\documentclass[a4paper,12pt,georgia,authoryear,print,index,custombib]{Classes/PhDThesisPSnPDF}

\input{Preamble/preamble}

\title{Applications of Gaussian Processes at Extreme Lengthscales: From Molecules to Black Holes}

\author{Ryan-Rhys Griffiths}

\dept{Department of Physics}

\university{University of Cambridge}
\crest{\includegraphics[width=0.2\textwidth]{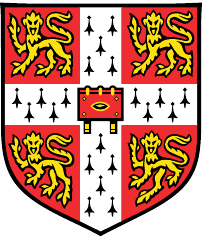}}

\supervisor{Dr. Alpha Lee}

\supervisorlinewidth{0.25\textwidth}

\renewcommand{\submissiontext}{This dissertation is submitted for the degree of}

\degreetitle{Doctor of Philosophy}

\college{Wolfson College}

\degreedate{August 2022} 

\subject{LaTeX} \keywords{{LaTeX} {PhD Thesis} {Engineering} {University of
Cambridge}}

\ifdefineAbstract
\fi

\ifdefineChapter
\fi

\begin{document}

\frontmatter

\maketitle

\begin{declaration}

This thesis is the result of my own work and includes nothing which is the outcome of work done in collaboration except as declared in the Preface and specified in the text. I further state that no substantial part of my thesis has already been submitted, or, is being concurrently submitted for any such degree, diploma or other qualification at the University of Cambridge or any other University or similar institution except as declared in the Preface and specified in the text. It does not exceed the prescribed word limit for the relevant Degree Committee

\end{declaration}

\begin{acknowledgements}      

I would like to thank all those who contributed in some part, indirect or otherwise, to my research productivity over the past years. I would like to thank Alpha Lee, my supervisor, firstly for giving me an opportunity to return from the working world to pursue research at a time when it seemed like all doors had been closed, and secondly, for striking a balance between giving me the freedom to explore new topics and offering excellent guidance and support when needed.\\
\indent In rough chronological order, I would like to thank Philippe Schwaller who introduced me to Alpha following a chance encounter on the streets of Cambridge in \textbf{April 2018}. My life may have been very different save for that meeting! Additionally, I would like to thank Philippe for his ongoing collaboration and sharing his expertise in all things involving sequence data and chemical reactions. \\
\indent From my time at Prowler.io (now Secondmind Labs) from \textbf{2017-2018}, I would like to thank Alexis Boukouvalas for acting as a fantastic mentor and supporting me in all my endeavours. I learned a great deal about both machine learning and software engineering during our pair programming sessions, especially when developing the code for adaptive sensor placement \citep{2019_Grant}. I would like to thank James Hensman, Richard Turner and Carl Rasmussen for giving lectures on Gaussian processes which sparked my interest in the topic. I would also like to thank Adithya Devraj for the numerous interesting conversations about machine learning and the philosophy of science. \\
\indent I would like to thank my colleagues from Cambridge Spark for keeping me up to speed with machine learning in industry from \textbf{2017-2022}. In particular, Raoul Gabriel-Urma, Petar Velickovic, Tim Hillel, Patrick Short, Catalina Cangea, Sahan Bulathwela, Ilyes Khemakhen, Chris Davis, Fred Hallgren and Kevin Lemagnen. Acting as a mentor for the Schmidt Data for Science Residency program was a highlight where I had the opportunity to learn about areas ranging from synthetic biology to geophysics and climate modelling. \\
\indent I would like to thank my colleagues from the Lee group at TCM from \textbf{2018-2022}, namely, Philip Verpoort, Alex Aldrick, Penelope Jones, Yunwei Zhang, William McCorkindale, Felix Faber, Alwin Bucher, Rhys Goodall, Janosh Riebesell, Rokas Elijosius, David Kovacs and Emma King-Smith for social interactions and academic discussions when I was not absent due to the global pandemic or internships. \\
\indent I would like to thank Bingqing Cheng, whom I first met in \textbf{2019} for involving me in her work on the ASAP library \citep{2020_Cheng}, and for taking the time to introduce me to a broad network of researchers applying machine learning to problems in physics and materials science. \\
\indent At the Institute of Astronomy I would like to thank Jiachen Jiang for introducing me to the world of high-energy astrophysics in the summer of \textbf{2019} and for providing excellent guidance on the contents of Chapter 3, namely modelling the multiwavelength variability of Mrk-335 using Gaussian processes \citep{2021_Mrk}. I would also like to thank Douglas Buisson, Dan Wilkins and Luigi Gallo for their feedback as well as Andy Fabian and Christopher Reynolds for being kind enough to attend an astrophysics talk given by a PhD student (myself) with no background in astrophysics! \\
\indent At the Computer Lab I would like to thank Ben Day, Simon Mathis and Arian Jamasb, whom I first met in \textbf{2020}, for discussions on graphs, molecules, proteins and antibodies. In particular, I would like to thank Arian for our almost daily Slack conversations and for answering my endless lists of questions! \\
\indent I would like to thank my colleagues at Huawei Noah's Ark Lab whom I began working with in \textbf{October 2020}. I would like to thank Haitham Bou-Ammar for his mentorship as well as Rasul Tutunov, Vincent Moens, Alexander Cowen-Rivers, Alexander Maraval, Antoine Grosnit and Hang Ren whom I learned a great deal from during our joint work \citep{2020_Grosnit, 2020_Rivers, 2021_Grosnit}. \\
\indent I would like to thank Anthony Bourached, George Cann, Gregory Kell and David Stork for their collaboration applying machine learning to artwork starting in late \textbf{2020} \citep{2021_Bourached_art, 2021_Cann, 2021_Stork, 2022_Kell}. In particular David has been an excellent mentor on the subject of research practices. \\
\indent I would like to thank Miguel Garcia-Ortegon, Vidhi Lalchand, James Wilson and Luke Corcoran for informative discussions on heteroscedastic Bayesian optimisation \citep{2021_Griffiths}, the topic of Chapter 6. I would like to thank Ajmal Aziz and Edward Kosasih at the Institute for Manufacturing for introducing me to supply chain logistics in the summer of \textbf{2021}, and in particular to applications of graph neural networks for supply chain problems \citep{2021_Aziz}. \\
\indent I would like to thank Jian Tang for supervising me at MILA from \textbf{January 2022} as well as Bojana Rankovic, Sang Truong, Leo Klarner, Aditya Ravuri, Yuanqi Du, Julius Schwartz, Austin Tripp, Alex Chan, Jacob Moss, Felix Opolka and Chengzhi Guo for their contributions to the GAUCHE library. \\
\indent I would like to thank Jake Greenfield for providing his photoswitch expertise and in particular, for helping out in a tight spot by rediscovering an old batch of lost molecules during a laboratory cleanup after the new batch had been mislaid by courier following a 3-month journey through customs. \\ 
\indent I would like to thank Henry Moss, whom I met virtually during the pandemic in the summer of \textbf{2020} and with whom I began developing a Gaussian process library for chemistry in the form of FlowMO \citep{2020_flowmo}. The evolved version, GAUCHE, comprises the contents of Chapter 4. I also appreciate the daily Slack discussions about all things to do with Bayesian optimisation and Gaussian processes. \\
\indent I would like to thank David Ginsbourger for hosting myself and Henry Moss in Bern in \textbf{April 2022} to discuss extensions of the ideas comprising Chapters 4 and 5 with Athénaïs Gautier and Anna Broccard. I would like to thank Ekansh Verma and Souradip Chakraborty for involving me in their work on invariances in Bayesian optimisation \citep{2021_Verma} in the summer of \textbf{2022}. I would like to thank my long-time friend S. F. for rigorously inspecting the notation of the final thesis in \textbf{July 2022}. I would like to thank Victor Prokhorov for always providing interesting food for thought during our many conversations in Cambridge. \\
\indent On a personal note, I would like to thank Leandro Charanga and Monika Jankauskaite for teaching me how to dance, Subhankar, Thomas C and Dan for their advice on ethical dilemmas, Thomas M for discussions on mathematics and Bachata, Teja for trying to teach me some gymnastics and my parents for their ongoing support.

\end{acknowledgements}

\begin{abstract}

In many areas of the observational and experimental sciences data is scarce. Observation in high-energy astrophysics is disrupted by celestial occlusions and limited telescope time while laboratory experiments in synthetic chemistry and materials science are both time and cost-intensive. On the other hand, knowledge about the data-generation mechanism is often available in the experimental sciences, such as the measurement error of a piece of laboratory apparatus. \\
\indent Both characteristics make Gaussian processes (\textsc{gp}s) ideal candidates for fitting such datasets. \textsc{gp}s can make predictions with consideration of uncertainty, for example in the virtual screening of molecules and materials, and can also make inferences about incomplete data such as the latent emission signature from a black hole accretion disc. Furthermore, \textsc{gp}s are currently the workhorse model for Bayesian optimisation, a methodology foreseen to be a vehicle for guiding laboratory experiments in scientific discovery campaigns. \\
\indent The first contribution of this thesis is to use \textsc{gp} modelling to reason about the latent emission signature from the Seyfert galaxy Markarian 335, and by extension, to reason about the applicability of various theoretical models of black hole accretion discs. The second contribution is to deliver on the promised applications of \textsc{gp}s in scientific data modelling by leveraging them to discover novel and performant molecules. The third contribution is to extend the \textsc{gp} framework to operate on molecular and chemical reaction representations and to provide an open-source software library to enable the framework to be used by scientists. The fourth contribution is to extend current \textsc{gp} and Bayesian optimisation methodology by introducing a Bayesian optimisation scheme capable of modelling aleatoric uncertainty, and hence theoretically capable of identifying molecules and materials that are robust to industrial scale fabrication processes.

\end{abstract}

\tableofcontents

\printnomenclature

\mainmatter

\chapter{Introduction}

\ifpdf
    \graphicspath{{Introduction/Figs/Raster/}{Introduction/Figs/PDF/}{Introduction/Figs/}}
\else
    \graphicspath{{Introduction/Figs/Vector/}{Introduction/Figs/}}
\fi

\section{Motivation}

The past decade has seen deep learning models achieve breakthroughs in computer vision \citep{2012_Krizhevsky}, speech recognition \citep{2013_Graves}, and natural language processing \citep{2017_Vaswani}. In fact, progress on developing deep learning architectures has proceeded so rapidly that, as of 2021, machine learning pioneer Andrew Ng has voiced the opinion that research into improving deep architectures has plateaued, at least in the traditional domains of vision, speech and language. Ng is now calling for a shift in focus towards data-centric AI, arguing that the dataset, as opposed to the model, is now the performance bottleneck in many real-world problems \citep{2021_Ng}.

In the natural sciences, however, model development is by no means a solved problem. Large scientific datasets have existed for some time, such as those generated by the Large Hadron Collider at CERN \citep{2013_Cern}, or the Chemical Universe Database, GDB-17 \citep{2012_Ruddigkeit}, which enumerates 166 billion small molecules. Developing effective models for scientific data is still an active and fast-moving field of research however \citep{2022_Kalinin}. In contrast to artificial data such as images, speech, and text, scientific data can often be inexorably tied to causal paradigms, entailing challenges for purely data-driven approaches seeking to achieve strong out-of-distribution (OOD) performance. Lines of inquiry in this direction include incorporating invariances due to symmetries into deep learning models for proteins and molecules \citep{2021_Jumper, 2020_Hermann}, as well as causal mechanisms for problems in physics \citep{2021_Scholkopf}.

A further challenge for building performant machine learning models for scientific applications stems from the availability of data. While large datasets in the sciences have undoubtedly been a key driver of research, there are also areas of scientific discovery which will always be limited to small data. Examples include molecular design, where one wishes to predict the properties of a new class of molecule for which few experimental measurements exist, as well as high-energy astrophysics, where one wishes to draw inferences from astronomical time series with short observation periods. In the past years researchers have achieved success in porting breakthroughs in deep learning to large scientific datasets \citep{2019_Bolgar, 2020_Chithrananda, 2022_White}. Deep learning models, however, are known to struggle in small data regimes to the extent that leading deep learning expert Yoshua Bengio previously voiced a preference for a model called a Gaussian process (\textsc{gp}) for small datasets \citep{2011_Bengio}. As such, leveraging them directly for small data scientific discovery could prove to be difficult.  

\textsc{gp}s have received comparatively less attention relative to deep learning over the past decade due to a variety of factors including a higher barrier to entry in terms of the mathematical background required to use them, fewer open-source software implementations, and perhaps most importantly, concerns over the ability of \textsc{gp}s to carry out representation learning, a stance summed up in the following prescient quote from \cite{2003_MacKay} which foreshadows some of the challenges currently encountered in supervised deep learning for the sciences.  

\begin{displayquote}
"According to the hype of 1987, neural networks were meant to be intelligent models that discovered features and patterns in data. Gaussian processes in contrast are simply smoothing devices. How can Gaussian processes possibly replace neural networks? Were neural networks over-hyped, or have we underestimated the power of smoothing methods? I think both these propositions are true. The success of Gaussian processes shows that many real-world data modelling problems are perfectly well-solved by sensible smoothing methods. The most interesting problems, the task of feature discovery for example are not ones that Gaussian processes will solve. But maybe multilayer perceptrons can't solve them either. Perhaps a fresh start is needed, approaching the problem of machine learning from a paradigm different from the supervised feedforward mapping."
\end{displayquote}

One of the motivations for focussing on \textsc{gp}s in this thesis, is the proposition that many scientific discovery problems are instances of the real-world problems described by MacKay. Furthermore, \textsc{gp}s are more than just smoothing devices. In addition to admitting exact Bayesian inference which can be used to perform plausible reasoning \citep{2003_Jaynes} over scientific hypotheses, \textsc{gp}s are also a longstanding workhorse of Bayesian optimisation (\textsc{bo}) and active learning \citep{2012_Settles}, two methodologies that have already shown promise in accelerating scientific discovery \citep{2020_Pyzer, 2021_Shields}. The goals of this thesis are twofold: First, to showcase some of the use-cases for \textsc{gp}s in modelling scientific data and second, to extend current \textsc{gp} methodology and software implementations to enable their application to scientific problems. Specifically, the problem domains considered are:

\begin{enumerate}
    \item \textbf{High-Energy Astrophysics} - It is challenging to test theories in high-energy astrophysics due to the inability to perform physical experiments at the far reaches of the universe. As such, the analysis of observational data is important to guide the development of theory. It is shown how \textsc{gp} modelling can play a role in performing inference over the structure of black hole accretion discs and hence inform the development of future accretion disk theories.
    \item \textbf{Photoswitch Chemistry} - In synthetic chemistry, new areas of chemical space are constantly being explored and often little experimental data exists to guide exploration. It is shown how \textsc{gp} modelling can be used for molecular property prediction to prioritise the synthesis of novel molecules. We validate the modelling approach with laboratory experiments, discovering new and performant photoswitch molecules.
    \item \textbf{Methodology/Software} - From a methodological standpoint a novel \textsc{bo} algorithm is introduced that identifies and penalises input-dependent (heteroscedasatic) measurement noise, an important consideration for the discovery of robust materials suitable for industrial scale manufacturing. From a software standpoint, an open-source \textsc{gp} library for chemistry is introduced, providing implementations of bespoke kernels designed for common molecular and chemical reaction representations.
\end{enumerate}

\begin{figure}[h]
\centering
{\includegraphics[width=\textwidth]{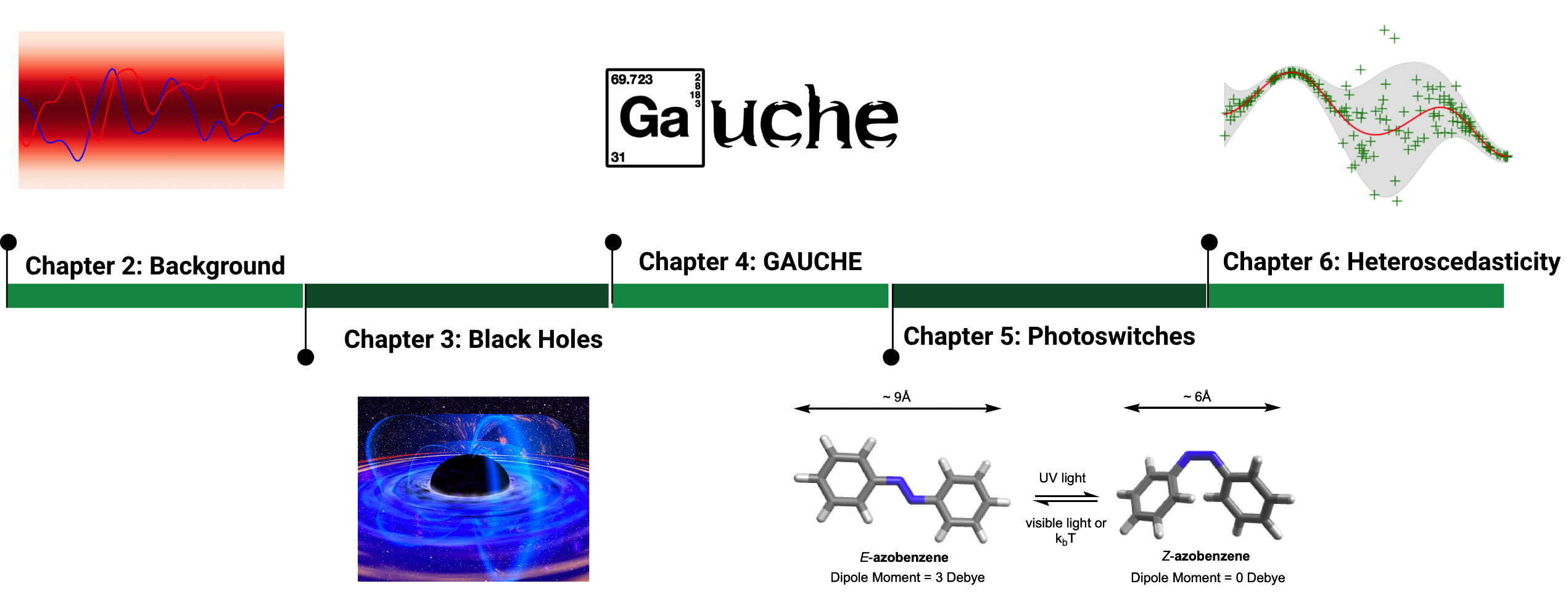}}
\caption{A pictorial overview of the thesis.}
\label{fig:overview_thesis}
\end{figure}

\section{Overview and Contributions}

A pictorial overview of the chapters of this thesis is available in \autoref{fig:overview_thesis}. The detailed summary and contributions of each chapter are as follows:

\paragraph{Chapter 2} The requisite background is provided on \textsc{gp}s and \textsc{bo}, the machine learning methodologies used across chapters of this thesis.

\paragraph{Chapter 3} A self-contained background is provided on the elements of high-energy astrophysics required to understand our findings. The gapped lightcurves of the Seyfert galaxy Markarian 335 (Mrk 335) are interpolated using \textsc{gp} modelling with the intention of inferring the structure of the black hole accretion disk through cross-correlation analysis. In a simulation study, Bayesian model selection through the marginal likelihood is investigated as a means of evaluating the most appropriate choice of \textsc{gp} kernel. Following \textsc{gp} modelling of the observational data, it is found that the distance between the UV and X-ray emission regions of Mrk 335 predicted by the Shakura-Sunyaev accretion disk model is shorter than the light travel time measured using \textsc{gp}-based inference. Tentative evidence is obtained for a short lag feature in the coherence and lag spectra which could indicate the presence of an extended UV
emission region on the accretion disk where reverberation happens.

\paragraph{Chapter 4} A self-contained background is provided on the elements of molecular machine learning required to understand the findings presented. GAUCHE is introduced, a software library for Gaussian processes in chemistry, tackling the problem of extending the \textsc{gp} framework to molecular representations such as graphs, strings and bit vectors. By designing bespoke molecular kernels, the door is opened to uncertainty quantification and \textsc{bo} directly on molecules and chemical reactions.

\paragraph{Chapter 5} A small dataset of experimentally-determined properties for $405$ photoswitch molecules is used in conjunction with the machinery made available in GAUCHE to train a multioutput \textsc{gp} with a Tanimoto kernel to screen a large virtual library of $7,265$ photoswitches, identifying 11 performant candidates validated through laboratory experiment. Additionally, a predictive performance comparison is conducted between the multioutput \textsc{gp} model and a cohort of trained human photoswitch chemists with the \textsc{gp} model outperforming the human experts. From a benchmark comparison against other machine learning models, it is concluded that the curated dataset, as opposed to the choice of model, is the key determinant of performance.

\paragraph{Chapter 6} A novel method for performing \textsc{bo} is introduced that is robust to experimental measurement noise featuring a heteroscedastic \textsc{gp} surrogate model. From an extensive empirical study, it is concluded that a moderately-sized initialisation set is required for the model to be able to distinguish heteroscedastic noise from intrinsic function variability. The chapter concludes with recommendations on how future research might enable the approach to be scaled to high-dimensional datasets.

\paragraph{Chapter 7} The thesis contributions are reviewed and discussed in the broader context of identifying and enabling further applications of \textsc{gp}s in the natural sciences.

\section{List of Publications}

What follows is the full list of publications co-authored during the PhD process starting in October 2018. J1 \citep{2021_Griffiths}, J2 \citep{2021_Mrk}, J3 \citep{2020_Griffiths} and W7 \citep{2022_Gauche} comprise the thesis. In J1 and J2, all coauthors acted in advisory roles, fine-tuning ideas and the final manuscripts. I conducted all experiments, mathematical derivations and implemented all code contributions. In J3, Aditya Raymond Thawani curated the training dataset, designed and recruited participants for the human performance comparison study and specified the set of performance criteria. Jake Greenfield performed the spectral characterisation of the discovered molecules in the Fuchter group laboratory at Imperial College London. I conducted all machine learning experiments and implemented all code contributions with the exception of the results in \autoref{tab1_photo} and \autoref{tab2_photo} of Appendix~\ref{benchmark_ml}, where Penelope Jones, William McCorkindale, Arian Jamasb and Henry Moss obtained results for the attentive neural process (ANP), smooth overlap of atomic positions (SOAP) kernel, graph neural network (GNN) and string kernel models respectively. 

In W7, I ran all experiments excluding the Buchwald-Hartwig reaction optimisation experiments which were run by Bojana Rankovic and the Weisfehler-Lehman (WL) graph kernel table entries which were run by Aditya Ravuri. The remaining co-authored articles are not included in the thesis for ease of exposition.

J4 \citep{2020_Griffiths} is a paper resulting from the continuation of my work from the MPhil in Machine learning at the University of Cambridge in 2017. J5 \citep{2020_Cheng} was work principally led by Dr. Bingqing Cheng. J6 \citep{2020_Rivers} and J7 \citep{2020_Grosnit} were articles written during an internship at the Huawei Noah's Ark Lab. J8 \citep{2020_Zagar} is a continuation of my work during an MSci at Imperial College London in 2016. J9 \citep{2022_Bourached} was principally led by Anthony Bourached. C1 \citep{2019_Grant} is a continuation of work undertaken whilst a machine learning researcher at Secondmind Labs prior to commencement of the PhD. C2-C5 \citep{2022_Kell, 2021_Stork, 2021_Cann, 2021_Bourached_art} are articles published in a domain unrelated to the topic of the thesis. The (unpublished) workshop contributions, W1-W3 \citep{2020_flowmo} are early versions of J3 and W7. W4 \citep{2018_Griffiths} is unrelated to the topic of the thesis although a figure from this paper is used as \autoref{reaction}. W5 \citep{2021_Aziz} was principally led by Ajmal Aziz and so is not included in the thesis. W6 is a condensed version of J8. P1 \citep{2021_Grosnit} resulted from work undertaken whilst at Huawei Noah's Ark Lab and so is not included in the thesis though the subject matter is related. P2 \citep{bourached2021hierarchical} was principally led by Anthony Bourached and so is not included in the thesis. P3 \citep{2023_Frieder} and P4 \citep{2022_Rankovic} was work led by Simon Frieder and Bojana Rankovic respectively and is not included in the thesis.\\

\noindent \textbf{Refereed Journal Papers} \\

\begin{enumerate}[label={[J\arabic*]}]
      \item \textbf{Griffiths RR}, Aldrick A, Garcia-Ortegon M, Lalchand V, Lee, AA. \href{https://iopscience.iop.org/article/10.1088/2632-2153/ac298c}{Achieving Robustness to Aleatoric Uncertainty with Heteroscedastic Bayesian Optimisation}. \textit{Machine Learning: Science and Technology}. 2021.
      \item \textbf{Griffiths RR}, Jiang J, Buisson D, Wilkins D, Gallo L, Ingram, A, Lee AA, Grupe D, Kara M, Parker ML, Alston W, Bourached A, Cann G, Young A, Komossa S. \href{https://iopscience.iop.org/article/10.3847/1538-4357/abfa9f/meta}{Modelling the Multiwavelength Variability of Mrk-335 using Gaussian Processes}. \textit{The Astrophysical Journal}. 2021. %
              \item \textbf{Griffiths RR}, Greenfield JL, Thawani AR, Jamasb A, Moss HB, Bourached A, Jones P, McCorkindale W, Aldrick AA, Fuchter, MJ, Lee AA. \href{https://pubs.rsc.org/en/content/articlelanding/2022/sc/d2sc04306h}{Data-Driven Discovery of Molecular Photoswitches with Multioutput Gaussian Processes}. \textit{Chemical Science}. 2022.
      \item \textbf{Griffiths RR}, Hern\'andez-Lobato JM. \href{https://pubs.rsc.org/en/content/articlelanding/2020/sc/c9sc04026a#!divAbstract}{Constrained Bayesian Optimization for Automatic Chemical Design using Variational Autoencoders}. \textit{Chemical Science}. 2020. %
      \item Cheng B, \textbf{Griffiths RR}, Wengert S, Kunkel C, Stenczel T, Zhu B, Deringer VL, Bernstein N, Margraf JT, Reuter K, Csanyi G. \href{https://pubs.acs.org/doi/abs/10.1021/acs.accounts.0c00403}{Mapping Datasets of Molecules and Materials}. \textit{Accounts of Chemical Research}. 2020.
      \item Cowen-Rivers A, Lyu W, Tutunov R, Wang Z, Grosnit A, \textbf{Griffiths RR}, Hao J, Wang J, Bou-Ammar H. \href{https://jair.org/index.php/jair/article/view/13643}{HEBO: Pushing the Limits of Sample-Efficient Hyper-parameter Optimisation}. \textit{Journal of Artificial Intelligence Research}, 2022.
      \item Grosnit A, Cowen-Rivers A, Tutunov R, \textbf{Griffiths RR}, Wang J, Bou-Ammar H. \href{https://www.jmlr.org/papers/v22/20-1422.html?ref=https://githubhelp.com}{Are We Forgetting About Compositional Optimisers in Bayesian Optimisation}. \textit{Journal of Machine Learning Research}. 2021. %
      \item Zagar C, \textbf{Griffiths RR}, Podgornik R, Kornyshev AA. \href{https://www.sciencedirect.com/science/article/pii/S1572665720305038}{On the Voltage-Controlled Self-Assembly of NP Arrays at Electrochemical Solid/Liquid Interfaces}. \textit{Journal of Electroanalytical Chemistry}. 2020. %
      \item Bourached A, \textbf{Griffiths RR}, Gray R, Jha A, Nachev P. \href{https://onlinelibrary.wiley.com/doi/full/10.1002/ail2.63}{Generative Model-Enhanced Human Motion Prediction}. \textit{Applied AI Letters}. 2021.
\end{enumerate}

\noindent \textbf{Refereed Conference Papers} \\

\begin{enumerate}[label={[C\arabic*]}]
      \item Grant J, Boukouvalas A, \textbf{Griffiths RR}, Leslie D, Vaikili S, Munoz de Cote E. \href{http://proceedings.mlr.press/v97/grant19a.html}{Adaptive Sensor Placement for Continuous Spaces}. \textit{International Conference on Machine Learning}. 2019. %
      \item Kell G, \textbf{Griffiths RR}, Bourached A, Stork D. \href{https://arxiv.org/abs/2203.07026}{Extracting Associations and Meanings of Objects Depicted in Artworks through Bi-Modal Deep Networks}, Electronic Imaging 2022.
      \item Stork D, Bourached A, Cann G, \textbf{Griffiths RR}. \href{https://arxiv.org/abs/2102.02732}{Computational Identification of Significant Actors in Paintings through Symbols and Attributes}, Electronic Imaging, 2021.
      \item Cann G, Bourached A, \textbf{Griffiths RR}, Stork D. \href{https://arxiv.org/abs/2102.00209}{Resolution Enhancement in the Recovery of Underdrawings Via Style Transfer by Generative Adversarial Deep Neural Networks}, Electronic Imaging, 2021.
      \item Bourached A, Cann G, \textbf{Griffiths RR}, Stork D. \href{https://arxiv.org/abs/2101.10807}{Recovery of Underdrawings and Ghost-Paintings via Style Transfer by Deep Convolutional Neural Networks: A Digital Tool for Art Scholars}, Electronic Imaging, 2021.
\end{enumerate}

\noindent \textbf{Refereed Workshop Papers} \\

\begin{enumerate}[label={[W\arabic*]}]
      \item \textbf{Griffiths RR*}, Moss H*. \href{https://arxiv.org/abs/2010.01118}{Gaussian Process Molecular Machine Learning with FlowMO}. \textit{NeurIPS Workshop on Machine Learning for Molecules}. 2020 (Contributed Talk - top 5\%, * joint first authorship). %
      \item \textbf{Griffiths RR}, Jones P, McCorkindale W, Aldrick AA, Jamasb A, Day B. Benchmarking Scalable Active Learning Strategies on Molecules. \textit{ICLR Workshop on Fundamental Science in the Era of AI}. 2020. %
      \item \textbf{Griffiths RR}, Thawani AR, Elijosius R. \textit{Enhancing the Diversity of Molecular Machine Learning Benchmarks: An Open-Source Dataset for Molecular Photoswitches}. \textit{ICLR Workshop on Fundamental Science in the Era of AI}. 2020. %
      \item \textbf{Griffiths RR}, Schwaller P, Lee AA. \href{https://chemrxiv.org/articles/Dataset_Bias_in_the_Natural_Sciences_A_Case_Study_in_Chemical_Reaction_Prediction_and_Synthesis_Design/7366973}{Dataset Bias in the Natural Sciences: A Case Study in Chemical Reaction Prediction and Synthesis Design}. \textit{NeurIPS Workshop on Critiquing and Correcting Trends in Machine Learning.} 2018. %
      \item Aziz A, Kosasih EE, \textbf{Griffiths RR}, Brintrup A. \href{https://arxiv.org/abs/2107.10609}{Data Considerations in Graph Representation Learning for Supply Chain Networks}. \textit{ICML Workshop on Machine Learning for Data: Automated Creation, Privacy, Bias}. 2021 %
      \item Bourached A, \textbf{Griffiths RR}, Gray R, Jha A, Nachev P. \href{https://arxiv.org/abs/2010.11699}{Generative Model-Enhanced Human Motion Prediction}. \textit{NeurIPS Workshop on Interpretable Inductive Biases and Physically-Structured Learning}. 2020. %
      \item \textbf{Griffiths RR}, Klarner L, Moss Henry B., Ravuri A, Rankovic B, Truong S, Du Y, Jamasb A, Schwartz J, Tripp A, Kell G, Bourached A, Chan A, Moss J, Guo C, Lee AA, Schwaller P, Tang J, \href{https://arxiv.org/abs/2010.11699}{GAUCHE: A Library for Gaussian Processes in Chemistry}. \textit{ICML Workshop on AI4Science}. 2022. %
      
\end{enumerate}

\noindent \textbf{Preprints} \\

\begin{enumerate}[label={[P\arabic*]}]
      \item \textbf{Griffiths RR*}, Grosnit A*, Tutunov R*, Maraval AM*, Cowen-Rivers A, Yang L, Lin Z, Lyu W, Chen Z, Wang J, Peters J, Bou-Ammar H. \href{https://arxiv.org/abs/2106.03609}{High-Dimensional Bayesian Optimisation with Variational Autoencoders and Deep Metric Learning}. \textit{arXiv}. 2021. (* joint first authorship)
      \item Bourached A, Gray R, \textbf{Griffiths RR}, Jha A, Nachev P. \href{https://arxiv.org/abs/2111.12602}{Hierarchical Graph-Convolutional Variational Autoencoding for Generative Modelling of Human Motion}. \textit{arXiv}. 2021.
      \item Frieder S, Pinchetti, L, \textbf{Griffiths RR}, Salvatori, T, Lukasiewicz, T, Petersen, PC, Chevalier, A and Berner, J, 2023. \href{https://arxiv.org/abs/2301.13867}{Mathematical capabilities of ChatGPT.}. \textit{arXiv}. 2023.
      \item Ranković, B, \textbf{Griffiths, RR}, Moss, HB and Schwaller, P. \href{https://chemrxiv.org/engage/chemrxiv/article-details/638e196ae6f9a162aa2ce493}{Bayesian optimisation for additive screening and yield improvements in chemical reactions–beyond one-hot encodings}. \textit{ChemRxiv}, 2022.
\end{enumerate}

\noindent \textbf{PhD Thesis} \\
\begin{enumerate}[label={[T\arabic*]}]
       \item \textbf{Griffiths RR}, \href{https://www.repository.cam.ac.uk/handle/1810/346223}{Applications of Gaussian Processes at Extreme Lengthscales: From Molecules to Black Holes}. \textit{University of Cambridge}. 2022.
\end{enumerate}

\section{List of Software}

The following list details the open-source software contributed to over the duration of the PhD process:

\begin{enumerate}[label={[S\arabic*]}]
    \item Constrained Bayesian optimisation for automatic chemical design: \textbf{Ryan-Rhys Griffiths} (2018). Code to reproduce the experiments from \cite{2020_Griffiths}.\\
    
    Available at: \href{https://github.com/Ryan-Rhys/Constrained-Bayesian-Optimisation-for-Automatic-Chemical-Design}{https://github.com/Ryan-Rhys/Constrained-Bayesian-Optimisation-for-Automatic-Chemical-Design}
    
    \item Mapping materials and molecules: Bingqing Cheng, \textbf{Ryan-Rhys Griffiths}, Tamas Stenczel, Bonan Zhu, Felix Faber (2020). A software library containing automatic selection tools for materials and molecules \citep{2020_Cheng}.\\
    
    Available at: \href{https://github.com/BingqingCheng/ASAP}{https://github.com/BingqingCheng/ASAP}
    
    \item Achieving robustness to aleatoric uncertainty with heteroscedastic Bayesian optimisation: \textbf{Ryan-Rhys Griffiths} (2019). Code to reproduce the experiments from \cite{2021_Griffiths}.\\
    
    Available at: \href{https://github.com/Ryan-Rhys/Heteroscedastic-BO}{https://github.com/Ryan-Rhys/Heteroscedastic-BO}
    
    \item The photoswitch dataset: \textbf{Ryan-Rhys Griffiths}, Aditya Raymond Thawani, Arian Jamasb, William McCorkindale, Penelope Jones (2020). Code to reproduce the experiments from Chapter 5.\\
    
    Available at: \href{https://github.com/Ryan-Rhys/The-Photoswitch-Dataset}{https://github.com/Ryan-Rhys/The-Photoswitch-Dataset}
    
    \item Modelling the multiwavelength variability of Mrk-335: \textbf{Ryan-Rhys Griffiths} (2021). Code to reproduce the experiments from \cite{2021_Mrk}.\\
    
    Available at: \href{https://github.com/Ryan-Rhys/Mrk_335}{https://github.com/Ryan-Rhys/Mrk\_335}
    
    \item An empirical study of assumptions in Bayesian optimisation: Alexander I. Cowen-Rivers, Wenlong Lyu, Rasul Tutunov, Zhi Wang, Antoine Grosnit, \textbf{Ryan-Rhys Griffiths}, Alexandre Max Maraval, Hao Jianye, Jun Wang, Jan Peters, Haitham Bou-Ammar (2021). Code to reproduce the experiments from \cite{2020_Rivers}.\\
    
    Available at: \href{https://github.com/huawei-noah/HEBO/tree/master/HEBO}{https://github.com/huawei-noah/HEBO/tree/master/HEBO}
    
    \item High-dimensional Bayesian optimisation with variational autoencoders and deep metric learning: Antoine Grosnit, Rasul Tutunov, Alexandre Max Maraval, \textbf{Ryan-Rhys Griffiths}, Alexander I. Cowen-Rivers, Lin Yang, Lin Zhu, Wenlong Lyu, Zhitang Chen, Jun Wang, Jan Peters, Haitham Bou-Ammar. Code to reproduce the experiments from \cite{2021_Grosnit}.\\
    
    Available at: \href{https://github.com/huawei-noah/HEBO/tree/master/T-LBO}{https://github.com/huawei-noah/HEBO/tree/master/T-LBO}
        
    \item Are we forgetting about compositional optimisers in Bayesian optimisation?: Antoine Grosnit, Alexander I. Cowen-Rivers, Rasul Tutunov, \textbf{Ryan-Rhys Griffiths}, Jun Wang, Haitham Bou-Ammar. Code to reproduce the experiments from \cite{2020_Grosnit}.\\
    
    Available at: \href{https://github.com/huawei-noah/HEBO/tree/master/T-LBO}{https://github.com/huawei-noah/HEBO/tree/master/T-LBO}
    
    \item FlowMO: \textbf{Ryan-Rhys Griffiths} and Henry Moss (2020). A GPflow library for training Gaussian processes on molecular data \citep{2020_Moss}.\\
    
    Available at: \href{https://github.com/Ryan-Rhys/FlowMO}{https://github.com/Ryan-Rhys/FlowMO}
    
    \item GAUCHE: \textbf{Ryan-Rhys Griffiths}, Leo Klarner, Henry Moss, Aditya Ravuri, Sang Truong, Arian Jamasb, Austin Tripp, Bojana Rankovic, Philippe Schwaller (2022). A software library for Gaussian processes in chemistry.\\
    
    Available at \href{https://github.com/leojklarner/gauche}{https://github.com/leojklarner/gauche}
    
    \item Extracting associations and meanings of objects depicted in artworks through bi-modal deep networks: Gregory Kell, \textbf{Ryan-Rhys Griffiths} (2021). Code to reproduce the experiments from \cite{2022_Kell}.\\
    
    Available at: \href{https://github.com/gck25/fine_art_asssociations_meanings}{https://github.com/gck25/fine\_art\_asssociations\_meanings}
\end{enumerate}

\nomenclature[Z-OOD]{OOD}{Out-Of-Distribution}
\nomenclature[Z-GP]{GP}{Gaussian Process}
\nomenclature[Z-BO]{BO}{Bayesian Optimisation}
\nomenclature[Z-Mrk 335]{Mrk 335}{Markarian 335}
\nomenclature[Z-GAUCHE]{GAUCHE}{GAUssian Processes in CHEmistry}
\nomenclature[Z-ANP]{ANP}{Attentive Neural Process}
\nomenclature[Z-SOAP]{SOAP}{Smooth Overlap of Atomic Positions}
\nomenclature[Z-GNN]{GNN}{Graph Neural Network}
\nomenclature[Z-WL]{WL}{Weisfehler-Lehman}

\chapter{Background}
\chapterimage[height=160pt]{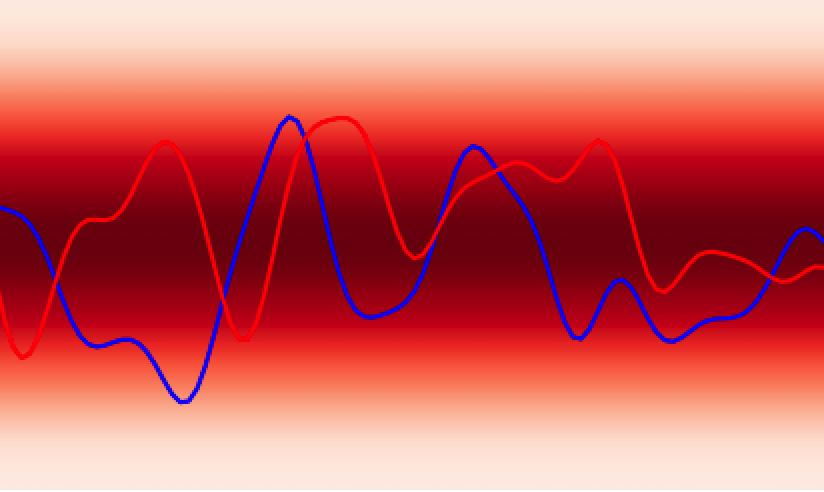}

\ifpdf
    \graphicspath{{Background/Figs/Raster/}{Background/Figs/PDF/}{Background/Figs/}}
\else
    \graphicspath{{Background/Figs/Vector/}{Background/Figs/}}
\fi

In this chapter the requisite background is provided on Gaussian processes (Chapters 3, 4, 5 and 6) and Bayesian optimisation (Chapters 4 and 6).

\section{Gaussian Processes}
\label{intro_to_gps}

In the context of machine learning, a Gaussian process (\textsc{gp}) is a Bayesian nonparametric model for functions. \textsc{gp}s are attractive models when limited data is available, a setting common to many areas of the natural sciences, with even notable deep learning experts voicing a preference for \textsc{gp}s in the small data regime \citep{2011_Bengio}. Furthermore, \textsc{gp}s possess several important properties for the applications in this thesis:

\begin{enumerate}
    \item \textbf{Bayesian optimisation:} \textsc{gp}s have few hyperparameters that need to be determined by hand which lends itself well to the repeated surrogate model hyperparameter optimisation required by Bayesian optimisation.
    \item \textbf{Astronomical time series:} For astronomical time series, where noise processes are often well understood, it is possible to incorporate this knowledge into the design of the \textsc{gp} model.
    \item \textbf{Molecules:} \textsc{gp}s maintain uncertainty estimates over molecular property values through exact Bayesian inference. Uncertainty estimates are particularly important when prioritising molecules for screening experiments.
\end{enumerate}

\noindent A Gaussian process (\textsc{gp}) may be defined as a collection of random variables, any finite subset of which have a joint Gaussian distribution \citep{2006_Rasmussen}. In the cases considered in this thesis, the random variables represent the value of the function $f(\mathbf{x})$ at location $\mathbf{x}$. A stochastic process $f$ that follows a \textsc{gp} is written as

\begin{equation}
f(\mathbf{x}) \sim \mathcal{GP}\big(m(\mathbf{x}), k(\mathbf{x}, \mathbf{x'})\big).
\end{equation}

\noindent The inputs to the \textsc{gp} may be scalars (e.g. time points in Chapter 3) or vectors (e.g. molecular representations in Chapters 4 and 5). In the current presentation we assume vector inputs $\mathbf{x} \in \mathbb{R}^d$ and we seek to perform Bayesian inference over the latent function $f$ that represents the mapping between the inputs $\{\mathbf{x_1}, \dotsc , \mathbf{x_N}\}$ and their function values $\{f(\mathbf{x_1}), \dotsc , f(\mathbf{x_N})\}$. The \textsc{gp} is characterised by a mean function,

\begin{equation}
m(\mathbf{x}) = \mathbb{E}[f(\mathbf{x})],
\end{equation}

\noindent and a covariance function

\begin{equation}
k(\mathbf{x}, \mathbf{x'}) = \mathbb{E}[(f(\mathbf{x}) - m(\mathbf{x}))(f(\mathbf{x'}) - m(\mathbf{x'}))].
\end{equation}

\noindent In the absence of prior information on trends in the data, the mean function is typically set to zero following standardisation of the outputs. Standardisation, in this case refers to the common practice of subtracting the mean and dividing by the standard deviation of the data when fitting the \textsc{gp} in order to facilitate the identification of appropriate hyperparameters \citep{2008_Murray}. The standardisation is reversed once the fitting procedure is complete in order to obtain predictions on the original scale of the data. $m(\mathbf{x}) \equiv \mathbf{0}$ will be assumed henceforth for the sake of the current presentation. The covariance function computes the pairwise covariance between two random variables (function values). In the \textsc{gp} literature, the covariance function is commonly referred to as the kernel. Informally, the kernel is responsible for determining the properties of the functions which the \textsc{gp} is capable of fitting e.g. smoothness and periodicity. The inductive bias created by the choice of kernel is an important consideration in \textsc{gp} modelling.

\begin{figure*}[h]
\centering
\subfigure[SQE small lengthscale ]{\label{fig:homo}\includegraphics[width=0.48\textwidth]{Background/Figs/sqexp_small_length.png}}
\subfigure[SQE large lengthscale]{\label{fig:het}\includegraphics[width=0.48\textwidth]{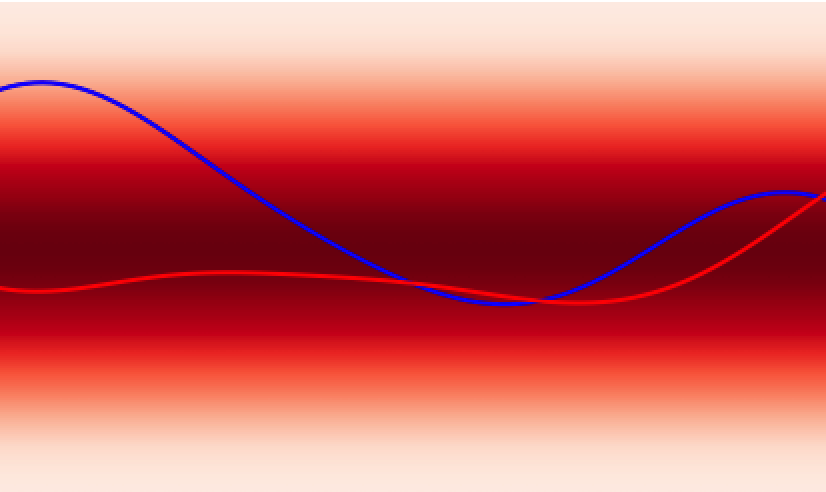}} 
\caption{\textsc{gp}s with small and large lengthscales.}
\label{fig:gp_kern}
\end{figure*}

\subsection{Kernels}

\noindent The most widely-known kernel is the squared exponential (SQE) or radial basis function (RBF) kernel,

\begin{equation}
\label{equation:sqe_}
k_{\text{SQE}}(\mathbf{x}, \mathbf{x'}) = \sigma_{f}^2 \exp\bigg(\frac{-\norm{\mathbf{x} - \mathbf{x'}}^{2}}{2\ell^2}\bigg),
\end{equation}

\noindent where $\norm{\cdot}$ is the Euclidean norm, $\sigma_{f}^2$ is the signal amplitude hyperparameter (vertical lengthscale) and $\ell$ is the (horizontal) lengthscale hyperparameter. Although \autoref{equation:sqe_} is written with a single lengthscale shared across dimensions, for multidimensional input spaces it is possible to optimise a lengthscale per dimension. We will adopt the notation of $\theta$ to represent the set of kernel hyperparameters. An illustration of \textsc{gp}s with different lengthscales is given in \autoref{fig:gp_kern}. It has been argued by \cite{2012_Stein} that the smoothness assumptions of the SQE kernel are unrealistic for many physical processes. As such, kernels such as the Matérn,

\begin{equation}
k_{\text{Matérn}}(\mathbf{x}, \mathbf{x'}) = \frac{2^{1-\nu}}{\Gamma(\nu)}\bigg(\frac{\sqrt{2\nu}-\norm{\mathbf{x} - \mathbf{x'}}}{\ell}\bigg)^\nu K_{\nu}\bigg(\frac{\sqrt{2\nu}-\norm{\mathbf{x} - \mathbf{x'}}}{\ell}\bigg),
\end{equation}

\noindent are more commonly seen in the machine learning literature. Here $K_{\nu}$ is a modified Bessel function of the second kind, $\Gamma$ is the gamma function and $\nu$ is a non-negative hyperparameter of the kernel which is typically taken to be either $\frac{3}{2}$ or $\frac{5}{2}$ \citep{2006_Rasmussen}. The lengthscale hyperparameter $\ell$ can be thought of loosely as a decay coefficient for the covariance between inputs as they become increasingly far apart in the input space; the further apart the inputs are, the less correlated they will be. The rational quadratic (RQ) kernel is defined as

\begin{equation}
k_{\text{RQ}}(\mathbf{x}, \mathbf{x'}) = \bigg(1 + \frac{\norm{\mathbf{x} - \mathbf{x'}}^{2}}{2\alpha\ell^2}\bigg)^{-\alpha},
\end{equation}

\noindent where $\alpha, \ell > 0$. The RQ kernel can be viewed as a scale mixture of SQE kernels with different characteristic lengthscales. A comparison of the functions drawn from \textsc{gp}s with SQE and Matérn $\frac{5}{2}$ kernels is given in \autoref{fig:gp_kern2}. The aforementioned kernels are defined over continuous input spaces and are used in Chapters 3 and 6. For discrete input spaces such as molecular representations it is necessary to define bespoke kernels which will be introduced in Chapters 4 and 5.

\begin{figure*}[h]
\centering
\subfigure[SQE]{\label{fig:homo}\includegraphics[width=0.48\textwidth]{Background/Figs/sqexp_small_length.png}}
\subfigure[Matérn $5/2$]{\label{fig:het}\includegraphics[width=0.48\textwidth]{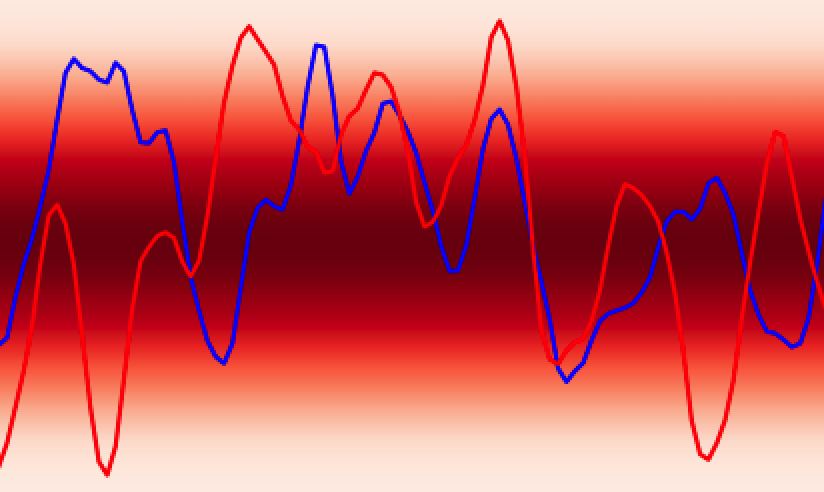}} 
\caption{A comparison of the SQE and Matérn $\frac{5}{2}$ kernels.}
\label{fig:gp_kern2}
\end{figure*}

\subsection{Predictions}

To obtain the predictive equations of \textsc{gp} regression, a mean function $m(\mathbf{x}) \equiv \mathbf{0}$ and kernel $k$ are specified and a \textsc{gp} prior $p$ is placed over $f$,

\begin{equation}
p(f(\mathbf{x})| \theta) = \mathcal{GP}\big(\mathbf{0}, K_{\theta}(X, X')\big).
\end{equation} 

\noindent The notation $K_{\theta}(X, X')$ denotes a kernel matrix with entries $[K]_{ij} = k(\mathbf{x}_i, \mathbf{x}_j)$ and the subscript notation is chosen to indicate the dependence on the set of hyperparameters $\theta$ (e.g. the signal variance $\sigma_f$ and lengthscale $\ell$ in \autoref{equation:sqe_}). We suppress the explicit dependence on $\theta$ in the subsequent notation. It is also necessary to specify a likelihood function 

\begin{equation}
p(y_i | f(\mathbf{x}_i)),
\end{equation}

\noindent which depends on $f(\mathbf{x}_i)$ only and is typically taken to be Gaussian i.e. $p(y_i | f(\mathbf{x}_i)) = \mathcal{N}(y_i | f(\mathbf{x}_i), \sigma_y^2)$. The noise level $\sigma_y^2$ is most frequently assumed to be homoscedastic, i.e. constant across the input domain. In Chapter 6, heteroscedastic (input-dependent) noise is considered by introducing a dependence $\sigma_y^2(\mathbf{x})$. The interpretation of $\mathbf{y}_i$ is a noise-corrupted observation of the latent function $f(\mathbf{x}_i)$. Once data $\{X, \mathbf{y}\}$ has been observed, where $X = \{\mathbf{x}_i\}_{i=1}^N$ and $\mathbf{y} = \{y_i\}_{i=1}^N$, the joint prior distribution over the observations $\mathbf{y}$ and the predicted function values $\mathbf{f}_*$ at test locations $X_*$ may be written

\begin{equation}
\label{eq:joint_prior}
    \begin{bmatrix} 
        \mathbf{y} \\
        \mathbf{f_*} \\
    \end{bmatrix}
    \sim
    \mathcal{N}
    \bigg(0,
    \begin{bmatrix}
        K(X, X) + \sigma_{y}^2 I & K(X, X_*)\: \\
        K(X_*, X)\phantom{+ \: \: \sigma{y}^2} & K(X_*, X_*)
    \end{bmatrix}
    \bigg),
\end{equation}

\noindent where $\mathcal{N}$ is the multivariate Gaussian probability density function and $I\sigma_{y}^2$ represents the variance of iid Gaussian noise on the observation vector $\mathbf{y}$. The joint prior in \autoref{eq:joint_prior} may be conditioned on the observations through

\begin{equation}
p(\mathbf{f_*}| \mathbf{y}) = \frac{p(\mathbf{f_*}, \mathbf{y})}{p(\mathbf{y})},
\end{equation}

\noindent which enforces that the joint prior agrees with the observations $\mathbf{y}$. The posterior predictive distribution is then

\begin{equation}
p(\mathbf{f_*}| X, \mathbf{y}, X_*) = \mathcal{N}\big(\mathbf{\bar{f}_*}, \text{cov}(\mathbf{f_*})\big),
\end{equation}

\noindent with predictive mean at test locations $X_*$,

\begin{equation}
\mathbf{\bar{f_*}} = K(X_*, X)[K(X, X) + \sigma_{y}^2 I]^{-1} \mathbf{y},
\end{equation}

\noindent and predictive uncertainty

\begin{equation}
\text{cov}(\mathbf{f_*}) = K(X_*, X_*) - K(X_*, X)[K(X, X) + \sigma_{y}^2 I]^{-1} K(X, X_*).
\end{equation}

\noindent Analysing the form of this expression one may notice that the first term $K(X_*, X_*)$ in the expression for the predictive uncertainty $\text{cov}(\mathbf{f}*)$ may be viewed as the prior uncertainty and the second term $K(X_*, X)[K(X, X) + \sigma_{y}^2 I]^{-1} K(X, X_*)$  can be thought of as a subtractive factor that accounts for the reduction in uncertainty when observing the data points $\mathbf{y}$. An illustration is given in \autoref{fig:gp_post} of the posterior predictive distribution updates following data observation.

\begin{figure*}[h]
\centering
\fbox{\subfigure{\label{fig:post1}\includegraphics[width=0.415\textwidth]{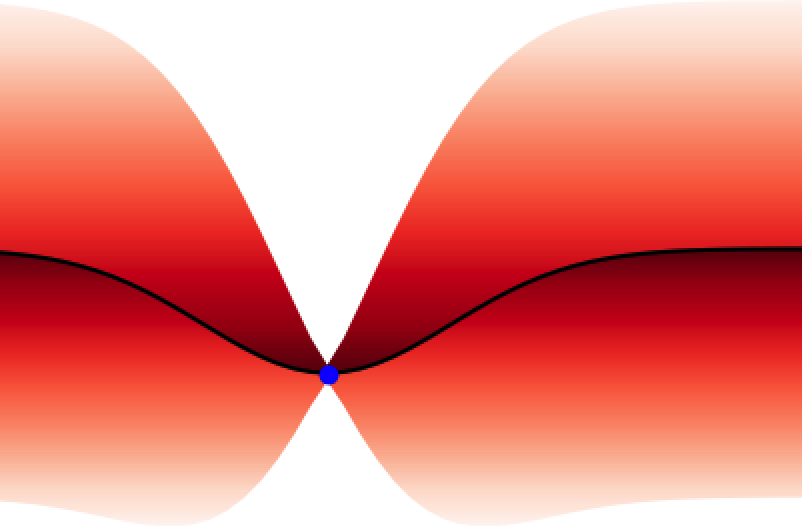}}}
\fbox{\subfigure{\label{fig:post2}\includegraphics[width=0.486\textwidth]{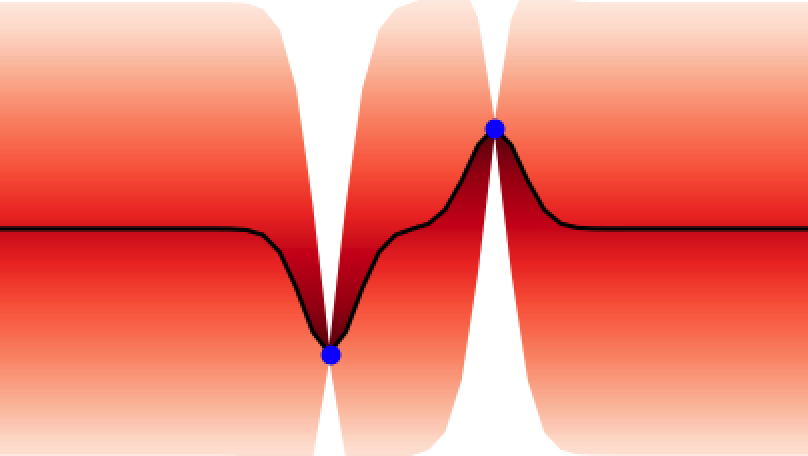}}}

\vspace{0.75cm}

\fbox{\subfigure{\label{fig:post4}\includegraphics[width=0.5\textwidth]{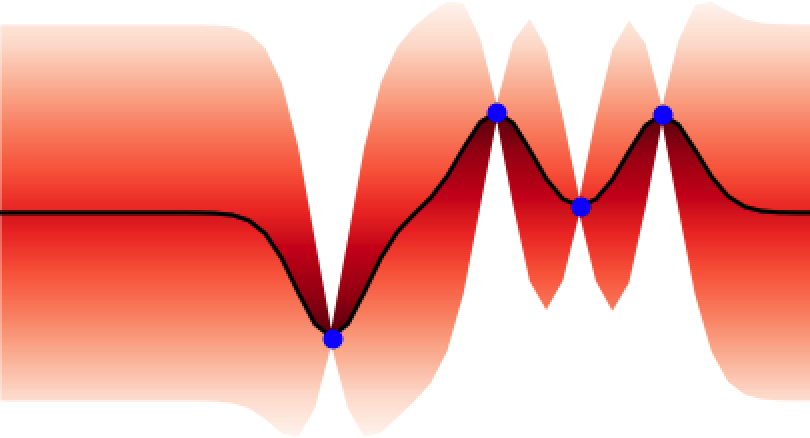}}}
\caption{An illustration of the \textsc{gp} posterior update on fitting 1, 2 and 4 data points (blue). The posterior distribution encodes the distribution over possible functions that may explain the data.}
\label{fig:gp_post}
\end{figure*}

\subsection{Training}
\label{gp_training}

An important objective function for training \textsc{gp}s is the log marginal likelihood or evidence \citep{1992_MacKay},

\begin{align}
\label{equation: log_lik_}
\log p(\mathbf{y}| X, \theta) =&  \underbrace{-\frac{1}{2} \mathbf{y}^{\top}(K_{\theta}(X, X) + \sigma_{y}^2I)^{-1} \mathbf{y}}_\text{encourages fit with data} \\ 
&\underbrace{-\frac{1}{2} \log | K_{\theta}(X, X) + \sigma_{y}^2 I |}_\text{controls model capacity} -\frac{N}{2} \log(2\pi) \nonumber.
\end{align}

\noindent $N$ is the number of observations and $\theta$ again represents the set of kernel hyperparameters to be optimised under the objective. The two terms in the expression for the log marginal likelihood embody Occam's Razor \citep{2001_Rasmussen} in their preference for selecting the simplest models that explain the data well as illustrated in \autoref{occam}. The first term in \autoref{equation: log_lik_} penalises functions that do not fit the data adequately whereas the second term acts as a regulariser, disfavouring overly complex models. The negative log marginal likelihood (NLML) is the \textsc{gp} training objective for all experiments performed in this thesis.

\begin{figure}[!htbp]
    \begin{center}
        \includegraphics[width=0.7\textwidth]{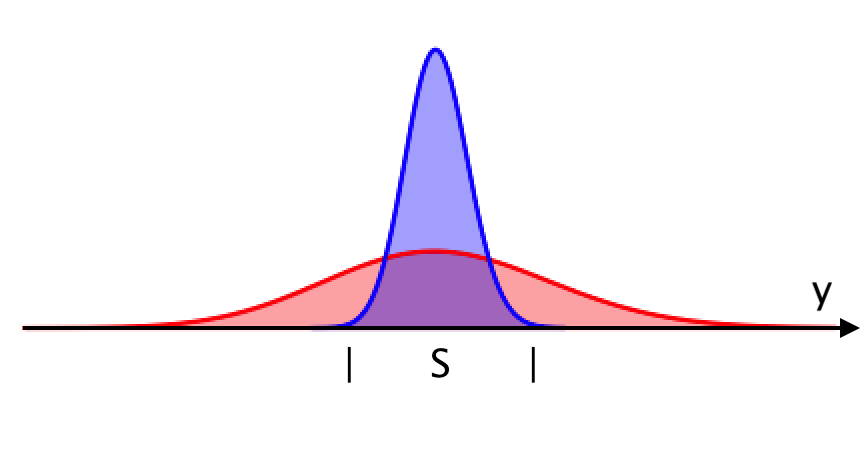}
    \end{center}
    \caption{An illustration of the Bayesian Occam’s razor effect introduced by \cite{2003_MacKay}. If models are interpreted as probability distributions over observations $y$, the x-axis may be viewed as the space of possible datasets. The simple model (blue) explains datasets inside $S$ well, but not outside. The complex model (red) explains datasets outside $S$ better, but worse inside $S$. The probability density in $S$ must be lower so as to explain the datasets outside $S$. Thus, if the dataset one wishes to model lies inside $S$, Occam's razor assigns preference to the simpler model.}
    \label{occam}
\end{figure}

\subsection{Bayesian Model Selection}

One desirable property of \textsc{gp}s, and Bayesian models in general, is the ability to carry out hierarchical modelling \citep{1992_MacKay_Hierarchical}. The three tiers of the modelling hierarchy are:

\begin{enumerate}
    \item Model Parameters
    \item Model Hyperparameters
    \item Model Structures
\end{enumerate}

\noindent In the case of the nonparametric \textsc{gp} framework, model parameters do not have the same meaning as in parametric Bayesian models and are instead obtained from the posterior distribution over functions. Model hyperparameters consist of parameters of the kernel function such as signal amplitudes and lengthscales as well as the likelihood noise. At the level of model structures, the fit achieved by different kernels can be quantitatively assessed by comparing the values of the optimised NLML objective permitting Bayesian model selection, a procedure that is used in Chapter 3. The next discussion point to be considered is an important application of \textsc{gp}s in mathematical optimisation.

\section{Bayesian Optimisation}

Bayesian optimisation (\textsc{bo})~\citep{1962_Kushner, 1964_Kushner, 1975_Mockus, 1975_Zhilinskas, 1978_Mockus} is a data-efficient methodology for solving black-box optimisation problems.

\subsection{Black-Box Optimisation}

In many problems in science and engineering we are interested in solving global optimisation problems of the form

\begin{equation}
    \label{Eq:ProbOne_}
    \mathbf{x}^{\star} = \arg\max_{\mathbf{x} \in \mathcal{X}} f(\mathbf{x}),
\end{equation}

\noindent where $f:\mathcal{X} \rightarrow\mathbb{R}$ is a function over an input domain $\mathcal{X}$ which is typically a compact subset of $\mathbb{R}^d$ (Chapter 6) but may also be non-numeric in the case of molecular representations such as graphs and strings (Chapter 4). \autoref{Eq:ProbOne_} is also a black-box optimisation problem in the sense that it possesses the following properties:

\begin{enumerate}
    \item Black-Box Objective: We do not have the analytic form of $f$ nor do we have access to its gradients. We can, however, evaluate $f$ pointwise anywhere in the input domain $\mathcal{X}$.
    \item Expensive Evaluations: Choosing an input $\mathbf{x}$ and evaluating $f(\mathbf{x})$ takes a very long time or incurs a large financial cost.
    \item Noise: The evaluation of a given $\mathbf{x}$ is a noisy process. In addition, this noise may vary across $\mathcal{X}$, making the underlying process heteroscedastic.
\end{enumerate}

A motivating example is molecular property optimisation where the input domain $\mathcal{X}$ is a set of molecular graphs $\{ \mathbf{x} : \mathbf{x} \in \mathcal{X}\}$ and the black-box function $f(\mathbf{x})$ is the property of the molecule to be optimised. $f$ maps a molecule to its property, but its analytic form is unknown and so instead $f$ must be queried through experiment by synthesising a molecule and measuring the value of its property under $f$. This is a time-consuming and financially expensive process. In addition, the measurement process using laboratory equipment is typically noisy.

\subsection{Solution Methods}

In the absence of an analytic form for the function to be optimised, strategies for solving black-box optimisation problems tend to proceed by sequentially evaluating the black-box function until the global optimum is found or the evaluation budget is exhausted. Such strategies may be represented by the abstract blueprint of sequential optimisation outlined in Algorithm~\ref{alg:sequential}.

\begin{algorithm}
\caption{Sequential Optimisation}
\begin{algorithmic}
      \State \textbf{input}: initial dataset $\mathcal{D}$ \algorithmiccomment{may be empty}
      \State \textbf{repeat}
            \State \hspace{\algorithmicindent} $\mathbf{x} \gets \text{Policy}(\mathcal{D})$ \algorithmiccomment{select the next input}
            \State \hspace{\algorithmicindent} $y \gets \text{Evaluate}(\mathbf{x})$ \algorithmiccomment{evaluate the black-box at the chosen input}
            \State \hspace{\algorithmicindent $\mathcal{D} \gets \mathcal{D} \cup \{(\mathbf{x}, y)\}$} \algorithmiccomment{update the dataset}
            
      \State \textbf{until} termination condition reached \algorithmiccomment{e.g. evaluation budget exhausted}
      \State \textbf{return} $\mathcal{D}$

\end{algorithmic}
\label{alg:sequential}
\end{algorithm}

Sequential optimisation algorithms differ in their choice of policy, or in other words, how they make use of the dataset of evaluations $\mathcal{D}$. Strategies may be non-adaptive in the sense that they ignore $\mathcal{D}$ completely, or they may be adaptive in the sense that they use the information about the black-box function stored within $\mathcal{D}$ to inform the selection of the next input $\mathbf{x}$ \citep{2022_Garnett}. Some of the most relevant solution methods for black-box optimisation include:

\noindent \textbf{Grid Search:} Perhaps the most well-known strategy for black-box optimisation problems, such as machine learning hyperparameter tuning, is grid search. Grid search is a deterministic, non-adaptive strategy where the policy consists of an exhaustive search through the input domain $\mathcal{X}$ by manually specifying a subset of inputs to query. Typically the manually-specified inputs are evenly spaced throughout the input domain and hence assume the form of a \say{grid}. Grid search suffers from the curse of dimensionality \citep{1957_Bellman} since the number of inputs to evaluate grows exponentially as a function of the dimensionality of $\textbf{x}$. Grid search is still a popular strategy in practice, however, due to its ease of implementation and the fact that it is \say{embarrassingly parallel} in so far as evaluations tend to be independent of each other.

\noindent \textbf{Random Search:} This stochastic, non-adaptive strategy consists of draws from a uniform density over the input domain $\mathcal{X}$. It has been demonstrated empirically that in high dimensions, random search can often outperform grid search due to its robustness to non-informative dimensions of the input space \citep{2012_Bergstra}. Random search is used as a baseline strategy in Chapters 4 and 6.

\noindent \textbf{Bayesian Optimisation:} A third solution method is an adaptive strategy where the policy is derived from Bayesian decision theory \citep{2005_De_Groot, 1985_Berger, 2007_Robert} and formalises the approach to decision-making under uncertainty with respect to the unknown objective function. \textsc{bo}, which is the principal subject of Chapter 6 and plays a major role in Chapter 4, has recently achieved notable and widely-publicised success as a component of AlphaGo \citep{2018_Yutian} as well as across applications including chemical reaction optimisation \citep{2021_Shields}, robotics \citep{2016_Calandra}, and machine learning hyperparameter optimisation \citep{2021_Turner, 2020_Rivers}. \textsc{bo} will be the focus from hereon in.

\begin{figure}[h]
    \begin{center}
        \includegraphics[width=1\textwidth]{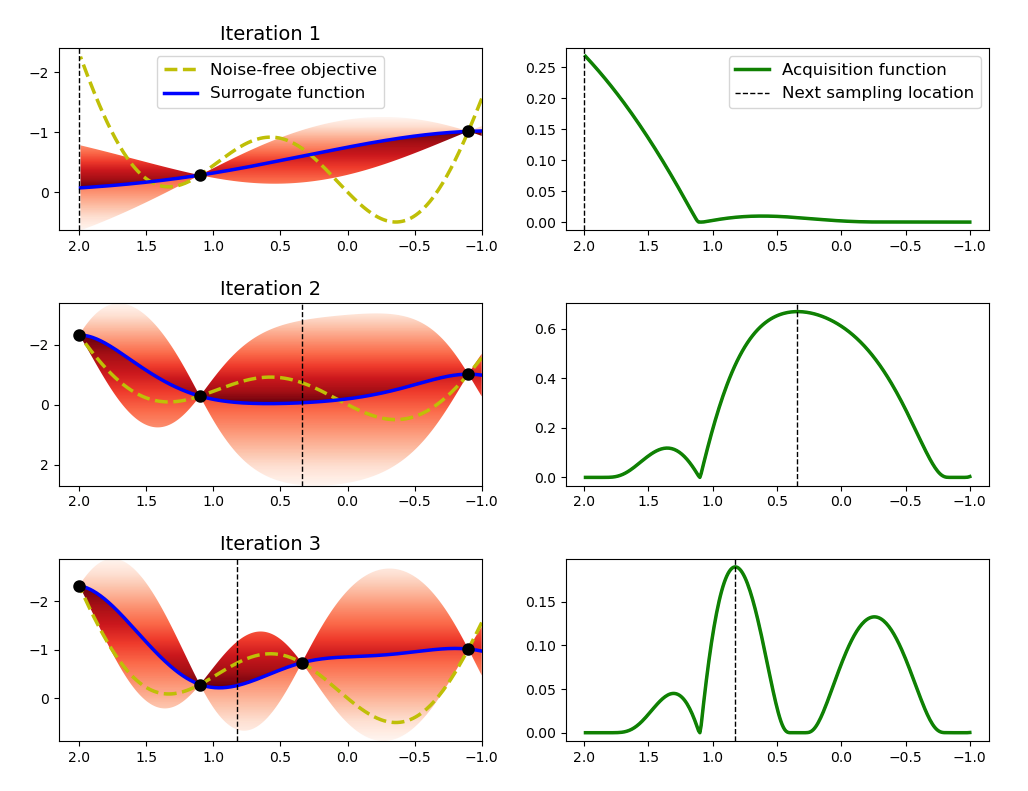}
    \end{center}
    \caption{An illustration of the Bayesian optimisation algorithm.}
    \label{bayesopt_illustration}
\end{figure}

\subsection{The Bayesian Optimisation Algorithm}

The \textsc{bo} algorithm, illustrated in Algorithm~\ref{alg:bayesian}, implements the policy from Algorithm~\ref{alg:sequential} through the use of two components:

\begin{algorithm}
\caption{Bayesian Optimisation}
\begin{algorithmic}
      \State \textbf{input}: initial dataset $\mathcal{D}$ \algorithmiccomment{may be empty}
      \State \textbf{repeat}
      
            \State choose $\mathbf{x}$ by optimising $\alpha$, the acquisition function
      
            \State \hspace{\algorithmicindent} \begin{equation*}\mathbf{x} \gets          \operatorname*{arg\,max}_{\mathbf{x} \in \mathcal{X}} \alpha(\mathbf{x}; \mathcal{D})\end{equation*}
         
            \State \hspace{\algorithmicindent} $y \gets \text{Evaluate}(\mathbf{x})$ \algorithmiccomment{evaluate the black-box at the chosen input}
            \State \hspace{\algorithmicindent $\mathcal{D} \gets \mathcal{D} \cup \{(\mathbf{x}, y)\}$} \algorithmiccomment{update the dataset and surrogate model}
            
      \State \textbf{until} termination condition reached \algorithmiccomment{e.g. evaluation budget exhausted}
      \State \textbf{return} $\mathcal{D}$

\end{algorithmic}
\label{alg:bayesian}
\end{algorithm}

\noindent \textbf{Surrogate Model:} A flexible probabilistic model that captures the prior belief about the behaviour of the black-box objective $f(\mathbf{x})$. A probabilistic model is necessary to ensure that uncertainty in the values of the black-box objective is maintained across the design space. This uncertainty measure is then used to inform the data collection policy known as the acquisition function. When a new data point is collected, the surrogate model is updated by means of re-training.

\noindent \textbf{Acquisition Function:} The acquisition function $\alpha(\mathbf{x}, \mathcal{D})$ determines the next input on a given iteration of \textsc{bo} by leveraging the uncertainty estimates of the surrogate model to trade off exploration and exploitation. It is beneficial to explore regions of the design space where the value of the objective is unknown, yet with a finite budget of function evaluations it is desirable to exploit the knowledge acquired to locate an input close to the global optimum of the function. From a computational standpoint, the acquisition function should be cheaper to evaluate relative to the black-box function. It should also be easy to optimise \citep{2018_Wilson, 2020_Grosnit, 2020_Schweidtmann}.

The pseudocode for \textsc{bo} in Algorithm~\ref{alg:bayesian} does not represent a single instantiation of an algorithm but rather a class of algorithms reflecting the broad range of choices available for both the surrogate model and the acquisition function. The set of criteria for choosing the surrogate model and the acquisition function will now be discussed.

\subsection{The Surrogate Model}

The desiderata for the surrogate model in \textsc{bo} are often related to the quality of the posterior distribution and the scalability of the model. In an idealised scenario, viewing Bayesian inference as an optimal calculus for dealing with incomplete information \citep{2010_Turner, 1961_Cox, 2003_Jaynes, 2003_MacKay}, one would obtain uncertainty estimates using full Bayesian inference over the surrogate model posterior. Full Bayesian inference is computationally demanding however and can be infeasible if the \textsc{bo} problem features a large dataset or a long horizon of function evaluations. To date, \textsc{gp}s have been the model of choice for \textsc{bo} on small datasets due to the ability to perform full Bayesian inference. \textsc{gp} surrogates have the following strengths and weaknesses from the point of view of \textsc{bo}:

\paragraph{Strengths}

\begin{enumerate}
    \item Full Bayesian inference, admits a closed-form posterior predictive distribution via exact inference. In contrast, approximate inference methods run the risk of degrading the quality of the uncertainty estimates \citep{2020_Foong}. The importance of uncertainty estimate quality in obtaining strong empirical performance is regularly emphasised in the \textsc{bo} literature \citep{2015_Shahriari, 2022_Garnett}.
    \item We can perform Bayesian model selection at the hyperparameter level meaning that we are more robust to overfitting. This is facilitated by an analytic form for the marginal likelihood.
    \item Few of the \textsc{gp}s hyperparameters needs to be determined by hand for example through hyperparameter search routines. This makes \textsc{gp}s well-suited to problems such as \textsc{bo} in which running hyperparameter search per iteration of the \textsc{bo} loop is not practically feasible \citep{2003_MacKay}. 
\end{enumerate}

\newpage

\paragraph{Weaknesses}

\begin{enumerate}
    \item Common choices of \textsc{gp} kernels are stationary kernels, meaning they cannot accurately model situations in which the complexity of the objective function varies in different regions of the input space. While non-stationary kernels, warping functions \citep{2020_Rivers, 2019_Balandat}, deep \textsc{gp}s \citep{2013_Damianou, 2021_Hebbal}, and normalising flows \citep{2020_Maronas} are potential solutions, they introduce additional complexity into the \textsc{bo} algorithm.

    \item The \textsc{gp} marginal distribution is not heavy-tailed. If outlier detection is a concern for example, one may wish to employ a heavier-tailed distribution such as the student T-process of \cite{2014_Shah} which has shown some success as a surrogate for \textsc{bo} \citep{2017_Cantin}.
    
    \item The observation model assumes homoscedastic Gaussian noise. While modifications to the standard \textsc{gp} framework exist to capture more complex noise distributions \citep{2021_Griffiths, 2021_Makarova}, they likely require more data in order to operate effectively.

    \item The most frequently cited downside of the \textsc{gp} framework for \textsc{bo} is the computational complexity of performing full Bayesian inference. Computing the inverse of the covariance matrix $[K(\mathbf{X}, \mathbf{X}) + I\sigma^2_y]^{-1}$ is $O(N^3)$ in the number of data points $N$. This covariance matrix appears in the expression for the marginal likelihood in addition to the predictive mean and covariance. A mitigating factor is that, for a fixed set of kernel hyperparameters, the Cholesky decomposition of this matrix may be computed once and stored, yielding a complexity of $O(N^2)$ for future predictions. In \textsc{bo}, however, the kernel hyperparameters are recomputed each time a new data point is collected. The $O(N^3)$ complexity cannot be avoided in this instance. Scalable surrogate model alternatives such as deep neural networks (DNNs) \citep{2015_Snoek, 2016_Springenberg, 2018_Perrone, 2021_White}, sparse \textsc{gp}s \citep{2018_Design, 2020_Griffiths}, and transformers \citep{2022_Maraval} have been trialled but face challenges in terms of the quality of the model uncertainty estimates.
    
    \item \textsc{gp}s often struggle to model functions in high-dimensional, continuous input spaces. In as little as $10$ input dimensions, the predictive capabilities of \textsc{gp}s can be impaired because the covariance function stipulates that inputs separated by more than a few lengthscales are negligibly correlated \citep{2014_Garnett}. As such, the majority of the input domain $\mathcal{X}$ may be uncorrelated with the observed data making prediction challenging. Some popular approaches in high-dimensional spaces include embedding methods such as variational autoencoders (VAEs) which seek to learn a low-dimensional embedding of the input data \citep{2018_Design, 2020_Griffiths, 2021_Grosnit, 2021_Verma, 2022_Hie, 2022_Maus}.

\end{enumerate}

In the \textsc{bo} problems considered in this thesis, however, many of the aforementioned limitations of \textsc{gp}s do not apply. The scientific datasets lie in the small data regime due to factors such as the expense of collecting laboratory measurements of synthesised molecules or the limited observational history of celestial objects and so the scalability of the surrogate model is not an issue. Similarly, the only high-dimensional input space considered is that of molecular fingerprints in which each input dimension is binary and so the problem of extrapolation in high-dimensional, continuous input spaces is avoided. The only exception is the case of the attempt to model heteroscedastic noise distributions in Chapter 6. In this case a bespoke heteroscedastic \textsc{gp} surrogate and acquisition function is devised.

\subsection{The Acquisition Function}

A sequential optimisation algorithm such as that defined in Algorithm~\ref{alg:sequential} requires a policy or acquisition function $\alpha:\mathcal{X} \to \mathbb{R}$ to provide a score for each potential observation.

\paragraph{Evaluation of Policies:}

The ideal performance metric for a \textsc{bo} scheme would quantify how close the set of queried inputs were to the global optimum of the black-box function. Regret is one such metric for quantifying optimisation performance that is commonly used in the analysis of optimisation algorithms. While there are many formulations of regret in different contexts, the central idea is to compare the values of the objective function visited during optimisation with the value of the global optimum. The larger the gap between these values is, the more regret is retrospectively incurred. The instantaneous regret is defined as $r_n(\mathbf{x}) = f^* - f(\mathbf{x}_n)$, where $f^*$ is the global optimum of the black-box and $f(\mathbf{x}_n)$ is the value of the function queried at the input $\mathbf{x}_n$ at iteration $n$. Two derived forms of regret are typically used in theoretical analysis of optimisation algorithm performance:

\begin{enumerate}
    \item \textbf{Simple Regret} - gives the instantaneous regret at the final iteration of \textsc{bo} as $r_\tau = f^* - \max f(\mathbf{x}_\tau)$, where $\mathbf{x}_\tau$ represent the set of inputs queried at the terminal iteration $\tau$. This metric has the advantage of not punishing the algorithm for explorative queries early in the search procedure.
    \item \textbf{Cumulative Regret} - is defined as $R_N = \frac{1}{N} \sum_{n=0}^N r_n$, where $r_n$ is the instantaneous regret at iteration $n$. Thus the cumulative regret is an average over all queries. The simple regret can be obtained by taking the last term, $r_N$, in the expression for the cumulative regret as the performance metric and setting $N = \tau$, where $\tau$ is the terminal iteration.
\end{enumerate}

\paragraph{Designing an Optimal Policy:}

In terms of designing an optimal policy, however, the regret metric cannot be used directly since the global optimum $f^*$ is unknown. In this instance a concept from Bayesian decision theory known as the utility $u(a, \psi, \mathcal{D})$ can be applied, where $a$ represents the action i.e. a choice of query location $\mathbf{x}$, $\psi$ represents the uncertain elements in the optimisation problem e.g. the objective function values, and $\mathcal{D}$ is the dataset of input/observation pairs collected so far. Maximising the expected utility of the data returned by the optimisation algorithm at a given iteration, however, requires consideration of the entire remainder of the optimisation query budget. Under such long time horizons the optimal policy becomes prohibitive to compute. Some attempts have been made to approximate an idealised look-ahead policy \citep{2009_Garnett, 2010_Ginsbourger, 2016_Osborne} but in practice most \textsc{bo} policies take the form of acquisition functions; myopic heuristics that attempt to trade off exploration and exploitation. When linked with the probabilistic surrogate model, this translates to greedily selecting queries which have high (for maximisation problems) predictive mean (exploitation) and high predictive variance (exploration).

\newpage

\paragraph{Classes of Acquisition Functions}

While there exist a broad range of acquisition functions \citep{2015_Shahriari, 2020_Grosnit}, a large subset of commonly-used acquisitions can be divided into three classes:

\begin{enumerate}
    \item \textbf{Optimistic Acquisition Functions} - An example of this type of acquisition function is the Upper Confidence Bound (UCB) \citep{2010_Srinivas}. In the bandits literature these methods are described by the term \say{optimism in the face of uncertainty} because they assign higher values to actions with high uncertainty.
    \item \textbf{Improvement-Based Acquisition Functions} - Are defined relative to some incumbent target, typically taken to be the best queried value found so far in the optimisation. Examples of this class include Probability of Improvement (PI) \citep{1964_Kushner} and Expected Improvement (EI) \citep{1971_Saltenis, 1978_Mockus, 1998_Jones}. The EI acquisition will be used and extended in Chapter 6.
    \item \textbf{Information-Based Acquisition Functions} - In these methods, the posterior over the unknown optimiser $\mathbf{x}_*$, induced implicitly by the posterior distribution over objective functions, is used as a means of selecting queries. Instances of this class of acquisition function include Thompson Sampling (TS) \citep{1933_Thompson}, Entropy Search (ES) \citep{2012_Hennig}, Predictive Entropy Search (PES) \citep{2014_Lobato}, General-Purpose Information-Based Bayesian Optimisation (GIBBON) \citep{moss2021gibbon}, and the Informational Approach to Global Optimization (IAGO) \citep{2009_Villemonteix}.
\end{enumerate}

Additionally, ensembles of acquisition functions known as portfolios are popular in practice and may perform better than any individual acquisition function \citep{2011_Hoffman, 2013_Hoffman, 2014_Shahriari, 2020_Rivers}.

\nomenclature[Z-SQE]{SQE}{Squared Exponential}
\nomenclature[Z-RBF]{RBF}{Radial Basis Function}
\nomenclature[Z-RQ]{RQ}{Rational Quadratic}
\nomenclature[Z-NLML]{NLML}{Negative Log Marginal Likelihood}
\nomenclature[Z-DNN]{DNN}{Deep Neural Network}
\nomenclature[Z-VAE]{VAE}{Variational AutoEncoder}
\nomenclature[Z-UCB]{UCB}{Upper Confidence Bound}
\nomenclature[Z-PI]{PI}{Probability of Improvement}
\nomenclature[Z-EI]{EI}{Expected Improvement}
\nomenclature[Z-TS]{TS}{Thompson Sampling}
\nomenclature[Z-ES]{ES}{Entropy Search}
\nomenclature[Z-PES]{PES}{Predictive Entropy Search}
\nomenclature[Z-GIBBON]{GIBBON}{General-Purpose Information-Based Bayesian Optimisation}
\nomenclature[Z-IAGO]{IAGO}{Informational Approach to Global Optimisation}

\chapter{Modelling Black Hole Signals with Gaussian Processes}
\chapterimage[height=130pt]{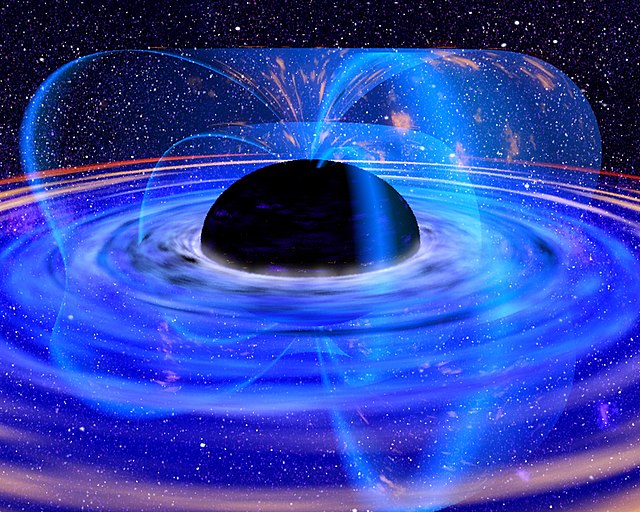}

\ifpdf
    \graphicspath{{Chapter3/Figs/Raster/}{Chapter3/Figs/PDF/}{Chapter3/Figs/}}
\else
    \graphicspath{{Chapter3/Figs/Vector/}{Chapter3/Figs/}}
\fi

\textbf{Status:} Published as Griffiths, RR., Jiang, J., Buisson, DJ., Wilkins, D., Gallo, LC., Ingram, A., Grupe, D., Kara, E., Parker, ML., Alston, W., Bourached, A. Cann, G., Young, A., Komossa, S., Modeling the Multiwavelength Variability of Mrk 335 Using Gaussian Processes. \textit{The Astrophysical Journal}, 2021.

\section{Background on High-Energy Astrophysics}

The chapter begins with a self-contained background on high-energy astrophysics to aid in contextualising the findings.

\subsection{Black Holes}

John Michell was the first to posit the existence of  black holes \citep{1784_Michell}, describing them as \say{dark stars} due to the fact that no light could escape from them. At the time, however, his work was largely ignored due to the absence of a theory of gravity describing the behaviour of light in a strong gravitational field. Following the introduction of the Einstein field equations from General Relativity \citep{1916_Einstein}, Schwarzchild was the first to calculate the radius of a black hole in the Schwarzschild metric \citep{1916_Schwarzschild}. The Schwarzchild radius is

\begin{equation}
    r_s = \frac{2GM}{c^2},
\end{equation}

\noindent where $G$ is the gravitational constant, $M$ is the mass of the object and $c$ is the speed of light. Black holes are characterised according to their mass and spin. When the mass of a black hole exceeds $10^5 M_{\odot}$, it is termed a supermassive black hole (SMBH), where $M_{\odot}$ is the solar mass unit, approximately equal to the mass of the Sun.

\subsection{Active Galactic Nuclei}

The term Active Galactic Nucleus (AGN) was coined by Viktor Ambartsumian in the early 1950s \citep{1997_Victor}. Ambartsumian argued that the nuclei of galaxies were subject to explosions which caused large amounts of mass to be expelled, and that for these explosions to occur, galactic nuclei must contain unknown bodies of huge mass. Moreover, AGN were observed to be highly luminous with unusual spectral properties, indicating that their power source could not be ordinary stars. In 1964, some insight on the nature of AGN was offered by Salpeter and Zeldovich  \citep{1964_Salpeter, 1964_Zeldovich}, who proposed accretion of gas onto a SMBH as the mechanism responsible for the power source of a powerful class of AGN known as quasars. \citet{1969_Lynden} later paid testament to the importance of the black hole accretion disc model by remarking that, 

\begin{displayquote}
\say{With different values of the black hole mass and accretion rate these discs are capable of providing an explanation for a large fraction of the incredible phenomena of high-energy astrophysics.}
\end{displayquote}

\noindent Lynden-Bell's statement is supported by the fact that AGN are one of the most persistent luminous sources of electromagnetic radiation in the universe and as such, may be leveraged to discover distant objects. Furthermore, the evolution of AGN in cosmic time may be used to inform theoretical models of the cosmos. It is estimated that one fifth of research astronomers work on AGNs \citep{1997_Peterson}. 

The observed properties of AGNs depend on the mass of the central SMBH, the extent that the nucleus is obscured by dust, the orientation of the accretion disc, the rate of gas accretion, as well as the presence or absence of outflows of ionised matter along the axis of rotation known as jets. Some subclasses of AGN include quasars, the most powerful form of AGN, blazars, which contain a jet pointed toward the Earth, and Seyfert galaxies which are characterised by broad emission lines in the optical band. It is the last of these categories of AGN that is the subject of this chapter and will be discussed next.

\subsection{Seyfert Galaxies}

In 1943, Carl Seyfert systematically studied a collection of bright AGN possessing broad emission lines in the optical band \citep{1943_Seyfert}. The eponymous Seyfert galaxies are further subdivided into Seyfert 1 (Sy1) and Seyfert 2 (Sy2) galaxies based on their emission line range, $1000-20,000 \text{kms}^{-1}$ and $300-1000 \text{kms}^{-1}$ respectively. The orientation-based unified model of is one of the most popular means of describing Seyfert galaxies and is based on the idea that classes of AGN are physically similar but are viewed at different orientations \citep{1993_Antonucci, 1995_Urry}. Some features of the model include:

\begin{itemize}
    \item In the narrow line region there is ionised, low-velocity and low-density gas extending to $100-1000$ parsecs (pc).
    \item In the broad line region there are high-density, dust-free gas clouds located at a distance of $0.01-1$ pc from the SMBH moving at Keplerian velocities.
    \item There is an antisymmetric dusty structure known as a torus located at a distance of $0.1=10$ pc from the SMBH.
    \item There is a sub-pc accretion disc located around the SMBH which may be optically thick or optically thin depending on the disc's state.
    \item There is an outflowing radio jet pointed in the general direction of the accretion disc.
\end{itemize}

\begin{figure}[!htbp]
    \begin{center}
        \includegraphics[width=0.75\textwidth]{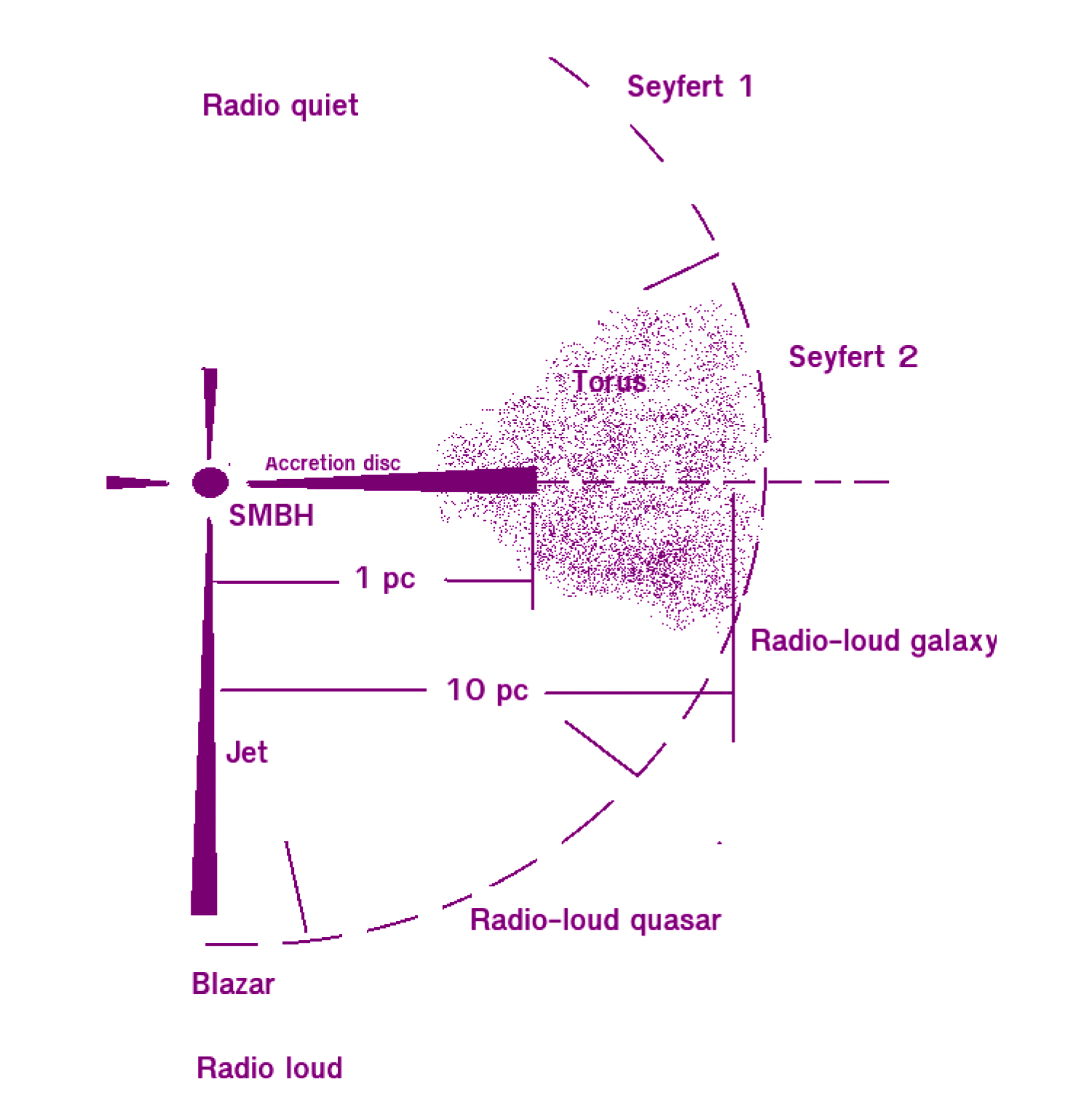}
    \end{center}
    \caption{The orientation-based unified model of AGN. Reprinted with permission from \cite{2019_Jiang}.}
    \label{smbh}
\end{figure}

A schematic for the orientation-based unified model is provided in \autoref{smbh}. In Sy2 galaxies, the narrow line region is viewable by an edge-on observer due to the fact that it is more extended relative to the broad line region. Both the broad line region and the accretion disc are obscured by the middle plane of the torus. In Sy1 galaxies, the observer is closer to the torus axis and has an unobscured line of sight towards the nuclear region of the AGN. While the orientation-based unified model can explain the spectral diversity of many AGN, it has recently been challenged. For example, a different torus shape is suggested by interferometry observations in the mid-infrared band of Sy1 galaxies which demonstrate that the majority of infrared emission originates from dust in the polar region as opposed to the disc plane \citep{2013_Honig}.

\subsection{Space Observatories}

High energy X-rays function as the primary tool for examining the innermost regions of AGN as they originate from the area closest to the central SMBH and can easily penetrate absorbing materials along the line of sight. Optical/UV emission is also useful in characterising the behaviour of AGN, for example, in verifying reprocessing models of X-ray emission by computing lags between X-ray and optical/UV lightcurves, where a lightcurve is a graph of the light intensity of a celestial object or region with respect to time.

All astronomical data in this thesis originates from the Neil Gehrels Swift observatory which was first launched by NASA in 2004 to detect and study Gamma-Ray Bursts (GRBs) using the Burst Alert Telescope (BAT) \citep{2005_Barthelmy}. While originally designed for the study of GRBs, Swift now also functions as a multiwavelength observatory containing an X-ray Telescope (XRT) \citep{2005_Burrows} and a UV/Optical Telescope (UVOT) \citep{2005_Roming}. Swift is also used to conduct long-term all sky surveys.

\subsection{Accretion Disc Models}

Theoretical models to explain accretion discs differ based on the physical processes considered. Four representative examples are the Polish doughnut (thick disc), Shakura-Sunyaev (thin disc), slim disc, and the advection-dominated accretion flow (ADAF) models \citep{2011_Abramowicz}. These theoretical models are not mutually exclusive in the sense that different aspects of real physical systems may be best described by different models.

\noindent \textbf{Polish Doughnut (Thick Disc) Model:} The \say{Polish Doughnut} introduced by Paczynski and collaborators in the 1970s and 80s \citep{1980_Jaroszynski, 1980_Paczynski, 1981_Paczynski, 1982_Paczynski} is the minimal analytic accretion disc model in so far as it only considers gravity and assumes a perfect fluid. The Polish Doughnut Model is predicated on a general method for constructing perfect fluid equilibria of matter orbiting an uncharged, rotating black hole known as a Kerr black hole \citep{1963_Kerr}.

\noindent \textbf{Thin Disc Models:} The majority of analytic accretion disc models assume a stationary and axially-symmetric state for matter undergoing accretion onto the black hole with all physical quantities depending only on two spatial coordinates, $r$, the radial distance from the disc centre, and $z$, the vertical distance from the equatorial plane of symmetry. Unlike the Polish Doughnut Model, which assumes vertically thick discs, in thin disc models, $\frac{z}{r} \ll 1$ applies at all points within the matter distribution. In 1973, Shakura and Sunyaev introduced the canonical thin disc model \citep{1973_Shakura} by specifying additional physically reasonable assumptions that allowed them to construct a set of algebraic equations from the standard set of thin disc equations. The relativistic extension of the Shakura-Sunyaev model was later proposed by \cite{1973_Novikov}.

\noindent \textbf{Slim Disc Models:} Slim discs are characterised by $\frac{z}{r} \leq 1$. Thin disc models such as the Shakura-Sunyaev and Novikov–Thorne models assume that viscous heating is balanced locally by radiative cooling i.e. the accretion process is radially efficient and as such, all viscosity-generated heat is radiated away. Although the assumption is valid if the accretion rate is small, at a luminosity $L = 0.3 L_{\text{Edd}}$ \footnote{$L_{\text{Edd}}$ is the Eddington limit, the maximum achievable luminosity of a body subject to the balance of an outward radiative force and an inward gravitational force.} the radial velocity is large and the disc is sufficiently thick to permit advection to function as a cooling mechanism. At the highest luminosities, thin disc models no longer apply as the cooling effect of advection becomes comparable to radiative cooling. The standard thin disc model equations become a two-dimensional system of ordinary differential equations with a critical point for the slim disc case. These equations were first solved by \cite{1988_Abramowicz} and extended to a fully relativistic treatment by \cite{1998_Beloborodov}.

\noindent \textbf{ADAF Models:} Advection-dominated accretion flow models, introduced first in a series of papers \citep{1995_Narayan, 1994_Narayan, 1995_Abramowicz, 1996_Abramowicz, 1998_Gammie}, assume that almost all viscously dissipated energy is not radiated but advected into the black hole and applies when the luminosity and mass accretion rate are low. As such, ADAF discs are typically far less luminous than thin discs. Fully relativistic solutions to such discs have been obtained numerically \citep{1997_Abramowicz, 1997_Beloborodov}. Further information on ADAFs is available in \cite{1997_Narayan}.

\subsection{Markarian 335}

The accretion disc of Markarian 335 (Mrk 335) is the focus of study in this chapter. Mrk 335 is a Sy1 galaxy located 324 million light-years from Earth in the constellation of Pegasus. The central SMBH of Mrk 335 is notable for the spinning rate of its corona at ca. 20\% the speed of light. Relativistic blurring of the reflection of the accretion disc has been used to infer the geometry of the corona \citep{2015_Wilkins_drive}. By using \textsc{gp}s to interpolate the unevenly-sampled lightcurves of Mrk 335 and performing a cross-correlation analysis, some insight into the structure of the accretion disc may be obtained, and subsequently used to inform future developments in accretion disc theories. Of the aforementioned disc theories, the Shakura-Sunyaev thin disc model is the most relevant for Mrk 335 as its predictions for the extent of UV emission match that from observation. The distance between the UV and X-ray emission regions however is shorter than the light travel time measured using \textsc{gp}-based inference on the observational data. The main contributions of this chapter are now introduced.

\section{Preface}

The optical and UV variability of the majority of AGN may be related to the reprocessing of rapidly-changing X-ray emission from a more compact region near the central black hole. Such a reprocessing model would be characterised by lags between X-ray and optical/UV emission due to differences in light travel time. Observationally, however, such lag features have been difficult to detect due to gaps in the lightcurves introduced through factors such as source visibility or limited telescope time. In this chapter, \textsc{gp} regression is employed to interpolate the gaps in the Swift X-ray and UV lightcurves of the narrow-line Seyfert 1 galaxy Mrk 335. In a simulation study of five commonly-employed analytic \textsc{gp} kernels, it is concluded that the Matérn $\frac{1}{2}$ and rational quadratic kernels yield the most well-specified models for the X-ray and UVW2 bands of Mrk 335. In analysing the structure functions of the \textsc{gp} lightcurves, a broken power law is obtained with a break point at 125 days in the UVW2 band. In the X-ray band, the structure function of the \textsc{gp} lightcurve is consistent with a power law in the case of the RQ kernel, whilst a broken power law with a break point at 66 days is obtained from the Matérn $\frac{1}{2}$ kernel. The subsequent cross-correlation analysis is consistent with previous studies and furthermore, shows tentative evidence for a broad X-ray-UV lag feature of up to 30 days in the lag-frequency spectrum. The significance of the lag depends on the choice of \textsc{gp} kernel. 

\section{Introduction} \label{intro}

AGN show strong and variable emission across multiple wavelengths. The UV emission from an AGN is believed to be dominated by thermal emission from an accretion disc close to the central SMBH \citep{pringle81}. The variability of optical and UV AGN \footnote{AGN with an UV and optical luminosity change of more than 1 magnitude such as changing-look AGN, are not discussed in this chapter cf. \cite{jiang21} for details.} emission is stochastic and described by random Gaussian fluctuations \citep{welsh11,gezari13,zhu16,sanchez18,smith18,xin20} with the autocorrelation functions of such fluctuations adhering to the `damped random walk' model. The X-ray emission from an AGN is often found to show faster variability relative to emission at longer wavelengths \citep{mushotzky93, gaskell03} and originates from a more compact region \citep{morgan08,chartas17}.

The relationship between the UV and X-ray emission has been well studied. For instance, correlations between the variability in two energy bands has been seen in some individual sources \citep{shemmer01,buisson17} while others do not show significant evidence for similar correlation \citep{smith07, 2018_Buisson}. In sources where correlation is found, lags that are related to the light travel time between two emission regions are frequently observed. These lags are often found to be on timescales of days and are longer than those predicted by classical disc theories \citep{1973_Shakura}. Such lag amplitudes indicate a disc of size a few times larger than expected \citep{edelson00,shappee14,troyer16, buisson17}. Alternatively, some modified models have been proposed for the underestimation of lags by the classical thin disc model, e.g. disc turbulence \citep{cai20}, additional varying FUV illumination \citep{gardner17}, a tilted or inhomogeneous inner disc \citep{dexter11,starkey17} or an extended coronal region \citep{kammoun20}. Much shorter lags, e.g. hundreds of seconds, in agreement with the Shakura-Sunyaev model \citep{1973_Shakura} have been rarely observed by comparison e.g. in NGC-4395 \citep{mchardy16}.

The Neil Gehrels \swift\ Observatory has been monitoring the X-ray sky in the past decade in tandem with simultaneous pointings in the optical and UV band. In this work, we focus on the X-ray and UVW2 ($\lambda=$212~nm) lightcurves of the narrow-line Seyfert 1 galaxy (NLS1) Mrk~335 obtained by XRT and UVOT, the soft X-ray and UV/optical telescopes on \swift.  Mrk~335 was one of the brightest X-ray sources prior to 2007, before its flux diminished by $10-50\times$ its original brightness \citep{grupe07}.  The X-ray brightness has not recovered since.  During this low X-ray flux period, the UV brightness remains relatively unchanged rendering Mrk~335 X-ray weak \citep{tripathi20}. The behavior has been explained as a possible collapse of the X-ray corona \citep{2013_Gallo_New, 2015_Gallo_New, 2014_Parker_New} and/or increased absorption in the X-ray emitting region \citep{grupe12, 2013_Longinotti, 2019_Longinotti, 2019_Parker}.

Mrk~335 has been continuously monitored since 2007 making it one of the best-studied AGN with \swift. Previous studies from the \swift\ monitoring program can be found in \citet{grupe07,grupe12, gallo18, tripathi20, 2020_Komossa}. The X-rays are constantly fluctuating and regularly display large amplitude flaring \citep{2015_Wilkins}.  The UV are significantly variable, but at a much smaller amplitude than the X-rays. \cite{gallo18} found tentative evidence for lags of $\approx20$ days based on cross-correlation analyses, suggesting a potential reprocessing mechanism of the more variable X-ray emission in the UV emitter of this source. One challenge faced by the \swift\ monitoring program is that the lightcurves are not continuously sampled and hence standard Fourier techniques cannot be applied. This uneven sampling of the lightcurves is imposed by limited telescope time.

In the context of cross-correlation analysis, methods have been developed to address the problem of unevenly-sampled lightcurves. In \citet{2000_Reynolds}, the method of \citet{1992_Press} is extended to interpolate the lightcurve gaps using a model of the covariance function, or equivalently the power spectrum, of the lightcurve. In \citet{1998_Bond, 2010_Miller, 2013_Zoghbi} a maximum likelihood approach is taken to fit models of the lightcurve power spectra which accounts for the correlation between the lightcurves. In this paper we focus on a relatively new approach to tackle unevenly-sampled lightcurves. 

Gaussian processes (\textsc{gp}s) confer a Bayesian nonparametric framework to model general time series data \citep{2013_Roberts, 2015_Tobar} and have proven effective in tasks such as periodicity detection \citep{2016_Durrande} and spectral density estimation \citep{2018_Tobar}. More broadly \textsc{gp}s have recently demonstrated modelling success across a wide range of spatial and temporal application domains including robotics \citep{2011_Deisenroth, 2020_Greeff}, Bayesian optimisation \citep{2015_Shahriari, 2020_Grosnit, 2020_Rivers} as well as areas of the natural sciences such as molecular machine learning \citep{2021_Nigam, 2020_Griffiths, 2020_flowmo, 2021_Griffiths, 2021_Hase} and genetics \citep{2020_Moss}. In the context of astrophysics there is a recent trend favouring nonparametric models such as \textsc{gp}s due to the flexiblity afforded when specifying the underlying data modelling assumptions. Applications have arisen in lightcurve modelling \citep{2021_Luger1, 2021_Luger_next}, continuous-time autoregressive moving average (CARMA) processes \citep{2021_Yu}, modelling stellar activity signals in radial velocity data \citep{2015_Rajpaul}, lightcurve detrending \citep{2016_Aigrain}, learning imbalances for variable star classification \citep{2020_Lyon}, inferring stellar rotation periods \citep{2018_Angus}, estimating the dayside temperatures of hot Jupiters \citep{2019_Pass}, exoplanet detection \citep{2017_Jones, 2017_Czekala, 2020_Gordon, 2020_Langellier}, spectral modelling \citep{2012_Gibson, 2018_Nikolov, 2020_Diamond}, as well as blazar variability studies \citep{2015_Karamanavis, 2017_Karamanavis, 2020_Covino, 2020_Yang}.

It has recently been demonstrated in lightcurve simulations by \citet{2019_Wilkins} that a \textsc{gp} framework can compute time lags associated with X-ray reverberation from the accretion disc that are longer and observed at lower frequencies than can be measured by applying standard Fourier transform techniques to the longest available continuous segments. It is for this principal reason that \textsc{gp}s are employed for the timing analysis in this chapter. Further desirable facets of \textsc{gp}s include the fact that, unlike parametric models, they do not make strong assumptions about the shape of the underlying light curve \citep{2012_Wang_light}. Additionally, Bayesian model selection may be performed at the level of the covariance function or kernel allowing the quantitative comparison of different models of the lightcurve power spectrum. Finally in the cross-correlation analysis, a weaker modelling assumption is made than in \citet{2013_Zoghbi} in treating the X-ray and UV lightcurves as being independent \citep{2019_Wilkins}.

The remainder of this chapter is outlined as follows: In Section~\ref{gp_mod} procedures used to fit \textsc{gp}s to the X-ray and UVW2 bands are described, including aspects such as identification of the flux distribution, consideration of measurement noise as well as a simulation study to determine the appropriate kernels. In Section~\ref{structure_analysis} the structure functions of the \textsc{gp}-interpolated lightcurves are compared with the observational structure functions from \cite{gallo18}. In Section \ref{lag_and_coherence} a cross-correlation analysis of the X-ray and UVW2 bands is presented using the \textsc{gp}-interpolated lightcurves. Finally, in Section~\ref{conc} concluding remarks are provided about the discrepancy between the observational and \textsc{gp}-derived structure functions as well as the implications of the cross-correlation analysis, namely that the broad lag features suggest an extended emission region of the disc in Mrk 335 during the reverberation process. All code for reproducing the analysis is available at \url{https://github.com/Ryan-Rhys/Mrk_335}.

\section{Modelling Markarian 335} 
\label{gp_mod}

\noindent This Chapter considers the Swift X-ray and UVW2 lightcurves in time bins of one day. The reader is referred to \cite{gallo18} for details of the data reduction processes. The observational measurements used in this work run from $54327-58626$ modified Julian days and comprise $509$ data points for the X-ray band and $498$ data points for the UVW2 band. The latest UVOT sensitivity calibration file (`swusenscorr20041120v006.fits') was considered so as to account for the sensitivity loss with time in the UVW2 band\footnote{The most up-to-date calibration files: \href{https://heasarc.gsfc.nasa.gov/docs/heasarc/caldb/swift}{https://heasarc.gsfc.nasa.gov/docs/heasarc/caldb/swift}. Only UVW2 data collected by UVOT is considered because the UVW2 filter was most frequently used in the archival observations.}.

\subsection{Identifying the Flux Distribution}

In order to assess the applicability of \textsc{gp}s in modelling the flux distribution of the X-ray and UVW2 bands of Mrk 335, a series of graphical distribution tests were performed to determine the sample distribution. The histograms of the log count rates for the X-ray, and flux for the UV bands, of Mrk 335 are shown in \autoref{Histograms}. The histograms show that the distribution of the UVW2 flux is approximately Gaussian-distributed whereas the X-ray count rate distribution appears to be log-Gaussian distributed in line with the general observation of \cite{2005_Uttley} that fluxes from accreting black holes tend to follow log-Gaussian distributions. Further graphical distribution tests based on probability-probability (PP) plots and empirical cumulative distribution functions (ECDFs) are provided in Appendix~\ref{dist_tests}. 

Furthermore, following \cite{2019_Wilkins} a Kolmogorov-Smirnov test for goodness-of-fit was performed, where the null hypothesis is that the sample was drawn from a Gaussian distribution. For the UVW2 flux values a p-value of $0.164$ was obtained. For the raw X-ray count rates a p-value of $1.017e^{-20}$ was obtained, and a p-value of $0.028$ for the log-transformed X-ray count rates. As such, the null hypothesis that either UVW2 flux or log-transformed X-ray count rates are drawn from a Gaussian distribution cannot be rejected at the $1\%$ level of significance. The null hypothesis may however be rejected in the case of the raw X-ray count rates, providing evidence that the raw X-ray count rates should be log-transformed in order to be well-modelled by a Gaussian distribution. As such, the raw X-ray count rates were log-transformed and the UVW2 flux values were left unchanged.

\begin{figure*}[ht]
\centering
\subfigure[X-Ray Log Count Rates]{\label{fig:4}\includegraphics[width=0.49\textwidth]{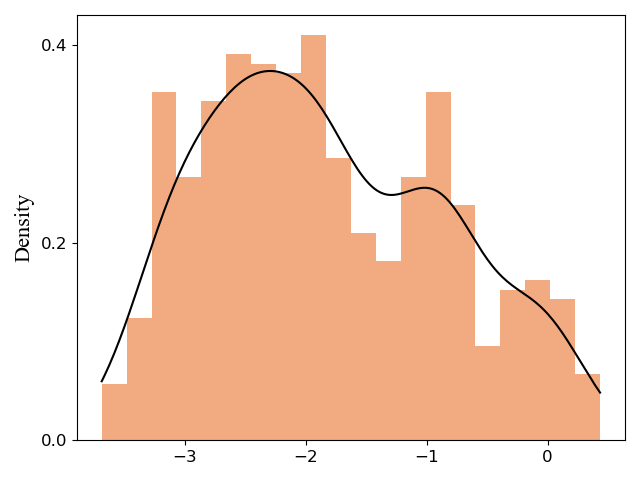}}
\subfigure[UVW2 Flux]{\label{fig:3}\includegraphics[width=0.49\textwidth]{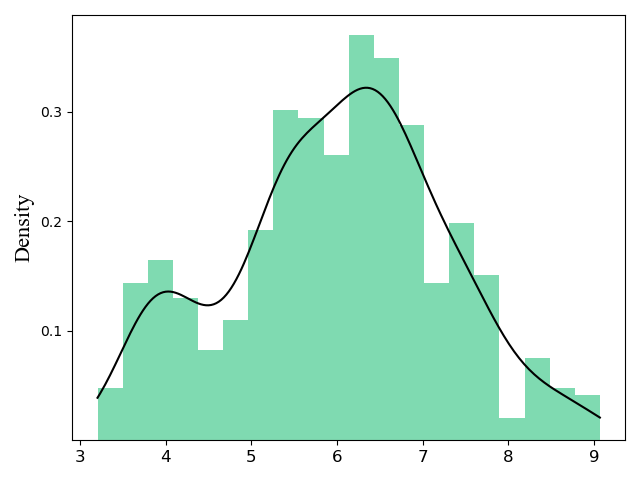}}
\caption{Histograms of the observed Swift X-ray log count rate and UVW2 flux overlaid with Gaussian kernel density estimates. The raw UVW2 flux values have been scaled by $1e^{14}$.}
\label{Histograms}
\end{figure*}

\subsection{Noise Considerations}
\label{noise}

As noted by \cite{2019_Wilkins} fitting a \textsc{gp} to the logarithm of the count rate is appropriate only in the limit of a large signal-to-noise ratio. In the case of Mrk 335, the Poisson (shot) noise intrinsic to the photon detectors used to obtain the flux measurements is over an order of magnitude smaller than the flux measurement itself. As such the choice of the log-Gaussian process would appear to be justified. 

\subsection{Lightcurve Simulations}
\label{sim_section}

\begin{figure*}[h]
    \centering
    \includegraphics[width=0.75\textwidth]{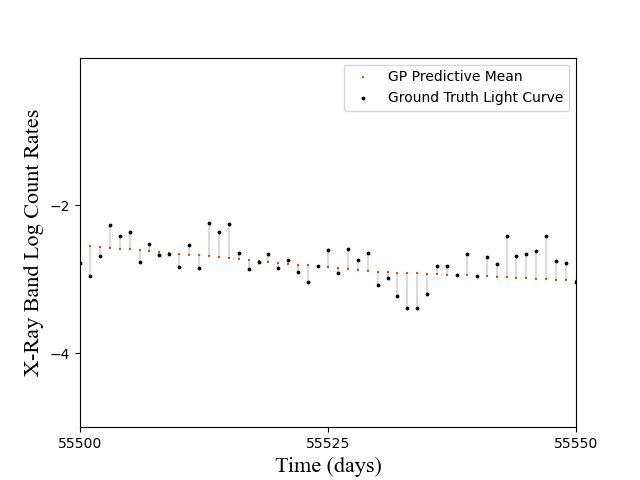}
    \caption{Residual plot. The normalised RSS metric is the sum of squared residuals divided by the total number of discretised points (4390) comprising the simulated lightcurve. A residual in this case represents the difference between the Gaussian process predictive mean and the ground truth value of the simulated lightcurve.}
    \label{rss_figure}
\end{figure*}

A simulation study was undertaken to quantitatively assess the abilities of different kernels to interpolate gapped simulated lightcurves. Observational power spectral densities (PSDs) of AGN are well-described by (broken) power laws \citep{2004_Mchardy}. As such, the simulations employed a power law PSD with index fit to the observational data. The goals with the study are twofold: Firstly, although one cannot be sure of the true PSD for the observational data, it is hoped that the simulations may afford a good proxy for identifying performant kernels based on the fact that AGN typically exhibit power law-like PSDs and secondly, it is desirable to asses the correlation between a kernel's ability to reconstruct the full simulated lightcurve and its marginal likelihood value for the gapped data on which it is trained. If there is a correlation, the marginal likelihood may be used as a metric for identifying the appropriate kernel on the observational data.

One thousand simulated light curves with gaps were generated for the \src \: X-ray and UV bands using the method of \cite{1987_Davies}, first applied in astrophysics by \cite{1995_Timmer}. For each lightcurve there is access to the ground truth functional form of the lightcurve before the introduction of gaps. Computationally, the ground truth lightcurve was evaluated on a fine, discrete grid of $4390$ time points whereas the gapped lightcurves were evaluated on a coarser, unevenly-spaced grid of $498$ time points for the UV simulations and $509$ time points for the X-ray simulations in line with the number of observational data points. How well each \textsc{gp} kernel performs in recovering the ground truth lightcurve was then quantified by measuring the normalised residual sum of squared errors,

\begin{equation}
    \text{RSS} = \frac{1}{N} \sum_{i=1}^{N} (f(t_i) - y_i)^2,
\end{equation}

\noindent where $f(t_i)$ is the \textsc{gp} prediction at grid point $t_i$ and $y_i$ is the true simulated count rate value. The RSS values were averaged over the one thousand simulated lightcurves. An illustration of the RSS metric is provided in \autoref{rss_figure}. In addition, the averaged negative log marginal likelihood (NLML) values are computed for each kernel. Kernel hyperparameters were selected via optimisation of the NLML using the SciPy optimiser of GPflow \citep{GPflow}. The jitter level was fixed at 0.001, a small positive number to ensure numerical stability. The output values (flux or the logarithm of the count rate) were standardised according to their empirical mean and standard deviation. A constant mean function set to the empirical mean of the data following standardisation was employed.

The results of the simulation study are reported in \autoref{sim_study}. The NLML values show correlation with RSS, thus providing evidence that NLML is an appropriate metric for determining the \textsc{gp} kernel for the real observational data (for which the ground truth lightcurve is of course not available). A paired t-test was conducted to determine whether the RSS results were significant in terms of identifying the best kernel. For the X-ray simulations, a t-statistic of $9$ was obtained corresponding to a two-sided p-value of $5^{-20}$. For the UVW2 simulations, a t-statistic of $-22$ was obtained corresponding to a two-sided p-value of $9^{-85}$. As such, the null hypothesis that the performance discrepancy between kernels on the RSS metric is due to chance variation across $1000$ simulations, may be rejected at the $1\%$ level of significance. Further rationalisation for why the top two performing kernels in the simulation study are the Matérn $\frac{1}{2}$ and RQ kernels is offered in Appendix~\ref{kern_rat} (plotted by Douglas Buisson).

\begin{table}[h]
\caption{Performance comparison of kernels based on the NLML on the simulated gapped X-ray and UV lightcurves and normalised residual sum of squared errors (RSS) on the ground truth simulated lightcurves. The mean NLML and RSS across 1000 simulations are reported with the standard error. UVW2 RSS values have an exponent of $-30$.}
\label{sim_study}
\begin{center}
\begin{tabular}{l|ll}
\toprule
Kernel & NLML & RSS \\ \midrule
\multicolumn{1}{c|}{\underline{\textbf{X-Ray}}} & & \\[5 pt]
Matérn$\frac{1}{2}$ & $\textbf{180.2} \pm \: \textbf{3.8}$ & $0.121 \pm \: 0.002$  \\
Matérn$\frac{3}{2}$ & $420.7 \pm \: 3.3$ & $0.309 \pm \: 0.003$ \\
Matérn$\frac{5}{2}$ & $523.5 \pm \: 2.9$ & $0.374 \pm \: 0.003$ \\
Rational Quadratic & $\textbf{184.2} \pm \: \textbf{3.6}$ & $\textbf{0.117} \pm \: \textbf{0.002}$ \\
Squared Exponential & $632.1 \pm \: 1.5$ & $0.554 \pm \: 0.004$ \\ \midrule
 \multicolumn{1}{c|}{\underline{\textbf{UVW2}}} & & \\[5 pt]
Matérn$\frac{1}{2}$ & $\textbf{-399.0} \pm \: \textbf{5.2}$ & $\textbf{2.9} \pm \: \textbf{0.08}$ \\
Matérn$\frac{3}{2}$ & $-298.3 \pm \: 6.0$ & $7.9 \pm \: 0.25$ \\
Matérn$\frac{5}{2}$ & $-219.6 \pm \: 6.5$ & $17.0 \pm \: 0.41$ \\
Rational Quadratic & $-349.2 \pm \: 5.4$ & $3.4 \pm \: 0.09$ \\
Squared Exponential & $-65.0 \pm \: 7.4$ & $32.8 \pm \: 0.55$ \\ \bottomrule
\end{tabular}
\end{center}
\end{table}

\subsection{Modelling Markarian 335 with Gaussian Processes}

The fits to the observational data for the UVW2 and X-ray bands are shown in \autoref{GP_uv_fits} and \autoref{GP_xray_fits} respectively. In an analogous fashion to the simulation experiments five stationary kernels were evaluated: Matérn $\frac{1}{2}$, Matérn $\frac{3}{2}$, Matérn $\frac{5}{2}$, rational quadratic and squared exponential. The two kernels, rational quadratic and Matérn $\frac{1}{2}$, which performed best in the simulation study in their abilities to model power law-like PSDs are displayed. These kernels also have the most favourable values under the NLML metric for the observational data. A constant mean function set to the empirical mean of the data following standardisation was again used. All kernel hyperparameters were optimised under the marginal likelihood save for the noise level which was fixed to a constant value in the standardised space. This constant noise value is computed by dividing the mean output value in the standardised space by the mean signal-to-noise ratio in the original space.

\begin{figure*}[]
\centering
\subfigure[UV Band | Matérn $\frac{1}{2}$ | Mean.]{\label{fig:4pp1}\includegraphics[width=0.49\textwidth]{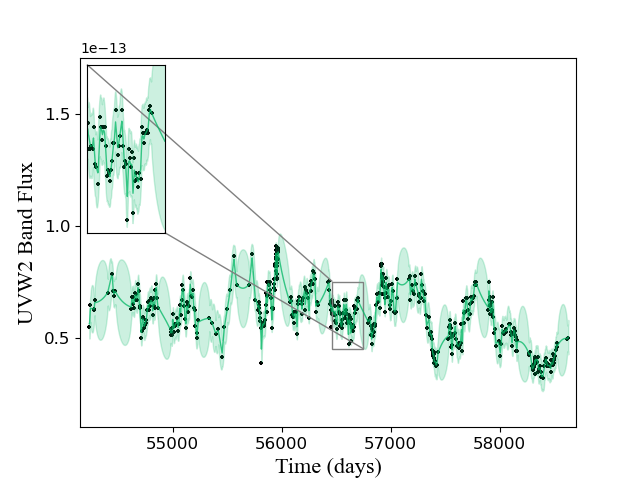}}
\subfigure[UV Band | Matérn $\frac{1}{2}$ | Sample.]{\label{fig:4pp2}\includegraphics[width=0.49\textwidth]{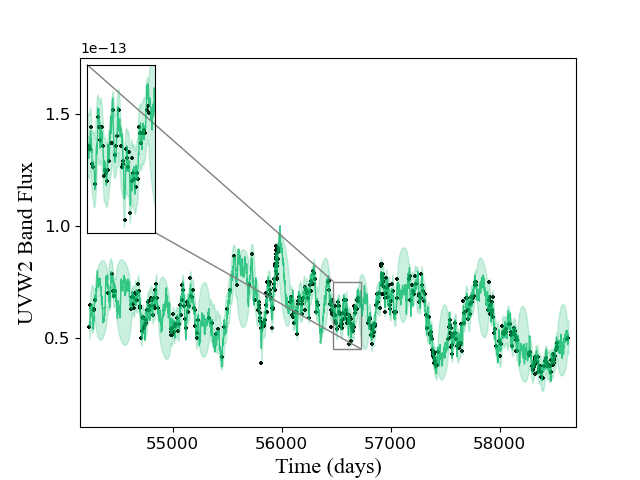}}
\subfigure[UV Band | Rational Quadratic | Mean.]{\label{fig:4pp3}\includegraphics[width=0.49\textwidth]{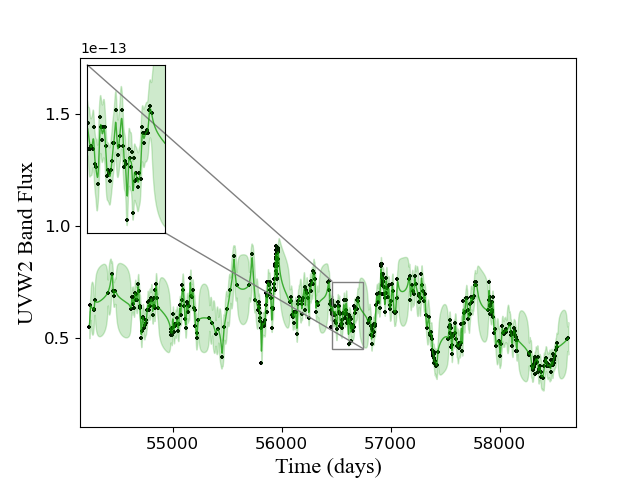}}
\subfigure[UV Band | Rational Quadratic | Sample]{\label{fig:4pp4}\includegraphics[width=0.49\textwidth]{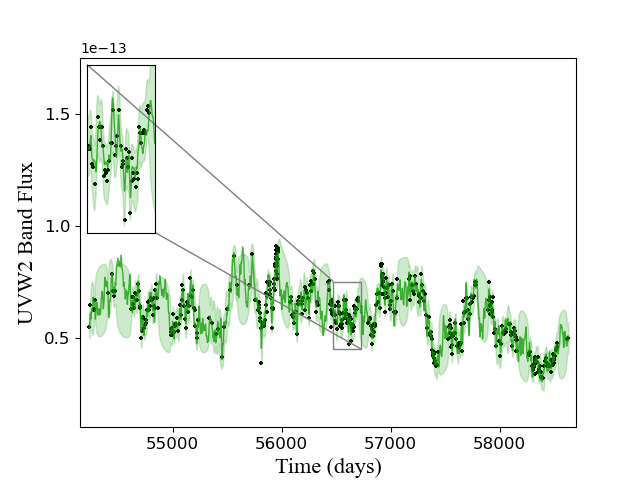}}  
\caption{\textsc{gp} lightcurves for the UVW2 band. The shaded regions denote the \textsc{gp} 95\% confidence interval. Both the \textsc{gp} mean and a sample from the \textsc{gp} posterior are shown in separate plots. The insets are included to highlight the variability of the fit.}
\label{GP_uv_fits}
\end{figure*}

\begin{figure*}[]
\centering
\subfigure[X-ray Band | Matérn $\frac{1}{2}$ | Mean.]{\label{fig:4pp}\includegraphics[width=0.49\textwidth]{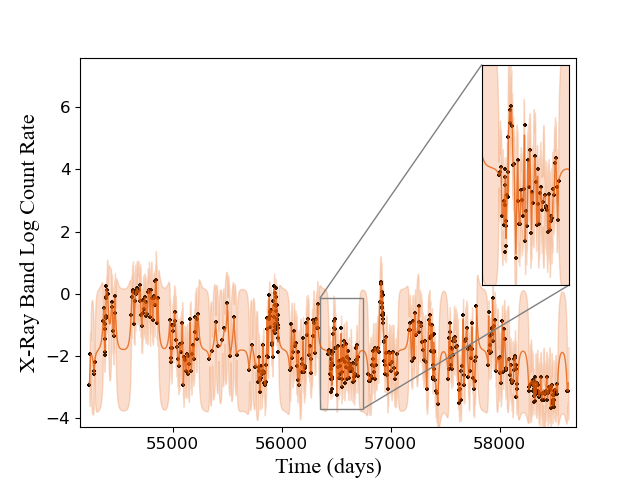}}
\subfigure[X-ray Band | Matérn $\frac{1}{2}$ | Sample.]{\label{fig:4pp}\includegraphics[width=0.49\textwidth]{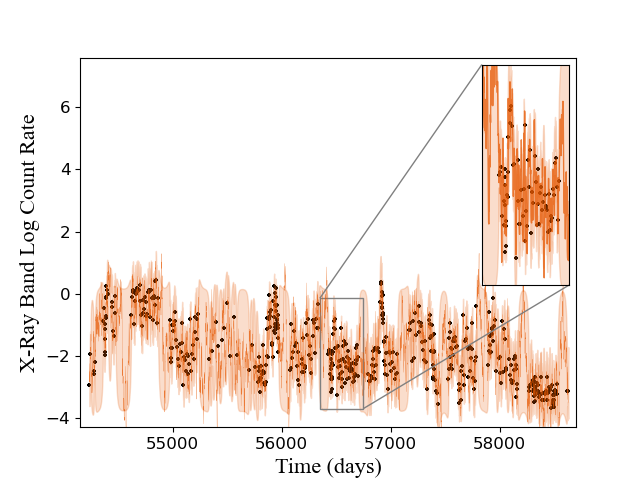}}
\subfigure[X-ray Band | Rational Quadratic | Mean.]{\label{fig:4pp}\includegraphics[width=0.49\textwidth]{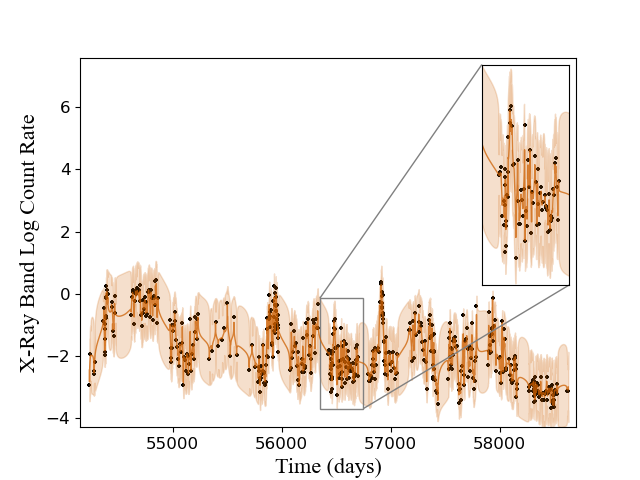}}
\subfigure[X-ray Band | Rational Quadratic | Sample]{\label{fig:4pp}\includegraphics[width=0.49\textwidth]{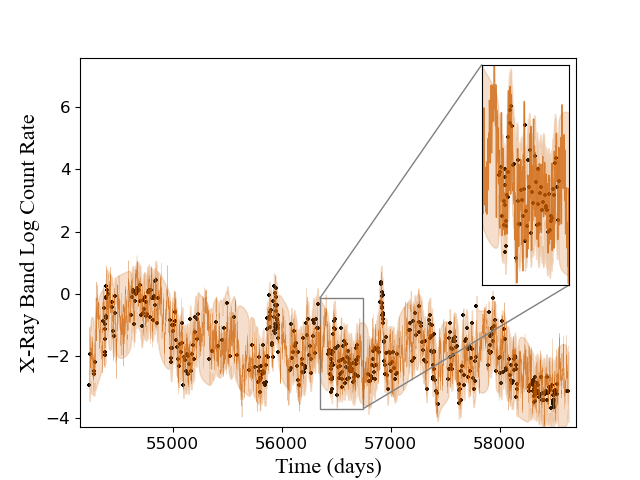}}  
\caption{\textsc{gp} lightcurves for the X-ray band. The shaded regions denote the \textsc{gp} 95\% confidence interval. Both the \textsc{gp} mean and a sample from the \textsc{gp} posterior are shown in separate plots. The insets are included to highlight the variability of the fit.}
\label{GP_xray_fits}
\end{figure*}

\section{Structure Function Analysis}
\label{structure_analysis}

Ideally one would like to examine the PSD of the \textsc{gp} fits to the observational data. The PSD characterises the distribution of power over frequencies of a given emission band and properties of the PSD can be linked to underlying physical processes in the accretion disc. Computation of the PSD, while possible, can be complicated by the uneven sampling of the observational data, leading previous studies to instead perform a structure function analysis on the \src \: data \citep{gallo18}. While it is possible to extract the PSD from the learned kernel \citep{2019_Wilkins}, in this work a structure function analysis of the \textsc{gp} lightcurves was performed in order to compare directly against the results of \cite{gallo18}. The method described in \cite{1985_Simonetti, 1992_Hughes, 1996_di_Clemente, 2001_Collier, gallo18} was followed. The binned structure function is defined as:

\begin{equation}
    \text{SF}(\tau) = \frac{1}{N(\tau)} \sum_{i} \: [f(t_i) - f(t_i + \tau)]^2,
\end{equation}

\noindent where $\tau = t_j - t_i$ is the distance between pairs of points $i$ and $j$ such that $t_j > t_i$. The structure function is binned according to $\tau$ where the centres of each bin are given by $\tau_i = (i - \frac{1}{2})\delta$. In this instance, $\delta$ is the structure function resolution. The same $\delta$ as in \cite{gallo18} was used, namely 5.3 days for the structure function computation over both the X-ray and UVW2 bands. $f(t_i)$ gives the count rate value at time point $t_i$ and $N(\tau)$ is the number of structure function pairs in each bin $i$ with centre $\tau_i$. Accounting for measurement noise by subtracting twice the mean noise variance from each structure function bin, as performed in \cite{gallo18} was found to have negligible effect on the \textsc{gp} structure functions and so was ignored. As in \cite{gallo18}, the structure function values were normalised by the global lightcurve variance. 

The \textsc{gp} structure functions for the interpolated lightcurves are shown in \autoref{GP Structure Functions}. The $1\sigma$ \textsc{gp} error bars were obtained by computing the structure function over 50 samples from the \textsc{gp} posterior. Each sample gives rise to highly similar structure functions and so the errors are not visible on the plot. The structure functions computed from the observational data, 509 and 498 data points for the X-ray and UV bands of \src \: respectively, are included for reference. In contrast to the \textsc{gp} structure function errorbars, in the case of the observational data the error bars are computed as $\frac{\sigma_i}{\sqrt(\frac{N_i}{2})}$ where $\sigma_i$ is the noise standard deviation in bin $i$ and $N_i$ is the number of pairs in bin $i$.\\ 

The \textsc{gp} structure functions are compared against the observational structure functions in \autoref{GP Structure Functions}. In addition, the broken power law fits to the \textsc{gp} structure functions are plotted, the parameters of which are given in \autoref{params}. In the UVW2 band, both \textsc{gp} kernels yield structure functions possessing a consistent break point at ca. 125 days. In the X-ray band the Matérn $\frac{1}{2}$ kernel yields a break point at 66 days whereas the rational quadratic kernel fit yields an unbroken power law. Given the discrepancy between \textsc{gp} kernels, definite evidence for a break in the X-ray power law is not found.

Of particular interest is whether the dip in the X-ray structure function is an expected feature of the latent lightcurve or a measurement artefact arising from uneven sampling. In order to assess the potential for the dip to be a sampling artefact, simulations were performed using the Timmer and König algorithm from \autoref{sim_section}. In this case structure functions of gapped lightcurves were computed and compared against structure functions derived from the ground truth lightcurves with no gaps. One representative simulation is depicted in \autoref{sf_sims}. In this instance a similar dip to that found in the observational data is observed in the X-ray band simulation. This highlights the possiblity that the dip seen in the observational X-ray structure function is a sampling artefact arising from gaps in the lightcurve.

\begin{figure*}[]
\centering
\subfigure[Matérn $\frac{1}{2}$ UVW2]{\label{fig:4pp}\includegraphics[width=0.49\textwidth]{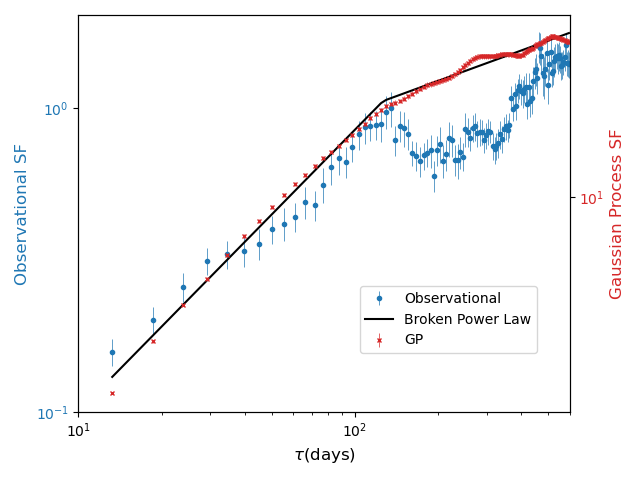}}
\subfigure[Rational Quadratic UVW2]{\label{fig:4pp}\includegraphics[width=0.49\textwidth]{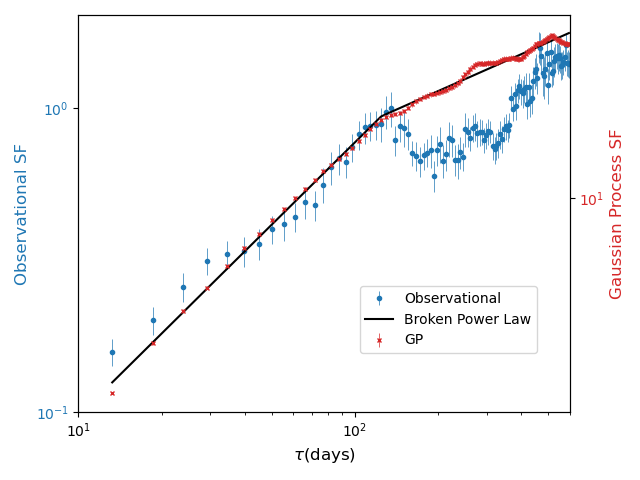}}  
\subfigure[Matérn $\frac{1}{2}$ X-ray]{\label{fig:4pp}\includegraphics[width=0.49\textwidth]{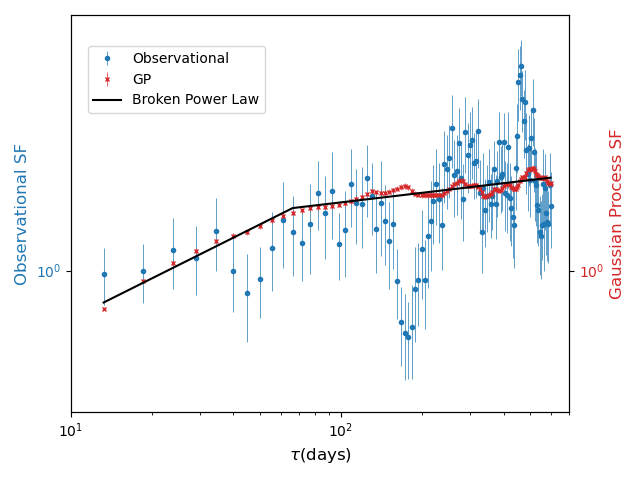}}
\subfigure[Rational Quadratic X-ray]{\label{fig:4pp}\includegraphics[width=0.49\textwidth]{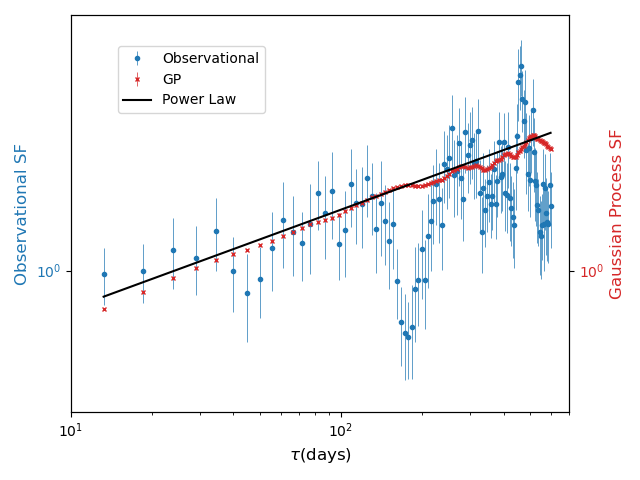}}
\caption{Comparison of observational and \textsc{gp} structure functions. The \textsc{gp} structure functions are consistent with those calculated from the observational data in the non-noise dominated regions. The dip at ca. $200$ days in the observational X-ray structure function is potentially a sampling artefact as demonstrated by simulation in \autoref{sf_sims}.}
\label{GP Structure Functions}
\end{figure*}

\begin{table}[]
\caption{Parameters for the broken power law fits to the \textsc{gp} structure functions. $\alpha_1$ and $\alpha_2$ are the indices for the power law before and after the break point $\tau_{char}$. The break point $\tau_{char}$ is reported in days. Errors were computed using 200 bootstrap samples of the data points corresponding to the \textsc{gp} structure functions. The X-ray rational quadratic structure function was fit using a power law and as such only has a single index as a parameter. The Astropy library \citep{astropy_1, astropy_2} was used to compute the (broken) power law fits using the simplex algorithm and least squares statistic for optimisation with the \textsc{gp} structure function uncertainties used as weights in the fitting.}
\begin{center}
\begin{tabular}{lllll}
\toprule
Waveband & Kernel & $\alpha_1$ & $\alpha_2$ & $\tau_{char}$ \\ \midrule
UVW2 & Matérn $\frac{1}{2}$ & $-0.72 \pm \: 0.03$ & $-0.26 \pm \: 0.01$ & $127 \pm \: 8$ \\
UVW2 & Rational Quadratic & $-0.62 \pm \: 0.01$ & $-0.28 \pm \: 0.01$ & $125 \pm \: 5$ \\
X-ray & Matérn $\frac{1}{2}$ & $-0.29 \pm \: 0.03$ & $-0.07 \pm \: 0.004$ & $66 \pm \: 8$ \\
X-ray & Rational Quadratic & $-0.21 \pm \: 0.002$ & N/A & N/A \\ \bottomrule
\end{tabular}
\end{center}
\label{params}
\end{table}

\begin{figure*}[]
\centering
\subfigure[Observational]{\label{fig:sf_sim}\includegraphics[width=0.49\textwidth]{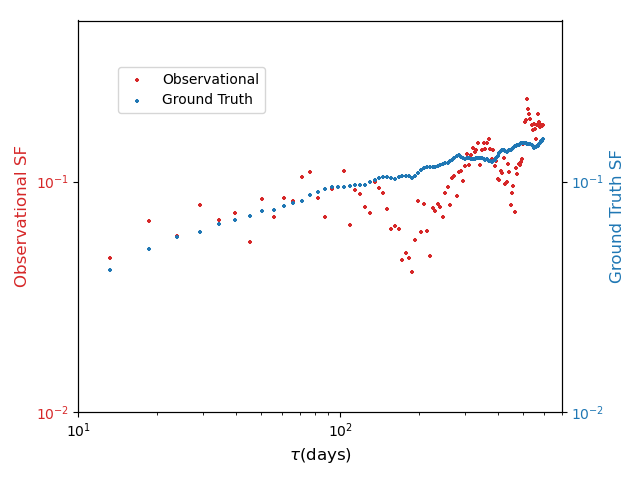}}
\subfigure[\textsc{gp} | Matérn $\frac{1}{2}$]{\label{fig:sf_sim2}\includegraphics[width=0.49\textwidth]{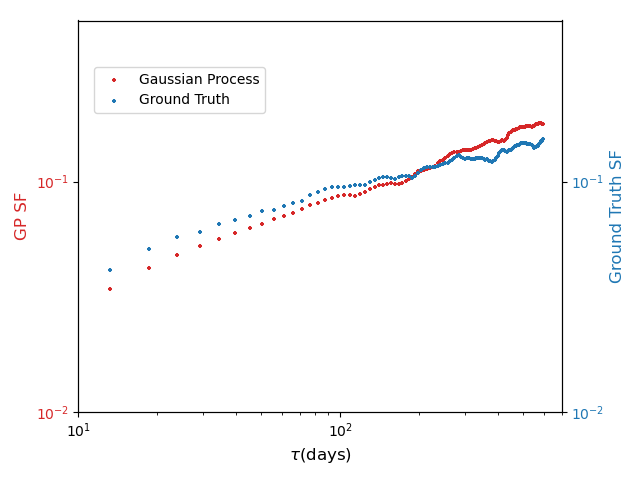}}
\caption{Structure function simulations. Pseudo-observational lightcurves are produced by introducing gaps into the simulated lightcurves. The structure function for the gapped lightcurve is shown in red in (a) whereas the structure function of the \textsc{gp} interpolation is shown in red in (b). Both structure functions are compared against the ground truth structure function obtained from the full simulated ground truth lightcurve (no gaps). The dips at $\tau = 200 \: \text{days}$ and $\tau = 400 \: \text{days}$ in the structure function derived from the gapped observational simulation in 3.7(a) are artefacts of the uneven sampling.}
\label{sf_sims}
\end{figure*}

\section{Lag and Coherence}
\label{lag_and_coherence}

\begin{figure*}
    \centering
    \includegraphics{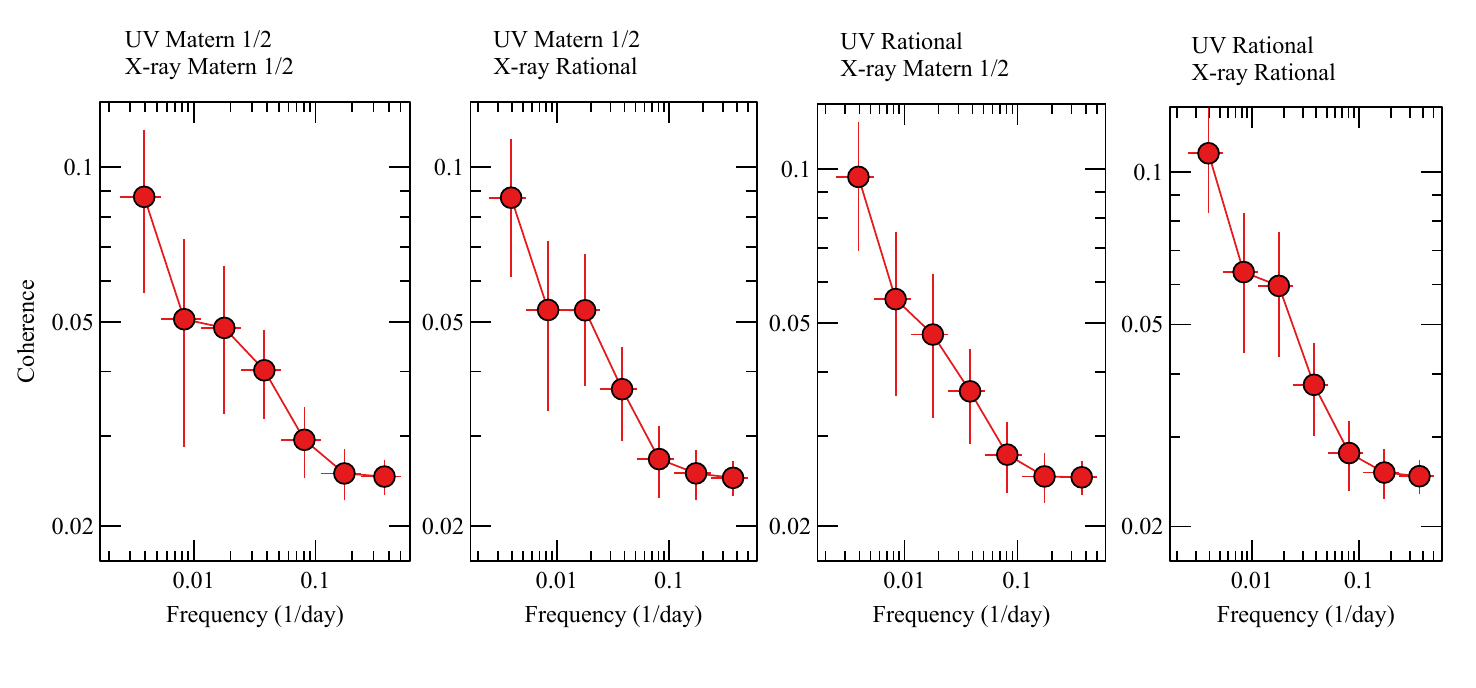}
    \includegraphics{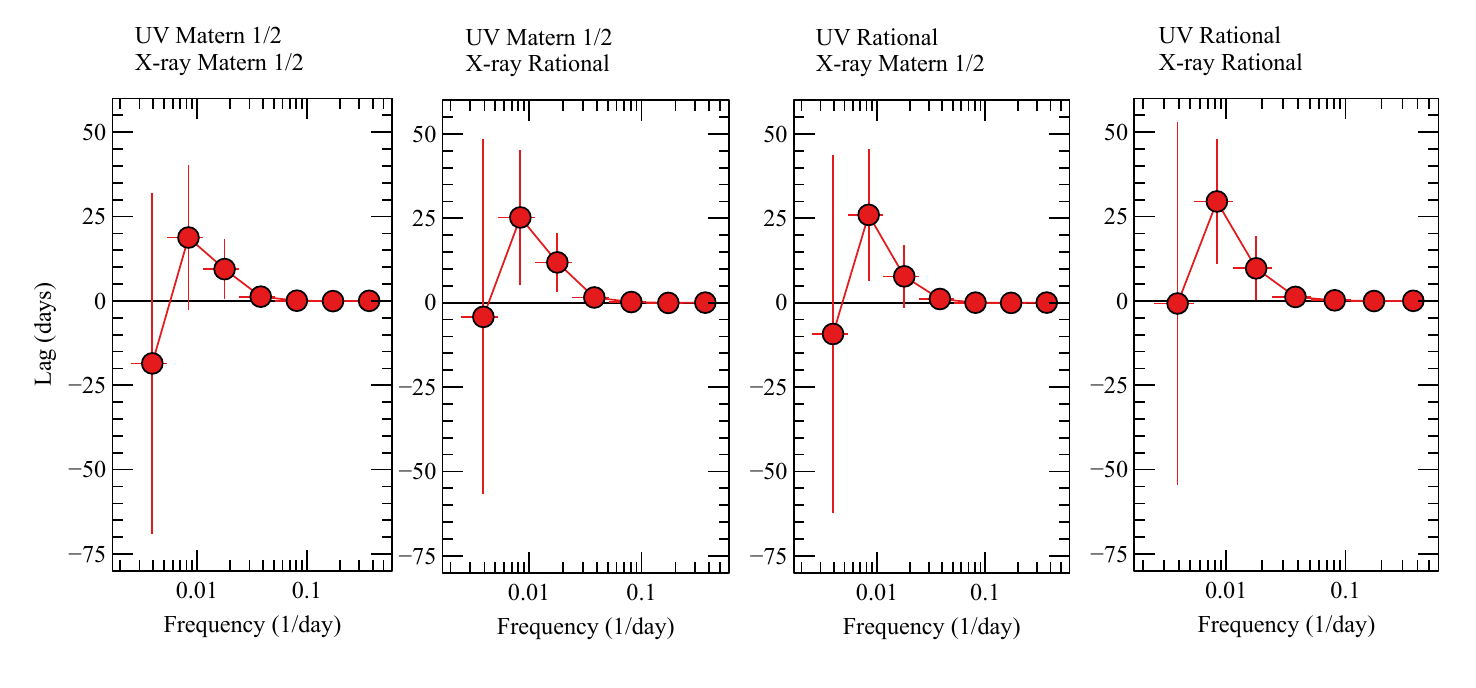}
    \caption{The coherence and lag spectra for Mrk~335, calculated by using 1000 pairs of \textsc{gp} lightcurve samples fit to the observed lightcurves. The error bars are the standard errors of the corresponding measurements for the 1000 samples. Different panels are for different kernels. Positive lags imply that the X-ray band leads the UVW2 band. Spectra plotted by Jiachen Jiang.}
    \label{pic_coh_lag}
\end{figure*}

In this section, the coherence between the UVW2 and X-ray emission from Mrk 335 is computed in search of evidence of lag features in the Fourier frequency domain. The coherence and lag spectra were estimated from one thousand pairs of UVW2 and X-ray \textsc{gp} lightcurve samples drawn from the \textsc{gp} posterior for each kernel. The lags are defined as the phase lags divided by the corresponding Fourier frequency. {A similar approach has been used in other disciplines \citep[e.g.][]{fabian09, kara13}.} Both Matérn $\frac{1}{2}$ and rational quadratic kernels are considered. The results are shown in \autoref{pic_coh_lag}. These spectra were plotted by Jiachen Jiang. Positive lags imply that the X-ray variability leads the UVW2 variability. The error bars in the figure are the standard errors of the corresponding measurements for the one thousand samples.

The coherence between the UVW2 and X-ray emission decreases with frequency, suggesting more coherent variability at lower frequency. Positive lag features are shown at the low frequencies in the range $f=$ 0.005--0.025 d$^{-1}$. The absolute value of the lag at $f=0.0039\pm0.0014$\,d$^{-1}$ is estimated to be $19 \pm 22$ days for the Matern $\frac{1}{2}$ kernel applied to both lightcurves and $29 \pm 19$ days for the rational quadratic kernel, however both measurements are consistent with zero lag in the $2\sigma$ uncertainty range.

Tentative evidence of a shorter time lag at a higher frequency of $f=0.018\pm0.006$\,d$^{-1}$ is also found. The longer lag feature at a lower frequency would correspond to a more extended emission region while the shorter lag feature at a higher frequency would correspond to a more compact region. This could be explained by the presence of an extended UV emission region on the disc where reverberation happens.

Given that the lags are consistent with zero lag within $2\sigma$ uncertainty ranges, it is concluded that only tentative evidence for a broad lag feature is found by applying \textsc{gp}s to the UVW2 and X-ray lightcurves of \src. Previous attempts to identify lags between two wavelengths of \src\ based on cross-correlation analysis in the time domain suggests similar results \citep[e.g.][]{gallo18}.

\section{Conclusions}
\label{conc}

Following the interpolation of the unevenly-sampled lightcurves of \src\ using \textsc{gp}s, tentative evidence for broad lag features is found in the Fourier frequency domain. The magnitude of the lags is consistent with previous cross-correlation analyses. In addition, the broad lag features {might} suggest an extended emission region e.g. of the disc in \src \: during the reverberation processes. {If the corona is compact within 5 $R_{\rm g}$ in Mrk~335 \citep{2015_Wilkins}, our data suggest a possibly wide range of UVW2 emission radii.}

The structure functions computed from the \textsc{gp}-interpolated lightcurves are consistent with those derived from the observational data and furthermore, illicit potential insights into the properties of the latent lightcurves. In particular, it is shown through a simulation study that it is possible that dips in the X-ray structure function may be produced by sampling artefacts arising from gaps in the lightcurve. In contrast, the \textsc{gp} structure functions show no dips. While this is not proof that the dip in the observational X-ray structure function is due to a sampling artefact, it does allude to the possibility. The UVW2 \textsc{gp} structure functions do not exhibit strong dependence on the choice of kernel with both Matérn $\frac{1}{2}$ and rational quadratic showing up a broken power law with breaks at 139 and 155 days respectively. The X-ray structure functions however do show up differences between kernels with the rational quadratic kernel predicting a power law and the Matérn 1/2 kernel predicting a broken power law.

From the \textsc{gp} modelling perspective, the ability to carry out Bayesian model selection affords a quantitative means of comparing analytic kernels under the marginal likelihood. It may be possible to incorporate further flexibility into the fitting procedure by making use of more sophisticated methods of kernel design \citep{2014_Duvenaud} to allow the assessment of fits of sums and products of analytic kernels or by leveraging advances in transforming \textsc{gp} priors via Deep \textsc{gp}s \citep{2013_Damianou} or normalising flows \citep{2020_Maronas}. Such approaches could be validated using simulation studies. Additionally, modelling the cross-correlation using multioutput \textsc{gp}s \citep{2020_de_Wolff} may be an interesting avenue for comparison against the approach taken here. Lastly, Bayesian spectral density estimation \citep{2018_Tobar} may afford further flexibility through nonparametric modelling of the PSD in addition to nonparametric modelling of the lightcurve in the time domain. These improvements in Bayesian modelling machinery may help to minimise model misspecification and as such, enable more robust inferences to be made about the functional forms of the latent lightcurves.

\nomenclature[Z-SMBH]{SMBH}{Supermassive Black Hole}
\nomenclature[Z-AGN]{AGN}{Active Galactic Nucleus}
\nomenclature[Z-Sy1]{Sy1}{Seyfert 1}
\nomenclature[Z-Sy2]{Sy2}{Seyfert 2}
\nomenclature[Z-GRB]{GRB}{Gamma-Ray Burst}
\nomenclature[Z-BAT]{BAT}{Burst Alert Telescope}
\nomenclature[Z-XRT]{XRT}{X-Ray Telescope}
\nomenclature[Z-UVOT]{UVOT}{UV/Optical Telescope}
\nomenclature[Z-ADAF]{ADAF}{Advection-Dominated Accretion Flows}
\nomenclature[Z-NLS1]{NLS1}{Narrow-Line Seyfert 1}
\nomenclature[Z-ECDF]{ECDF}{Empirical Cumulative Distribution Function}
\nomenclature[Z-PSD]{PSD}{Power Spectral Density}
\nomenclature[Z-RSS]{RSS}{Residual Sum of Squares} %
\chapter{Modelling Molecules with Gaussian Processes}
\chapterimage[height=130pt]{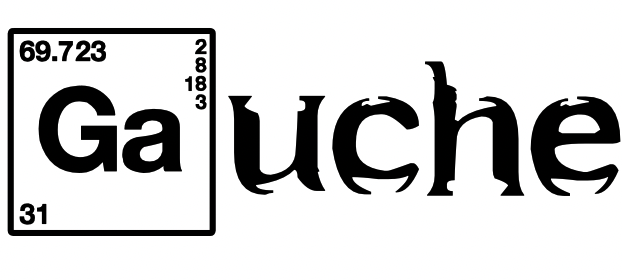}

\ifpdf
    \graphicspath{{Chapter4/Figs/Raster/}{Chapter4/Figs/PDF/}{Chapter4/Figs/}}
\else
    \graphicspath{{Chapter4/Figs/Vector/}{Chapter4/Figs/}}
\fi

\textbf{Status:} Accepted as Griffiths, RR., Klarner, L., Moss, HB., Ravuri, A., Truong, S., Rankovic, B., Du, Y., Jamasb, A., Schwartz, J., Tripp, A., Kell, G. Bourached, A., Chan, A., Moss, J., Chengzhi, G, Lee AA, Schwaller P., Tang, J.  GAUCHE: A library for Gaussian processes in chemistry. \textit{ICML Workshop on AI4Science}, 2022.

\section{Background on Molecular Machine Learning}
\label{background_gauche}

The chapter begins with a self-contained background on molecular machine learning required to contextualise the findings of this chapter.

\begin{figure}[h]
    \begin{center}
        \includegraphics[width=0.75\textwidth]{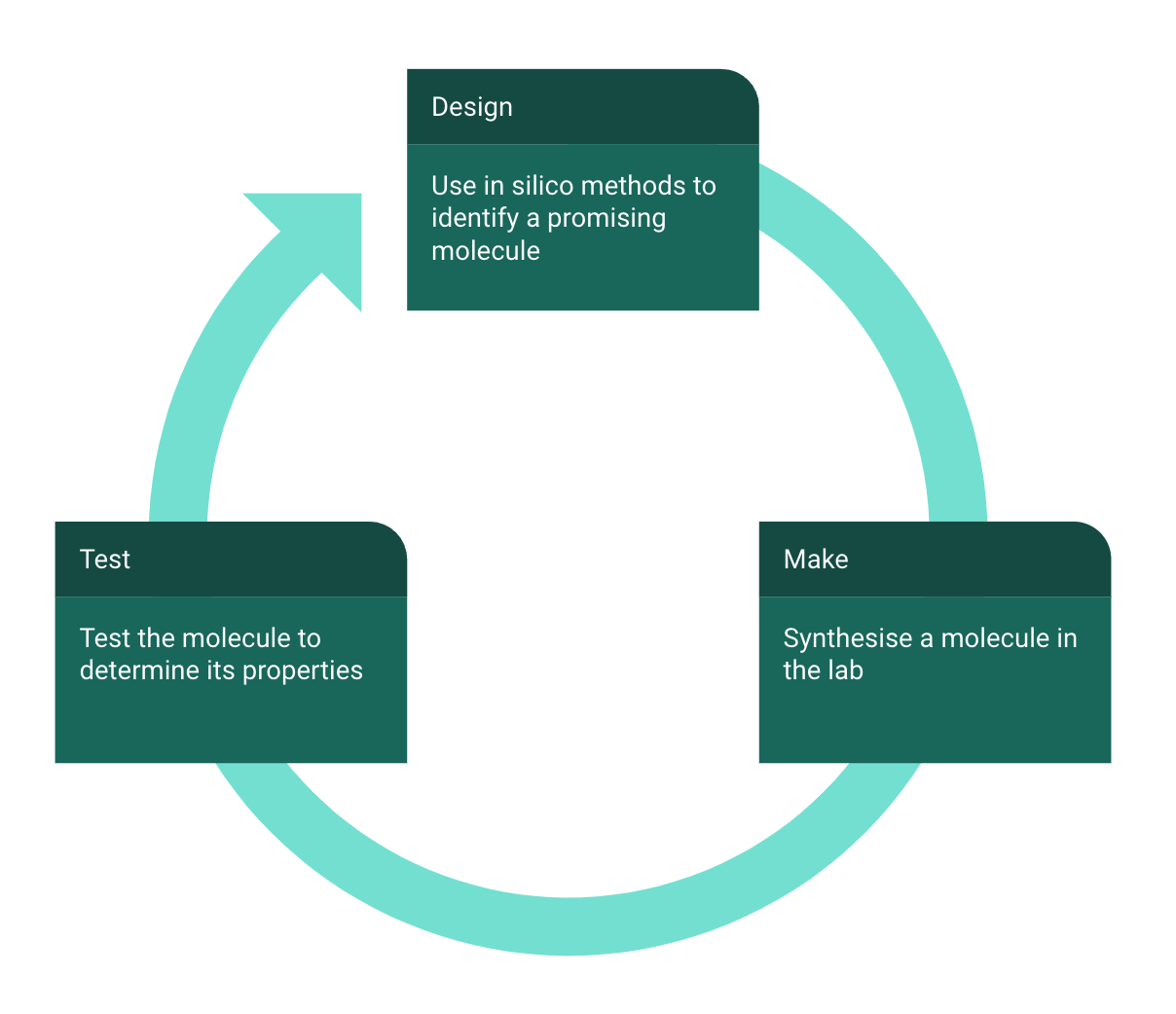}
    \end{center}
    \caption{The design-make-test cycle of molecular discovery.}
    \label{dmt}
\end{figure}

Although application domains of machine learning in the molecular sciences are constantly expanding \citep{2020_Struble}, an important subset of applications can be taxonimised according to the role they play in enhancing the design-make-test cycle \citep{2012_Plowright} of molecular discovery campaigns, illustrated in \autoref{dmt}. Molecule generation \citep{2018_Design, 2020_Griffiths, 2017_Grammar, 2020_Jin, 2022_Du, 2022_Gao} is concerned with designing novel molecules using generative models such as variational autoencoders (VAEs) \citep{2014_Kingma} and generative adversarial networks (GANs) \citep{2014_Goodfellow}. Chemical reaction prediction \citep{2019_Schwaller}, reaction planning \citep{2018_Coley_new}, and synthesis design \citep{2022_Schwaller}, illustrated in \autoref{reaction}, are focussed on improving the throughput of the \say{make} stage of the design-make-test cycle by using machine learning to suggest synthetic pathways to target molecules.

The molecular machine learning tasks considered in this thesis, are molecular property prediction and chemical reaction optimisation. Molecular property prediction is concerned with the \say{design} phase of the design-make-test cycle in so far as it allows molecules to be prioritised for laboratory synthesis. Chemical reaction optimisation, on the other hand, is concerned with the \say{make} phase as it is a means of improving the yield of chemical reactions. Both applications will be described in detail next.

\begin{figure}[h]
    \begin{center}
        \includegraphics[width=0.75\textwidth]{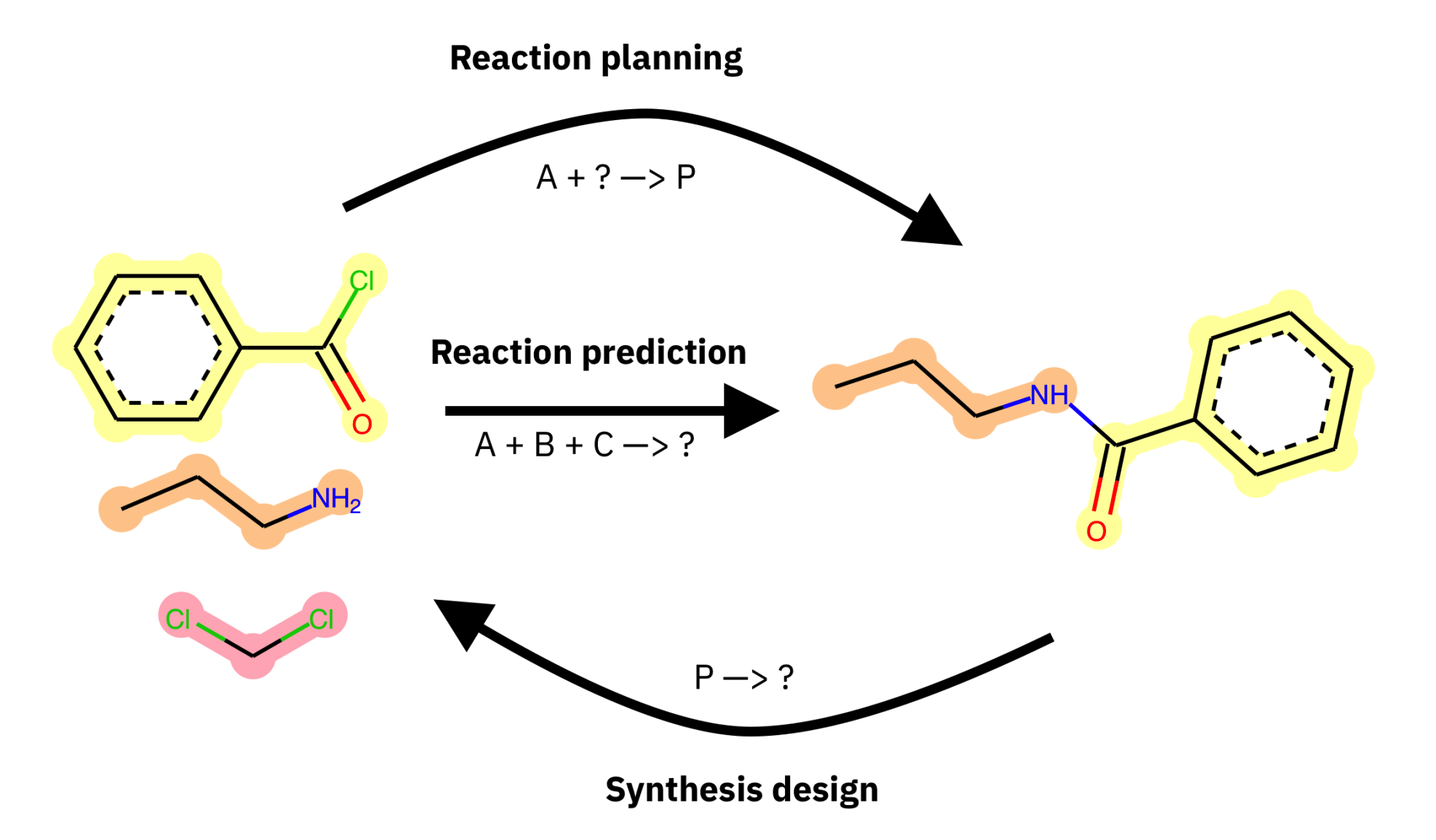}
    \end{center}
    \caption{A, B and C are starting materials and P is the product of the chemical reaction. Reaction planning involves finding a set of reagents to transform a starting material into a product. Reaction prediction involves predicting the product given a set of reactants. Synthesis design involves working backwards from the product towards a set of reactants and reagents. Machine learning-based solutions to all of these tasks would yield a blueprint for a chemist to follow in synthesising a novel molecule.}
    \label{reaction}
\end{figure}

\subsection{Molecular Property Prediction}

The laboratory synthesis of a novel molecule is a highly time-intensive process. As such, there has been a great deal of interest in computational approaches to prioritising molecules drawn from vast molecular databases in a process known as high-throughput virtual screening \citep{2015_Knapp}. In this fashion, theoretical techniques can be used to winnow the database down to a few promising candidates for experimental chemists to follow up on. Before the advent of machine learning, the dominant approach was to use first principles quantum chemical calculations to compute molecular properties. Below, one such first principles method is reviewed, density functional theory (DFT), which is compared against machine learning model performance in Chapter 5, before discussing machine learning approaches to molecular property prediction. 

\subsubsection{Density Functional Theory}

DFT is a method of modelling the electronic structure of many-body systems \citep{brazdova_atomistic_2013}, and has been applied across problems in physics, chemistry, biology, and materials science \citep{becke_perspective_2014}. DFT is an \emph{ab initio}, or first principles computational method because physical constants are the only inputs to calculations based on the postulates of quantum mechanics \citep{leach_molecular_2001}. Since the inception of DFT in 1964-1965, Kohn-Sham DFT (KS-DFT) has been one of the most frequently applied electronic structure methods \citep{becke_perspective_2014}. 

KS-DFT makes use of the Hohenberg-Kohn theorems \citep{hohenberg_inhomogeneous_1964}, a trial electron density, and a self-consistency scheme. KS-DFT executes a computational loop by starting with a trial density, solving the Kohn-Sham equations, and obtaining the single electron wavefunctions for the trial density; in the next step, an electron density may be computed. If the computed density is consistent i.e. within a tolerance threshold of the trial density, the theoretical ground state density has been identified. If the densities are not consistent, however, the computed density is taken as the new trial density, and the iterative loop continues to be executed until the tolerance threshold is met. The accuracy of DFT calculations, with exchange and correlation functionals, can be very high, yet may exhibit significant fluctuations with the choice of functional, pseudopotential, basis sets and cutoff energy \citep{howard_assessing_2015}. Furthermore, these quantities are not always trivial to optimise.

\subsubsection{Time-Dependent Density Functional Theory}

Time-Dependent Density Functional Theory (TD-DFT) is a time-dependent analogue of DFT based on the Runge-Gross (RG) theorem in place of the Hohenberg-Kohn theorems \citep{runge_density-functional_1984}. The RG theorem states that a unique delineation exists between the time-dependent electron density and the time-dependent external potential. As such, a computational, time-dependent Kohn-Sham system may be implemented \citep{van_leeuwen_causality_1998} in a similar fashion to KS-DFT. When TD-DFT has been used together with a linear response theory \citep{ullrich_time-dependent_2012}, it has enjoyed success in the calculation of electromagnetic spectra of medium and large molecules \citep{casida_progress_2012, burke_time-dependent_2005}. A relevant application of TD-DFT in this thesis is the computation of the $\pi-\pi{^*}/\emph{n}-\pi{^*}$ electronic transitions wavelengths for photoswitch molecules in Chapter 5.

\subsubsection{Machine Learning Approaches}

In contrast to first principles methods such as DFT, machine learning approaches seek to carry out data-driven prediction. A key issue in data-driven molecular property prediction is how best to featurise molecules. This problem is commonly referred to as choosing a molecular representation. While a great many base representations of molecules exist \citep{2022_Wigh}, some of the most popular featurisations include graph-based, string-based and fingerprint representations. The field of molecular representation learning is concerned with learning representations on top of these base representations e.g. \citep{2015_Duvenaud} typically via deep learning. Below, a brief review is provided of the molecular representations used in this thesis.

\paragraph{Graphs:} Molecules may be represented as an undirected, labeled graph $\mathcal{G}=(\mathcal{V}, \mathcal{E})$ where vertices, $\mathcal{V}=\{v_1, \ldots, v_N\}$, represent the atoms of an $N$-atom molecule and edges, $\mathcal{E}\subset\mathcal{V}\times\mathcal{V}$, represent covalent bonds between these atoms. Additional information may be incorporated in the form of vertex and edge labels $\mathcal{L}:\mathcal{V}\times\mathcal{E}\to\Sigma_V\times\Sigma_E$. Common label spaces including attributes such as atom types (i.e. hydrogen, carbon) as vertex labels and bond orders (i.e. single, double) as edge labels.

\paragraph{Strings:} The Simplified Molecular-Input Line-Entry System (SMILES) is a text-based representation of molecules \citep{1987_Anderson, 1988_Weininger}, examples of which are given in \autoref{smiles_figure}. Self-Referencing Embedded Strings (SELFIES) \citep{2020_Krenn} is an alternative string representation to SMILES such that a bijective mapping exists between a SELFIES string and a molecule.

\begin{figure}[h]
    \centering
    \includegraphics[width=0.39\textwidth]{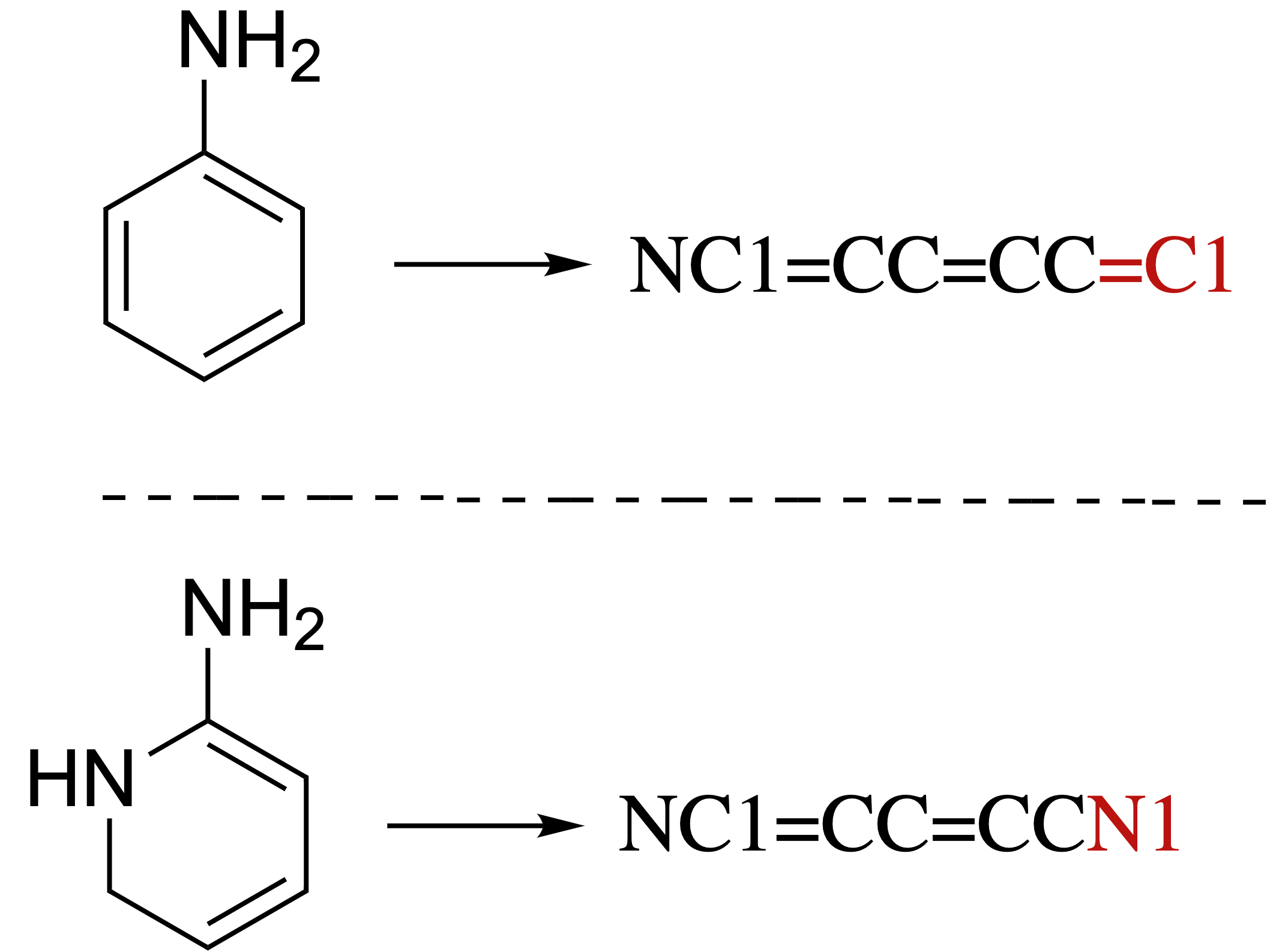}
    \caption{SMILES strings for structurally similar molecules. Similarity is encoded in the string through common contiguous subsequences (black). Local differences are highlighted in red. Note the molecules are chosen solely for the purposes of illustrating the SMILES syntax.}
    \label{smiles_figure}
\end{figure}

\paragraph{Fingerprints:} Molecular fingerprints were first introduced for chemical database substructure searching \citep{1993_Christie}, but were later repurposed for similarity searching \citep{1990_Johnson}, clustering \citep{1997_McGregor} and classification \citep{2017_Breiman}. Extended Connectivity FingerPrints (ECFP) \citep{2010_Rogers} were introduced as part of the Pipeline project \citep{2006_Hassan} with the explicit goal of capturing features relevant for molecular property prediction \citep{2004_Xia}. ECFP fingerprints operate by assigning initial numeric identifiers to each atom in a molecule. These identifiers are subsequently updated in an iterative fashion based on the identifiers of their neighbours. The number of iterations corresponds to half the \textit{diameter} of the fingerprint and the naming convention reflects this. For example, ECFP6 fingerprints have a diameter of 6, meaning that 3 iterations of atom identifier reassignment are performed. Each level of iteration appends substructural features of increasing non-locality to an array and the array is then hashed to a bit vector reflecting the presence of absence of those substructures in the molecule.

For property prediction applications a radius of 3 or 4 is recommended. A radius of 3 is used for all experiments in the thesis. Fragment descriptors are also used, which are count vectors, each component of which indicates the number of a certain functional group present in a molecule. For example row 1 of the count vector could be an integer representing the number of aliphatic hydroxl groups present in the molecule. Both fingerprint and fragment features computed using RDKit are made use of \citep{rdkit}, as well as the concatenation of the fingerprint and fragment feature vectors, a representation termed fragprints which has shown strong empirical performance. Example representations $\mathbf{x_f}$, for fingerprints and $\mathbf{x_{fr}}$, for fragments might be

\begin{align*}
 \mathbf{x_{f}} &= \begin{bmatrix}
       1 \\
       0 \\
       \vdots \\
       1
     \end{bmatrix}, \:\:\:
  \mathbf{x_{fr}} = \begin{bmatrix}
       3 \\
       0 \\
       \vdots \\
       2
     \end{bmatrix}.
\end{align*}

\subsubsection{Bayesian Optimisation and Active Learning for Molecules}

\textsc{bo} and active learning hold particular promise for accelerating high-throughput virtual screening efforts for molecules \citep{2017_Knapp, 2020_Pyzer}. The idea in this instance is that data-efficient molecular property optimisation can be performed in a \textsc{bo}/active learning loop featuring laboratory synthesis of the suggested molecules.

\subsection{Chemical Reaction Optimisation}

The \say{yield} of a chemical reaction is an important consideration for large-scale production of a desired molecule. The percentage yield, of a chemical reaction may be defined as

\begin{equation}
     \text{Percentage Yield} = \frac{\text{Actual Yield}}{\text{Theoretical Yield}}.
\end{equation}

\noindent The Theoretical Yield assumes a flawless chemical reaction in which all starting materials are converted to the desired product. In practice, chemical reactions are not perfectly efficient due to factors such as reverse reactions, in which the reactants and products exist in a state of chemical equilibrium, as well as competing chemical reactions that form unwanted side products. Factors such as the temperature of the reaction, the concentration of reactant species, the choice of solvent as well as the presence of reagents, molecular species that enhance the reaction without contributing atoms to the product, are all determinants of the efficiency of the reaction and hence the Actual Yield. Optimising such reaction parameters has recently been tackled by machine learning approaches such as \textsc{bo} \citep{2021_Shields}. Common reaction representations are detailed next. 

\subsubsection{Chemical Reaction Representations}
A chemical reaction comprises reactants and reagents that transform into one or more products together with reaction parameters such as temperature and concentration. The reactants and reagents are instances of molecular species which play different roles in the reaction. By means of illustration, the high-throughput experiments by \cite{ahneman2018predicting} on Buchwald-Hartwig reactions feature a reaction design space consisting of 15 aryl and heteroaryl halides, 4 Buchwald ligands, 3 bases, and 23 isoxazole additives. Below, various means of featurising the reactant and reagent components of a chemical reaction are introduced.

\paragraph{Concatenated molecular representations:} If the number of reactant and reagent categories is constant, the molecular representations discussed above can be used to encode reactants and reagents. The vectors for the individual reaction components may then be concatenated to build the reaction representation \citep{ahneman2018predicting, sandfort2020structure}. An additional concatenated representation, is the one-hot-encoding (OHE) of the reaction categories where bits indicate the presence or absence of a particular reactant/reagent. In the Buchwald-Hartwig example above, the OHE would describe which of the aryl halides, Buchwald ligands, bases, and additives are used in the reaction, resulting in a 44-dimensional bit vector \citep{chuang2018comment}. 

\paragraph{Differential reaction fingerprints:} Inspired by the hand-engineered difference reaction fingerprints by \citet{schneider2015development}, \citet{probst2022reaction} recently introduced the differential reaction fingerprint (DRFP). This reaction fingerprint is constructed by taking the symmetric difference of the sets containing the molecular substructures on both sides of the reaction arrow. The size of the reaction bit vector generated by DRFP is independent of the number of reaction components. 

\paragraph{Data-driven reaction fingerprints:} \citet{schwaller2021mapping} described data-driven reaction fingerprints using Transformer models such as BERT \citep{devlin2018bert}, trained in a supervised or an unsupervised fashion on reaction SMILES. The Transformer models can then be fine-tuned on the task of interest to learn more specific reaction representations \citep{schwaller2021prediction}. These representations are designated using the acronym RXNFP. As with the DRFP, the size of data-driven reaction fingerprints is also independent of the number of reaction components.

In the next section, the focus of this chapter is introduced.

\section{Preface}

This chapter introduces GAUCHE, a library for GAUssian processes in CHEmistry. \textsc{gp}s have long been a cornerstone of probabilistic machine learning, affording particular advantages for uncertainty quantification (UQ) and \textsc{bo}. Extending \textsc{gp}s to chemical representations however is nontrivial, necessitating kernels defined over structured inputs such as graphs, strings and bit vectors. By defining such kernels in GAUCHE, the door is opened to powerful tools for UQ and \textsc{bo} in chemistry. Motivated by scenarios frequently encountered in experimental chemistry, applications for GAUCHE are showcased in molecule discovery and chemical reaction optimisation.

\section{Introduction}
\label{intro}

Early-stage scientific discovery is typically characterised by the small data regime due to the limited availability of high-quality experimental data \citep{2018_Zhang}. Much of the novelty of discovery relies on the fact that there is a lot of knowledge to gain in the small data regime. By contrast, in the big data regime, discovery offers diminishing returns as much of the knowledge about the space of interest has already been acquired. As such, machine learning methodologies that facilitate search in small data regimes such as \textsc{bo} \citep{2018_Bombarelli, 2020_Griffiths, 2021_Shields} and active learning (AL) \citep{2019_Zhang, 2021_Jablonka} have great potential to expedite the rate at which performant molecules, molecular materials, chemical reactions and proteins are discovered.

To date in molecular machine learning, BNNs have been the surrogate of choice to produce the uncertainty estimates that underpin \textsc{bo} and AL \citep{2019_Ryu, 2019_Zhang, 2020_Hwang, 2020_Scalia}. For small datasets, however, DNNs are often not the model of choice. Notably, certain deep learning experts have voiced a preference for \textsc{gp}s in the small data regime \cite{2011_Bengio}. Furthermore, for \textsc{bo}, \textsc{gp}s possess particularly advantageous properties; first, they admit exact as opposed to approximate Bayesian inference and second, few of their parameters need to be determined by hand. In the words of Sir David MacKay \cite{2003_MacKay}, 

\begin{displayquote}
"Gaussian processes are useful tools for automated tasks where fine tuning for each problem is not possible. We do not appear to sacrifice any performance for this simplicity.''
\end{displayquote}

The iterative model refitting required in \textsc{bo} makes it a prime example of such an automated task. Although BNN surrogates have been trialled for \textsc{bo} \citep{2015_Snoek, 2016_Springenberg}, \textsc{gp}s remain the model of choice as evidenced by the results of the recent NeurIPS Black-Box Optimisation Competition \cite{2021_Turner}.

Training \textsc{gp}s on molecular inputs is non-trivial however. Canonical applications of \textsc{gp}s assume continuous input spaces of low and fixed dimensionality. The most popular molecular input representations are SMILES/SELFIES strings \citep{1987_Anderson, 1988_Weininger, 2020_Krenn}, fingerprints \citep{2010_Rogers, probst2018probabilistic, capecchi2020one} and graphs \citep{2015_Duvenaud,2016_Kearnes}. Each of these input representations poses problems for \textsc{gp}s. SMILES strings have variable length, fingerprints are high-dimensional and sparse bit vectors, while graphs are also a form of non-continuous input. To construct a \textsc{gp} framework over molecules, GAUCHE provides GPU-based implementations of kernels that operate on molecular inputs, including string, fingerprint and graph kernels. Furthermore, GAUCHE includes support for protein and chemical reaction representations and interfaces with the GPyTorch \citep{2018_Gardner} and BoTorch \citep{2019_Balandat} libraries to facilitate usage for advanced probabilistic modelling and \textsc{bo}. The detailed contributions of this chapter may be summarised as:

\newpage

\begin{enumerate}
\setlength\itemsep{0.3em}
    \item The introduction of a \textsc{gp} framework for molecules and chemical reactions.
    \item The provision of an open-source, GPU-enabled library building on GPyTorch \citep{2018_Gardner}, BoTorch \citep{2019_Balandat}, and RDKit \cite{rdkit}.
    \item The use of black box graph kernels, from GraKel, \citep{2020_GraKel}, is extended to \textsc{gp} regression via a GPyTorch interface, along with a limited set of graph kernels implemented in native GPyTorch to enable optimisation of the graph kernel hyperparameters under the marginal likelihood.
    \item Benchmark experiments are conducted, evaluating the utility of the \textsc{gp} framework on regression, UQ and \textsc{bo} tasks.
\end{enumerate}

GAUCHE is made available at \href{https://github.com/leojklarner/gauche}{https://github.com/leojklarner/gauche} and includes tutorials to guide users through the tasks considered in this paper. 

\section{Related Work}
General-purpose \textsc{gp} and \textsc{bo} libraries do not specifically cater for molecular representations. Likewise, general-purpose molecular machine learning libraries do not specifically consider \textsc{gp}s and \textsc{bo}. Here, existing libraries are reviewed, highlighting the niche GAUCHE fills in bridging the \textsc{gp} and molecular machine learning communities. To date, there has been little work on Gaussian processes applied to discrete molecular representations, some notable exceptions being \citep{2020_Gardiner_gp, 2022_Gosnell, 2023_Jablonka}.

GAUCHE is a PyTorch extension of FlowMO \citep{2020_flowmo}, which introduces a molecular \textsc{gp} library in the GPflow framework. It is upon FlowMO which GAUCHE builds, extending the scope of the library to a broader class of molecular representations (graphs), problem settings (\textsc{bo}), and applications (reaction optimisation).

\paragraph{Gaussian Process Libraries:} \textsc{gp} libraries include GPy (Python) \citep{2014_gpy}, GPflow (TensorFlow) \citep{2017_Matthews, 2020_Wilk}, and GPyTorch (PyTorch) \citep{2018_Gardner}, while examples of recent Bayesian optimisation libraries include BoTorch (PyTorch) \citep{2019_Balandat}, Dragonfly (Python) \citep{2019_dragonfly}, and HEBO (PyTorch) \citep{2020_Rivers}. The aforementioned libraries do not explicitly support molecular representations. Extension to cover molecular representations, however, is nontrivial, requiring implementations of bespoke \textsc{gp} kernels for bit vector, string and graph inputs together with modifications to Bayesian optimisation schemes to consider acquisition function evaluations over a discrete set of heldout molecules, a setting commonly encountered in virtual screening campaigns \citep{2020_Pyzer, 2022_Graff}.

\paragraph{Molecular Machine Learning Libraries:} Molecular machine learning libraries include DeepChem \citep{2019_Ramsundar}, DGL-LifeSci \citep{2020_Li} and TorchDrug \citep{2022_Zhu}. DeepChem features a broad range of model implementations and tasks, while DGL-LifeSci focuses on graph neural networks. TorchDrug caters for applications including property prediction, representation learning, retrosynthesis, biomedical knowledge graph reasoning and molecule generation.

\textsc{gp} implementations are not included, however, in the aforementioned libraries. In terms of atomistic systems, DScribe \citep{2020_Himanen} features, amongst other methods, the Smooth Overlap of Atomic Positions (SOAP) representation \citep{2013_Bartok} which is typically used in conjunction with a \textsc{gp} model to learn atomistic properties. Automatic Selection And Prediction (ASAP) \citep{2020_Cheng} also principally focusses on atomistic properties as well as dimensionality reduction and visualisation techniques for materials and molecules. Lastly, the Graphein library focusses on graph representations of proteins \citep{2020_Jamasb}.

\paragraph{Graph Kernel Libraries:} Graph kernel libraries include GraKel \citep{2020_GraKel}, graphkit-learn \citep{2021_Linlin}, graphkernels \citep{2018_Sugiyama}, graph-kernels \citep{2015_Sugiyama}, pykernels (\href{https://github.com/gmum/pykernels}{https://github.com/gmum/pykernels}) and ChemoKernel \citep{2012_Gauzere}. The aforementioned libraries focus on CPU implementations in Python. Extending graph kernel computation to GPUs has been noted as an important direction for future research \citep{2018_Ghosh}. In our work, the GraKel library is built upon by interfacing it with GPyTorch, facilitating \textsc{gp} regression with GPU computation. Furthermore, this enables the graph kernel hyperparameters to be learned through the marginal likelihood objective as opposed to being pre-specified and fixed upfront. 

\paragraph{Molecular Bayesian Optimisation:} \textsc{bo} over molecular space can be split into two classes of methods. In the first class, molecules are encoded into the latent space of a VAE \citep{2018_Bombarelli}. \textsc{bo} is then performed over the continuous latent space and queried molecules are decoded back to the original space. Much work on VAE-BO has focussed on improving the synergy between the surrogate model and the VAE \citep{2018_Griffiths, 2020_Griffiths, 2020_Tripp, 2021_Deshwal, 2021_Grosnit, 2021_Verma, 2022_Maus, 2022_Stanton}. One of the defining characteristics of VAE-BO is that it enables the generation of new molecular structures.\\

In the second class of methods, \textsc{bo} is performed directly over the original discrete space of molecules \citep{2022_Tom} In this setting it is not possible to generate new structures and so a candidate set of queryable molecules is defined. The inability to generate new structures however, is not necessarily a bottleneck to molecule discovery in many cases, as the principle concern is how best to explore existing candidate sets. These candidate sets are also known as molecular libraries in the virtual screening literature \citep{2015_Pyzer}. 

To date, there has been little work on \textsc{bo} directly over discrete molecular spaces. In \citet{2020_Moss}, the authors use a string kernel \textsc{gp} trained on SMILES to perform \textsc{bo} to select from a candidate set of molecules. In \citet{2020_Korovina}, an optimal transport kernel \textsc{gp} is used for \textsc{bo} over molecular graphs. In \citet{2021_Hase} a surrogate based on the Nadarya-Watson estimator is defined such that the kernel density estimates are inferred using BNNs. The model is then trained on molecular descriptors. Lastly, in \citet{2017_Knapp} and \citet{2021_Moss} a BNN and a sparse \textsc{gp} respectively are trained on fingerprint representations of molecules. In the case of the sparse \textsc{gp} the authors select an ArcCosine kernel. It is a long term aim of the GAUCHE Project to compare the efficacy of VAE-BO against vanilla \textsc{bo} on real-world molecule discovery tasks.

\paragraph{Chemical Reaction Optimisation:} Chemical reactions describe how reactant molecules transform into product molecules. Reagents (catalysts, solvents, and additives) and reaction conditions heavily impact the outcome of chemical reactions. Typically the objective is to maximise the reaction yield (the amount of product compared to the theoretical maximum) \citep{ahneman2018predicting}, to maximise the enantiomeric excess in asymmetric synthesis, where the reactions could result in different enantiomers \citep{zahrt2019prediction}, or to minimise the E-factor, which is the ratio between waste materials and the desired product \citep{schweidtmann2018machine}. 

A range of studies have evaluated the optimisation of chemical reactions in single and  multi-objective settings \citep{schweidtmann2018machine,muller2022automated}. \citet{felton2021summit} and \citet{hase2021olympus} benchmarked reaction optimisation algorithms in low-dimensional settings including reaction conditions, such as time, temperature, and concentrations. \citet{2021_Shields} suggested \textsc{bo} as a general tool for chemical reaction optimisation and benchmarked their approach against human experts. \citet{haywood2021kernel} compared the yield prediction performance of different kernels and \citet{pomberger2022}, the impact of various molecular representations. 

In all reaction optimisation studies above, the representations of the different categories of reactants and reagents are concatenated to generate the reaction input vector, which could lead to limitations if another type of reagent is suddenly considered. Moreover, most studies concluded that simple one-hot encodings (OHE) perform at least on par with more elaborate molecular representations in the low-data regime \citep{2021_Shields, pomberger2022, hickman2022}. In GAUCHE, reaction fingerprint kernels are introduced, based on existing reaction fingerprints \citep{schwaller2021mapping, probst2022reaction} and work independently of the number of reactant and reagent categories. The molecular kernels, constituting the backbone of the GAUCHE library, are described next.

\section{Molecular Kernels}
\label{mol_gauss}

Here, examples are given of the classes of GAUCHE kernel designed to operate on the molecular representations introduced in \autoref{background_gauche}.

\subsection{Fingerprint Kernels}

\paragraph{Scalar Product Kernel:} The simplest kernel to operate on fingerprints is the scalar product or linear kernel defined for vectors $\mathbf{x}, \mathbf{x'} \in \mathbb{R}^d$ as

\begin{equation}
    k_{\text{Scalar Product}}(\mathbf{x}, \mathbf{x'}) \coloneqq \sigma_{f}^2 \cdot \langle\mathbf{x}, \mathbf{x'}\rangle,
\end{equation}

\noindent where $\sigma_{f}$ is a scalar signal variance hyperparameter and $\langle\cdot, \cdot\rangle$ is the Euclidean inner product.

\paragraph{Tanimoto Kernel:} Introduced as a general similarity metric for binary attributes \citep{gower1971general}, the Tanimoto kernel was first used in chemoinformatics in conjunction with non-\textsc{gp}-based kernel methods \citep{2005_Ralaivola}. It is defined for binary vectors $\mathbf{x}, \mathbf{x'} \in \{0, 1\}^d$ for $d \geq 1$ as

\begin{equation}
\label{equation: tanimoto}
    k_{\text{Tanimoto}}(\mathbf{x}, \mathbf{x'}) \coloneqq \sigma_{f}^2 \cdot \frac{\langle\mathbf{x}, \mathbf{x'}\rangle}{\norm{\mathbf{x}}^2 + \norm{\mathbf{x'}}^2 - \langle\mathbf{x}, \mathbf{x'}\rangle},
\end{equation}

\noindent where $||\cdot||$ is the Euclidean norm.

\subsection{String Kernels}
String kernels \citep{lodhi2002text, cancedda2003word} measure the similarity between strings by examining the degree to which their sub-strings differ. In GAUCHE, the SMILES string kernel \citep{cao2012silico} is implemented, which calculates an inner product between the occurrences of sub-strings, considering all contiguous sub-strings made from at most $n$ characters ($n=5$ was chosen in all experiments). Therefore, for the sub-string count featurisation $\phi : \mathcal{S} \rightarrow \mathbb{R}^p$ (also known as a bag-of-characters representation \citep{jurasfky2000introduction}), where $p$ is the number of unique n-grams from the alphabet, the SMILES string kernel between two strings $\mathcal{S}$ and $\mathcal{S}'$ is given by

\begin{equation}
    k_{\textrm{String}}(\mathcal{S},\mathcal{S}')\coloneqq \sigma^2\cdot \langle \phi(\mathcal{S}), \phi(\mathcal{S}') \rangle.
\end{equation}

Although more complicated string kernels do exist in the literature, for example those that allow non-contiguous matches \citep{2020_Moss}, it was found that the significant extra computational overhead of these methods did not provide improved performance over the more simple SMILES string kernel in the context of molecular data. Note that although named the SMILES string kernel, this kernel can also be applied to any other string representation of molecules e.g. SELFIES.

\subsection{Graph Kernels}

\paragraph{Graph Kernels:} 
Graph kernel methods $\phi_\lambda:\mathcal{G}\to\mathcal{H}$, map elements from a graph domain $\mathcal{G}$ to a reproducing kernel Hilbert space (RKHS) $\mathcal{H}$, in which an inner product between a pair of graphs $g,g'\in\mathcal{G}$ is derived as a measure of similarity,

\begin{equation}
    k_{\text{Graph}}(g, g') \coloneqq \sigma^2 \cdot \langle\phi_\lambda(g),\phi_\lambda(g')\rangle_\mathcal{H},
\end{equation}

\noindent where $\lambda$ denotes kernel-specific hyperparameters and $\sigma^2$ is a scale factor.
Depending on how $\phi_\lambda$ is defined \citep{Nikolentzos_2021}, the kernel considers different substructural motifs and is characterised by different hyperparameters.

Frequently-employed approaches include the random walk kernel \citep{2010_Viswanathan}, given by a geometric series over the count of matching random walks of increasing length with coefficient $\lambda$,
and the Weisfeiler-Lehman (WL) kernel \citep{shervashidze2011weisfeiler}, given by the inner products of label count vectors over $\lambda$ iterations of the Weisfeiler-Lehman algorithm.

\paragraph{Graph Embedding:}
Pretrained graph neural networks (GNNs)~\cite{hu2019} may also be used to embed molecular graphs in a vector space. Since the GNN is trained on a large amount of data, the representation it produces has the potential to be a more expressive method to encode a molecule (Note: this assumes access to a large pool of in-domain data). Given a vector representation from a pretrained GNN model, any \textsc{gp} kernel for continuous input spaces may be applied, such as the RBF kernel.

\section{Experiments}

GAUCHE (available at \href{https://github.com/leojklarner/gauche}{https://github.com/leojklarner/gauche}) is evaluated on regression, UQ and \textsc{bo} tasks. The principle goal in conducting regression and UQ benchmarks is to gauge whether performance on these tasks may be used as a proxy for \textsc{bo} performance. \textsc{bo} is a powerful tool for automated scientific discovery but one would prefer to avoid model misspecification in the surrogate when deploying a scheme in the real world. The following datasets were chosen:

\paragraph{The Photoswitch Dataset:} The labels, $y$ are the experimentally-determined values of the \textit{E} isomer $\pi-\pi^*$ transition wavelength (nm) for 392 photoswitch molecules.

\paragraph{ESOL:} \citep{2004_Delaney}: The labels $y$ are the experimentally-determined logarithmic aqueous solubility values (mols/litre) for 1128 organic small molecules.

\paragraph{FreeSolv:} \citep{2014_Mobley}: The labels $y$ are the experimentally-determined hydration free energies (kcal/mol) for 642 molecules.

\paragraph{Lipophilicity:} The labels $y$ are the experimentally-determined octanol/water distribution coefficient (log D at pH 7.4) of 4200 compounds curated from the ChEMBL database \citep{2012_Gaulton, 2014_Bento}.

\paragraph{Buchwald-Hartwig reactions:} \citep{ahneman2018predicting}: The labels $y$ are the experimentally-determined yields for 3955 Pd-catalysed Buchwald–Hartwig C–N cross-couplings. 

\paragraph{Suzuki-Miyaura reactions:} \citep{perera2018platform}: The labels $y$ are the experimentally-determined yields for 5760 Pd-catalysed Suzuki-Miyaura C-C cross-couplings.

\subsection{Regression}

The regression results for molecular property prediction are reported in Table ~\ref{table: regression} and for reaction yield prediction in Table~\ref{table: reaction} of \autoref{app_b}. The datasets are split in a train/test ratio of 80/20 (note that validation sets are not required for the \textsc{gp} models since training uses the marginal likelihood objective). Errorbars represent the standard error across 20 random initialisations. All \textsc{gp} models are trained using the L-BFGS-B optimiser \citep{1989_Liu}. If not mentioned, default settings in the GPyTorch and BoTorch libraries apply. For the SELFIES representation, some molecules could not be featurised and corresponding entries are left blank. The results of Table~\ref{table: reaction} indicate that the best choice of representation (and hence the choice of kernel) is task-dependent.

\begin{table*}[h]
\caption{Molecular property prediction regression benchmark. RMSE values for 80/20 train/test split across 20 random trials. All WL kernel entries computed by Aditya Ravuri.}
\adjustbox{width=\textwidth}{%
\centering
\begin{tabular}{l l | c c c c}
    \toprule
    \multicolumn{2}{c|}{{\bf GP Model}} & \multicolumn{4}{c}{{\bf Dataset}}  \\
    Kernel & Representation & Photoswitch & ESOL & FreeSolv & Lipophilicity \\
    \hline
    
    Tanimoto & fragprints & ${\bf 20.9} \pm {\bf 0.7}$ & $0.71 \pm 0.01$ & $1.31 \pm 0.06 $ & $\textbf{0.67} \pm \textbf{0.01}$ \\
    & fingerprints & $23.4 \pm 0.8$ & $1.01 \pm 0.01$ & $1.93 \pm 0.09$ & $0.76 \pm 0.01$ \\ 
    & fragments & $26.3 \pm 0.8$ & $0.91 \pm 0.01$ & $1.49 \pm 0.05$ & $0.80 \pm 0.01$ \\ 
    \hdashline
    
    Scalar Product & fragprints & $22.5 \pm 0.7$ & $0.88 \pm 0.01$ & $\textbf{1.27} \pm \textbf{0.02} $ & $0.77 \pm 0.01$ \\
    & fingerprints & $24.8 \pm 0.8$ & $1.17 \pm 0.01$ & $1.93 \pm 0.07$  & $0.84 \pm 0.01$ \\
    & fragments & $36.6 \pm 1.0$ & $1.15 \pm 0.01$ &  $1.63 \pm 0.03$ & $0.97. \pm 0.01$ \\
    \hdashline
    
    String & SELFIES & $24.9 \pm 0.6$ & - & - & - \\
    & SMILES & $24.8 \pm 0.7$ & $\textbf{0.66} \pm \textbf{0.01}$  & $1.31 \pm 0.01$ & $\textbf{0.68} \pm \textbf{0.01}$  \\
    
    \hdashline
    WL Kernel (GraKel) & graph & $22.4 \pm 1.4$ & $1.04 \pm 0.02$ & $1.47 \pm 0.06$ & $0.74 \pm 0.05$ \\

    \bottomrule
\end{tabular}}
\label{table: regression}
\end{table*}

\subsection{Uncertainty Quantification (UQ)}

To quantify the quality of the uncertainty estimates three metrics were used, the negative log predictive density (NLPD), the mean standardised log loss (MSLL) and the quantile coverage error (QCE). The NLPD results are provided in Table~\ref{table: nlpd} and the MSLL and QCE results in Table~\ref{table: MSLL} and Table~\ref{table: QCE} respectively. One trend to note is that uncertainty estimate quality is roughly correlated with regression performance. Numerical errors were encountered with the WL kernel on the large lipophilicity dataset which invalidated the results and so the corresponding entry is left blank. The native random walk kernel was discontinued (for the time being) due to poor performance.

\begin{table*}[h]
\caption{UQ benchmark. NLPD values for 80/20 train/test split across 20 random trials.}
\adjustbox{width=\textwidth}{%
\centering
\begin{tabular}{l l | c c c c}
    \toprule
    \multicolumn{2}{c|}{{\bf GP Model}} & \multicolumn{4}{c}{{\bf Dataset}}  \\
    Kernel & Representation & Photoswitch & ESOL & FreeSolv & Lipophilicity \\
    \hline

    Tanimoto & fragprints & $\textbf{0.22} \pm \textbf{0.03}$ & $0.33 \pm 0.01$ & $0.28 \pm 0.02$ & $\textbf{0.71} \pm \textbf{0.01}$ \\
    & fingerprints & $0.33 \pm 0.03$ & $0.71 \pm 0.01$ & $0.58 \pm 0.03$ & $0.85 \pm 0.01$ \\ 
    & fragments & $0.50 \pm 0.04$ & $0.57 \pm 0.01$ & $0.44 \pm 0.03$ & $0.94 \pm 0.02$\\ 
    \hdashline
    
    Scalar Product & fragprints & $\textbf{0.23} \pm \textbf{0.03}$ & $0.53 \pm 0.01$ & $0.25 \pm 0.02$ & $0.92 \pm 0.01$ \\
    & fingerprints & $0.33 \pm 0.03$ & $0.84 \pm 0.01$ & $0.64 \pm 0.03$ & $1.03 \pm 0.01$\\
    & fragments & $0.80 \pm 0.03$ & $0.82 \pm 0.01$ & $0.54 \pm 0.02$ & $0.88 \pm 0.10$  \\
    \hdashline
    
    String & SELFIES & $0.37 \pm 0.04$ & - & - & - \\
    & SMILES & $0.30 \pm 0.04$ & $\textbf{0.29} \pm \textbf{0.03}$  & $\textbf{0.16} \pm \textbf{0.02}$  & $\textbf{0.72} \pm \textbf{0.01}$ \\
    \hdashline
    
    WL Kernel (GraKel) & graph & $0.39 \pm 0.11$ & $0.76 \pm 0.001$ & $0.47 \pm 0.02$ & - \\
    
    \bottomrule
    
\end{tabular}}
\label{table: nlpd}
\end{table*}

\begin{table}[h]
\caption{UQ Benchmark. MSLL Values ($\downarrow$) for 80/20 Train/Test Split.}
\adjustbox{width=\textwidth}{%
\centering
\begin{tabular}{l l | c c c c}
    \toprule
    \multicolumn{2}{c|}{{\bf GP Model}} & \multicolumn{4}{c}{{\bf Dataset}}  \\
    Kernel & Representation & Photoswitch & ESOL & FreeSolv & Lipophilicity \\
    \hline
    
    Tanimoto & fragprints & $\textbf{0.06} \pm \textbf{0.01}$ & $0.17 \pm 0.04$ & $0.16 \pm 0.02$ & $\textbf{0.50} \pm \textbf{0.006}$ \\
    & fingerprints & $0.16 \pm 0.01$ & $0.55 \pm 0.01$ & $0.42 \pm 0.02$ & $0.63 \pm 0.004$ \\ 
    & fragments & $0.27 \pm 0.01$ & $0.34 \pm 0.04$ & $0.24 \pm 0.02$& $0.72 \pm 0.003$\\ 
    \hdashline
    
    Scalar Product & fragprints & $0.03 \pm 0.01$ & $0.32 \pm 0.004$ & $0.06 \pm 0.01$ & $0.67 \pm 0.003$ \\
    & fingerprints & $0.11 \pm 0.01$ & $0.64 \pm 0.006$ & $0.41 \pm 0.02$ & $0.79 \pm 0.003$\\
    & fragments & $0.56 \pm 0.01$ & $0.58 \pm 0.005$ & $0.29 \pm 0.01$ & $0.94 \pm 0.003$  \\
    \hdashline
    
    String & SELFIES & $0.13 \pm 0.01$ & - & - & - \\
    & SMILES & $\textbf{0.08} \pm \textbf{0.02}$ & $\textbf{0.03} \pm \textbf{0.005}$  & $\textbf{0.03} \pm \textbf{0.02}$  & $\textbf{0.52} \pm \textbf{0.002}$ \\
    \hdashline
    
    WL Kernel (GraKel) & graph & $0.14 \pm 0.03$ & $0.54 \pm 0.01$ & $0.26 \pm 0.01$ & - \\

    \bottomrule
    
\end{tabular}}
\label{table: MSLL}
\end{table}

\begin{table}[h]
\caption{UQ benchmark. QCE values ($\downarrow$) for 80/20 train/test split across 20 random trials.}
\adjustbox{width=\textwidth}{%
\centering
\begin{tabular}{l l | c c c c}
    \toprule
    \multicolumn{2}{c|}{{\bf GP Model}} & \multicolumn{4}{c}{{\bf Dataset}}  \\
    Kernel & Representation & Photoswitch & ESOL & FreeSolv & Lipophilicity \\
    \hline
    
    Tanimoto & fragprints & $\textbf{0.019} \pm \textbf{0.003}$ & $0.023 \pm 0.002$ & $0.023 \pm 0.002$ & $0.006 \pm 0.002$ \\
    & fingerprints & $0.023 \pm 0.003$ & $0.022 \pm 0.002$ & $0.018 \pm 0.003$ & $0.006 \pm 0.001$ \\ 
    & fragments & $0.025 \pm 0.005$ & $0.012 \pm 0.002$ & $0.014 \pm 0.002$& $0.009 \pm 0.002$\\ 
    \hdashline
    
    Scalar Product & fragprints & $0.033 \pm 0.006$ & $0.010 \pm 0.002$ & $0.017 \pm 0.003$ & $0.010 \pm 0.001$ \\
    & fingerprints & $0.036 \pm 0.006$ & $0.014 \pm 0.002$& $0.016 \pm 0.002$ & $0.009 \pm 0.001$ \\
    & fragments & $0.027 \pm 0.004$ & $0.012 \pm 0.003$ & $0.021 \pm 0.003$ & $0.010 \pm 0.001$ \\
    \hdashline
    
    String & SELFIES & $0.031 \pm 0.006$ & -  & - & - \\
    & SMILES & $0.024 \pm 0.003$ & $0.016 \pm 0.002$  & $0.019 \pm 0.003$  & $0.005 \pm 0.001$ \\
    \hdashline
    
    WL Kernel (GraKel) & graph & $0.025 \pm 0.007$ & $0.011 \pm 0.004$ & $0.019 \pm 0.009$ & $0.066 \pm 0.014$ \\
    \bottomrule
\end{tabular}}
\label{table: QCE}
\end{table}

\subsection{Chemical Reaction Yield Prediction}
\label{app_b}

Further regression and UQ experiments are presented in Table~\ref{table: reaction}. The differential reaction fingerprint in conjunction with the Tanimoto kernel is the best-performing reaction representation.

\begin{table*}[ht]
\caption{Chemical reaction regression benchmark. 80/20 train/test split across 20 random trials. Experiments performed by Bojana Rankovic. Kernel code written by Ryan-Rhys Griffiths.}
\adjustbox{width=\textwidth}{%
\centering
\begin{tabular}{l l | c c c c}
    \toprule
    \multicolumn{2}{c|}{\bf GP Model} & \multicolumn{4}{c}{\bf Buchwald-Hartwig}  \\
    Kernel & Representation & RMSE $\downarrow$ & $R^{2}$ score $\uparrow$ & MSLL $\downarrow$ & QCE $\downarrow$ \\ \hline
    
    Tanimoto & OHE & $7.94 \pm 0.05$ & $0.91 \pm 0.001$ & $-0.06 \pm 0.002$& $0.011 \pm 0.001$\\
    & DRFP & $\textbf{6.48} \pm \textbf{0.45}$ & $\textbf{0.94} \pm \textbf{0.015}$ & $\textbf{-0.15} \pm \textbf{0.07} $& $0.027 \pm 0.002$ \\ 
    \hdashline
    
    Scalar Product & OHE & $15.23 \pm 0.052$ & $0.69 \pm 0.002$ & $0.57 \pm 0.002$ & $0.008 \pm 0.001$\\ 
    & DRFP & $14.63 \pm 0.050$ & $0.71 \pm 0.002$ &$0.55 \pm 0.002$ & $0.010 \pm 0.001$ \\
    \hdashline
    
    RBF & RXNFP & $10.79 \pm 0.049$ &$0.84 \pm 0.001 $ & $0.37 \pm 0.005$ & $0.024 \pm 0.001$ \\
    \hline
    
    \multicolumn{2}{c|}{} & \multicolumn{4}{c}{\bf Suzuki-Miyaura}  \\ \hline
    
    Tanimoto & OHE & $11.18 \pm 0.036$ & $0.83 \pm 0.001$ & $0.23 \pm 0.001 $ & $0.007 \pm 0.001 $  \\
    & DRFP & $11.46 \pm 0.038 $ & $0.83 \pm 0.001 $ &$0.25 \pm 0.006$ &$0.019 \pm 0.000$  \\
    \hdashline
    
    Scalar Product & OHE & $19.91 \pm 0.042$ & $0.47 \pm 0.003$ & $0.82 \pm 0.001 $  & $0.012 \pm 0.001$ \\
    & DRFP & $19.66 \pm 0.042$ & $0.52 \pm 0.003$ & $0.81 \pm 0.001$  & $0.014 \pm 0.001$  \\
    \hdashline
    
    RBF & RXNFP & $13.83 \pm 0.048 $ & $0.75 \pm 0.002$ & $0.50 \pm 0.001$ & $0.007 \pm 0.001$ \\
    \bottomrule
\end{tabular}}
\label{table: reaction}
\end{table*}

\subsection{Bayesian Optimisation}

Two of the best-performing kernels were taken forward, the Tanimoto-fragprint kernel and the bag of SMILES kernel to undertake \textsc{bo} over the photoswitch and ESOL datasets. Random search is used as a baseline. \textsc{bo} is run for 20 iterations of sequential candidate selection (EI acquisition) where candidates are drawn from 95\% of the dataset. The results are provided in \autoref{bayesopt}. The models are initialised with 5\% of the dataset. In the case of the photoswitch dataset this corresponds to just 19 molecules. In this ultra-low data setting, common to many areas of synthetic chemistry, both models outperform random search, highlighting the real-world use-case for such models in supporting human chemists prioritise candidates for synthesis. Furthermore, one may observe that \textsc{bo} performance is tightly coupled to regression and UQ performance. In the case of the photoswitch dataset, the better-performing Tanimoto model on regression and UQ also achieves relatively better BO performance. Additionally, results are reported on the Buchwald-Hartwig reaction dataset.

\begin{figure*}[]
\centering
\subfigure[Photoswitch]{\label{fig:4}\includegraphics[width=0.3\textwidth]{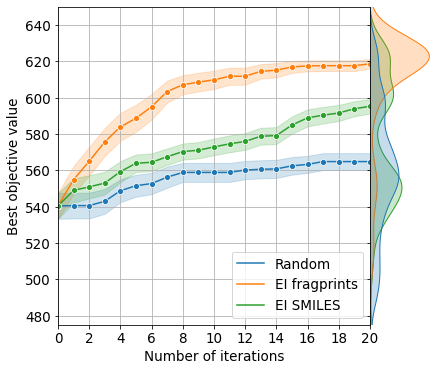}}
\subfigure[ESOL]{\label{fig:3}\includegraphics[width=0.3\textwidth]{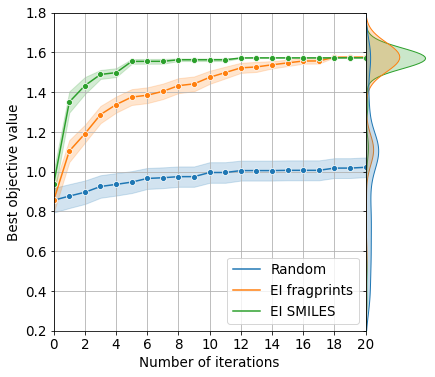}}
\subfigure[Buchwald-Hartwig]{\label{fig:5}\includegraphics[width=0.3\textwidth]{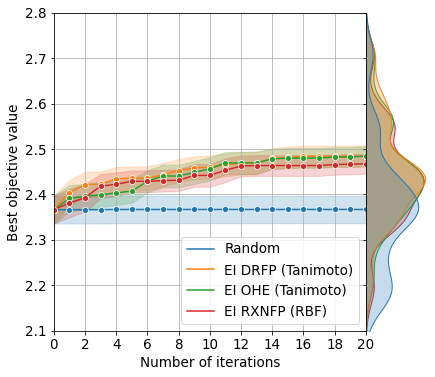}}
\caption{\textsc{bo} performance. Standard error confidence interval from 50 random initialisations, 20 for Buchwald-Hartwig reactions. Marginal density plots for the trace shown on the right axis. Data for Buchwald-Hartwig plot produced by Bojana Rankovic.}
\label{bayesopt}
\end{figure*}

\section{Conclusions}

This chapter introduces GAUCHE, a library for GAUssian Processes in CHEmistry with the aim of providing tools for UQ and \textsc{bo} that may hopefully be deployed for screening in laboratory settings. Future work, will seek to:

\begin{enumerate}
\setlength\itemsep{0.3em}
    \item Expand the range of \textsc{gp} kernels considered, most notably to include \textit{deep kernels} based on GNN embeddings.
    \item Perform more extensive benchmarking for UQ and active learning against models such as BNNs.
    \item Exploit the benefits of the Autodiff framework to facilitate the learning of graph kernel hyperparameters through the \textsc{gp} marginal likelihood.
    \item Broaden the application domains considered by GAUCHE to include examples in protein engineering.
    \item Investigate more sophisticated \textsc{gp}-based optimization and active learning loops in chemistry applications \citep{eyke2020iterative}, featuring ideas from batch \citep{gonzalez2016batch}, multi-task \citep{swersky2013multi}, multi-fidelity \citep{moss2020mumbo}, multi-objective \citep{daulton2020differentiable}, controllable experimental noise \citep{moss2020bosh}, or quantile \citep{torossian2020bayesian} optimisation.
\end{enumerate}

\nomenclature[Z-GAN]{GAN}{Generative Adversarial Networks}
\nomenclature[Z-DFT]{DFT}{Density Functional Theory}
\nomenclature[Z-KS-DFT]{KS-DFT}{Kohn-Sham Density Functional Theory}
\nomenclature[Z-TD-DFT]{TD-DFT}{Time-Dependent Density Functional Theory}
\nomenclature[Z-SMILES]{SMILES}{Simplified Molecular-Input Line-Entry System}
\nomenclature[Z-SELFIES]{SELFIES}{Self-Referencing Embedded Strings}
\nomenclature[Z-ECFP]{ECFP}{Extended Connectivity Fingerprints}
\nomenclature[Z-OHE]{OHE}{One-Hot Encoding}
\nomenclature[Z-DRFP]{DRFP}{Differential Reaction Fingerprint}
\nomenclature[Z-AL]{AL}{Active Learning}
\nomenclature[Z-BNN]{BNN}{Bayesian Neural Networks}
\nomenclature[Z-ASAP]{ASAP}{Automatic Selection and Prediction}
\nomenclature[Z-RKHS]{RKHS}{Reproducing Kernel Hilbert Space}
\nomenclature[Z-UQ]{UQ}{Uncertainty Quantification}
\nomenclature[Z-RMSE]{RMSE}{Root Mean Square Error}
\nomenclature[Z-NLPD]{NLPD}{Negative Log Predictive Density}
\nomenclature[Z-MSLL]{MSLL}{Mean Standardised Log Loss}
\nomenclature[Z-QCE]{QCE}{Quantile Coverage Error}

\chapter{Molecular Discovery with Gaussian Processes}  %
\chapterimage[height=100pt]{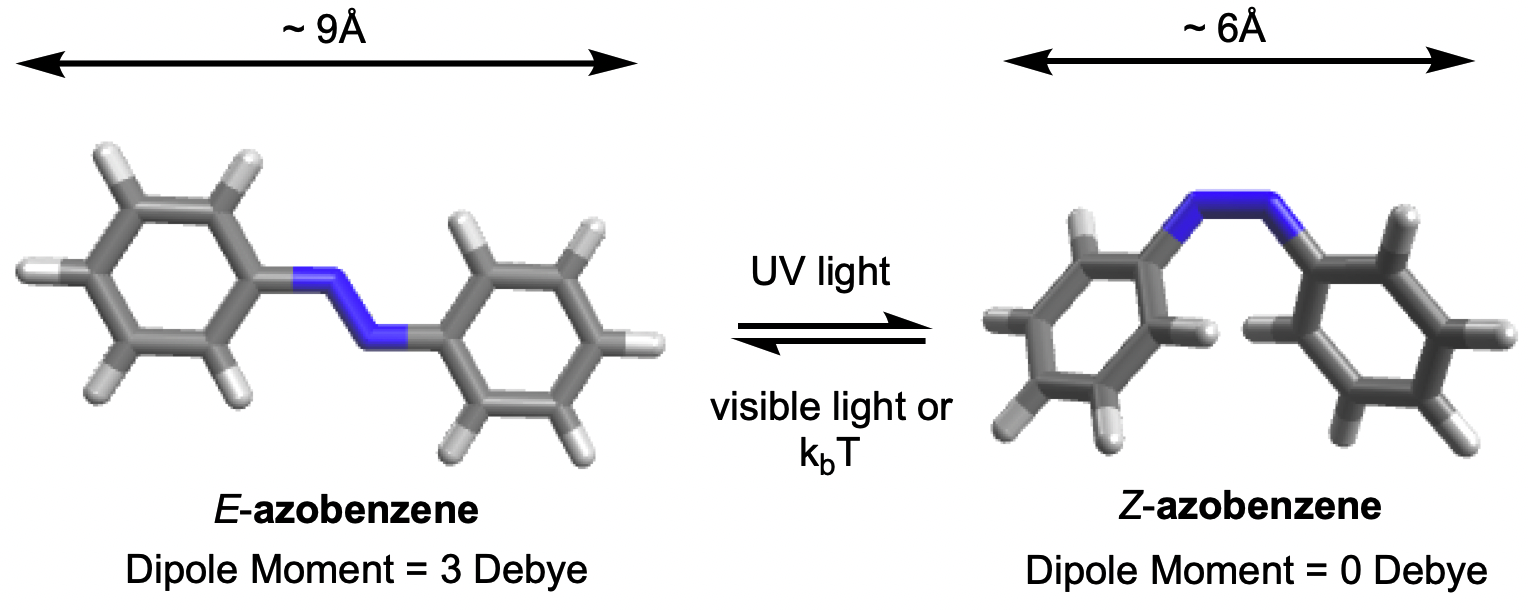}

\ifpdf
    \graphicspath{{Chapter1/Figs/Raster/}{Chapter1/Figs/PDF/}{Chapter1/Figs/}}
\else
    \graphicspath{{Chapter1/Figs/Vector/}{Chapter1/Figs/}}
\fi

\textbf{Status:} Accepted as Griffiths, RR., Greenfield, JL., Thawani, AR., Jamasb, AR., Moss HB., Bourached A., Jones, P., McCorkindale W., Aldrick AA., Fuchter MJ., Lee, AA. Data-Driven Discovery of Molecular Photoswitches with Multioutput Gaussian
Processes, \textit{Chemical Science} 2022.

\section{Preface}

This chapter is focussed on leveraging the predictive capabilities of \textsc{gp}s for molecular discovery. The discovery campaign is focussed on photoswitches, a particular class of molecule defined by their ability to convert between two or more isomeric forms in response to light. Photoswitches may be employed for information transfer and photopharmacological applications. Key  photoswitch properties in these domains include separation of the electronic absorption bands of the isomers as well as red-shifting of the absorption bands. The former property is useful for addressing a specific isomer and achieving high photostationary states (PSS), while the latter limits material damage from UV exposure and serves to increase the penetration depth for photopharmacological applications. The ability to engineer these properties, however, is challenging. As such, a predictive model is highly desirable for identifying novel and performant molecules.

In this chapter, a data-driven discovery pipeline for molecular photoswitches is presented, underpinned by dataset curation and multitask learning with \textsc{gp}s. In the prediction of electronic transition wavelengths, it is demonstrated that a multioutput Gaussian process (\textsc{mogp}) trained using labels from four photoswitch transition wavelengths yields the strongest predictive performance relative to single-task models as well as operationally outperforming time-dependent density functional theory (TD-DFT). The proposed approach is validated experimentally, by screening a library of commercially-available photoswitch molecules. Through this screen, several motifs are identified that displayed separated electronic absorption bands of their isomers and exhibit red-shifted absorptions. The curated dataset and all models are made available at \url{https://github.com/Ryan-Rhys/The-Photoswitch-Dataset}.

\section{Introduction}
\label{intro}

Photoswitch molecules are capable of reversible structural isomerisation upon irradiation with light as depicted in \autoref{fig:photo}, a characteristic behaviour that has led to a broad range of molecular \citep{Eisenreich2018,Dorel2019,Neilson2013}, supramolecular \citep{Corra2022, Han2016, Lee2022}, and materials applications \citep{Wang2021, Garcia-Amoros2012a, Hou2019, GouletHanssens2020}. Efficient light addressability is key to many of these applications for which photophysical properties of the photoswitch are the core determinant.

\begin{figure*}[ht!]
\centering
{\includegraphics[width=0.68\textwidth]{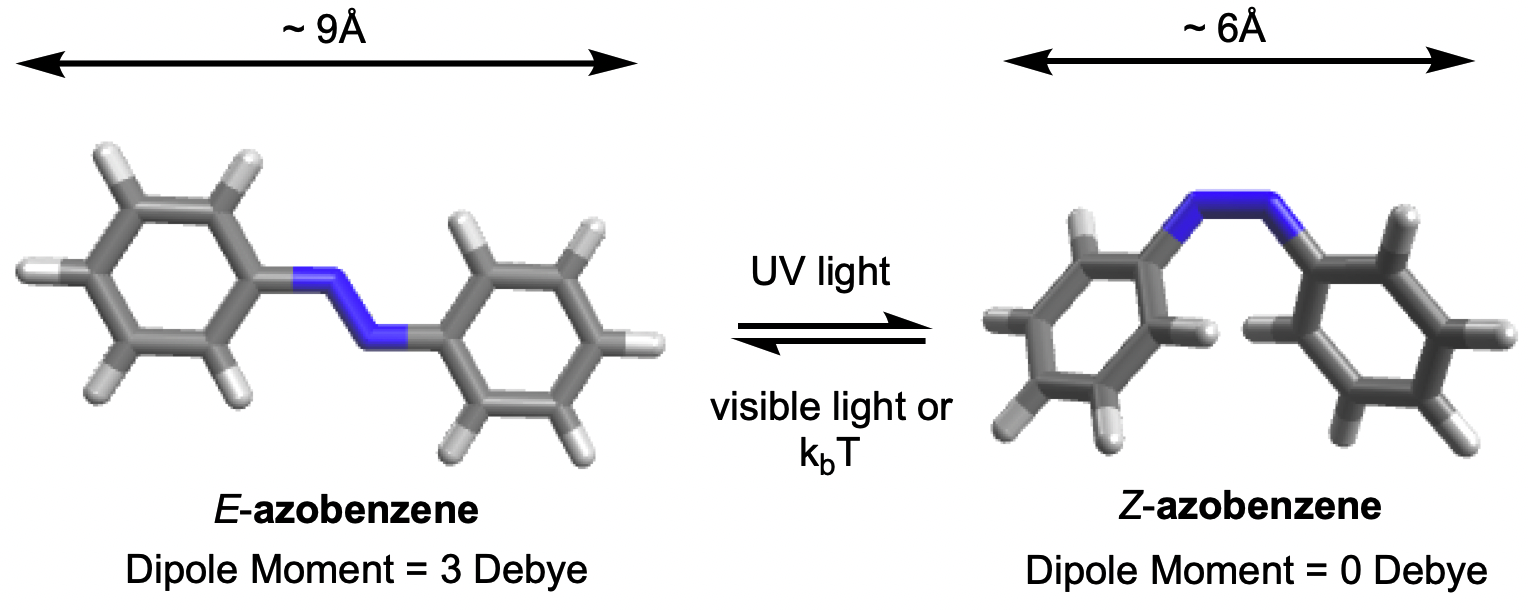}\label{fig:photo_mech}}
\caption{Azobenzene, an example of a photoswitch that undergoes a reversible structural change upon irradiation with light. }
\label{fig:photo}
\end{figure*}

Properties which govern the utility of a photoswitch include quantum yields of photoswitching, the steady-state distribution of a particular isomer at a given irradiation wavelength (known as the photostationary state - PSS) as well as the thermal half-life of the metastable isomer. The desired thermal half-life depends on the application. Information transfer applications benefit from short thermal half-life photoswitches \citep{Garcia-Amoros2012a} whilst, in contrast, photoswitches used in energy storage are serviced by long thermal half-lives \citep{Dong2018}. In contrast, the attainment of separated isomeric electronic absorption bands and a high PSS are uniformly favourable properties for photoswitches as they dictate the light addressability of the isomeric forms. Minimal spectral overlap for a set irradiation wavelength is made possible by modulating the $\pi-\pi^*$ and $n-\pi^*$ bands of the \emph{E} and \emph{Z} isomers. Low spectral overlap maximises the composition of a given isomer at a set PSS. Inducing red-shifted absorption spectra away from the UV region is also desirable given that the use of high wavelength light decreases photo-induced material degradation and simultaneously improves tissue penetration depth.

To date, laboratory synthesis or quantum chemical calculations such as TD-DFT have been the choice approaches for measuring ground truth and computing predicted estimates of photoswitch properties. Both approaches are cost-intensive in terms of synthesis or compute time, although it should be noted that high-throughput DFT approaches have potential to mitigate the wall-clock time to some extent in the future \citep{2017_Lopez, 2018_Wilbraham, 2020_Choudhary}. In light of this, human intuition remains the guide for candidate selection in many photoswitch chemistry laboratories. Advances in molecular machine learning, however, have taken great strides in recent years. In particular, machine learning property prediction has the potential to cut the attrition rate in the discovery of novel and impactful molecules by virtue of its short inference time. A rapid, accessible, and accurate machine learning prediction of a photoswitch's properties prior to synthesis would allow promising structures to be prioritised, facilitating photoswitch discovery as well as revealing new structure-property relationships.

Recently, work by Lopez and co-workers \citep{2021_Mukadum} employed machine learning to accelerate a quantum chemistry screening workflow for photoswitches. The screening library in this case is generated from 29 known azoarene photoswitches and their derivatives yielding a virtual library of 255,991 photoswitches in total. The authors observed that screening using active search tripled the discovery rate of photoswitches compared to random search according to a binary labelling system which assigns a positive label to a molecule possessing a $\lambda_{\text{max}} > 450 \text{nm}$ and a negative label otherwise. The approach highlights the potential for AL and \textsc{bo} methodology to accelerate DFT-based screening. Nonetheless, to the best of our knowledge, open questions remain in terms of the utility of machine learning-based predictive models for experimental photoswitch properties, in addition to experimental validation of machine learning approaches.

In this chapter an experimentally-validated framework for molecular photoswitch discovery is presented based on the curation of a large dataset of experimental photophysical data, and multitask learning with \textsc{mogp}s. This framework was designed with the goals of: (i) performing faster prediction relative to TD-DFT and directly training on experimental data; (ii) obtaining improved accuracy relative to human experts; (iii) operationalising model predictions in the context of laboratory synthesis. To achieve these goals, a dataset of the electronic absorption properties of 405 photoswitches in their \emph{E} and \emph{Z} isomeric forms was curated originally by Aditya Raymond Thawani, a full description of the dataset and collated properties is provided in \autoref{sec:data_description}. 

Following an extensive benchmark study, an appropriate machine learning model and molecular representation was identified for prediction, as detailed in \autoref{sec:model_choice}. A key feature of this model is that it is performant in the small data regime as photoswitch properties (data labels) obtained via laboratory measurement are expensive to collect in both financial cost and time. The chosen model uses a \textsc{mogp} approach due to its ability to operate in the multitask learning setting, amalgamating information obtained from molecules with multiple labels. In \autoref{sec:dft} it is shown that the \textsc{mogp} model trained on the curated dataset obtains comparable predictive accuracy to TD-DFT (at the CAM-B3LYP level of theory) and only suffers slight degradations in accuracy relative to TD-DFT methods with data-driven linear corrections whilst maintaining inference time on the order of seconds. A further benchmark against a cohort of human experts is included in \autoref{sec:human}. In \autoref{sec:valid} the approach is used to screen a set of commercially-available azoarene photoswitches, and in the process, identify several motifs displaying separated electronic absorption bands of their isomers as well as red-shifted absorptions, thus making them suitable for information transfer
and photopharmacological applications.

\section{Dataset Curation}
\label{sec:data_description}

Experimentally-determined properties of azobenzene-derived photoswitch molecules reported in the literature were curated initially by Aditya Raymond Thawani. Azobenzene derivatives in possession of diverse substitution patterns and functional groups were included to cover as large a fraction of chemical space as possible. Azoheteroarenes and cyclic azobenzenes were also included. The dataset includes properties for 405 photoswitches denoted using the SMILES syntax. A full list of references for the data sources is provided in Appendix~\ref{exp_source}. 

The following properties from the literature were collated, where available. (i) The rate of thermal isomerisation (units = $s^{-1}$), a solution-based measure of the thermal stability of the metastable isomer. For cyclic azophotoswitches, this corresponds to the \textit{E} isomer, whereas for non-cyclic azophotoswitches the rate is for the \textit{E} isomer. (ii) The PSS of each isomer at the set wavelength of photoirradiation. Such values are obtained through continuous, solution-based irradiation of a photoswitch until the point at which the steady-state distribution of the \emph{E} and \emph{Z} isomers is observed. The PSS values reported in the literature all correspond to solution-phase measurements. (iii) The irradiation wavelength (nm)  corresponds to the wavelength of light employed to irradiate samples, such that a PSS is attained, from \emph{E}-\emph{Z} or \emph{Z}-\emph{E}. (iv) Experimental transition wavelengths (nm) correspond to the wavelength at which the $\pi-\pi{^*}$/$\emph{n}-\pi{^*}$ electronic transition attains a maximum for the given isomer. This data was curated from solution-phase measurements. (v) DFT-computed transition wavelengths (nm), obtained using solvent continuum TD-DFT methods, correspond to the predicted $\pi-\pi{^*}$/$\emph{n}-\pi{^*}$ electronic transition maximum for a given isomer. (vi) The extinction coefficient (M$^{-1}$cm$^{-1}$), corresponds to the extent to which a molecule absorbs light, conditioned on the solvent. (vii) The theoretically-computed Wiberg Index \citep{1968_Wiberg} (through the analysis of the SCF density calculated at the PBE0/6-31G** level of theory), a measure of the bond order of the N=N bond in an azo-based photoswitch, provides an indication of the ‘strength’ of the azo bond.

Following the curation of the Photoswitch dataset, the goal is to use a machine learning model to predict the four experimentally-determined transition wavelengths. These four properties were chosen as they are core determinants of quantitative, bidirectional photoswitching \citep{2019_Crespi}. The wavelength properties include, the $\pi-\pi{^*}$ transition wavelength of the \emph{E} isomer (labels for 392 molecules), the \emph{n}$-\pi{^*}$ transition wavelength of the \emph{E} isomer (labels for 141 molecules), the $\pi-\pi{^*}$ transition wavelength of the \emph{Z} isomer (labels for 93 molecules), and the \emph{n}$-\pi{^*}$ transition wavelength of the \emph{Z} isomer (labels for 123 molecules). While other photophysical or thermal properties, such as the thermal half-life of the metastable state, could also be investigated using machine learning approaches, there are fewer reported measurements of thermal half-lives which significantly reduces the amount of data that may be used to train a model.

\section{Machine Learning Prediction Pipeline}\label{sec:model_choice}

There are three constituents to the prediction pipeline: A dataset, a model and a representation. The effects of the choice of dataset are examined in Appendix~\ref{sec:big_data}, where performance is compared between models trained on the curated dataset against those trained on a large out-of-domain dataset of 6,142 photoswitches \citep{2019_Beard}. In terms of the choice of model, a broad range of models are evaluated including Gaussian processes (\textsc{gp}), random forest (\textsc{rf}), Bayesian neural networks (BNNs), graph convolutional networks (GCNs), message-passing neural networks (MPNNs), graph attention networks (GATs), LSTMs with augmented SMILES, attentive neural processes (\textsc{anp}), as well as multioutput Gaussian processes (\textsc{mogp}), which aggregate information across prediction tasks to perform multitask learning \citep{1997_Caruana}.

Full model benchmark results, as well as all hyperparameter settings, are provided in Appendix~\ref{benchmark_ml}, where Wilcoxon signed rank tests \citep{1945_Wilcoxon} determine that there is weak evidence to support that multitask learning affords improvements over the single task setting in the case where auxiliary task labels (i.e. not the label being predicted) are available for test molecules. All subsequent experiments in this chapter assume that the \textsc{mogp} is not provided with auxiliary task labels for test molecules. All experiments may be reproduced via the scripts provided at \url{https://github.com/Ryan-Rhys/The-Photoswitch-Dataset}.
The \textsc{mogp} was chosen to take forward to the comparison against TD-DFT and experimental screening due to its predictive performance in the multitask setting as well as its ability to represent uncertainty estimates. Some use-cases for the \textsc{gp} uncertainty estimates with confidence-error curves are illustrated in Appendix~\ref{conf_error}.

\begin{figure}[!htbp]
    \begin{center}
        \includegraphics[width=0.98\textwidth]{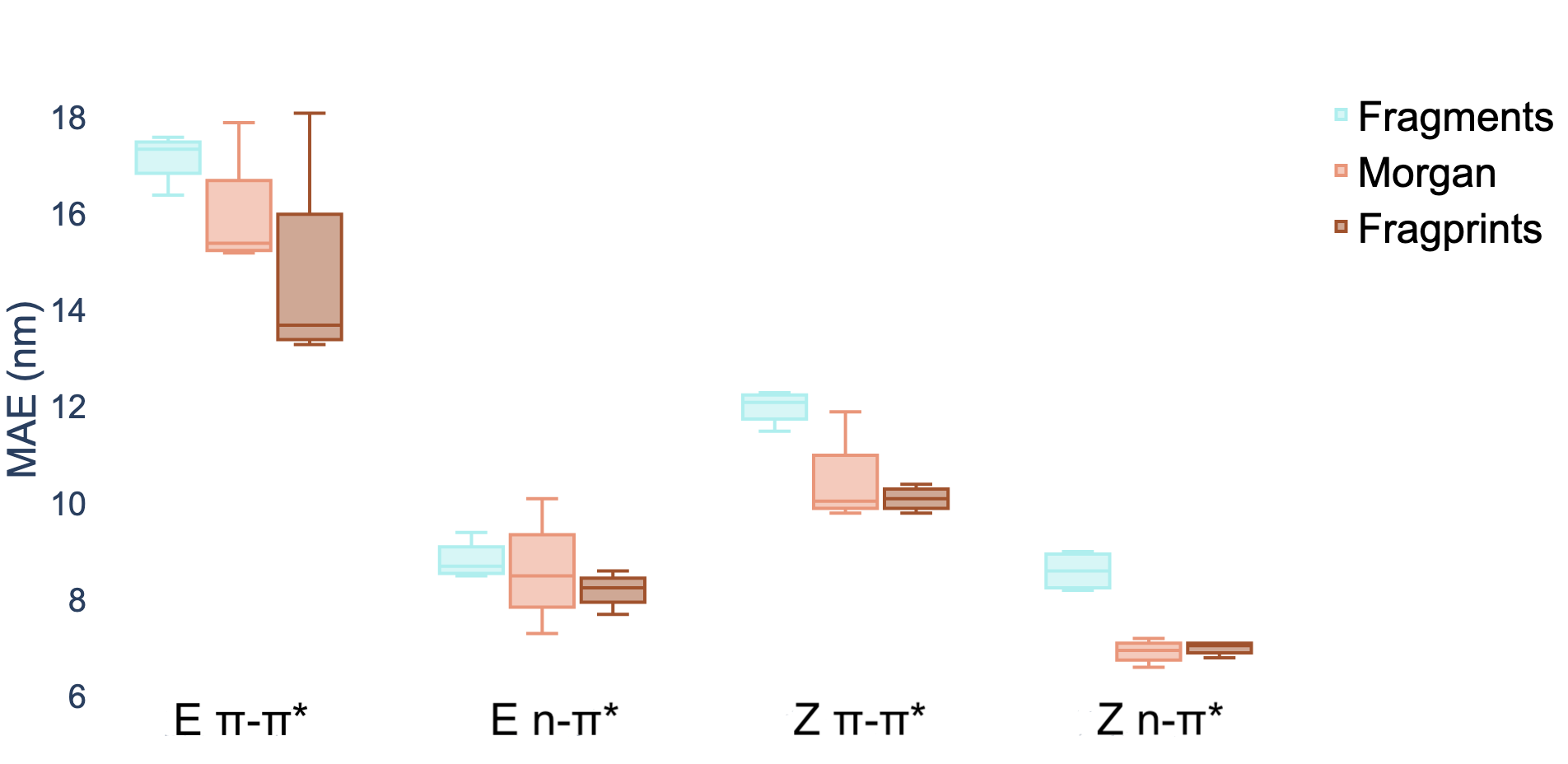}
    \end{center}
    \caption{Marginal boxplot showing the performance of representations aggregated over different models (\textsc{rf}, \textsc{gp}, \textsc{mogp} and \textsc{anp}). Performance is evaluated on 20 random train/test splits of the Photoswitch Dataset in a ratio of 80/20 using the mean absolute error (MAE) as the performance metric. An individual box is computed using the mean values of the MAE for the four models for the representation indicated by the associated colour and shows the range in addition to the upper and lower quartiles of the error distribution. The plot indicates that fragprints are the best representation on the \emph{E} isomer $\pi - \pi^*$ prediction task and RDKit fragments alone are disfavoured across all tasks.}
    \label{boxplot}
\end{figure}
 
In terms of the choice of representation, three commonly-used descriptors are evaluated: RDKit fragment features \citep{rdkit}, ECFP fingerprints \citep{2010_Rogers} as well as a hybrid 'fragprints' representation formed by concatenating the Morgan fingerprint and fragment feature vectors. The performance of the RDKit fragment, ECFP fingerprint, and fragprint representations on the wavelength prediction tasks is visualised in \autoref{boxplot} where aggregation is performed over the \textsc{rf}, \textsc{gp}, \textsc{mogp} and \textsc{anp} models. This analysis motivated the use of the fragprints representation in conjunction with the \textsc{mogp} to take forward to the TD-DFT comparison and experimental screening. The \textsc{mogp} with Tanimoto kernel employed for prediction will now be described.

\subsection{Multioutput Gaussian Processes (MOGPs)}

A \textsc{mogp} generalises the idea of the \textsc{gp} to multiple outputs and a common use case is multitask learning. In multitask learning, tasks are learned in parallel using a shared representation; the idea being that learning for one task may benefit from the training signals of related tasks. In the context of photoswitches, the tasks constitute the prediction of the four transition wavelengths. We wish to perform Bayesian inference over a stochastic function $f: \mathbb{R}^D \to \mathbb{R}^P$ where $P$ is the number of tasks and we possess observations $\{(\mathbf{x_{11}}, y_{11}), \dotsc , (\mathbf{x_{1N}}, y_{1N}), \dotsc , (\mathbf{x_{P1}}, y_{P1}), \dotsc , (\mathbf{x_{PN}}, y_{PN})\}$. We do not necessarily have property values for all tasks for a given molecule. 

To construct a \textsc{mogp} we compute a new kernel function $k(\mathbf{x}, \mathbf{x'}) \cdot B[i, j]$ where $B$ is a positive semi-definite $P \times P$ matrix, where the $(i, j)^{\text{th}}$ entry of the matrix $B$ multiplies the covariance of the $i$-th function at $\mathbf{x}$ and the $j$-th function at $\mathbf{x'}$. Such a \textsc{mogp} is termed the intrinsic coregionalisation model (ICM) \citep{2007_Williams}. Inference proceeds in the same manner as for vanilla \textsc{gp}s, substituting the new expression for the kernel into the equations for the predictive mean and variance. Positive semi-definiteness of $B$ may be guaranteed through parametrising the Cholesky decomposition $LL^{\top}$, where $L$ is a lower triangular matrix and the parameters may be learned alongside the kernel hyperparameters through maximising the marginal likelihood in \autoref{equation: log_lik_} substituting the appropriate kernel. In all our experiments we use bit/count vectors to represent molecules and hence we choose the Tanimoto kernel defined in \autoref{equation: tanimoto}.

While it has been widely cited that \textsc{gp}s scale poorly to large datasets due to the $O(N^3)$ cost of training, where $N$ is the number of datapoints \citep{2006_Rasmussen}, recent advances have seen \textsc{gp}s scale to millions of data points using multi GPU parallelisation \citep{2019_Pleiss}. Nonetheless, on CPU hardware, scaling \textsc{gp}s to datasets on the order of $10,000$ data points can prove challenging. For the applications considered in this chapter, however, we are unlikely to be fortunate enough to encounter datasets of relevant experimental measurements on the order of tens of thousands of data points and so CPU hardware is sufficient for this purpose.

\section{TD-DFT Performance Comparison}\label{sec:dft}

The \textsc{mogp}, Tanimoto kernel and fragprints combination are compared against two widely-utilised levels of TD-DFT: CAM-B3LYP \citep{2004_Yanai} and PBE0 \citep{1996_Perdew, 1999_Adamo}. While the CAM-B3LYP level of theory offers highly accurate predictions, its computational cost is high relative to that of machine learning methods. To obtain the predictions for a single photoswitch molecule one is required to perform a ground state energy minimisation followed by a TD-DFT calculation \citep{2015_Belostotskii}. In the case of photoswitches these calculations need to be performed for both molecular isomers and possibly multiple conformations which further increases the wall-clock time. When screening multiple molecules is desirable, this cost, in addition to the expertise required to perform the calculations may be prohibitive, and so in practice it is easier to screen candidates based on human chemical intuition. In contrast, inference in a data-driven model is on the order of seconds but may yield poor results if the training set is out-of-domain relative to the prediction task. 

In \autoref{tab_merge1} a performance comparison is presented against 99 molecules and 114 molecules for CAM-B3LYP and PBE0 respectively, both using the 6-31G** basis set taken from the results of a benchmark quantum chemistry study \citep{2011_Jacquemin}, to which the reader is referred for all information pertaining to the details of the calculations. \footnote[1]{The TD-DFT CPU runtime in \autoref{tab_merge1} estimates are taken from \citep{2015_Belostotskii} and hence represent a ballpark figure that is liable to decrease with advances in high performance computing.} An additional $15$ molecules are included in the test set for PBE0. These molecules are not featured in the study by \citet{2011_Jacquemin}, but are included from the other literature sources present in the Photoswitch Dataset which use the same basis set. It should also be noted that the data presented in \citet{2011_Jacquemin} contains measurements for the same molecules under different solvents. In this chapter, solvent effects are absorbed into the noise. Specifically, the solvent is not treated as part of the molecular representation. As such, for duplicated molecules a single solvent measurement is chosen at random. We report the mean absolute error (MAE) and the mean signed error (MSE), presented in Appendix~\ref{spearman_section}, to assess systematic deviations in predictive performance for the TD-DFT methods. For the \textsc{mogp} model, leave-one-out validation is performed, testing on a single molecule and training on the others as well as the experimentally-determined property values for molecules acquired from the Photoswitch Dataset. The prediction errors are then averaged and the standard error is reported.

\begin{table*}[h]
\caption{\textsc{mogp} against TD-DFT performance comparison on the PBE0 benchmark consisting of 114 molecules, and the CAM-B3LYP benchmark consisting of 99 molecules. Best metric values for each benchmark are highlighted in bold.}
\resizebox{0.98\textwidth}{!}{
\centering
\begin{tabular}{l l | c  c | c}
    \toprule
    \multicolumn{2}{c|}{{\bf Method}} & \multicolumn{2}{c|}{{\bf Accuracy Metric (nm)}} & \multicolumn{1}{c}{{\bf CPU Runtime ($\downarrow$)}}  \\
     &  & MAE ($\downarrow$) & MSE &  \\
    \hline
    \multicolumn{2}{c|}{{\bf \underline{PBE0 Benchmark}}} & 
    & 
    \\
    \textsc{mogp} & & $15.5 \pm 1.3$ & $\textbf{0.0} \pm \textbf{2.0}$ & $\textbf{<}$ \textbf{1 minute} \\
    PBE0 & uncorrected & $26.0 \pm 1.8$ & $-19.1 \pm 2.5$ &  \\
    & linear correction & $\textbf{12.4} \pm \textbf{1.3}$ & $-1.2 \pm 1.8$ & ca. 228 days \\ 
    \hline
    \multicolumn{2}{c|}{{\bf \underline{CAM-B3LYP Benchmark}}} & 
    & 
    \\
    
    \textsc{mogp} & & $15.3 \pm 1.4$ & $-0.2 \pm 2.1$ & $\textbf{<}$ \textbf{1 minute} \\
    
    CAM-B3LYP & uncorrected & $16.5 \pm 1.6$ & $6.7 \pm 2.2$ &  \\
    & linear correction & $\textbf{10.7} \pm \textbf{1.2}$ & $\textbf{0.0} \pm \textbf{1.6}$ & ca. 396 days \\
    \bottomrule
\end{tabular}}
\label{tab_merge1}
\end{table*}

The \textsc{mogp} model outperforms PBE0 by a large margin and provides comparable performance to CAM-B3LYP in terms of accuracy. The MSE values for the TD-DFT methods, however, indicate that there is systematic deviation in the TD-DFT predictions. This motivates the addition of a data-driven correction to the TD-DFT predictions. As such, a Lasso model, with an $L_1$ multiplier of $0.1$, is trained on the prediction errors of the TD-DFT methods and this correction is applied when evaluating the TD-DFT methods on the heldout set in leave-one-out validation. Lasso is chosen because it outperforms linear regression empirically in fitting the errors, likely due to inducing sparsity in the high-dimensional fragprint feature vectors. The Spearman rank-order correlation coefficients of all methods as well as the error distributions are provided in Appendix~\ref{spearman_section}. There, it is observed that an improvement is obtained in the correlation between TD-DFT predictions on applying the linear correction. Furthermore, the error distribution becomes more symmetric on applying the correction. 

\section{Human Performance Benchmark}\label{sec:human}

\begin{figure*}[!htbp]
\centering
\subfigure{\includegraphics[width=0.41\textwidth]{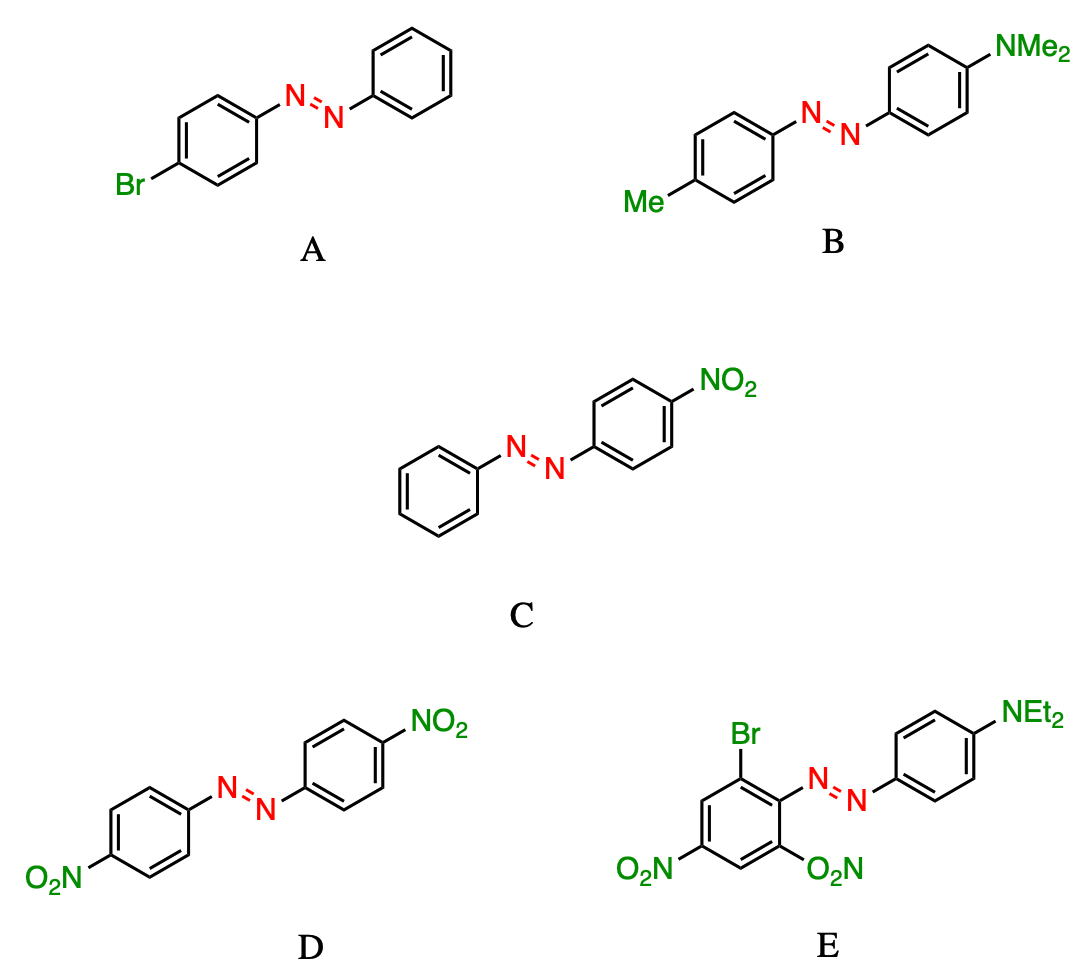}}  
\subfigure{\label{fig:4ppb}\includegraphics[width=0.55\textwidth]{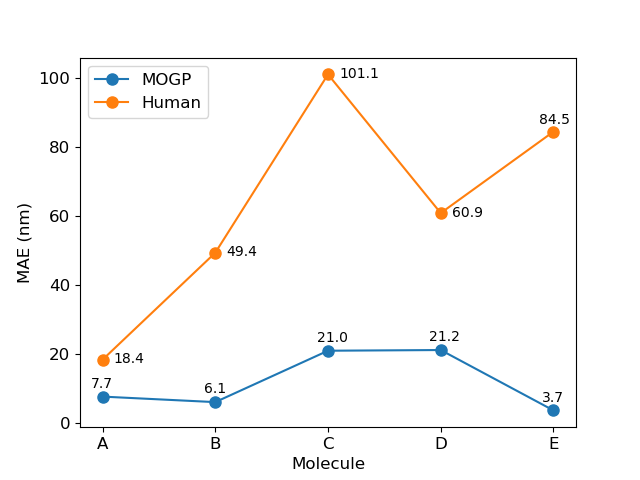}}
\caption{A performance comparison between human experts (orange) and the \textsc{mogp}-fragprints model (blue). MAEs are computed on a per molecule basis across all human participants.}
\label{human}
\end{figure*}

In practice, candidate screening is undertaken based on the opinion of a human chemist due to the speed at which predictions may be obtained. While inference in a data-driven model is comparable to the human approach in terms of speed, the aim in this section is to compare the predictive accuracy of the two approaches. To achieve this, a panel of 14 photoswitch chemists were assembled, comprising Postdoctoral Research Assistants and PhD students in photoswitch chemistry with a median research experience of 5 years. The assigned task was to predict the \emph{E} isomer $\pi-\pi{^*}$ transition wavelength for five molecules taken from the dataset. The study was designed by Aditya Raymond Thawani who was also responsible for recruiting human participants. All model predictions and plots were generated by Ryan-Rhys Griffiths.

All participants had prior knowledge of UV-vis spectroscopy. It should be noted that one of the limitations of this study is that the human chemists were not provided with the full dataset of 405 photoswitch molecules in advance of making their predictions. As such, the goal in constructing the study was to enable a comparison of the benefits of dataset curation, together with a machine learning model to internalise the information contained in the dataset, against the experience acquired over a photoswitch chemist's research career. Analysing the MAE across all humans per molecule \autoref{human}, it is observed that the human chemists perform worse than the \textsc{mogp} model in all instances. In going from molecule A to E, the number of point changes on the molecule increases steadily, thus, increasing the difficulty of prediction. Noticeably, the human performance is approximately five-fold worse on molecule E (three point changes) relative to molecule A (one point change). This highlights the fact that in instances of multiple functional group modifications, human experts are unable to reliably predict the impact on the \emph{E} isomer $\pi-\pi^*$ transition wavelength. The full results breakdown is provided in Appendix~\ref{sec:human_app}.

\section{Screening for Novel Photoswitches using the MOGP}
\label{sec:valid}

Having determined that the \textsc{mogp} approach does not suffer substantial degradation in accuracy relative to TD-DFT, the model was subsequently used to perform experimental screening. Diazo-containing compounds supplied by Molport and Mcule were identified. There were 7,265 commercially-available diazo molecules as of November 2020, when experiments were planned. The full list is made available at \href{https://github.com/Ryan-Rhys/The-Photoswitch-Dataset/tree/master/dataset}{https://github.com/Ryan-Rhys/The-Photoswitch-Dataset/tree/master/dataset}. The \textsc{mogp} was then used to score the list. A subset of 11 molecules were chosen to screen which satisfied the criteria detailed in the following section. The goal of the screening was to discover a novel azophotoswitch motif satisfying the performance criteria.

\subsection{Screening Criteria}

To demonstrate the utility of the machine learning prediction pipeline, commercially-available photoswitches were screened based on a set of performance criteria. The experimental properties of the screened photoswitches were subsequently measured and compared against the predictions made by the \textsc{mogp} model. The criteria were selected to demonstrate that beneficial properties for materials and photopharmacological applications, which are difficult to engineer, could be obtained using the \textsc{mogp} model. The criteria are:

\begin{enumerate}
    \item A $\pi-\pi^*$ maximum between 450-600 nm for the \textit{E} isomer.
    \item A separation in excess of 40 nm between the $\pi-\pi^*$ of the \textit{E} isomer and the $\pi-\pi^*$ of the \textit{Z} isomer.
\end{enumerate}

The first criterion was imposed to limit UV-included material damage and enhance tissue penetration depths. The second criterion was chosen to provide complete bidirectional photoswitching as the specified degree of separation between the $\pi-\pi^*$ bands of the isomers facilitates a given isomer to be selectively addressed using light emitting diodes (LEDs), commonly used for their low power consumption and ability to express broad emission profiles relative to laser diodes.

 \begin{figure*}[!htbp]
    \begin{center}
        \includegraphics[width=0.8\textwidth]{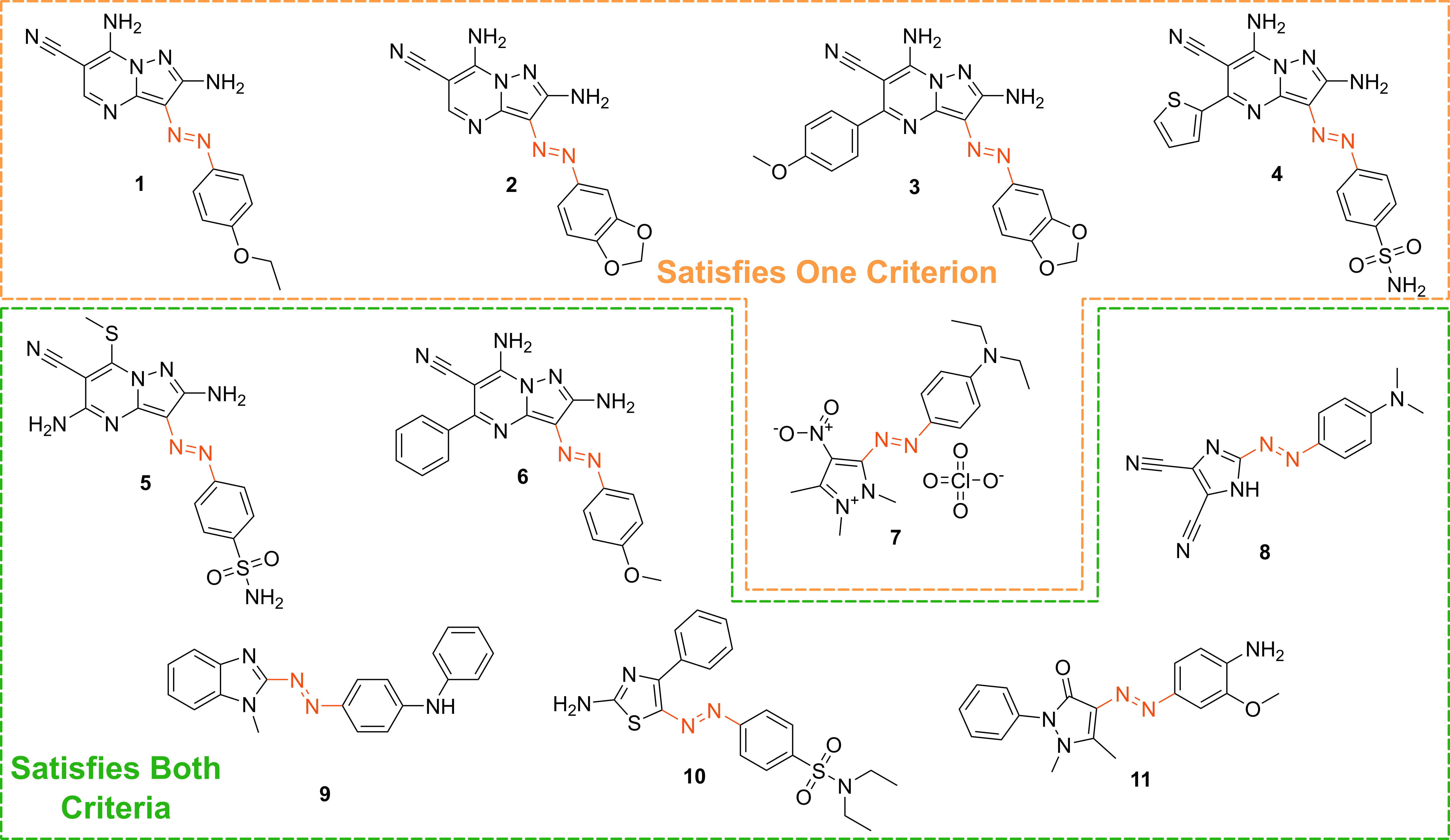}
    \end{center}
    \caption{The chemical structures of the 11 commercially-available azo-based photoswitches that were predicted to meet the criteria. Figure produced by Jake Greenfield.}
    \label{structures2}
\end{figure*}

\subsection{Lead Candidates}

Based on the stated selection criteria, 11 commercially-available molecules were identified via the predictions of the \textsc{mogp} model. The molecular structures are shown in \autoref{structures2}. Solutions of the 11 photoswitches were prepared in the dark to a concentration of 25 $\mu$M in DMSO. The UV-vis spectra of the photoswitches were recorded using a photodiode array spectrometer where the photoswitches were in their thermodynamically stable \emph{E} isomeric form. Samples were continuously irradiated with wavelengths of light at an angle of 90$^{\circ}$ relative to the measurement path. UV-vis spectra were recorded during irradiation until no further change in the UV-vis trace was observed, indicating attainment of the PSS. The \emph{in situ} irradiation procedure was implemented such that compounds displaying short thermal half-lives could be reliably measured. Through repetition of the measurement process with one or more distinct irradiation wavelengths, the PSS could be quantified and subsequently used to predict the UV-vis spectrum of the pure \emph{Z} isomer using the method detailed by \citet{Fischer1967}. With spectra of the \emph{E} and \emph{Z} isomers in hand, the experimental wavelength of the $\pi-\pi^*$ band of each isomer was determined and compared with that predicted by the \textsc{mogp}. Full experimental details are made available in Appendix~\ref{exp_app}. 

Model predictions are compared against the experimentally-determined values in \autoref{tab_preds}. The \textsc{mogp} MAE on the \emph{E} isomer $\pi-\pi^*$ wavelength prediction task was $22.7$ nm and $21.6$ nm on the \emph{Z} isomer $\pi-\pi^*$ wavelength prediction task, comparable for the \emph{E} isomer $\pi-\pi^*$ and slightly higher for the \emph{Z} isomer $\pi-\pi^*$ relative to the benchmark study in Appendix~\ref{benchmark_ml}, reflecting the challenge of achieving strong generalisation performance when extrapolating to large regions of chemical space. The first criterion is a requirement on the absolute rather than the relative value of the $\pi-\pi^*$ transition wavelengths and so the experimental values may be subject to shifts depending on the solvent.

Molecules can display solvatochromism in so far as the dielectric of the solvent, as well as hydrogen-bonding interactions, can influence the electronic transitions giving rise to hypsochromic or bathochromic shifts in the absorption spectra. This can manifest as changes in the position, intensity and shape of the UV-vis absorption spectrum. As such, the 450 nm criterion could be considered a rough guide and candidates that are just short of the threshold may fulfill the criterion in a different solvent. Nonetheless, given that the \textsc{mogp} model is trained on just a few hundred data points and is required to extrapolate to several thousand structures, the accuracy is promising with the advent of further experimental data. In terms of satisfying the pre-specified criteria, 7 of the 11 molecules possessed an \emph{E} isomer $\pi-\pi^*$ wavelength greater than 450 nm, 10 of the 11 molecules possessed a separation between the \emph{E} and \emph{Z} isomer $\pi-\pi^*$ wavelengths of greater than 40 nm, and 6 of the 11 molecules satisfied both criteria. Compound 7 did not photoswitch under irradiation.

\definecolor{caribbeangreen}{rgb}{0.0, 0.8, 0.6}
\definecolor{carrotorange}{rgb}{0.93, 0.57, 0.13}
\definecolor{cinnabar}{rgb}{0.89, 0.26, 0.2}

\begin{table}[h]
\caption{\textsc{mogp} predictions compared against experimental values (nm). A traffic light system indicates whether the molecules satisfied the criteria. Both criteria are indicated by (\color{caribbeangreen}{green}\color{black}{) and one criterion is indicated by}(\color{carrotorange}{orange}\color{black}{). All molecules satisfied at least one criterion. The model MAE was 22.7 nm for the \emph{E} isomer $\pi - \pi^*$} and 21.6 nm for the \emph{Z} isomer $\pi - \pi^*$. Experimental measurements were taken by Jake Greenfield.}
\resizebox{0.983\textwidth}{!}{
\centering
\begin{tabular}{c|cc|cccc}
\toprule
& \multicolumn{2}{c|}{{ \bf \underline{Model}}} & \multicolumn{4}{c}{{ \bf \underline{Experimental}}} \\
Switch & \begin{tabular}[c]{@{}c@{}}\emph{E} $\pi - \pi^*$ \end{tabular} & \begin{tabular}[c]{@{}c@{}}\emph{Z} $\pi - \pi^*$. \end{tabular} & \begin{tabular}[c]{@{}c@{}}\emph{E} $\pi - \pi^*$ \end{tabular} & \begin{tabular}[c]{@{}c@{}}\emph{Z} $\pi - \pi^*$ \end{tabular} & \emph{Z} PSS (\%) & \begin{tabular}[c]{@{}c@{}}ca. t$\frac{1}{2}$  (s)\end{tabular}  \\
\midrule
\color{carrotorange}
\textbf{1}  & \color{carrotorange}456 & \color{carrotorange}368 & \color{carrotorange}446 &  \color{carrotorange}355 &  \color{carrotorange}90 (405 nm) & \color{carrotorange}<5 \\ 
\color{carrotorange}
\textbf{2}  & \color{carrotorange}459 & \color{carrotorange}377 & \color{carrotorange}441 &  \color{carrotorange}356 &  \color{carrotorange}96 (405 nm) &  \color{carrotorange}<1 \\
\color{carrotorange}
\textbf{3}  & \color{carrotorange}457 & \color{carrotorange}377 & \color{carrotorange}399 &  \color{carrotorange}331 &  \color{carrotorange}66 (405 nm) & \color{carrotorange}<10 \\
\color{carrotorange}
\textbf{4}  & \color{carrotorange}463 & \color{carrotorange}373 & \color{carrotorange}445 &  \color{carrotorange}357 &  \color{carrotorange}94 (405 nm) &  \color{carrotorange}<1 \\
\color{caribbeangreen}
\textbf{5}  & \color{caribbeangreen}471 &\color{caribbeangreen} 381 & \color{caribbeangreen}450 &  \color{caribbeangreen}370 &  \color{caribbeangreen}68 (450 nm) &  \color{caribbeangreen}<1 \\
\color{caribbeangreen}
\textbf{6}  & \color{caribbeangreen}460 & \color{caribbeangreen}368 & \color{caribbeangreen}451 &  \color{caribbeangreen}360 &  \color{caribbeangreen}92 (405 nm) & \color{caribbeangreen}<30 \\
\color{carrotorange}
\textbf{7}  & \color{carrotorange}467 & \color{carrotorange}369 & \color{carrotorange}534 &  \color{carrotorange}\textit{n/a} & \color{carrotorange}\textit{n/a} & \color{carrotorange}\textit{n/a} \\
\color{caribbeangreen}
\textbf{8}  & \color{caribbeangreen}450 & \color{caribbeangreen}359 & \color{caribbeangreen}465 &  \color{caribbeangreen}376 &  \color{caribbeangreen}87 (405 nm) & \color{caribbeangreen}<10 \\
\color{caribbeangreen}
\textbf{9}  & \color{caribbeangreen}453 & \color{caribbeangreen}369 & \color{caribbeangreen}468 &  \color{caribbeangreen}399 &  \color{caribbeangreen}60 (450 nm) & \color{caribbeangreen}<10 \\
\color{caribbeangreen}
\textbf{10} & \color{caribbeangreen}453 & \color{caribbeangreen}363 & \color{caribbeangreen}471 &  \color{caribbeangreen}398 &  \color{caribbeangreen}15 (450 nm) &  \color{caribbeangreen}<1 \\
\color{caribbeangreen}
\textbf{11} & \color{caribbeangreen}453 & \color{caribbeangreen}360 & \color{caribbeangreen}452 &  \color{caribbeangreen}379 & \color{caribbeangreen}88 (405 nm)  &  \color{caribbeangreen}<1 \\
\bottomrule
\end{tabular}}
\label{tab_preds}
\end{table}

The correlation between the ML-predicted electronic absorption bands and the experimental measurements provided in \autoref{tab_preds} highlights the utility of the \textsc{mogp} model in identifying photoswitches with red-shifted and separated $\pi-\pi^*$ transitions. However, it should be noted that several photoswitches exhibit low PSS compositions of the metastable isomer at the irradiation wavelengths employed. Low PSS values of the \emph{Z} isomer may be attributed to overlap of broad electronic transitions for the isomeric forms. It is envisage that the composition of the \emph{Z} isomer at the PSS may be enhanced by expanding the curated dataset to consider the full-width-at-half-maximum (FWHM) of the electronic absorption bands. Moreover, the thermal half-lives of the photoswitches in \autoref{tab_preds} are short (less than 1 minute). This rapid thermal relaxation is to be expected for the push-pull type photoswitches the \textsc{mogp} predicted. Despite showing some potential applications for information transfer, it is envisaged that consideration of the thermal half-life properties would be beneficial for future work. Prediction of the thermal half-lives would enable further selectivity in the choice of photoswitch for a given application. It is anticipated that machine learning-based prediction, using the \textsc{mogp} model or otherwise, will be of use for synthetic photoswitch chemists who aim to design photoswitches with red-shifted absorption bands.

 \begin{figure*}[!htbp]
    \begin{center}
        \includegraphics[width=\textwidth]{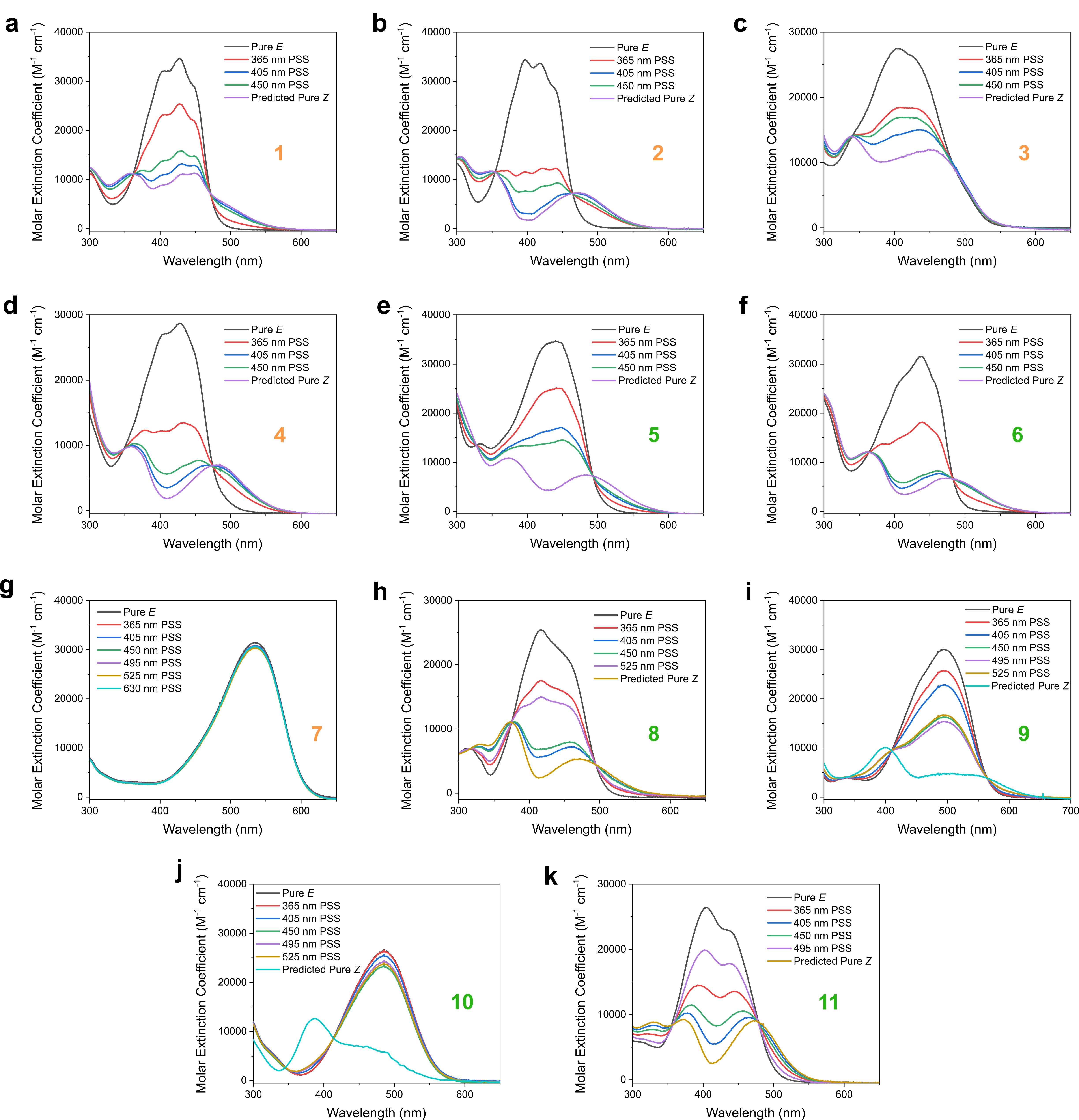}
    \end{center}
    \caption{The experimental UV-vis absorption spectrum of photoswitches \textbf{1}-\textbf{11} measured at 25 $\mu$M in DMSO and shown as the molar extinction coefficient (M$^{-1}$ cm$^{-1}$). Distinct irradiation wavelengths were chosen to predict the "pure" \emph{Z} spectra by applying the procedure detailed by \citet{Fischer1967} The chemical structures of these photoswitches are shown in Figure \ref{structures2}. Spectra generated by Jake Greenfield.}
    \label{UV_vis}
\end{figure*}

\section{Conclusions}\label{sec:conc}

This chapter introduces a data-driven prediction pipeline underpinned by dataset curation and multioutput Gaussian processes. It is demonstrated that a \textsc{mogp} model trained on a small curated azophotoswitch dataset can achieve comparable predictive accuracy to TD-DFT, and only slightly degraded performance relative to TD-DFT with a data-driven linear correction, in near-instantaneous time. The methodology is applied to discover several motifs that displayed separated electronic absorption bands of their isomers and which exhibit red-shifted absorption. The discovered motifs are hence suited for information transfer materials and photopharmacological applications. Sources of future work include the curation of a dataset of the thermal reversion barriers to improve the predictive capabilities of machine learning models as well as investigating how synthetic chemists may use model uncertainty estimates in the decision process to screen molecules e.g. via active learning \citep{2021_Mukadum} and Bayesian optimisation. The confidence-error curves in Appendix~\ref{conf_error} show initial promise in this direction and indeed understanding how best to tailor calibrated Bayesian models to molecular representations \citep{2020_flowmo, 2022_Gauche} is an avenue worthy of pursuit. The curated dataset and all code to train models is released under an MIT licence at \url{https://github.com/Ryan-Rhys/The-Photoswitch-Dataset}.

\nomenclature[Z-RF]{RF}{Random Forest}
\nomenclature[Z-GCN]{GCN}{Graph Convolutional Network}
\nomenclature[Z-GAT]{GAT}{Graph Attention Network}
\nomenclature[Z-MPNN]{MPNN}{Message-Passing Neural Network}
\nomenclature[Z-MOGP]{MOGP}{Multioutput Gaussian Process}
\nomenclature[Z-ICM]{ICM}{Intrinsic Coregionalisation Model}
\nomenclature[Z-MAE]{MAE}{Mean Absolute Error}
\nomenclature[Z-MSE]{MSE}{Mean Signed Error} %

\chapter{Modelling Experimental Noise with Gaussian Processes}
\chapterimage[height=130pt]{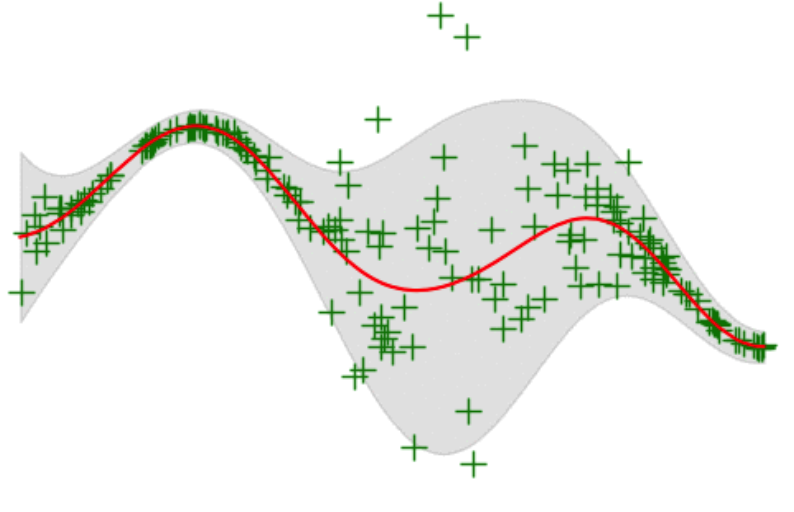}

\ifpdf
    \graphicspath{{Chapter2/Figs/Raster/}{Chapter2/Figs/PDF/}{Chapter2/Figs/}}
\else
    \graphicspath{{Chapter2/Figs/Vector/}{Chapter2/Figs/}}
\fi

\textbf{Status:} Published as Griffiths, RR., Aldrick AA., Garcia-Ortegon, M., Lalchand, V., Lee, AA. Achieving Robustness to Aleatoric Uncertainty with Heteroscedastic Bayesian Optimisation, \textit{Machine Learning: Science and Technology}, 2021.

\section{Preface}

    Bayesian optimisation (\textsc{bo}) is a sample-efficient search methodology that holds great promise for accelerating drug and materials discovery programs. A frequently-overlooked modelling consideration in \textsc{bo} strategies however, is the representation of heteroscedastic aleatoric uncertainty. In many practical applications, it is desirable to identify inputs with low aleatoric noise, an example of which might be a material composition which displays robust properties in response to a noisy fabrication process. In this chapter, a heteroscedastic \textsc{bo} scheme is proposed that is capable of representing and minimising aleatoric noise across the input space. The scheme employs a heteroscedastic Gaussian process (\textsc{gp}) surrogate model in conjunction with two straightforward adaptations of existing acquisition functions. First, the augmented expected improvement (AEI) heuristic is extended to the heteroscedastic setting, and second, the aleatoric noise-penalised expected improvement (ANPEI) heuristic is introduced. Both methodologies are capable of penalising aleatoric noise in the suggestions. In particular, the ANPEI acquisition yields improved performance relative to homoscedastic \textsc{bo} and random search on toy problems as well as on two real-world scientific datasets. All code is made available at: \url{https://github.com/Ryan-Rhys/Heteroscedastic-BO}
    
\section{Introduction}

\textsc{bo} is proving to be a highly effective search methodology in areas such as drug discovery \citep{2018_Design, 2020_Griffiths, 2020_Hoffman}, materials discovery \citep{2021_Hase, hase2021olympus, 2020_Terayama}, chemical reaction optimisation \citep{2020_Felton, felton2021summit, 2020_Yehia}, robotics \citep{2016_Calandra}, sensor placement \citep{2019_Grant}, tissue engineering \citep{2018_Olofsson} and genetics \citep{2020_Moss}. Heteroscedastic aleatoric noise, however, is rarely accounted for in these settings despite being an important consideration for real-world applications. Aleatoric uncertainty refers to uncertainty inherent in the observations (measurement noise) \citep{2017_Kendall}. In contrast, epistemic uncertainty corresponds to model uncertainty and may be explained away given sufficient data. Heteroscedastic aleatoric noise refers to aleatoric noise which varies across the input domain and is a prevalent feature of many scientific datasets; perhaps suprisingly not only experimental datasets, but also datasets where properties are predicted computationally. One such source of heteroscedasticity in the computational case might be situations in which the accuracy of first-principles calculations deteriorate as a function of the chemical complexity of the molecule being studied \citep{2018_Griffiths}. 

\begin{figure*}[!ht]
\centering
\subfigure[Density plot of computational errors]{\label{fig:hist_DFT}\includegraphics[width=0.48\textwidth]{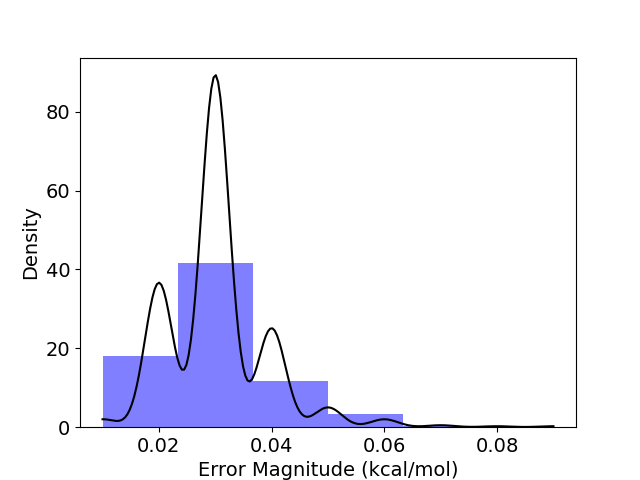}}
\subfigure[Density plot of experimental errors]{\label{fig:hist_exp}\includegraphics[width=0.48\textwidth]{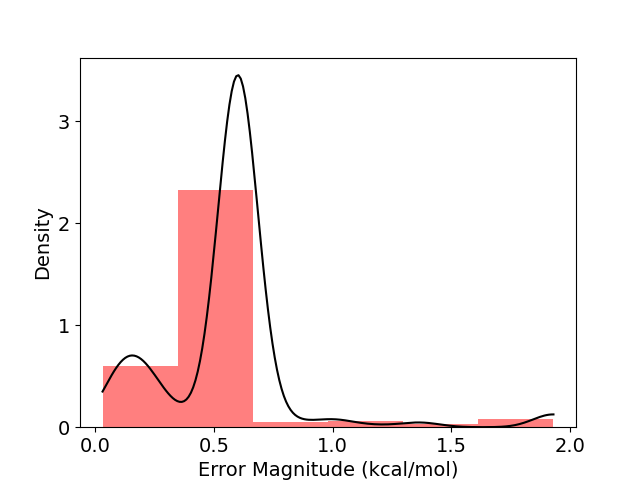}}
\caption{(a) The density histogram of computational errors (kcal/mol) for the FreeSolv hydration energy dataset \citep{2017_Duarte}. The computational errors in the hydration free energy arise from systematic errors in the force field used in alchemical free energy calculations based on classical molecular dynamics (MD) simulations. (b) A similar density histogram for the experimental errors where the source of uncertainty stems from the instrumentation used to obtain the measurement. The histograms are overlaid with kernel density estimates.}
\label{fig:hist}
\end{figure*}

\begin{figure*}[ht!]
\centering
\subfigure[Homoscedastic \textsc{gp} Fit ]{\label{fig:homo}\includegraphics[width=0.48\textwidth]{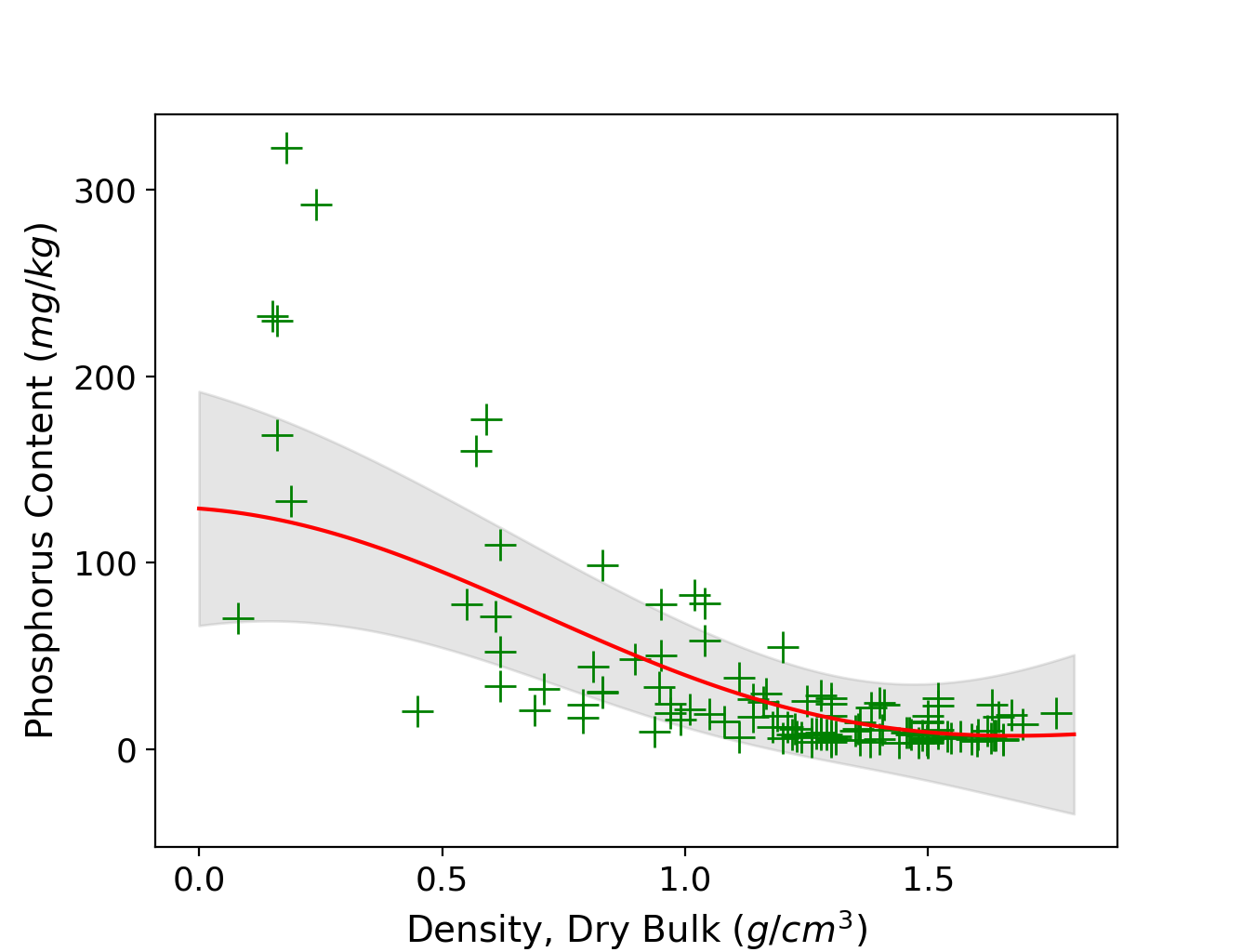}}
\subfigure[Heteroscedastic \textsc{gp} Fit]{\label{fig:het}\includegraphics[width=0.48\textwidth]{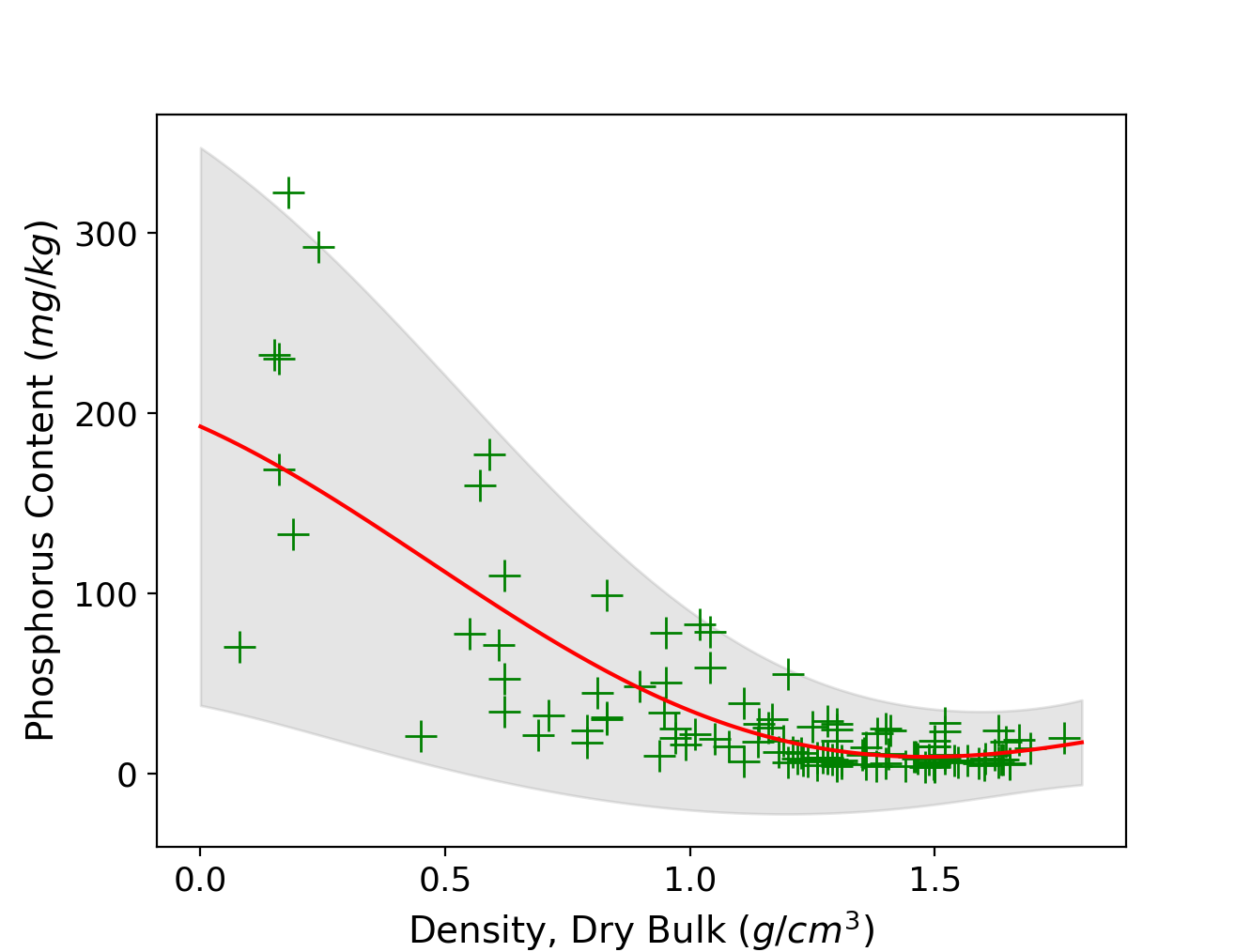}} 
\caption{Comparison of homoscedastic and heteroscedastic \textsc{gp} fits to the soil phosphorus fraction dataset \citep{houGlobalDatasetPlant2018}.}
\label{fig:soil}
\end{figure*}

In \autoref{fig:hist}, real-world sources of heteroscedasticity are illustrated using the FreeSolv dataset of \cite{2017_Duarte}. The consequences of misrepresenting heteroscedastic noise as being homoscedastic, i.e. constant across the input domain, are illustrated using a second dataset \citep{houGlobalDatasetPlant2018} in \autoref{fig:soil}. The homoscedastic model can underestimate noise in certain regions of the input space which in turn could induce a \textsc{bo} scheme to suggest values possessing large aleatoric noise. In an application such as high-throughput virtual screening \citep{2015_Knapp}, the cost of misrepresenting noise during the screening process could lead to a substantial loss of time in material fabrication \citep{2017_Knapp}. In this chapter, a heteroscedastic \textsc{bo} algorithm is introduced which is capable of both representing and minimising aleatoric noise in its suggestions. The chapter contributions are:

\newpage

\begin{enumerate}[label=(\arabic*)]
    \item The introduction of a novel combination of surrogate model and acquisition function designed to minimise heteroscedastic aleatoric uncertainty.
    \item A demonstration of the scheme's ability to outperform naive schemes based on homoscedastic \textsc{bo} and random search on toy problems as well as two real-world scientific datasets.
    \item The provision of an open-source implementation.
    
\end{enumerate}

The chapter is structured as follows: \autoref{rel_het_work} introduces related work on heteroscedastic \textsc{bo}. \autoref{method_section} provides background on the heteroscedastic \textsc{gp} surrogate model used in this chapter and introduces the novel HAEI and ANPEI acquisitions functions. \autoref{experiments} considers experiments on synthetic and scientific datasets possessing heteroscedastic noise where the goal is to be robust to, i.e. minimise, aleatoric noise in the suggestions. \autoref{first_ablation} presents an ablation study on noiseless tasks as well as tasks with homoscedastic and heteroscedastic noise in order to determine whether there is a detrimental effect to using a heteroscedastic surrogate when the noise properties of the problem are a priori unknown. \autoref{het_conc} concludes with some limitations of the approach presented as well as fruitful sources for future work.

\section{Related Work}
\label{rel_het_work}

The most similar work to our own is that of \citet{2017_Calandra} where experiments are reported on a heteroscedastic Branin-Hoo toy function using the variational heteroscedastic Gaussian process (\textsc{gp}) approach of \citet{2011_Lazaro}. This work defines and optimises a robustness index, making a compelling case for penalisation of aleatoric noise in real-world \textsc{bo} problems. A modification to expected improvement (EI), expected risk improvement is introduced in \citet{2013_Kuindersma} and is applied to problems in robotics where robustness to aleatoric noise is desirable. In this framework, however, the relative weights of performance and robustness cannot be tuned \citep{2017_Calandra}. \citet{2014_Assael, 2014_Ariizumi} implement heteroscedastic \textsc{bo} but do not introduce an acquisition function that penalises aleatoric noise. \citet{2015_Sui, 2016_Berkenkamp} consider the related problem of safe \textsc{bo} through implementing constraints in parameter space. In this instance, the goal of the algorithm is to enforce a performance threshold for each evaluation of the black-box function. Recently, the winners of the 2020 NeurIPS Black-Box Optimisation Competition applied non-linear output transformations in their solution to tackle heteroscedasticity. The authors however are not interested in explicitly penalising aleatoric noise in this case. In terms of acquisition functions, \cite{2009_Frazier, 2019_Letham} propose principled approaches to handling aleatoric noise in the homoscedastic setting that could be extended to the heteroscedastic setting. Our primary focus in this chapter, however, is to highlight that heteroscedasticity in the surrogate model is beneficial and so an examination of a subset of acquisition functions is sufficient for this purpose.

\section{Heteroscedastic Bayesian Optimisation}
\label{method_section}

The goal is to perform \textsc{bo} whilst minimising input-dependent aleatoric noise. In order to represent input-dependent aleatoric noise, a heteroscedastic surrogate model is required. 

\subsection{The Most Likely Heteroscedastic Gaussian Process}

The most likely heteroscedastic Gaussian process (\textsc{mlhgp}) approach of \citet{2007_Kersting} is adopted, and for consistency, the same notation as the source work is used in the presentation. There is a dataset $\mathbf{\mathcal{D}} = \{(\mathbf{x}_{i}, t_{i})\}^{n}_{i = 1}$ in which the target values $t_{i}$ have been generated according to $t_{i} = f(\mathbf{x}_{i}) + \epsilon_{i}$. Independent Gaussian noise terms $\epsilon_{i} \sim \mathcal{N}(0, \sigma_{i}^2)$ are assumed, with variances given by $\sigma_{i}^2 = r(\mathbf{x}_{i})$. In the heteroscedastic setting $r$ is typically a non-constant function over the input domain $\mathbf{x}$. In order to perform \textsc{bo}, the predictive distribution $P(\mathbf{t}^{*} \mid \mathbf{x}^{*}_{1}, \ldots, \mathbf{x}^{*}_{q})$ needs to be modelled at the query points $\mathbf{x}^{*}_{1}, \ldots, \mathbf{x}^{*}_{q}$. Placing a \textsc{gp} prior on $f$ and taking $r(\textbf{x})$ as the assumed noise function, the predictive distribution is multivariate Gaussian $\mathcal{N}(\mathbf{\mu}^{*}, \Sigma^{*})$ with mean

\begin{equation}
    \mathbf{\mu}^{*} = E[\mathbf{t^{*}}] = K^{*}(K + R)^{-1} \mathbf{t},
\end{equation}

\noindent and covariance matrix

\begin{equation}
    \Sigma^{*} = \text{var}[\mathbf{t^{*}}] = K^{**} + R^{*} - K^{*}(K + R)^{-1}K^{*T},
\end{equation}

\noindent where $K \in \mathbb{R}^{n \times n}$, $K_{ij} = k(\mathbf{x}_{i},\mathbf{x}_{j})$, 
$K^{*} \in \mathbb{R}^{q \times n}$, $K_{ij}^{*} = k(\mathbf{x}_{i}^{*},\mathbf{x}_{j})$, $K^{**} \in \mathbb{R}^{q \times q}$, $K^{**}_{ij} = k(\mathbf{x}^{*}_{i}, \mathbf{x}^{*}_{j})$, $\mathbf{t} = (t_{1}, t_{2}, \ldots, t_{n})^{T}$, $R = \text{diag}(\mathbf{r}) \text{ with } \mathbf{r} = (r(\mathbf{x}_{1}),r(\mathbf{x}_{2}), \ldots, r(\mathbf{x}_{n}))^{T}$, and $R^{*} = \text{diag}(\mathbf{r}^{*}) \text{ with }\mathbf{r}^{*} = (r(\mathbf{x}^{*}_{1}), r(\mathbf{x}^{*}_{2}),\ldots, r(\mathbf{x}^{*}_{q}))^{T}$.\\

\noindent The \textsc{mlhgp} algorithm executes the following steps:

\begin{enumerate}
    \item Estimate a homoscedastic \textsc{gp}, $G_1$ on the dataset $\mathbf{\mathcal{D}} = \{(\mathbf{x}_{i}, t_{i})\}^{n}_{i = 1}$
    \item Given $G_1$, estimate the empirical noise levels for the training data using $z_i = \log(\text{var}[t_i, G_1(\textbf{x}_i, \mathcal{D})])$, where $\text{var}[t_i, G_1(\textbf{x}_i, \mathcal{D})] \approx \frac{1}{s} \: \sum_j^s \: 0.5 \: (t_i - t_i^j)^2$ and $t_i^j$ is a sample from the predictive distribution induced by the \textsc{gp} at $\mathbf{x}_i$, forming a new dataset 
    $\mathbf{\mathcal{D}'} = \{(\mathbf{x}_{i}, z_{i})\}^{n}_{i = 1}$.
    \item Estimate a second \textsc{gp}, $G_2$ on $\mathbf{\mathcal{D}'}$.
    \item Estimate a combined \textsc{gp}, $G_3$ on $\mathbf{\mathcal{D}}$ using $G_2$ to predict the logarithmic noise levels $r_i$.
    \item If not converged, set $G_3$ to $G_1$ and repeat.
\end{enumerate}

\noindent In essence, the defining characteristic of the \textsc{mlhgp} approach is that $G_1$ learns the latent function and $G_2$ learns the noise function. 

\subsection{Bayesian Optimisation with Aleatoric Noise Penalisation}

The heteroscedastic \textsc{bo} problem may be framed as

\begin{equation}
    \boldsymbol{x}^{*} = \operatorname*{arg\,min}_{\boldsymbol{x} \in \mathcal{X}} h(\boldsymbol{x}),
\end{equation}

\noindent where the black-box objective $h$, to be minimised has the form

\begin{equation}
    \label{noise_equation}
    h(\boldsymbol{x}) = \alpha f(\boldsymbol{x}) + (1 - \alpha)g(\boldsymbol{x}),
\end{equation}

\noindent where $f(\boldsymbol{x})$ is the black-box function of the principal objective i.e. the objective corresponding to classical \textsc{bo} where noise is not optimised, and $g(\boldsymbol{x})$ is the latent heteroscedastic noise function which governs the magnitude of the noise at a given input location $\boldsymbol{x}$. $\alpha$ is a parameter chosen, for the purposes of evaluation, by a domain expert that trades off the weight of the principal objective relative to the noise objective. It is worth noting that $\alpha$ is a parameter that affects only the evaluation of an algorithm and not the execution. The evaluation criteria, however, will dictate the optimal hyperparameters of the acquisition function.

\subsection{Heteroscedastic Acquisition Functions}

Extensions of the EI acquisition criterion \citep{1998_Jones} are investigated. The EI acquisition may be written in terms of the targets $t$ and the incumbent best objective function value, $\eta$, found so far as

\begin{equation}
    \text{EI}(\boldsymbol{x}) = \mathbb{E}\big[\,(\eta - t)_{+}\big] = \int_{-\infty}^{\infty}(\eta - t)_{+}\,p(t\,|\,\boldsymbol{x})\,dt,
\end{equation}

\noindent where $p(t\,|\,\boldsymbol{x})$ is the posterior predictive marginal density of the objective function evaluated at $\boldsymbol{x}$ and \hspace{3mm}$(\eta - t)_{+} \equiv \text{max}\,(0,\, \eta - t)$ is the improvement over the incumbent best objective function value $\eta$. Evaluations of the objective are noisy in all problems considered and so EI with plug-in \citep{2013_Picheny} is used, the plug-in value being the \textsc{gp} predictive mean \citep{2008_Vasquez}.

Two extensions to the EI criterion are proposed. The first is an extension of the augmented expected improvement (AEI) criterion

\begin{equation}
    \text{AEI}(\boldsymbol{x}) = \text{EI}(\boldsymbol{x}) \Bigg( 1 - \frac{\sigma_n}{\hphantom{e}\vphantom{\bigg|}\sqrt{\text{var}[t] + \sigma_n^{2}}\hphantom{e}} \Bigg),
\end{equation}

\noindent of \citet{2006_Huang}, where $\sigma_n$ is the fixed aleatoric noise level. AEI is introduced as a heuristic for the optimisation of noisy functions. EI is recovered in the case that $\sigma_n^{2} = 0$ and in the case that $\sigma_n^2 > 0$, AEI operates as a rescaling of the EI acquisition function, penalising test locations where the \textsc{gp} predictive variance is small relative to the fixed noise level $\sigma_n^2$. AEI is extended to the heteroscedastic setting by exchanging the fixed aleatoric noise level with the input-dependent one:

\begin{equation}
\label{eq:12}
    \text{HAEI}(\boldsymbol{x}) = \text{EI}(\boldsymbol{x}) \Bigg( 1 - \frac{\gamma\sqrt{r(\boldsymbol{x})}}{\hphantom{e}\vphantom{\big|}\sqrt{\text{var}[t] + \gamma^2r(\boldsymbol{x})}\hphantom{e}} \Bigg), 
\end{equation}

\noindent where $r(\boldsymbol{x})$ is the predicted aleatoric uncertainty at input $\boldsymbol{x}$ under the \textsc{mlhgp} and $\text{var}[t]$ is the predictive variance of the \textsc{mlhgp} at input $\boldsymbol{x}$. In this instance, $\gamma$ is defined to be a positive penalty parameter for regions with high aleatoric noise. \\

\begin{proposition}[Limit of Large Epistemic Uncertainty]
\label{prop:prop1}

The HAEI acquisition function reduces to EI when the ratio of epistemic uncertainty to aleatoric uncertainty is much greater than $\gamma^2$. \end{proposition}

\begin{proof}

Let $k = \frac{\text{var}[t]}{r(\boldsymbol{x})}$ denote the ratio of epistemic to aleatoric uncertainty at an arbitrary input location $\boldsymbol{x}$. Dividing the numerator and the denominator of the second term in the second factor of \autoref{eq:12} by $\sqrt{r(\boldsymbol{x})}$ yields

\begin{equation}
\label{eq:13}
    \text{HAEI}(\boldsymbol{x}) = \text{EI}(\boldsymbol{x})\Bigg( 1 - \frac{\gamma}{\hphantom{e}\vphantom{\big|}\sqrt{k + \gamma^2}\hphantom{e}} \Bigg).
\end{equation}

\noindent Taking the limit as $k$ tends to infinity and assuming finite $\gamma$

\begin{equation}
    \lim_{k \to \infty}\text{EI}(\boldsymbol{x})\Bigg( 1 - \frac{\gamma}{\hphantom{e}\vphantom{\big|}\sqrt{k + \gamma^2}\hphantom{e}} \Bigg) = \text{EI}(\boldsymbol{x}),
\end{equation}

\noindent recovers the EI acquisition. 

\end{proof}

\begin{proposition}[Limit of Large Aleatoric Uncertainty]\label{prop2}

The HAEI acquisition function goes to zero as the ratio of epistemic uncertainty to aleatoric uncertainty goes to zero. \end{proposition}

\begin{proof}

Taking the limit as $k$ tends to zero in \autoref{eq:13} yields 

\begin{equation}
    \lim_{k \to 0}\text{EI}(\boldsymbol{x})\Bigg( 1 - \frac{\gamma}{\hphantom{e}\vphantom{\big|}\sqrt{k + \gamma^2}\hphantom{e}} \Bigg) = 0.
\end{equation}

\end{proof}

\begin{remark} In the limit of large aleatoric uncertainty there is an approximation that is linear in $k$ for the HAEI scaling factor.
\end{remark}

Letting $S(k) = 1 - \frac{\gamma}{\sqrt{k + \gamma^2}}$ such that $\text{HAEI} = \text{EI}(\boldsymbol{x})S(k)$, consider the MacLaurin expansion of $S(k)$,

\begin{equation}
    S(k) = S(0) + S'(0)k + \frac{S''(0)}{2!}k^2 + \frac{S'''(0)}{3!}k^3 + \dotsc ,
\end{equation}

\noindent Dropping terms of $O(k^2)$ and higher we obtain

\begin{equation}
    S(k) \approx \frac{k}{2\gamma^2}.
\end{equation}

\noindent This approximation may be used when $k$ is small relative to $\gamma$ and could provide guidance in setting the $\gamma$ parameter if prior knowledge about $k$ and the desired trade-off is available between the principal and noise objectives. In \autoref{scalig} some insight is provided into the effect that different values of $\gamma$ will have on the scaling factor $S(k)$. \\

\begin{figure*}[t]
\centering
    \includegraphics[width=.7\textwidth]{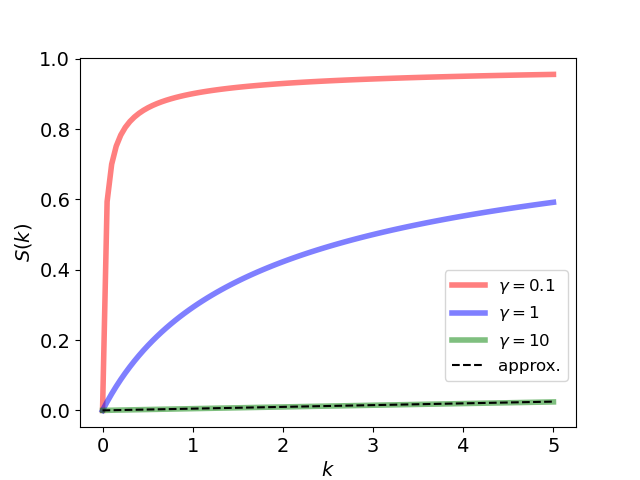}
    \caption{The HAEI scaling factor $S(k)$, now written as a function of $k$ for different values of $\gamma$. When $k$, the ratio of epistemic to aleatoric uncertainty is small, the scaling factor goes to zero to reflect the penalty for regions of high aleatoric uncertainty. The $\gamma$ parameter controls the decay rate of this penalty. Also shown is the linear approximation to the scaling factor for $\gamma = 10$.}
    \label{scalig}
\end{figure*}

In addition to HAEI, a simple modification to EI is proposed that explicitly penalises regions of the input space with large aleatoric noise. This acquisition function is termed aleatoric noise-penalised expected improvement (ANPEI) and denoted

\begin{equation}\label{eq:anpei_equation}
    \text{ANPEI} = \beta\text{EI}(\boldsymbol{x}) - (1 - \beta)\sqrt{r(\boldsymbol{x})},
\end{equation}

\noindent where $\beta$ is a scalarisation constant. In the multiobjective optimisation setting a particular value of $\beta$ will correspond to a point on the Pareto frontier. The advantages of both HAEI and ANPEI acquisition functions, in conjunction with the \textsc{mlhgp} surrogate model, are showcased in \autoref{experiments}.

\section{Experiments on Robustness to Aleatoric Uncertainty}
\label{experiments}

The experiments below are designed to test the performance of the heteroscedastic \textsc{bo} schemes against homoscedastic \textsc{bo} schemes as well as random search when the task section dictates that it is desirable to minimise (be robust to) aleatoric noise.

\begin{figure*}
\centering
\subfigure[Latent Function]{\label{fig:sin_1_first}\includegraphics[width=0.32\textwidth]{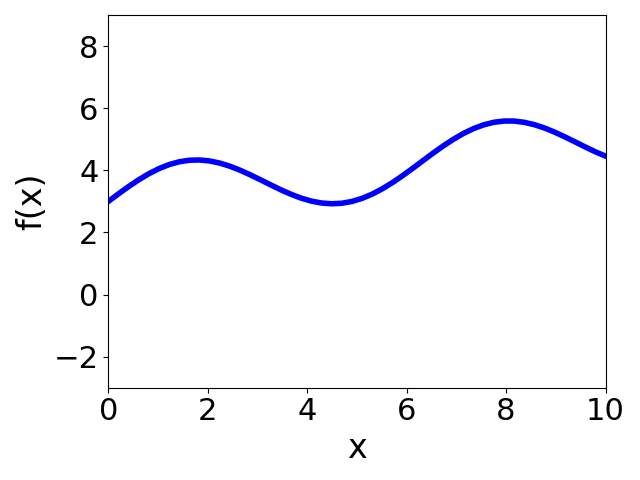}}
\subfigure[Noise Function ]{\label{fig:sin_2_first}\includegraphics[width=0.32\textwidth]{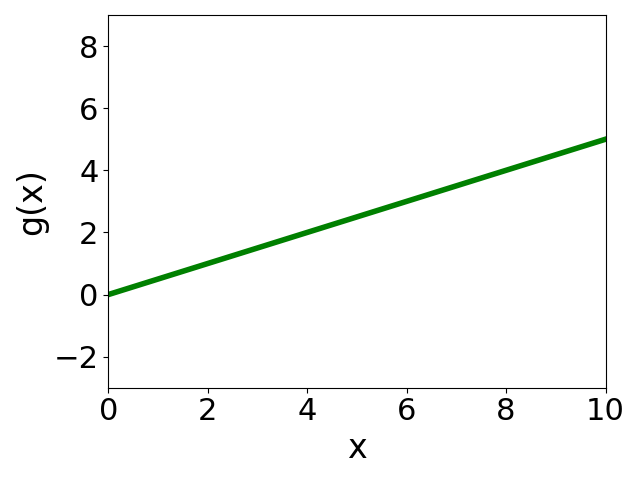}}
\subfigure[Objective Function]{\label{fig:sin_3_first}\includegraphics[width=0.32\textwidth]{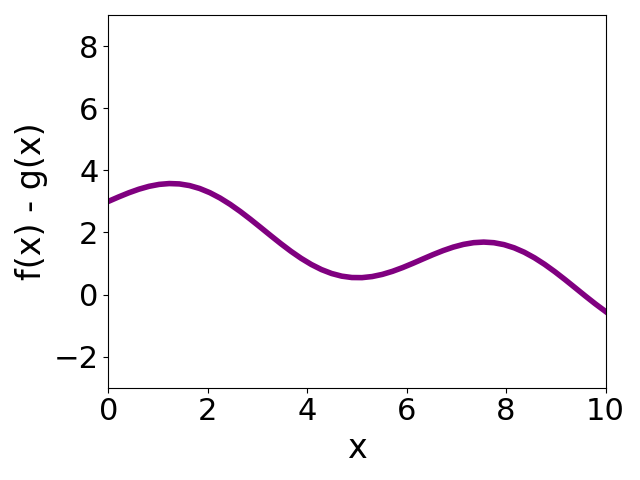}} 
\caption{Illustrative Toy Problem. The latent function in a) is corrupted with heteroscedastic Gaussian noise according to the function in b) where $g(x)$ is a constant multiplier of a sample from a standard Gaussian. The combined objective is given in c) and is obtained by subtracting the noise function from the latent function.}
\label{fig:sin_first}
\end{figure*}

\begin{figure*}[t]
\centering
    \includegraphics[width=.5\textwidth]{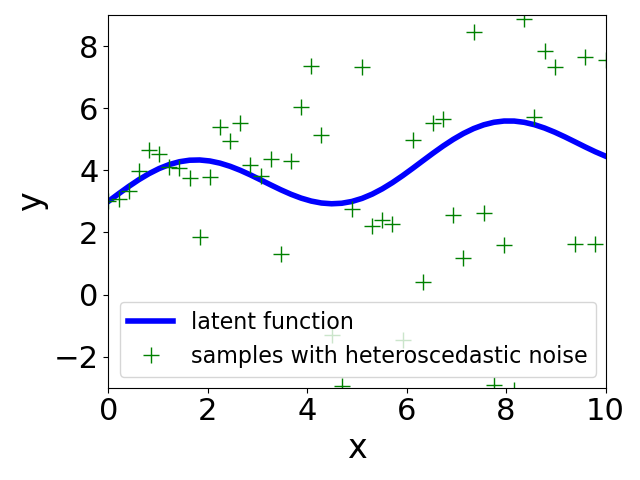}
    \caption{Noisy samples $y_i = f(x_i) + g(x_i)\epsilon$ from the heteroscedastic sin wave function.}
    \label{sin_wave_samples}
\end{figure*}

\subsection{Implementation}

Experiments were run using a custom NumPy \citep{2020_numpy} implementation of \textsc{gp} regression and \textsc{mlhgp} regression. All code to reproduce the experiments is available at \url{https://github.com/Ryan-Rhys/Heteroscedastic-BO}. The squared exponential kernel was chosen as the covariance function for both the homoscedastic \textsc{gp} as well as $G_1$ and $G_2$ of the \textsc{mlhgp}. Across all datasets, the lengthscales, $\ell$, of the homoscedastic \textsc{gp} were initialised to 1.0 for each input dimension. The signal amplitude $\sigma_{f}^{2}$ was initialised to a value of 1.0. The lengthscale, $\ell$, of $G_2$ of the \textsc{mlhgp} was initialised to 1.0 and the initial noise level of $G_2$ was set to 1.0. The EM-like procedure required to train the \textsc{mlhgp} was run for 10 iterations and the sample size required to construct the variance estimator producing the auxiliary dataset was 100. All standard error confidence bands are computed using 50 independent random seed initialisations. Hyperparameter values, including the noise level of the homoscedastic \textsc{gp}, were obtained by optimising the marginal likelihood using the SciPy implementation of the L-BFGS-B optimiser \citep{1997_Zhu}, taking the best of 20 random restarts. The objective function is

\begin{equation}
    \label{example1}
    h(x) = \alpha f(x) - (1 - \alpha)g(x)
\end{equation}

\noindent for the one-dimensional sin wave experiment which is a maximisation problem and as such has a subtractive penalty for regions of large noise. For the remaining experiments, which are minimisation problems, the objective is

\begin{equation}
    \label{example2}
    h(\boldsymbol{x}) = \alpha f(\boldsymbol{x}) + (1 - \alpha)g(\boldsymbol{x}).
\end{equation}

\noindent The sin wave and Branin-Hoo tasks are initialised with 25 and 100 data points respectively drawn uniformly at random within the bounds of the design space. The soil and FreeSolv experiments are initialised with 36 and 129 data points respectively drawn uniformly at random from the datasets. The $\alpha$ parameter is set to 0.5 for all experiments while $\beta$ is set to 0.5, $\frac{1}{11}$, 0.5 and 0.5 for the sin, Branin-Hoo, soil and FreeSolv experiments. The $\gamma$ parameter is set to 1, 500, 1 and 1 for the sin, Branin-Hoo, soil and FreeSolv experiments. Five acquisition functions were run in all experiments: random search, homoscedastic EI, AEI, HAEI and ANPEI. Homoscedastic EI is included as a baseline to demonstrate the difference consideration of aleatoric noise yields in the optimisation of the objective. AEI is included to demonstrate the difference consideration of heteroscedastic aleatoric noise yields and random search is included as a baseline as it is known to be highly competitive with \textsc{bo} in noisy settings.

\subsection{Heteroscedastic Sin Wave Function}

The objective function has the form 

\begin{equation}
    \label{example3}
    h(x) = f(x) - g(x),
\end{equation}

\noindent where $f(x) = \sin(x) + 0.2(x) + 3$ and $g(x) = 0.5(x)$. In this instance $\alpha$ from \autoref{example1} has a setting of $0.5$ but is omitted explicitly as the objectives have equal weight. Over the course of the experiment samples 

\begin{align*}
    y_i = f(x_i) + g(x_i)\epsilon, \: \: \: \: \: \:\epsilon \sim \mathcal{N}(0, 1),
\end{align*}

\noindent are observed. The problem setup is depicted in \autoref{fig:sin_first} and \autoref{sin_wave_samples}. The \textsc{bo} problem is constructed such that the first maximum in Figure~\autoref{fig:sin_1_first} is to be preferred as samples from this region of the input space will have low aleatoric noise. The black-box objective in Figure~\autoref{fig:sin_3_first} illustrates this trade-off. In \autoref{fig:sin_bayesopt} the performance of all surrogate model/acquisition function combinations is compared. The low aleatoric noise-seeking behaviour of HAEI and ANPEI on $g(x)$ is observed as well as their ability to optimise the composite objective $h(x)$.

\begin{figure*}
\centering
\subfigure[Best Objective Value Found so Far]{\label{fig:bo_1}\includegraphics[width=0.486\textwidth]{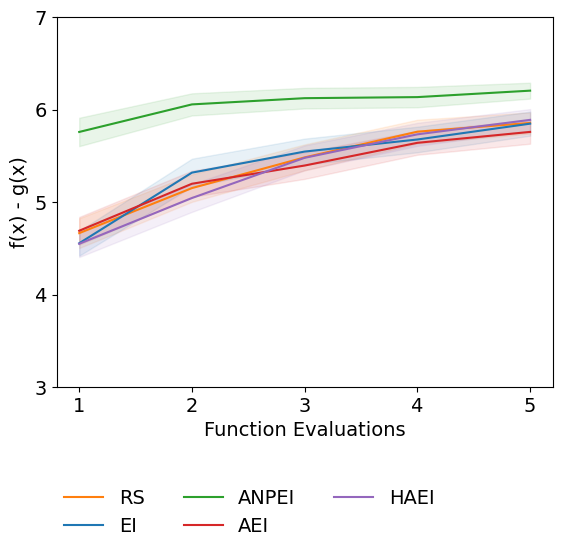}}
\subfigure[Lowest Aleatoric Noise Found so Far ]{\label{fig:bo_2}\includegraphics[width=0.501\textwidth]{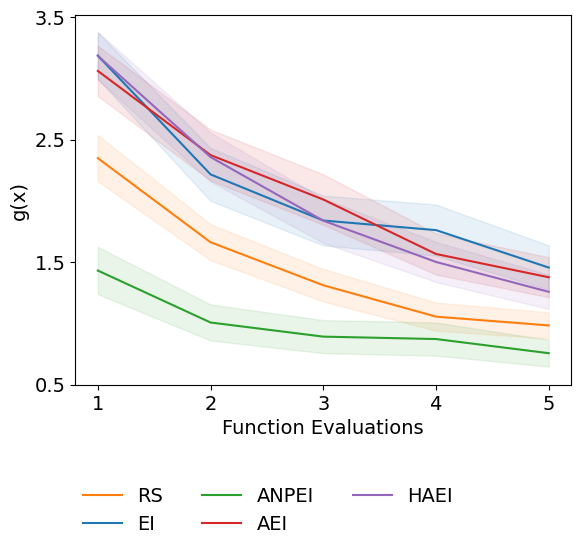}}
\caption{Comparison of heteroscedastic and homoscedastic \textsc{bo} on the sin wave problem. (a) shows the optimisation of $h(x) = f(x) - g(x)$ (higher is better) whereas (b) shows the values $g(x)$ obtained over the course of the optimisation of $h(x)$. This latter plot demonstrates the propensity of ANPEI to seek low aleatoric noise solutions.}
\label{fig:sin_bayesopt}
\end{figure*}

\subsection{Heteroscedastic Branin-Hoo Function}\label{het_bran_section}

In the second experiment the objective

\begin{align*}
    h(\boldsymbol{x}) = f(\boldsymbol{x}) + g(\boldsymbol{x})
\end{align*}

\noindent is considered with an additive penalty because the task is a minimisation problem and an $\alpha$ setting of $0.5$ for equal-weight objectives.

\begin{align}
\label{branin_eq}
f(\boldsymbol{x})=\frac{1}{51.95}\left[\left(\bar{x}_{2}-\frac{5.1 \bar{x}_{1}^{2}}{4 \pi^{2}}+\frac{5 \bar{x}_{1}}{\pi}-6\right)^{2}+\left(10-\frac{10}{8 \pi}\right) \cos \left(\bar{x}_{1}\right)-44.81\right]
\end{align}

\noindent with $\bar{x}_{1} = 15x_1 - 5$, $\bar{x}_{2} = 15x_2$ and $\boldsymbol{x} = (x_1, x_2)$ is the standardised Branin-Hoo function introduced in \cite{2013_Picheny}. The noise function $g(\boldsymbol{x})$ is in this instance

\begin{align}
\label{branin_noise_eq}
    g(\boldsymbol{x}) = 15 - 8x_{1} + 8x_{2}^2.
\end{align}

\noindent Samples are again generated according to

\begin{align*}
    y_i = f(\boldsymbol{x}_i) + g(\boldsymbol{x}_i)\epsilon, \: \: \: \: \: \:\epsilon \sim \mathcal{N}(0, 1).
\end{align*}

The problem setup is shown in \autoref{fig:branin} and the performance of all surrogate model/acquisition function pairs is depicted in \autoref{fig:branin_bayesopt}. The gulf in performance between the heteroscedastic and homoscedastic surrogate models is more pronounced in this case because the noise function is more severe relative to the sin wave problem.

\begin{figure*}
\centering
\subfigure[Latent Function]{\label{fig:branin_1}\includegraphics[width=0.32\textwidth]{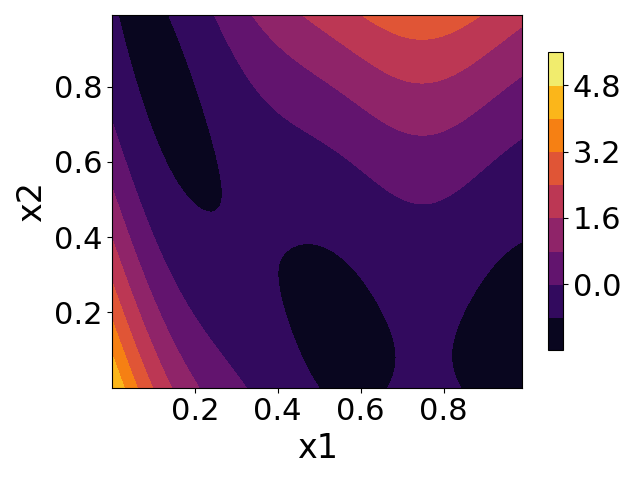}}
\subfigure[Non-linear Noise Function]{\label{fig:branin_2}\includegraphics[width=0.32\textwidth]{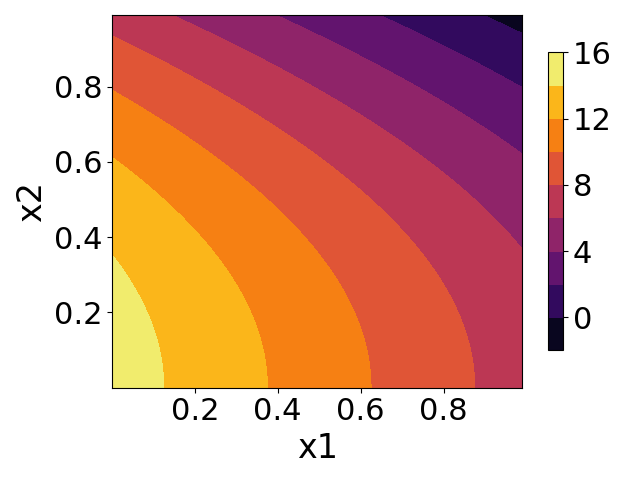}}
\subfigure[Objective Function]{\label{fig:branin_3}\includegraphics[width=0.32\textwidth]{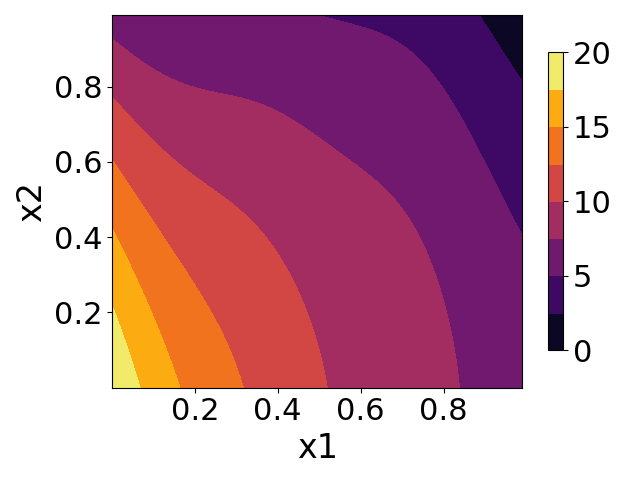}} 
\caption{Branin-Hoo Optimisation Problem. The latent function in a) is corrupted by heteroscedastic Gaussian noise function according to the function in b) The combined objective function is given in c) and is obtained by summing the functions in a) and b). The sum is required to penalise regions of large aleatoric noise because the objective is being minimised.}
\label{fig:branin}
\end{figure*}

\begin{figure*}
\centering
\subfigure[Best Objective Value Found so Far]{\label{fig:bo_1_branin}\includegraphics[width=0.49\textwidth]{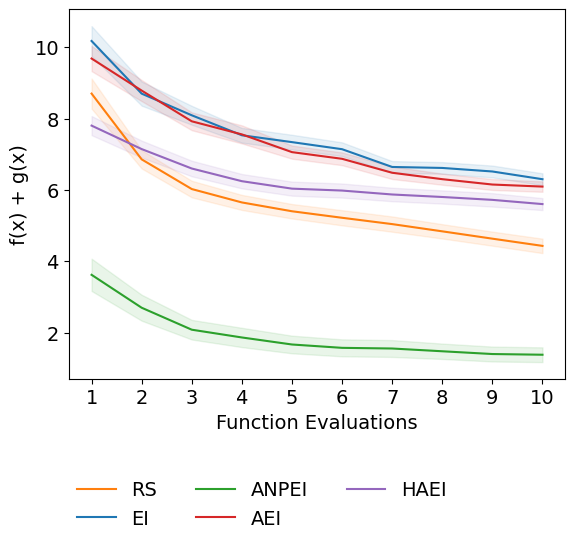}}
\subfigure[Lowest Aleatoric Noise Found so Far ]{\label{fig:bo_2_branin}\includegraphics[width=0.49\textwidth]{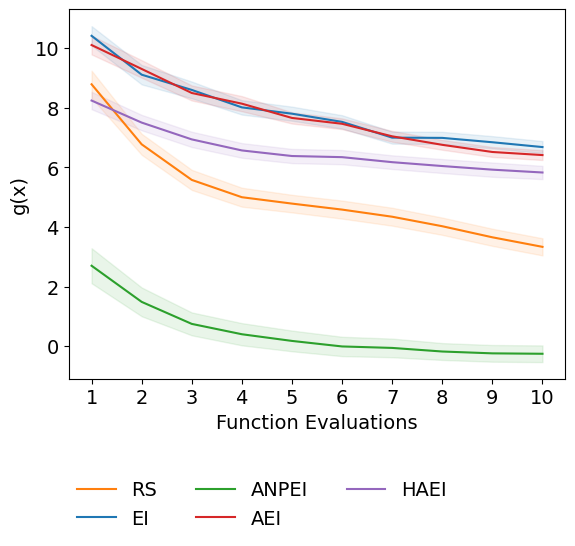}}
\caption{Comparison of heteroscedastic and homoscedastic \textsc{bo} on the Branin-Hoo problem. (a) shows the optimisation of $h(\boldsymbol{x}) = f(\boldsymbol{x}) + g(\boldsymbol{x})$ (lower is better) whereas (b) shows the values $g(\boldsymbol{x})$ obtained over the course of the optimisation of $h(\boldsymbol{x})$.}
\label{fig:branin_bayesopt}
\end{figure*}

\subsection{Soil Phosphorus Fraction Optimisation}

In this experiment the optimisation of the phosphorus fraction of soil is considered. Soil phosphorus is an essential nutrient for plant growth and is widely used as a fertiliser in agriculture. While the amount of arable land worldwide is declining, global population is expanding concomitantly with food demand. As such, understanding the availability of plant nutrients that increase crop yield is a topic worthy of attention. To this end, \citep{houGlobalDatasetPlant2018} have curated a dataset on soil phosphorus, relating phosphorus content to variables such as soil particle size, total nitrogen, organic carbon, and bulk density. The relationship between bulk soil density and the phosphorus fraction was chosen for study, the goal being to minimise the phosphorus content of soil subject to heteroscedastic noise. In lieu of performing a formal test for heteroscedasticity, evidence is provided that there is heteroscedasticity in the dataset by comparing the fits of a homoscedastic \textsc{gp} and the \textsc{mlhgp} in \autoref{fig:soil} and a predictive performance comparison based on negative log predictive density values is provided in Appendix~\ref{soil_het_demo}. 

In this problem, there is no access to a continuous-valued black-box function or a ground truth noise function. As such, the surrogate models were initialised with a subset of the data and the query locations selected by \textsc{bo} were mapped to the closest datapoints in the heldout data. The following kernel smoothing procedure was used to generate pseudo ground-truth noise values:

\begin{enumerate}[label=(\arabic*)]
    \item Fit a homoscedastic \textsc{gp} to the full dataset.
    \item At each point $x_i$, compute the corresponding squared error $s_i^2 = (y_i - \mu(x_i))^2$.
    \item Estimate variances by computing a moving average of the squared errors, where the relative weight of each $s_i^2$ is assigned with a Gaussian kernel.
\end{enumerate}

\noindent The performances of heteroscedastic and homoscedastic \textsc{bo} are compared in \autoref{fig:soil_bayesopt}. Given that regions of low phosphorus fraction coincide with regions of small aleatoric noise, an $\alpha$ value of $\frac{1}{6}$ was applied to the composite objective $h(x)$ to admit a finer granularity for distinguishing between degrees of low aleatoric noise in the solutions.

\begin{figure*}
\centering
\subfigure[Best Objective Value Found so Far]{\label{fig:bo_1_soil}\includegraphics[width=0.50\textwidth]{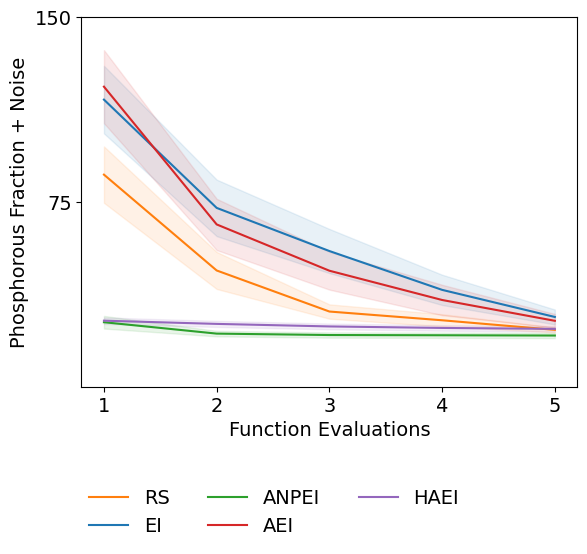}}
\subfigure[Lowest Aleatoric Noise Found so Far ]{\label{fig:bo_2_soil}\includegraphics[width=0.488\textwidth]{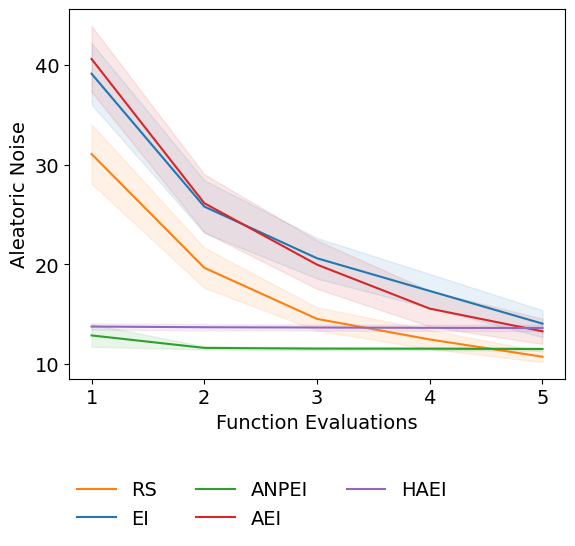}}
\caption{Comparison of heteroscedastic and homoscedastic \textsc{bo} on the soil phosphorus fraction optimisation problem. (a) shows the optimisation of $h(x) = f(x) + g(x)$ (lower is better) where $x$ is the dry bulk density of the soil. (b) shows the values $g(x)$ obtained over the course of the optimisation of $h(x)$.}
\label{fig:soil_bayesopt}
\end{figure*}

\subsection{Molecular Hydration Free Energy Optimisation}

A retrospective virtual screening experiment was performed with the aim of identifying molecules with favourable hydration free energy, a property important in determining the binding affinity of a drug candidate. Experiments were performed with an initialisation of 129 out of the 642 molecules in the FreeSolv dataset \citep{2014_Mobley, 2017_Duarte} over 10 iterations of data collection. Unlike the soil phosphorus fraction dataset, ground truth measurement error (aleatoric noise $g(\mathbf{x})$) values are available for the FreeSolv dataset. The remaining 513 molecules were reserved as a heldout set where at each iteration of data collection one of the heldout molecules was selected. Chemical fragments computed using RDKit \citep{rdkit} were used as the molecular representation based on the fact that these global features, unlike local Morgan fingerprints, act as good predictors of the hydration free energy. The fragment features were projected down to 14 components using principal component analysis, retaining more than 90\% of the variance on average across random trials. The results are shown in \autoref{fig:freesolv_bayesopt}. Compared to previous experiments, the noise is smaller in this instance relative to the magnitude of the hydration free energy (signal-to-noise ratio of approximately 10) and as such, the heteroscedastic modelling problem is more difficult, leading to only very marginal gains in obtaining low noise solutions. While ANPEI obtains the lowest objective function value over the \textsc{bo} trace, the results are unlikely to be statistically significant according to the standard error bands.

\begin{figure*}
\centering
\subfigure[Best Objective Value Found so Far]{\label{fig:bo_freesolv}\includegraphics[width=0.498\textwidth]{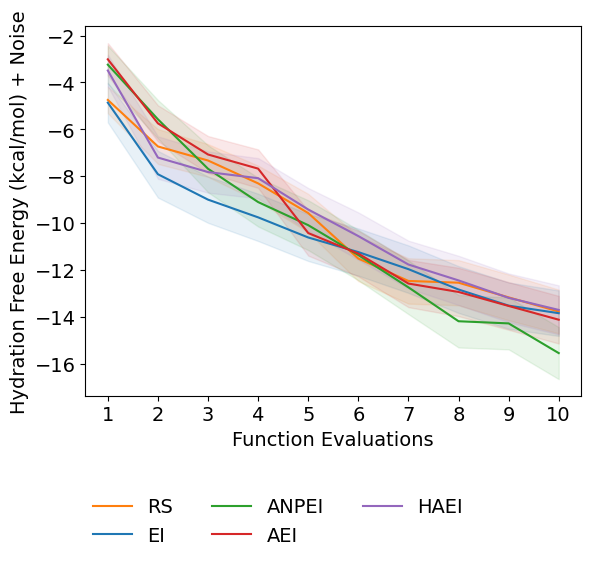}}
\subfigure[Lowest Aleatoric Noise Found so Far ]{\label{fig:bo_2_freesolv}\includegraphics[width=0.49\textwidth]{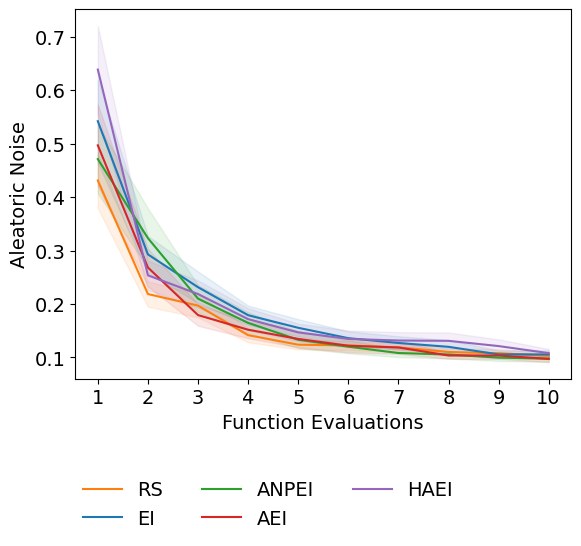}}
\caption{Comparison of heteroscedastic and homoscedastic \textsc{bo} on the FreeSolv hydration free energy optimisation problem. (a) shows the optimisation of $h(\boldsymbol{x}) = f(\boldsymbol{x}) + g(\boldsymbol{x})$ (lower is better) where $\boldsymbol{x}$ is the fragment set of molecular descriptors, $f(\boldsymbol{x})$ is the hydration free energy and $g(\boldsymbol{x})$ is the aleatoric noise. (b) shows the values $g(\boldsymbol{x})$ obtained over the course of the optimisation of $h(\boldsymbol{x})$.}
\label{fig:freesolv_bayesopt}
\end{figure*}

\subsection{Heteroscedastic Acquisition Function Hyperparameters}

The $\beta$ hyperparameter of ANPEI in \autoref{eq:anpei_equation} and the $\gamma$ hyperparameter of HAEI in \autoref{eq:13} are designed to modulate the avoidance of aleatoric noise in the acquisitions. In \autoref{fig:weight_comparison} we offer some intuition as to the effect of various settings of $\beta$ and $\gamma$ by examining the heteroscedastic Branin-Hoo function introduced in \autoref{het_bran_section}. The results demonstrate that the performance of the algorithms is strongly dependent on the setting of the $\beta$ hyperparameter for ANPEI whereas $\gamma$ is less influential on the performance of HAEI. It is worth noting that in Figure~\autoref{fig:bo_2_branin_hyper} if too large a value of $\gamma$ is chosen the principal objective $f(\mathbf{x})$ may be compromised through overly aggressive avoidance of aleatoric noise. In practice choosing the value of $\beta$ in line with the value of the evaluation criterion parameter $\alpha$ in \autoref{example2} is likely to be a sensible approach i.e. if the noise objective is more important relative to the principal objective by a factor of 10 then the value of $\beta$ should be $\frac{1}{11}$.

\begin{figure*}
\centering
\subfigure[ANPEI]{\label{fig:bo_branin_hyper}\includegraphics[width=0.49\textwidth]{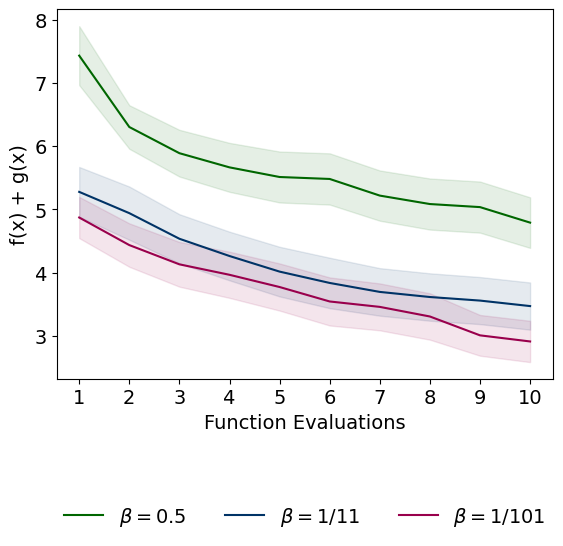}}
\subfigure[HAEI]{\label{fig:bo_2_branin_hyper}\includegraphics[width=0.49\textwidth]{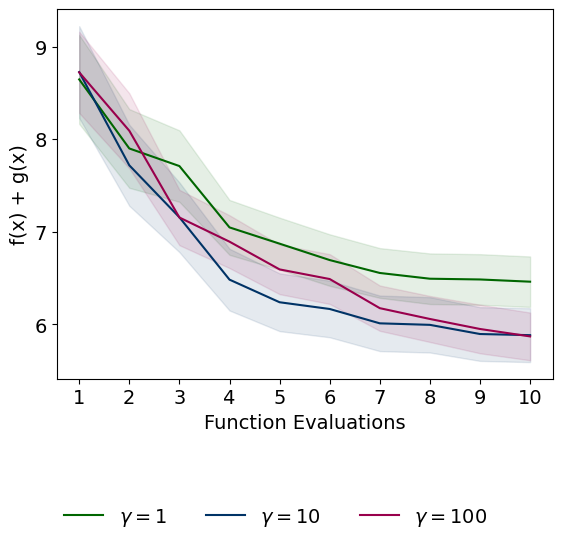}}
\caption{Performance of ANPEI and HAEI plotted for different values of the $\beta$ and $\gamma$ hyperparameters respectively. \textit{Smaller} values of $\beta$ encourage avoidance of regions of high aleatoric noise whilst \textit{larger} values of $\gamma$ encourage avoidance of regions of high aleatoric noise.}
\label{fig:weight_comparison}
\end{figure*}

\subsection{Summary of Robustness Experiments}

The experiments of this section provide strong evidence that modelling heteroscedasticity in \textsc{bo} is a useful approach for problems in which there is a strong degree of aleatoric noise present. The ANPEI acquisition tends to outperform HAEI on the majority of the tasks where there is a small degree of aleatoric noise whilst the acquisitions are more evenly matched when the extent of the aleatoric noise is high. The outstanding questions for these methods however, is how well they perform on tasks where heteroscedastic noise is not present. Such a situation may easily arise for real-world problems where the noise properties of the tasks are a priori unknown and as such, it is important to ascertain whether there is a deleterious effect on performance in noiseless and homoscedastic noise settings.

\section{Ablation Study}\label{first_ablation}

In this section an ablation study is performed where components of the ablation constitute different noise properties. The noiseless case is examined as a base task before adding first a homoscedastic noise component, and second, a heteroscedastic noise component. Additionally, The effect of the size of the initialisation grid on performance is examined in the heteroscedastic noise tasks.

\subsection{Noise Properties}

The ablation study makes use of three synthetic optimisation functions: The Branin-Hoo function, the Hosaki function and the Goldstein-Price function. The form of the Branin-Hoo function is the same standardised Branin-Hoo function introduced in \autoref{branin_eq} with heteroscedastic noise function given in \autoref{branin_noise_eq}. The Hosaki function, defined on the domain $x_1, x_2 \in [0, 5]$, is

\begin{equation}
\text{Hosaki}(x_1, x_2) = \Big(1 - 8x_1 + 7{x_1}^2 - \frac{7}{3} {x_1}^3 + \frac{1}{4} {x_1}^4\Big) {x_2}^2 \exp(-x_2).
\end{equation}

\noindent To facilitate the \textsc{gp} fit, the Hosaki function is subsequently standardised by its mean (0.817) and standard deviation (0.573). The noise function is

\begin{equation}\label{eq:hos_noise}
g_{\text{Hosaki}}(x_1, x_2) = 50 \cdot \frac{1}{(x_1 - 3.5)^2 + 2.5}\cdot \frac{1}{(x_2 - 2)^2 + 2.5}.
\end{equation}

\noindent The logarithmic Goldstein-Price function \cite{2013_Picheny} is

\begin{align*}
    \text{G-P}(x_1, x_2) = \frac{1}{2.427} \Bigg[\log \Big([1 + {(\bar{x}_{1} + \bar{x}_{2} + 1)}^2 (19 - 14\bar{x}_{1} + 3{\bar{x}_{1}}^2 - 14\bar{x}_{2} + 6\bar{x}_{1} \bar{x}_{2} + 3{\bar{x}_{2}}^2)] \\ [30 + {(2\bar{x}_{1} - 3\bar{x}_{2})}^2 (18 - 32\bar{x}_{1} + 12{\bar{x}_{1}}^2 + 48\bar{x}_{2} - 36\bar{x}_{1} \bar{x}_{2} + 27{\bar{x}_{2}}^2)]\Big) - 8.693\Bigg],
\end{align*}

\noindent where $\bar{x}_1 = 4x_1 - 2 $ and $\bar{x}_2 = 4x_2 - 2 $. The Goldstein-Price noise function is

\begin{equation}\label{eq_g-p_noise}
   g_{\text{G-P}}(x_1, x_2) = \frac{3}{2} \cdot \frac{1}{(x_1 - 0.5)^2 + 0.2} \cdot \frac{1}{(x_2 - 0.3)^2 + 0.3}.
\end{equation}

\noindent For clarity, only the results of the Hosaki function are presented in this chapter with the Branin-Hoo and Goldstein-Price results presented in Appendix~\ref{more_ablation}. The Hosaki function is visualised in \autoref{fig:hos_diagram}. The value of $\beta$ for ANPEI is set to 0.5 and the value of $\gamma$ is set to 500 for all Hosaki function experiments.

\begin{figure*}
\centering
\subfigure[Latent Function]{\label{fig:hos_1}\includegraphics[width=0.32\textwidth]{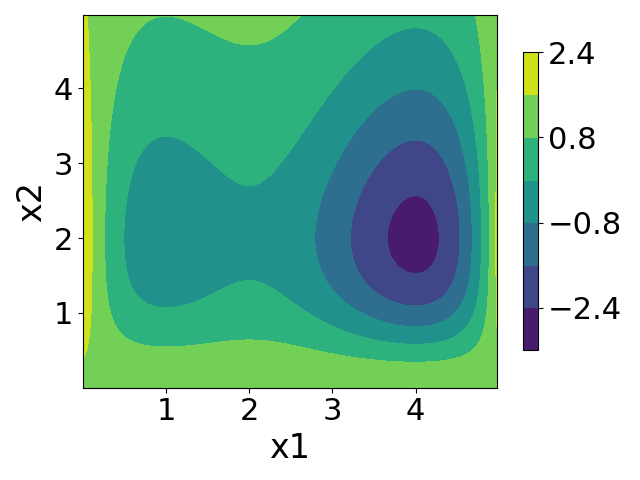}}
\subfigure[Noise Function ]{\label{fig:hos_2}\includegraphics[width=0.32\textwidth]{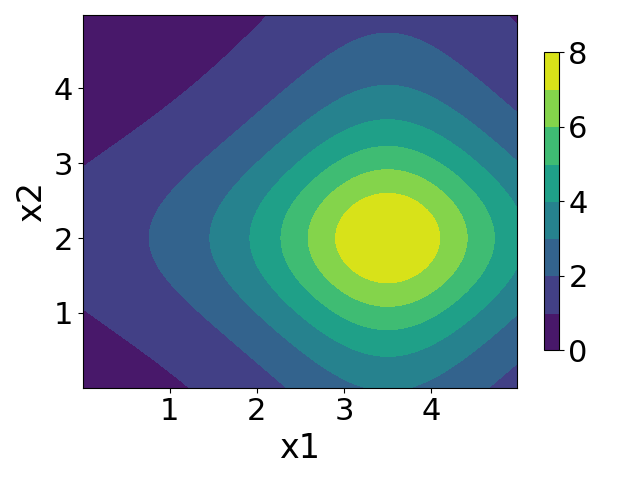}}
\subfigure[Objective Function]{\label{fig:hos_3}\includegraphics[width=0.32\textwidth]{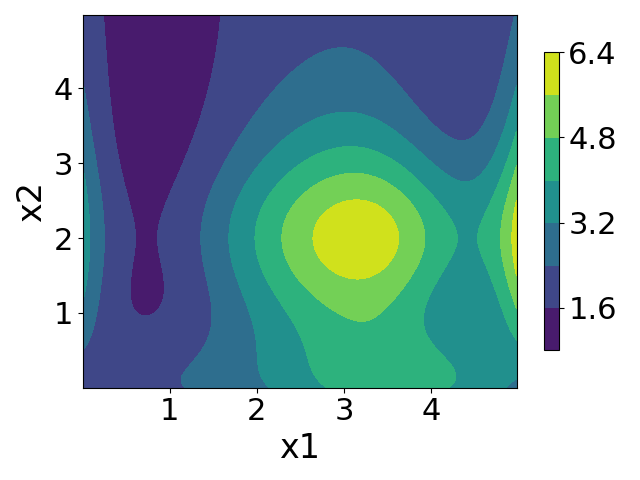}} 
\caption{(a) The latent Hosaki Function $f(\mathbf{x})$ together with (b) its heteroscedastic noise function $g(\mathbf{x})$ and (c) the objective function $f(\mathbf{x}) + g(\mathbf{x})$.}
\label{fig:hos_diagram}
\end{figure*}

\subsubsection{Noiseless Case}

In this case, the synthetic functions do not possess any observation noise and the optimisation function corresponds to the situation in Figure~\autoref{fig:hos_1}. 9 points sampled uniformly at random are used for initialisation and the results are displayed in \autoref{fig: noiseless_hos}. As expected, all \textsc{bo} methods outperform random search in the noiseless case. In this example it is unclear as to whether heteroscedastic \textsc{bo} methods are detrimental as HAEI performs best whereas ANPEI performs worst.

\begin{figure*}[t]
\centering
    \includegraphics[width=.7\textwidth]{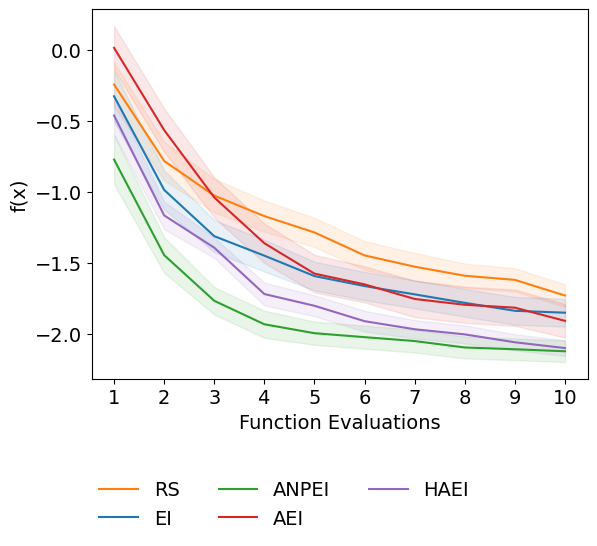}
    \caption{Hosaki function noiseless case. All \textsc{bo} methods outperform random search. HAEI performs best and ANPEI performs worst.}
    \label{fig: noiseless_hos}
\end{figure*}

\subsubsection{Homoscedastic Noise Case}

In this case the functions are subject to homoscedastic noise of the form $25\epsilon$, where epsilon is noise sampled from a standard Gaussian $\mathcal{N}(0, 1)$. The \textsc{gp} surrogates are again initialised with 9 points. The results are displayed in \autoref{homo_hos}. The \textsc{bo} methods perform worse in the homoscedastic noise case relative to the noiseless case although the rank order of the methods mirrors that of the noiseless case.

\begin{figure*}[t]
\centering
    \includegraphics[width=.7\textwidth]{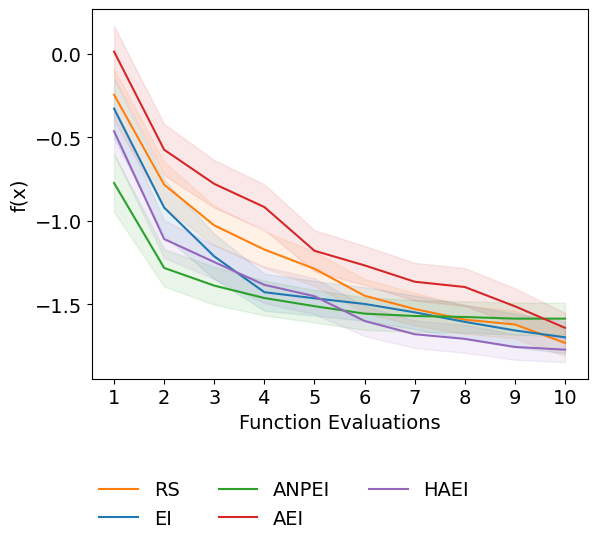}
    \caption{Hosaki function homoscedastic noise case. All \textsc{bo} methods outperform random search with HAEI the best and ANPEI the worst.}
    \label{homo_hos}
\end{figure*}

\subsubsection{Heteroscedastic Noise}

In the heteroscedastic noise case the Hosaki function is subject to the noise function given in \autoref{eq:hos_noise} and is visualised in \autoref{fig:hos_diagram}. 144 points were used to initialise the \textsc{gp} surrogates. The results are shown in \autoref{fig:hosaki_hetero}. In this instance, given that the extent of heteroscedastic noise is very strong (relative to the homoscedastic noise case), random search is highly competitive with the \textsc{bo} methods. ANPEI however, is the best-performing algorithm. The large number of initialisation points chosen for this experiment reflects one limitation of the heteroscedastic surrogate approach; for the \textsc{mlhgp} to effectively learn a decomposition of the function into signal and noise components it needs access to more samples. As such, this merits an investigation into the effect of the number of samples on the performance of the heteroscedastic acquisitions.

\begin{figure*}
\centering
\subfigure[Best Objective Value Found so Far]{\label{fig:bo_hosa}\includegraphics[width=0.498\textwidth]{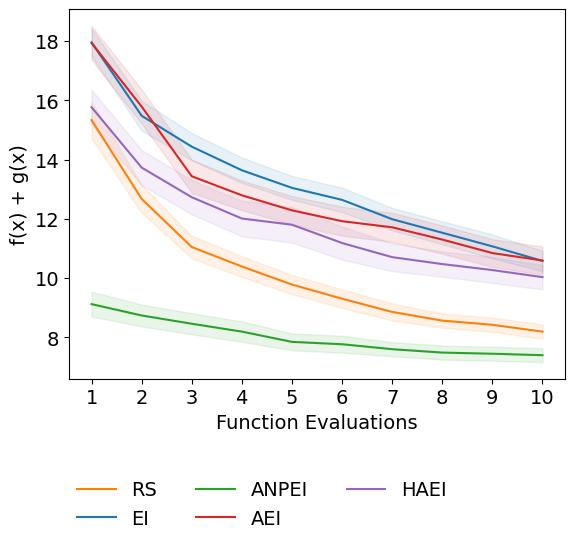}}
\subfigure[Lowest Aleatoric Noise Found so Far ]{\label{fig:bo_2_hosa}\includegraphics[width=0.482\textwidth]{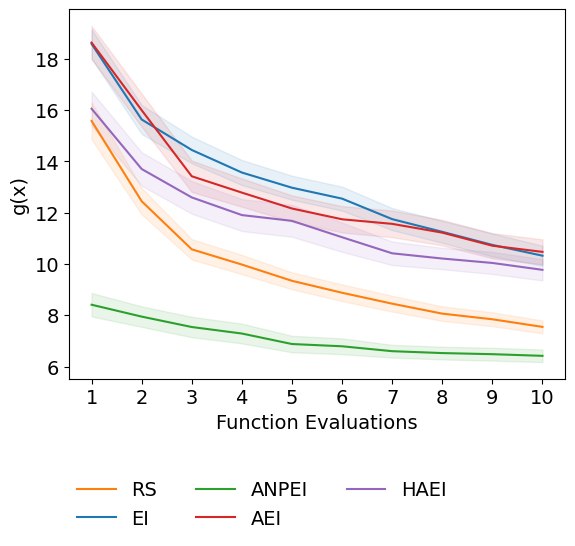}}
\caption{Comparison of heteroscedastic and homoscedastic \textsc{bo} on the heteroscedastic 2D Hosaki function. (a) shows the optimisation of $h(\boldsymbol{x}) = f(\boldsymbol{x}) + g(\boldsymbol{x})$ (lower is better) where $g(\boldsymbol{x})$ is the aleatoric noise. (b) shows the values $g(\boldsymbol{x})$ obtained over the course of the optimisation of $h(\boldsymbol{x})$.}
\label{fig:hosaki_hetero}
\end{figure*}

\subsection{Initialisation Set Size}

The effect of the size of the initialisation set on the heteroscedastic Branin-Hoo task is shown in \autoref{fig:init_grid}. The value of $\beta$ used for ANPEI is $\frac{1}{11}$ and the value of $\gamma$ used for HAEI is 500. The performance of the heteroscedastic acquisitions ANPEI and HAEI is observed to improve as the size of the initialisation set increases. In contrast, the homoscedastic methods EI and AEI do not improve on obtaining access to more samples as they are unable to model the heteroscedastic noise component of the task.

\begin{figure*}
\centering
\subfigure[9 Points]{\label{fig:init_grid_1}\includegraphics[width=0.315\textwidth]{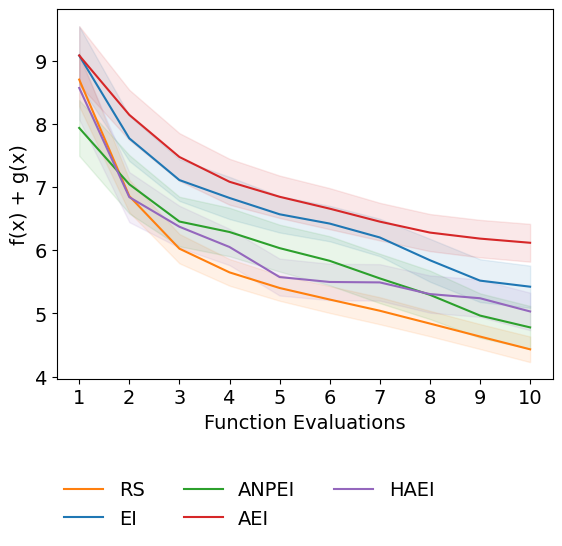}}
\subfigure[49 Points]{\label{fig:init_grid_2}\includegraphics[width=0.32\textwidth]{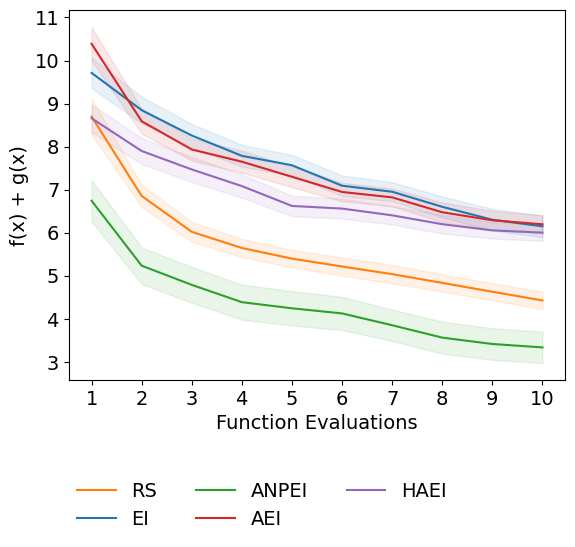}}
\subfigure[100 Points]{\label{fig:init_grid_3}\includegraphics[width=0.32\textwidth]{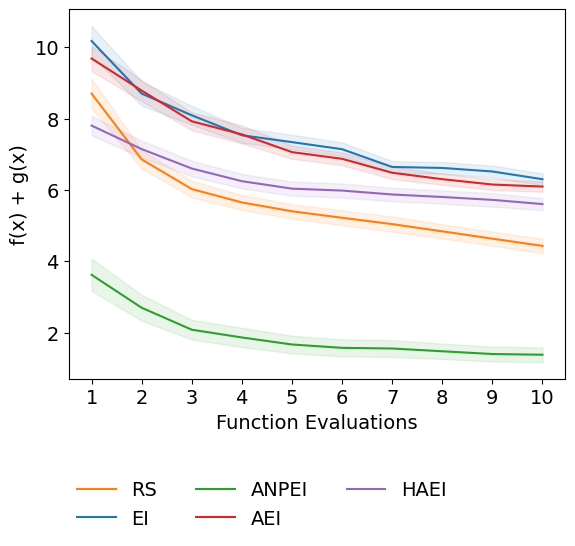}} 
\caption{The effect of the initialisation set size on the heteroscedastic Branin-Hoo function. The performance of heteroscedastic acquisitions ANPEI and HAEI increases as they are given access to more samples. An excess of samples do not help the homoscedastic \textsc{bo} methods as they are unable to model the heteroscedastic noise component.}
\label{fig:init_grid}
\end{figure*}

\subsection{Summary of Ablation Experiments}

Synthesising the results from the additional ablation experiments in Appendix~\ref{more_ablation} some trends may be observed:

\begin{enumerate}
    \item All \textsc{bo} methods outperform random search in the noiseless case and homoscedastic noise cases on aggregate, across the three synthetic functions.
    \item On aggregate, there is no significant difference between \textsc{bo} methods in the noiseless or homoscedastic noise cases (HAEI marginally outperforms ANPEI on 2/3 of the noiseless tasks and 2/3 of the homoscedastic noise tasks).
    \item The heteroscedastic acquisitions ANPEI and HAEI perform competitively on the noiseless and homoscedastic noise tasks, most likely because the \textsc{mlhgp} is capable of effecting nonstationary behaviour by "fantasising" heteroscedastic noise. As such, the \textsc{mlhgp} surrogate may be achieving enhanced flexibility relative to the homoscedastic \textsc{gp} in this setting.
    \item The heteroscedastic acquisitions tend to outperform other \textsc{bo} approaches on the heteroscedastic noise tasks although crucially this depends on the size of the initialisation set. In order to detect heteroscedastic noise, the \textsc{mlhgp} surrogate needs access to more samples relative to the noiseless and homoscedastic cases.
    \item ANPEI outperforms HAEI.
\end{enumerate}

In summary, the experiments would appear to show that there is no significant downside to employing a heteroscedastic surrogate and acquisition function on noiseless tasks or tasks with homoscedastic noise save for the increased training time for the model.

\section{Conclusions}
\label{het_conc}

This chapter has presented an approach for performing \textsc{bo} with the explicit goal of minimising aleatoric noise in the suggestions. It is posited that such an approach can prove useful for the natural sciences in the search for molecules and materials that are robust to experimental measurement noise. The synthetic function ablation study highlights no particular downside to the use of the \textsc{mlhgp} in conjunction with ANPEI or HAEI in cases where the noise structure of the problem is a priori unknown i.e. the black-box optimisation problem is either noiseless or homoscedastic. Nonetheless, it is anticipated that this type of approach may be particularly relevant for the experimental natural sciences where noiseless objectives or those with homoscedastic noise are highly uncommon. In terms of concrete recommendations on when to apply the algorithm, the best performance is foreseen in situations where the user has access to a moderately-sized initialisation set in order to provide the \textsc{mlhgp} with enough samples to distinguish heteroscedastic noise from intrinsic function variability. There are a number of possible extensions to the current approach which may facilitate its application to high-dimensional datasets and act as promising avenues for future work:

\begin{enumerate}[label=(\arabic*)]
    \item \textbf{Surrogate Model:} One disadvantage of the \textsc{mlhgp} model is the lack of convergence guarantees for the EM-like procedure required for fitting. Various other forms of heteroscedastic \textsc{gp} exist \citep{2005_Le, 2018_Binois, 2017_Almosallam, 2011_Munoz, 2017_Wangb, 2012_Wang, 2020_Zhang} and have demonstrated success in modelling applications \citep{2018_Rodrigues, 2018_Tabor, 2020_Rogers, 2019_Wang}. Of particular interest for real-world problems are scalable heteroscedastic \textsc{gp}s \citep{2019_Wang_gp, 2020_Liu} which could circumvent the computationally-intensive bottleneck of fitting multiple exact \textsc{gp}s as a subroutine of the \textsc{mlhgp} \textsc{bo} procedure.
    \item \textbf{Advances in Surrogate Model Machinery}: Advances in areas such as efficient sampling of \textsc{gp}s \citep{2020_Wilson} are liable to yield improvements to sampled-based acquisition functions such as Thompson sampling \citep{1933_Thompson}, while fully Bayesian approaches to hyperparameter estimation for sparse \textsc{gp}s \citep{2019_Lalchand} are liable to yield improvements in model fitting procedures.
    \item \textbf{Scalable BO:} Scalable \textsc{bo} can also be enabled via dimensionality reduction techniques \citep{2020_Moriconi, 2019_Candelieri, 2021_Grosnit}. Such approaches, when combined with efficient libraries \citep{2019_Balandat, 2019_dragonfly} could facilitate heteroscedastic \textsc{bo} in high-dimensional settings.
    \item \textbf{Acquisition Function Optimisation:} Recent developments in acquisition function optimisation including Monte Carlo reformulations \citep{2018_Wilson, 2020_Grosnit}, compositional optimisers \citep{2020_Tutunov, 2020_Grosnit}, and tight relaxations \citep{2020_Schweidtmann} of common acquisition functions have the potential to yield gains in empirical performance.
    \item \textbf{Data Transformation:} Input-warping \citep{2020_Wiebe} and output transformations \citep{2020_Rivers} have recently shown success when working with heteroscedastic datasets.
    \item \textbf{Exploration in the Noise Objective:} Incorporating exploration in the noise objective in the multi-objective setting as in \citet{2013_Kuindersma}.
\end{enumerate}

\noindent Lastly, a further use-case of the machinery developed in this paper is obtained by turning the noise minimisation problem into a noise maximisation problem. As an example, in materials discovery, we may derive benefit from being antifragile \citep{2012_taleb} towards (i.e. derive benefit from) high aleatoric noise. In an application such as the search for performant perovskite solar cells, we are faced with an extremely large compositional space, with millions of potential candidates possessing high aleatoric noise for identical reproductions \citep{2019_Zhou}. In this instance it may be desirable to guide search towards a candidate possessing a high photoluminescence quantum efficiency with high aleatoric noise. If the cost of repeating material syntheses is small relative to the cost of the search, the large aleatoric noise will present opportunities to synthesise materials possessing efficiencies far in excess of their mean values.

\nomenclature[Z-MD]{MD}{Molecular Dynamics}
\nomenclature[Z-Het-GP]{Het-GP}{Heteroscedastic Gaussian Process}
\nomenclature[Z-AEI]{AEI}{Augmented Expected Improvement}
\nomenclature[Z-ANPEI]{ANPEI}{Aleatoric Noise-Penalised Expected Improvement}
\nomenclature[Z-MCMC]{MCMC}{Markov Chain Monte Carlo}
\nomenclature[Z-ARD]{ARD}{Automatic Relevance Determination}
\nomenclature[Z-SGP]{SGP}{Sparse Gaussian Process}
\nomenclature[Z-RMSLE]{RMSLE}{Root Mean Square Log Error} %
\chapter{Conclusion}

\ifpdf
    \graphicspath{{Conclusion/Figs/Raster/}{Conclusion/Figs/PDF/}{Chapter3/Figs/}}
\else
    \graphicspath{{Conclusion/Figs/Vector/}{Conclusion/Figs/}}
\fi

\section{Summary of Contributions}

The goals of this thesis were first, to examine new use-cases for the existing \textsc{gp} framework in modelling scientific data and second, to extend current \textsc{gp}
methodology and software implementations to tackle a broader range of scientific modelling problems. Below, the chapter-by-chapter contributions are summarised with particular attention given to derivative works that have used or built on ideas and results introduced in the papers authored as part of the thesis.

\begin{itemize}
    \item In Chapter 3, \textsc{gp}s are used to infer the latent lightcurves of the Seyfert galaxy Mrk 335. Bayesian model selection is used to quantitatively compare choices of the \textsc{gp} kernel and the efficacy of the \textsc{gp} model is assessed via simulation. \textsc{gp} modelling of the observational data from Mrk 335 together with cross-correlation analysis provides weak evidence for a lag feature at high frequency with potential implications for the development of future accretion disk theories. Bayesian model selection over kernels is also employed in \cite{2022_Covino}, where the authors use \textsc{gp}s to detect periodicities in the quasar SDSS J025214.67-002813.7. The authors suggest that the rational quadratic kernel may outperform the squared exponential in their application due to its ability to pick up meaningful correlations for very long lags. 
    
    In \cite{2022_Lewin} the authors again use Bayesian model selection over kernels for X-ray reverberation mapping of the Seyfert galaxy Ark-564. The authors observe, similarly to \cite{2021_Mrk}, that as a component of time lag modelling, the rational quadratic and Matérn kernels outperform the squared exponential kernel in the marginal likelihood metric. Through their analysis, the authors constrain the fundamental black hole parameters and make a case for the use of \textsc{gp} models in the development of future theoretical reverberation models. 
    
    Lastly, in \cite{2022_Cackett} the authors conduct a frequently-resolved lag analysis of the AGN lightcurves across the full UV/optical range, obtaining results that are consistent with \cite{2021_Mrk} for the Seyfert galaxy NGC 5548. This work highlights that \textsc{gp} approaches are already being used as a point of comparison for alternative statistical inference methods.
    
    \item In Chapter 4, the software library GAUCHE  is introduced which extends the \textsc{gp} framework to operate on molecular and chemical reaction representations (\href{https://leojklarner.github.io/gauche/}{https://leojklarner.github.io/gauche/}). Specifically, the library provides support for graph, string and bit vector representations of molecules and reactions. GAUCHE subsumes an earlier TensorFlow version of the library, FlowMO by the same authors, \citep{2020_flowmo}, available at \href{https://github.com/Ryan-Rhys/FlowMO}{https://github.com/Ryan-Rhys/FlowMO}, which contains the first open-source implementation of the \textsc{gp}-Tanimoto model combination used in Chapter 5 as well as a wrapper around the GraKel library \citep{2020_GraKel} which, for the first time, makes a wide range of graph kernels available to be used in conjunction with \textsc{gp}s in a modern machine learning framework supporting GPU acceleration.
    
    \cite{2021_Broccard} performs an extensive analysis of the introduced \textsc{gp}-Tanimoto model in the context of \textsc{gp} regression and \textsc{bo}. The author proves the positive-definiteness of the generalised Tanimoto kernel and conducts experiments on the photoswitch dataset introduced in Chapter 5, showing that the \textsc{gp}-Tanimoto model outperforms popular kernels such as the squared exponential and Matérn kernels, likely due to the fact that it possesses only a single hyperparameter (the signal amplitude) and hence optimisation is more stable.
    
    In \cite{2021_Deshwal}, the authors leverage molecular kernels and data processing functionality in \cite{2020_flowmo} to enhance molecule generation architectures featuring deep generative models.
    
    In \cite{2022_Rankovic}, the authors use GAUCHE to perform \textsc{bo} for reaction screening, paying particular attention to optimisation over the space of reaction additives.
    
    \item In Chapter 5, the tools made available in GAUCHE, namely the \textsc{gp}-Tanimoto model used as a component of a \textsc{mogp}, are put to use in discovering novel photoswitch molecules. The model is trained on a curated dataset of photoswitch molecules and is subsequently used to screen a set of candidates satisfying a pre-specified set of performance criteria related to photoswitch use in light-emitting diodes (LEDs). 
    
    In \cite{2021_Mukadum}, the authors introduce an approach for active learning to prioritise molecules for comparatively expensive DFT calculations. The authors make use of similar visualisation techniques for photoswitch space and arrive at similar conclusions regarding the appropriateness of different molecular representations for wavelength prediction.
    
    \item In Chapter 6, the \textsc{bo} framework is extended to incorporate penalisation of experimental measurement noise (heteroscedastic noise). The experiments show that the methodology requires a larger initialisation set relative to standard \textsc{bo} in order to fully enable the desired noise penalisation.
    
    In \cite{2021_Makarova}, the authors introduce a complementary model to \cite{2021_Griffiths} which operates by repeating measurements at the same input location in order to obtain noise estimates. The model of \cite{2021_Makarova} is likely to be useful in scenarios where the cost of repeating noisy measurements is cheap relative to querying a new input location whereas the model introduced in \cite{2021_Griffiths} is likely to be useful if repeat measurements are as expensive as measurements at a different location.
\end{itemize}

\section{Future Work}

There is a long history of \textsc{gp}s being used to model scientific data with the first recorded instance being astronomer T.N. Thiele using \textsc{gp}s for time series analysis in 1880 \citep{1981_Lauritzen}. However, unlike deep learning, some of the more recent advances in \textsc{gp} and \textsc{bo} machinery \citep{2022_Garnett} have yet to be ported to the natural sciences. As an example, although the strengths of \textsc{gp} modelling for astrophysical time lag analysis were realised as early as 1992, there was no knowledge of an automated mechanism for learning the kernel hyperparameters through marginal likelihood optimisation and so it was necessary to specify hyperparameters by hand, significantly complicating the fitting procedure \citep{1992_Rybicki}. Similarly, there is anecdotal evidence that cheminformaticians have been known to abandon \textsc{gp} models due to Cholesky decomposition errors when training \textsc{gp}s on molecular representations with standard kernels defined over continuous input spaces. Although the range of potential applications for \textsc{gp}s in the sciences is vast, ranging from genetics  \citep{2011_Kalaitzis, 2021_Bintayyash} to protein modelling \citep{2022_Hie}, some avenues of future work relevant for astrophysics and chemistry, the applications considered in this thesis, are given below together with suggestions for research into adaptations of \textsc{gp} and \textsc{bo} machinery that may be particularly relevant for scientists.

\subsection{GPs in Astrophysics}

Unlike other areas of the natural sciences, there is an abundance of astrophysical data consisting of one-dimensional time series where the noise process is well understood. Much of the knowledge of the physical process can hence be incorporated into the \textsc{gp} model in order to improve the resulting fit, for example, by specifying known observation noise prior to optimisation of the kernel hyperparameters. As such, performing inference over latent functions for which only irregularly-sampled observations are available, is a task very well-suited to \textsc{gp} models. In the future, given the level of structure present in the data, it may be appropriate to use more sophisticated \textsc{gp} models to identity and capture new properites of the data. Example applications include modelling non-stationarity with deep \textsc{gp}s \citep{2013_Damianou} or transformed \textsc{gp}s \citep{2020_Maronas}, as well as the use of spectral mixture kernels \citep{2013_Wilson} to detect periodicities.

\subsection{GPs in Chemistry}

There is growing excitement in the chemistry community about machine learning approaches for predicting molecular properties. The 1981 Nobel Laureate in Chemistry, Roald Hoffmann, even went so far as to speculate the following about the future of molecular machine learning and quantum chemistry:

\begin{displayquote}
In view of the progress of machine learning and neural networks, it is likely that these two tools will compete efficiently - in quality, in cost - with the best quantum chemistry tools in the near future. Then the community of number-oriented quantum chemists will face a dramatic problem. Will their function be relegated to providing reliable training data sets for the production of improved neural networks? Or will they follow the destiny of super-market cashiers these days and that of taxi drivers tomorrow?
\end{displayquote}

The following Twitter counter-commentary by Max Welling, however, sheds light on some of the limitations and opportunities for machine learning approaches in molecular property prediction:

\begin{displayquote}
Interesting statement by Roald Hoffmann. While this is what happened to CV and NLP, it may not happen like that to chemistry. ML will become a very powerful tool for the computational chemist, but data is expensive and the microscopic equations are known! So inductive bias will remain key! \footnote{Abbreviations used in the original tweet due to the Twitter character limit have been expanded for clarity.}
\end{displayquote}

\noindent In particular, Welling's point about the expense of real-world data remains a valid concern. From a modelling standpoint, the ability to fit small data is a point in favour of the use of \textsc{gp}s in place of deep learning architectures. In terms of generalising beyond the domain of the training data, however, inductive bias does remain a key concern for both deep learning and \textsc{gp} models. While \textsc{gp}s have lagged behind deep learning in terms of incorporating inductive biases, there has been recent progress in learning invariances through the marginal likelihood \citep{2018_Wilk, 2021_Verma} as well as considering causal mechanisms \citep{2020_Aglietti, 2021_Aglietti}. While it remains to be seen how successful deep learning or \textsc{gp} approaches will be at encoding the inductive biases present in the microscopic equations of chemistry, the incorporation of inductive biases and causal mechanisms into machine learning models will no doubt be useful across many scientific applications \citep{2022_Kalinin}.

\subsection{GPs in Scientific Experiments}

Perhaps one of the most challenging obstacles to the adoption of machine learning in the laboratory is convincing an experimentalist to use it. In particular, when focussing on new areas of chemical or materials space for example, it can be challenging to specify an appropriate \textsc{bo} scheme upfront. The ability to do so would entail some knowledge of the underlying black-box function and noise processes for the properties of interest. While some experimental groups have been very successful in applying \textsc{bo} methodology for laboratory experiments \citep{2022_Jorayev}, this has often been paired with an understanding of the design space. As such, for unexplored domains it will become important to develop high-fidelity offline simulators to benchmark \textsc{bo} schemes and/or to develop means for counterfactual evaluation.

\begin{spacing}{0.9}

\bibliographystyle{apalike}
\cleardoublepage
\bibliography{References/references} %

\end{spacing}

\begin{appendices} %

\chapter{Modelling Black Hole Signals with Gaussian Processes}

\section{Additional Graphical Tests for Identifying the Flux Distribution}
\label{dist_tests}

In \autoref{PP Plots} probability-probability (PP) plots and empirical cumulative distributions functions (ECDFs) are shown as graphical distribution tests for Gaussianity. It may be observed qualitatively that both X-ray band log count rates and UVW2 flux are well-modelled by a Gaussian distribution.

\begin{figure}[h!]
\centering
\subfigure[PP plot for X-ray log count rates]{\label{fig:4pt1}\includegraphics[width=0.49\textwidth]{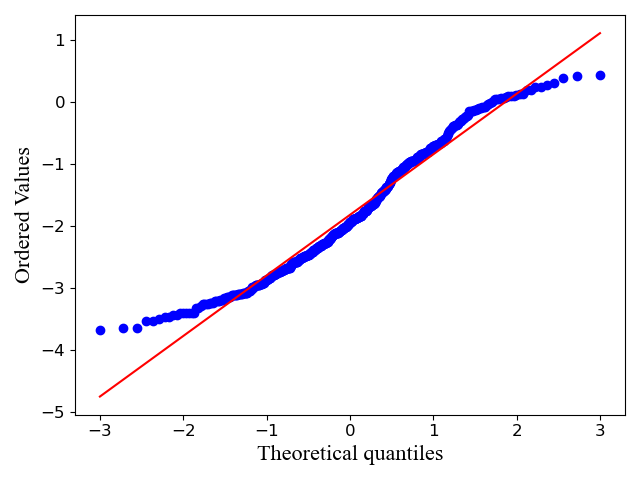}}
\subfigure[PP plot for UVW2 flux]{\label{fig:4pt2}\includegraphics[width=0.49\textwidth]{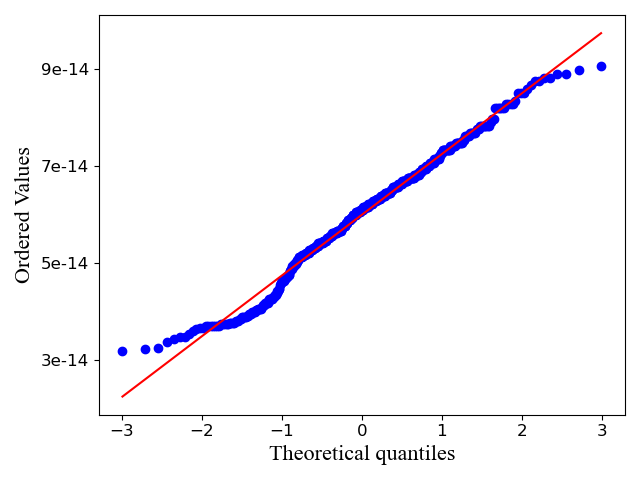}}
\subfigure[ECDF for X-ray log count rates]{\label{fig:4pt3}\includegraphics[width=0.49\textwidth]{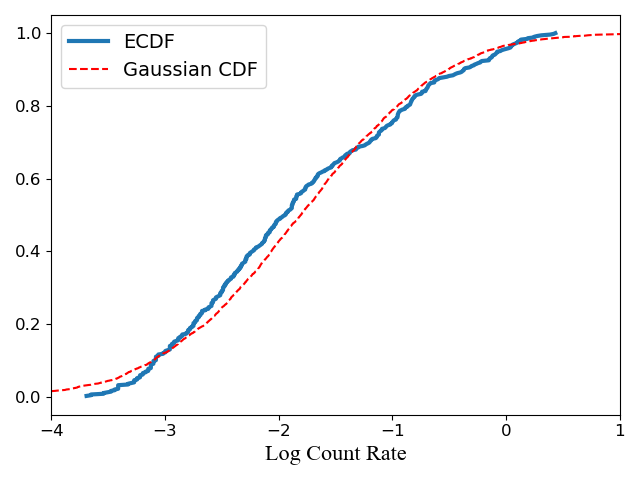}}
\subfigure[ECDF for UVW2 flux]{\label{fig:4pt4}\includegraphics[width=0.49\textwidth]{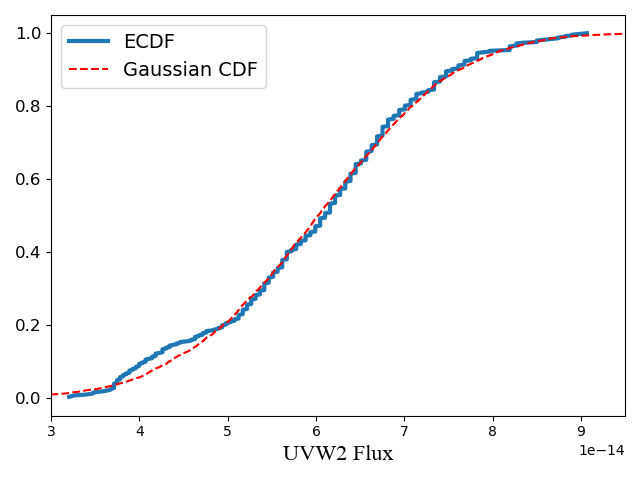}}  
\caption{PP plots and ECDFs for X-ray log count rates and UVW2 flux, graphical tests of Gaussianity. In the case of the PP plots, proximity to the line is an indicator of Gaussianity. In the case of the ECDF plots, resemblance to the cumulative distribution function of a Gaussian is indicative of Gaussianity. The  figures above were generated by Douglas Buisson.}
\label{PP Plots}
\end{figure}

\section{Spectral Properties of the Examined Kernels}
\label{kern_rat}

The autocorrelation functions, log autocorrelation functions and PSDs are illustrated for the Matérn, squared exponential and rational quadratic kernels in \autoref{kern_acfs}. The figures were generated by Douglas Buisson.

\begin{figure}[h!]
\centering
\subfigure[Kernel autocorrelation functions]{\label{fig:4k}\includegraphics[width=0.4\textwidth]{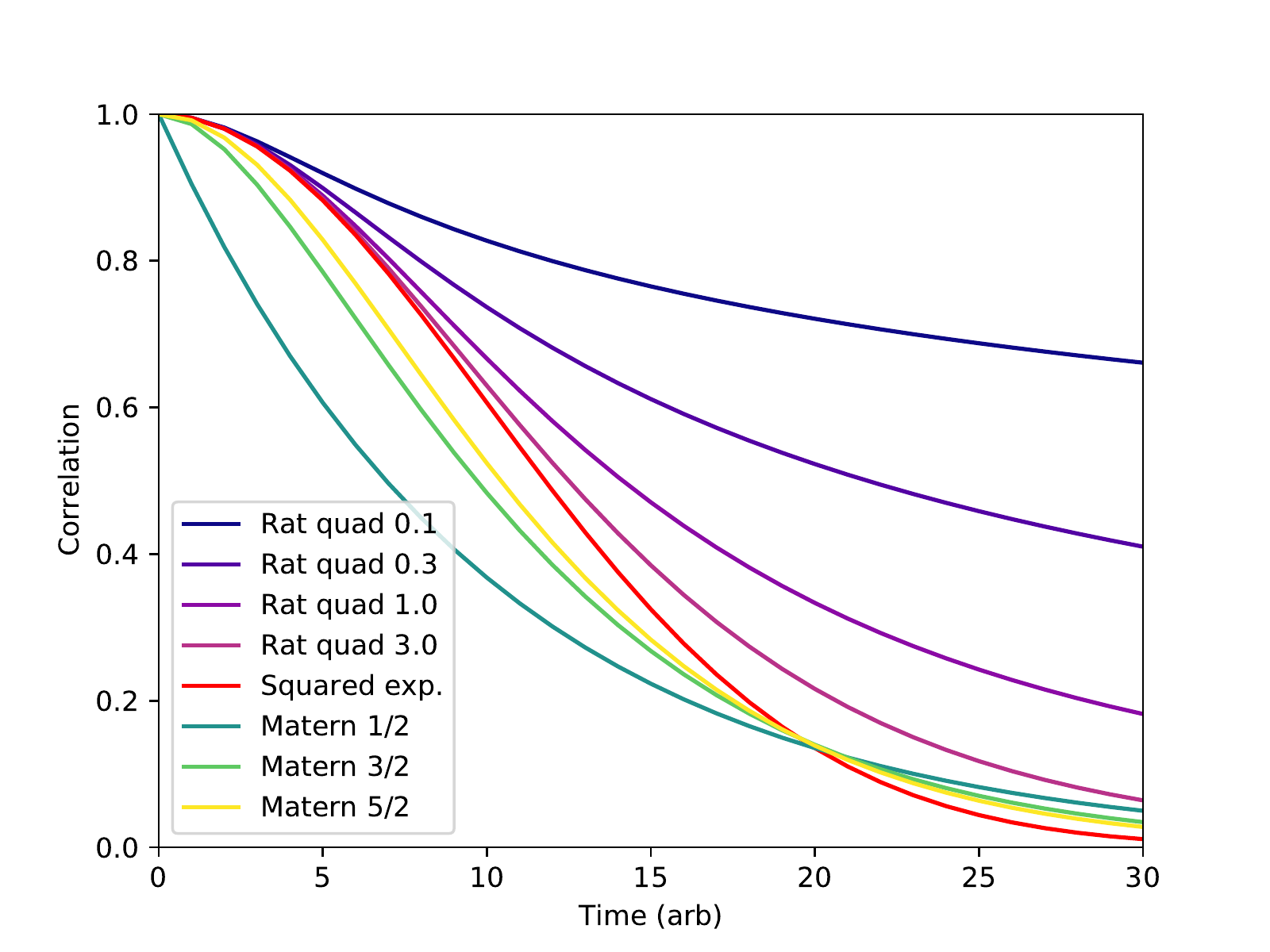}}
\subfigure[Kernel log autocorrelation functions]{\label{fig:3k}\includegraphics[width=0.4\textwidth]{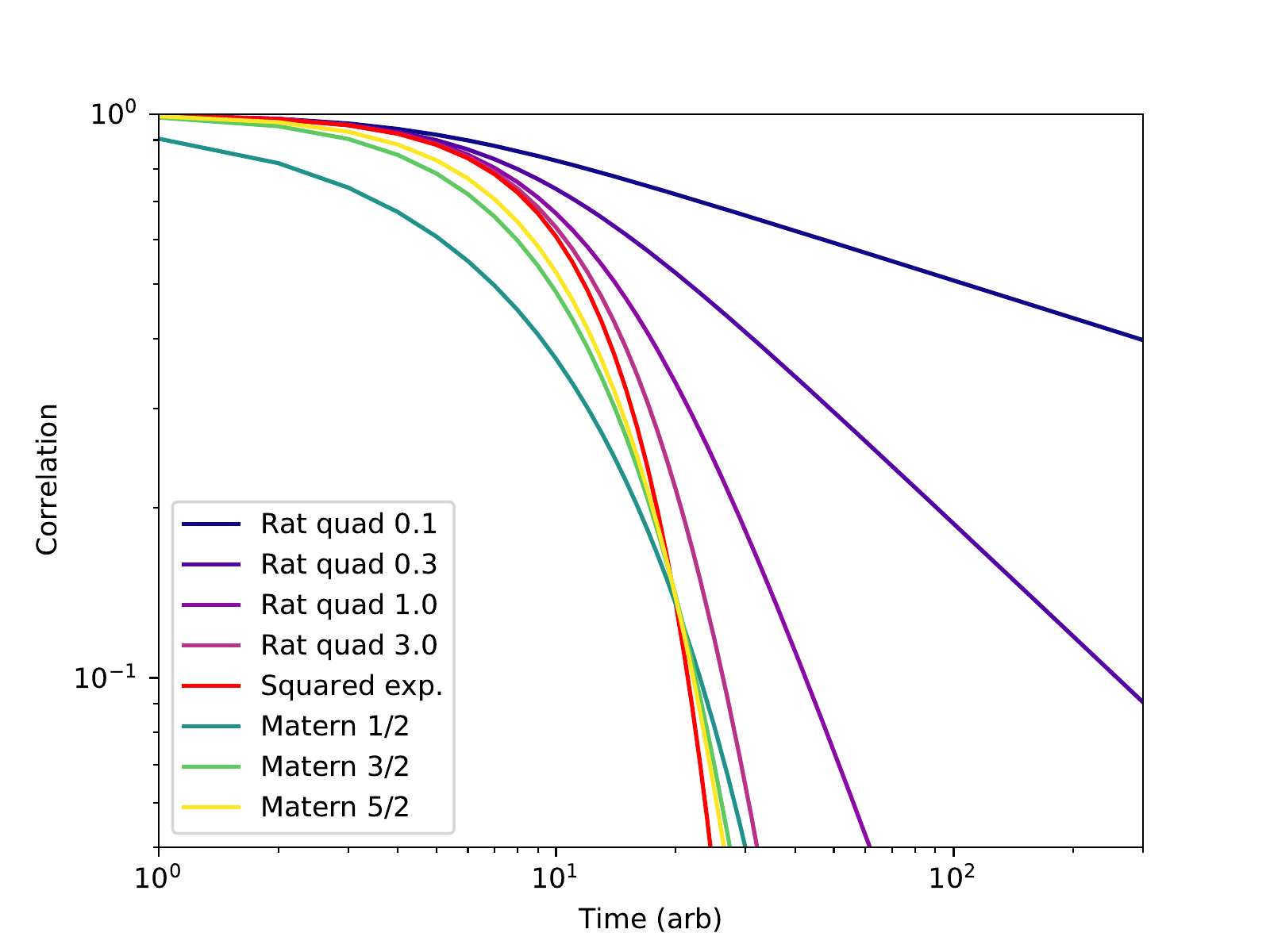}}
\subfigure[Kernel PSDs]{\label{fig:5k}\includegraphics[width=0.4\textwidth]{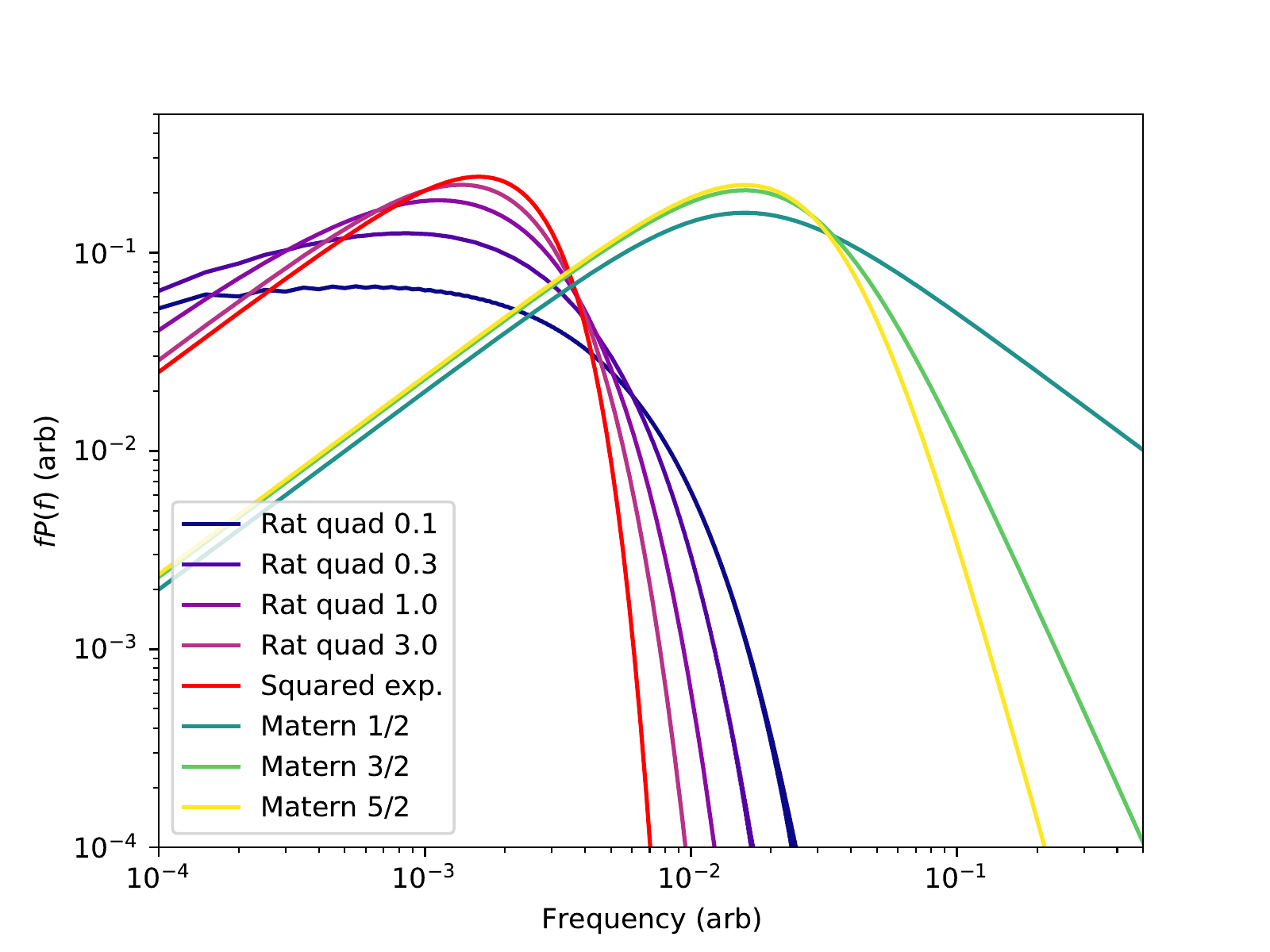}}
\caption{Kernel autocorrelation functions and PSDs. The rational quadratic kernel is plotted for different values of the $\alpha$ parameter. The Matérn kernel plots in the PSD figure are offset by a factor of 10 for clarity. A PSD of $f^{-2}$ will match the high frequency part of the Matérn $\frac{1}{2}$ kernel and the rational quadratic is endowed with additional flexibility to model PSDs by virtue of its $\alpha$ parameter. Such characteristics may explain why these kernels are preferred in the simulation study.}
\label{kern_acfs}
\end{figure}

\chapter{Modelling Experimental Noise with Gaussian Processes} 

\section{Heteroscedasticity of the Soil Phosphorus Fraction Dataset}
\label{soil_het_demo}

\autoref{the_table} is used to demonstrate the efficacy of modelling the soil phosphorus fraction dataset using a heteroscedastic \textsc{gp}. The heteroscedastic \textsc{gp} outperforms the homoscedastic \textsc{gp} on prediction based on the metric of negative log predictive density (NLPD)

\begin{equation*}
    \text{NLPD} = \frac{1}{n} \sum_{i=1}^n - \log p(t_i | \boldsymbol{x_i}),
\end{equation*}

\noindent which penalises both over and under-confident predictions.

\begin{table}[ht]
\centering
\caption{Comparison of NLPD values on the soil phosphorus fraction dataset. Standard errors are reported for 10 independent train/test splits. Lower scores are better.}
\label{the_table}
\vspace{3mm}
\resizebox{0.75\textwidth}{!}{%
\begin{tabular}{@{}lll@{}}
\toprule
\textbf{Soil Phosphorus Fraction Dataset} & \multicolumn{1}{c}{\textbf{GP}} & \multicolumn{1}{c}{\textbf{Het GP}} \\ \midrule
NLPD & $1.35 \pm 1.33$ & $1.00 \pm 0.95$ \\ \bottomrule
\end{tabular}%
}
\end{table}

\section{Additional Ablation Experiments}\label{more_ablation}

In this section the ablation results are presented on noiseless, homoscedastic and heteroscedastic noise tasks in line with \autoref{first_ablation} of the main thesis.

\subsection{Goldstein-Price Function}

The form of the Goldstein-Price function is the same as in the main thesis with noise function in \autoref{eq_g-p_noise}. The function is visualised in \autoref{fig:g-p_function}. 9 data points are used for initialisation in the noiseless and homoscedastic noise cases whereas 100 data points are used for initialisation in the heteroscedastic noise case. $\beta$ is set to 0.5 for the noiseless and homoscedastic noise tasks and $\frac{1}{11}$ for the heteroscedastic noise task. $\gamma$ is set to 500 for all experiments.

\begin{figure*}
\centering
\subfigure[Latent Function]{\label{fig:gp_1}\includegraphics[width=0.32\textwidth]{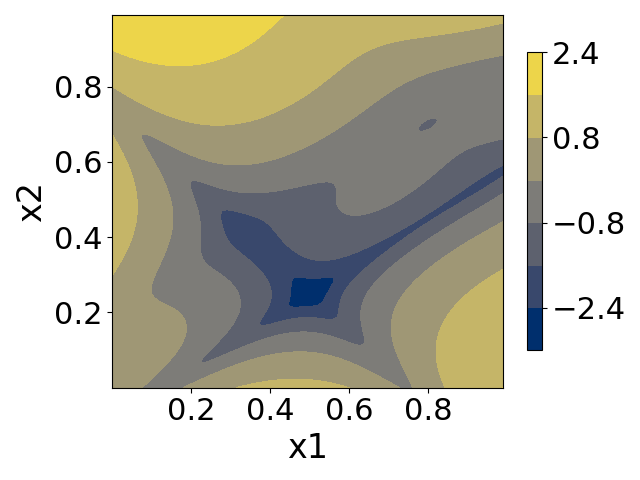}}
\subfigure[Noise Function ]{\label{fig:gp_2}\includegraphics[width=0.32\textwidth]{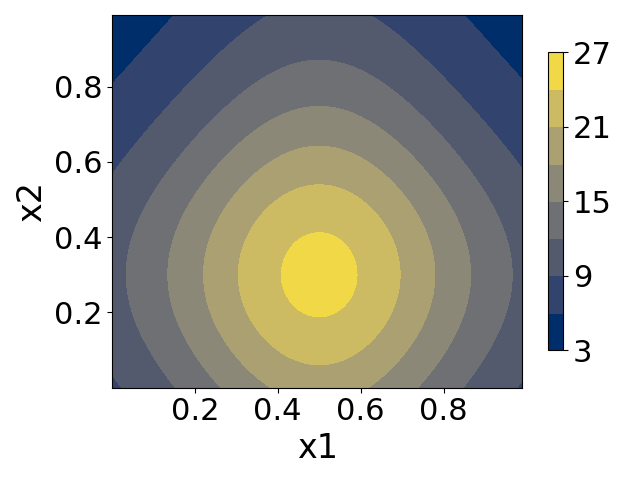}}
\subfigure[Objective Function]{\label{fig:gp_3}\includegraphics[width=0.32\textwidth]{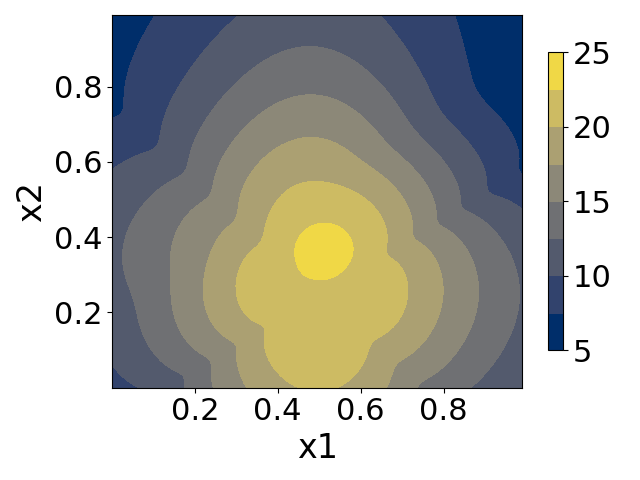}} 
\caption{(a) The latent Goldstein-Price Function $f(\mathbf{x})$ together with (b) its heteroscedastic noise function $g(\mathbf{x})$ and (c) the objective function $f(\mathbf{x}) + g(\mathbf{x})$.}
\label{fig:g-p_function}
\end{figure*}

\subsubsection{Noiseless Case}

The results of the noiseless case for Goldstein-Price are given in \autoref{noiseless_g-p}. All \textsc{bo} methods outperform random search with ANPEI best and HAEI second best.

\begin{figure*}[t]
\centering
    \includegraphics[width=.7\textwidth]{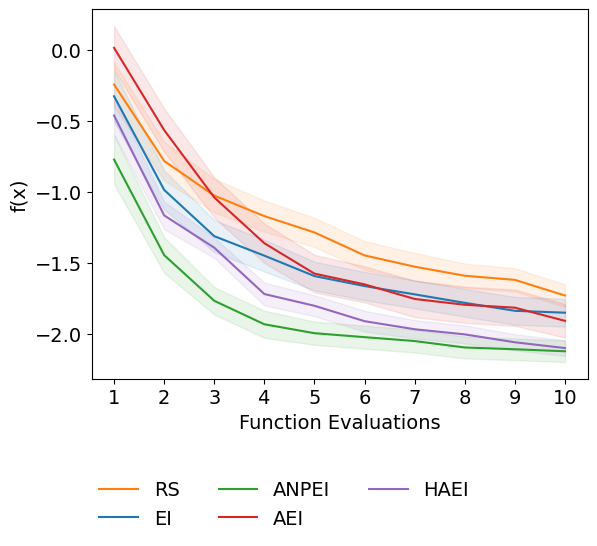}
    \caption{Goldstein-Price function noiseless case. All \textsc{bo} methods outperform random search. ANPEI performs best and HAEI is runner-up.}
    \label{noiseless_g-p}
\end{figure*}

\pagestyle{fancy}
\fancyhf{}
\lhead{Appendix}
\rhead{\thepage}

\subsubsection{Homoscedastic Noise Case}

The results of the homoscedastic noise case for Goldstein-Price are shown in \autoref{homo_g-p}. In this instance HAEI performs best.

\begin{figure*}[t]
\centering
    \includegraphics[width=.7\textwidth]{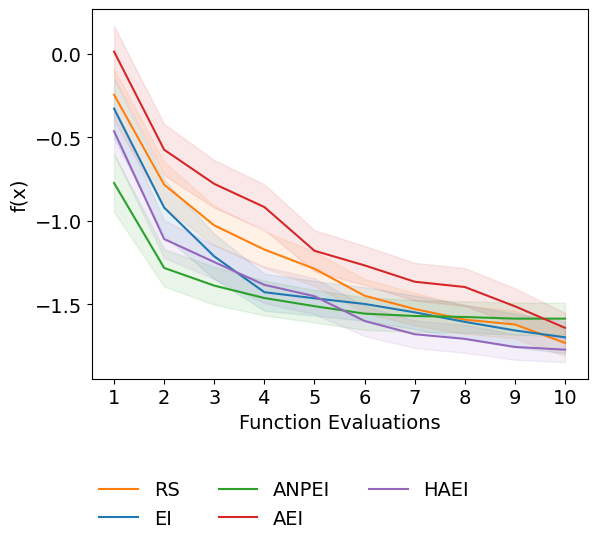}
    \caption{Goldstein-Price function homoscedastic noise case. HAEI performs best.}
    \label{homo_g-p}
\end{figure*}

\subsubsection{Heteroscedastic Noise}

The results of the heteroscedastic noise case for Goldstein-Price are shown in \autoref{fig:g-p_hetero}. ANPEI performs best whilst HAEI performs worse than random search.

\begin{figure*}
\centering
\subfigure[Best Objective Value Found so Far]{\label{fig:bo_g-p}\includegraphics[width=0.49\textwidth]{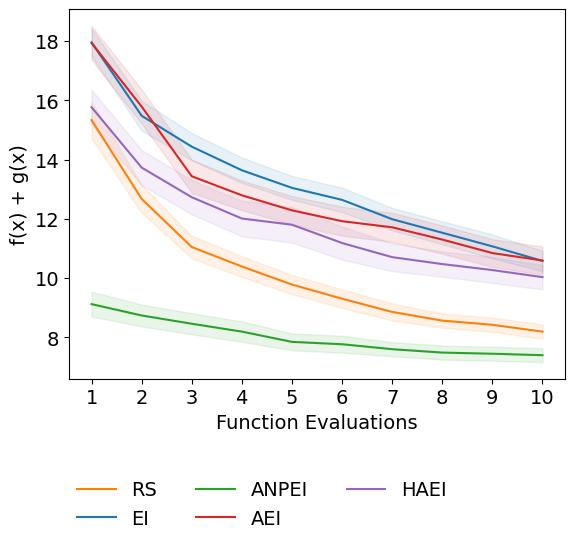}}
\subfigure[Lowest Aleatoric Noise Found so Far ]{\label{fig:bo_2_g-p}\includegraphics[width=0.49\textwidth]{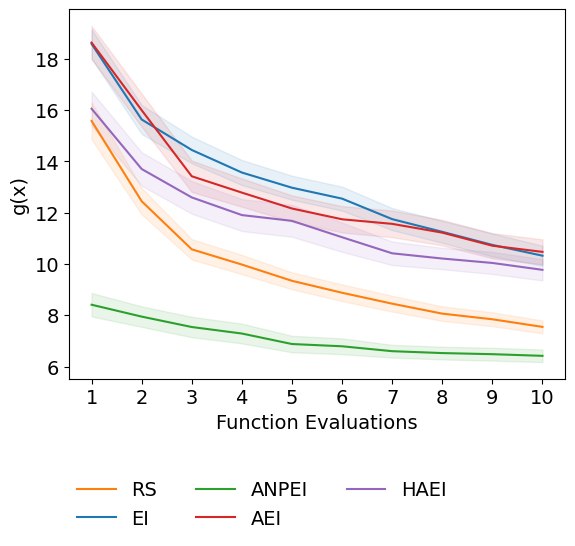}}
\caption{Comparison of heteroscedastic and homoscedastic \textsc{bo} on the heteroscedastic 2D Goldstein-Price function. (a) shows the optimisation of $h(\boldsymbol{x}) = f(\boldsymbol{x}) + g(\boldsymbol{x})$ (lower is better) where $g(\boldsymbol{x})$ is the aleatoric noise. (b) shows the values $g(\boldsymbol{x})$ obtained over the course of the optimisation of $h(\boldsymbol{x})$.}
\label{fig:g-p_hetero}
\end{figure*}

\subsection{Branin-Hoo Function}

The form of the Branin-Hoo function is given in \autoref{branin_eq} with noise function in \autoref{branin_noise_eq}. The function is visualised in \autoref{fig:double_bran}, a figure from the main thesis repeated here for clarity. 9 data points are used for initialisation in the noiseless and homoscedastic noise cases whereas 100 data points are used for initialisation in the heteroscedastic noise case. $\beta$ is set to 0.5 and $\gamma$ is set to 500 for all experiments.

\begin{figure*}
\centering
\subfigure[Latent Function]{\label{fig:double_bran_1}\includegraphics[width=0.32\textwidth]{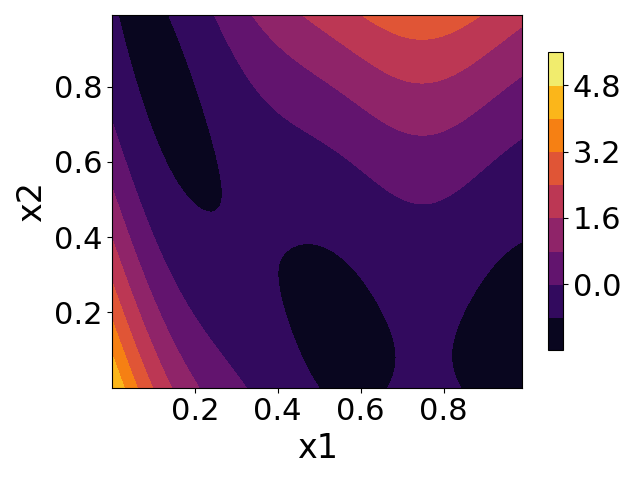}}
\subfigure[Noise Function ]{\label{fig:double_bran_2}\includegraphics[width=0.32\textwidth]{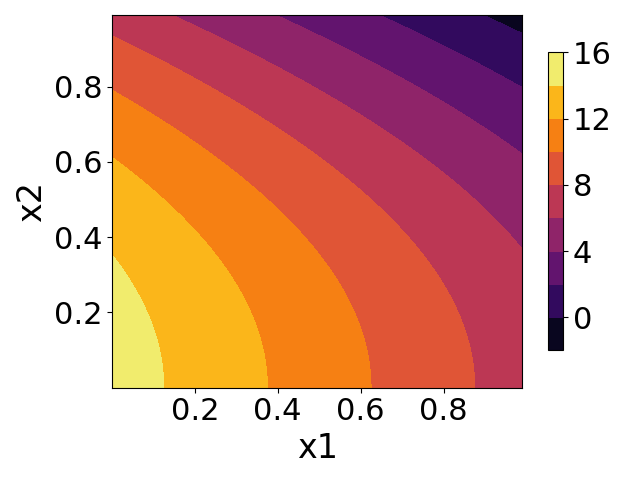}}
\subfigure[Objective Function]{\label{fig:double_bran_3}\includegraphics[width=0.32\textwidth]{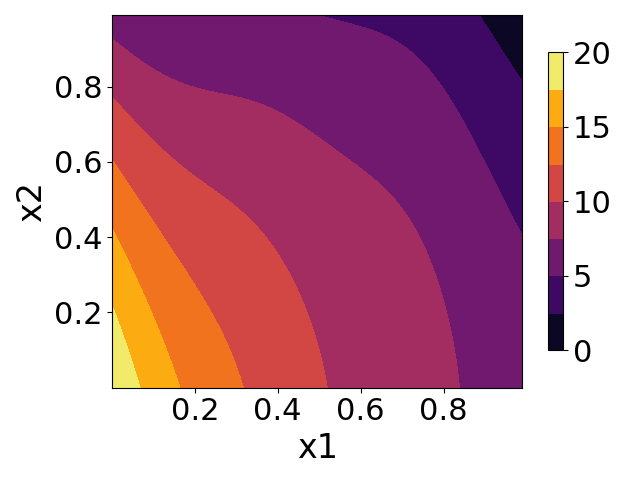}} 
\caption{Heteroscedastic Branin Function.}
\label{fig:double_bran}
\end{figure*}

\subsubsection{Noiseless Case}

The results of the noiseless case for the Branin-Hoo function are given in \autoref{noiseless_double_branin}. HAEI performs best in this case whereas ANPEI performs worst.

\begin{figure*}[t]
\centering
    \includegraphics[width=.7\textwidth]{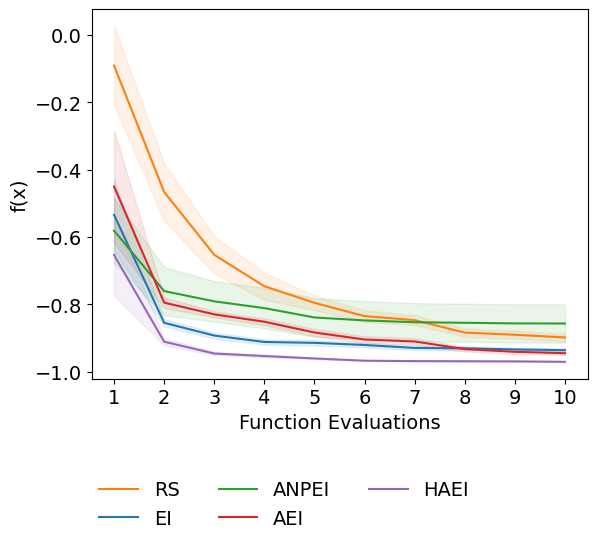}
    \caption{Branin-Hoo function noiseless case. HAEI performs best. ANPEI performs worst.}
    \label{noiseless_double_branin}
\end{figure*}

\subsubsection{Homoscedastic Noise Case}

The results of the homoscedastic noise case for the Branin-Hoo function are given in \autoref{homoscedastic_braninhoo}. All \textsc{bo} methods outperform random search yet perform comparably against each other.

\begin{figure*}[t]
\centering
    \includegraphics[width=.7\textwidth]{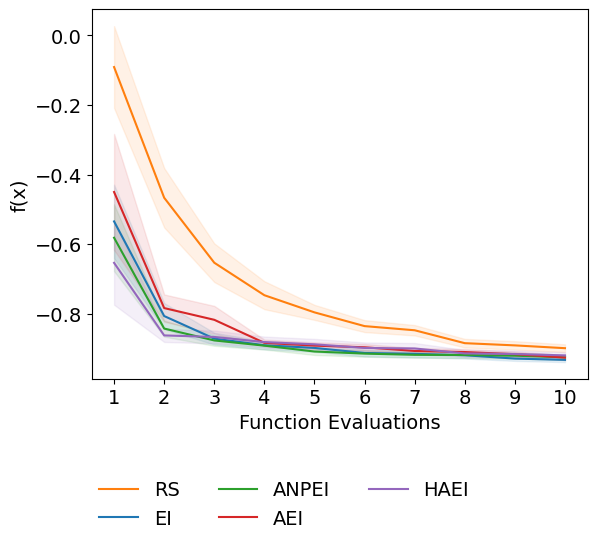}
    \caption{Branin-Hoo function homoscedastic noise case. All \textsc{bo} methods outperform random search.}
    \label{homoscedastic_braninhoo}
\end{figure*}

\subsubsection{Heteroscedastic Noise}

The results of the heteroscedastic noise case for the Branin-Hoo function are shown in \autoref{fig:bran_hetero2}. ANPEI performs best whilst HAEI performs worse than random search.

\begin{figure*}
\centering
\subfigure[Best Objective Value Found so Far]{\label{fig:bo_bran2}\includegraphics[width=0.495\textwidth]{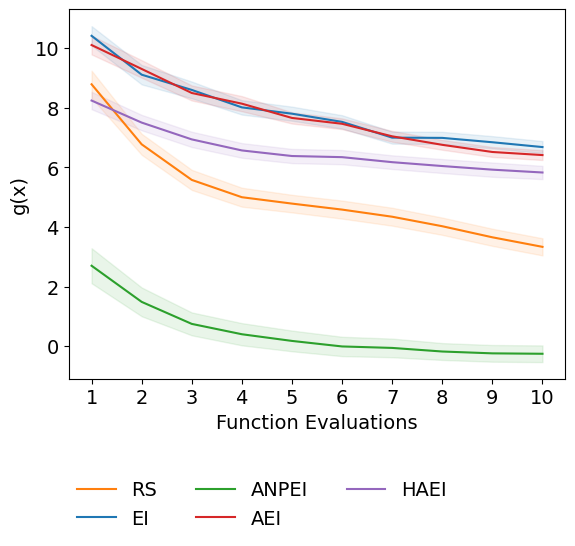}}
\subfigure[Lowest Aleatoric Noise Found so Far ]{\label{fig:bo_2_bran2}\includegraphics[width=0.493\textwidth]{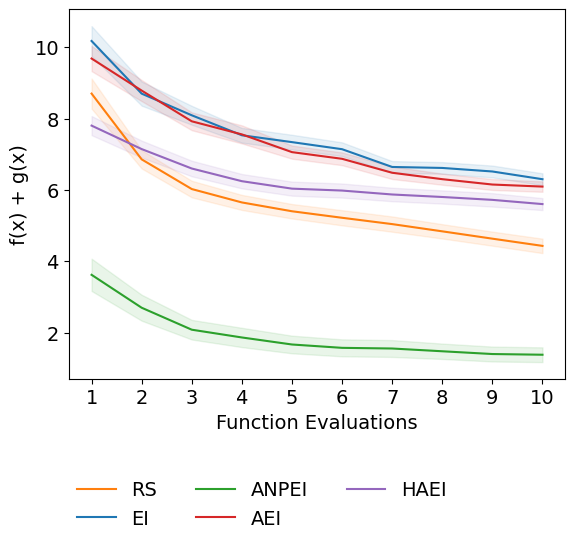}}
\caption{Comparison of heteroscedastic and homoscedastic \textsc{bo} on the heteroscedastic 2D Branin function. (a) shows the optimisation of $h(\boldsymbol{x}) = f(\boldsymbol{x}) + g(\boldsymbol{x})$ (lower is better) where $g(\boldsymbol{x})$ is the aleatoric noise. (b) shows the values $g(\boldsymbol{x})$ obtained over the course of the optimisation of $h(\boldsymbol{x})$.}
\label{fig:bran_hetero2}
\end{figure*}

\section{Performance Impact of the Kernel Choice}\label{kernel_exps}

In this section the impact that the choice of \textsc{gp} kernel has on \textsc{bo} performance is analysed. Three kernels are selected for this purpose: the RBF kernel

\begin{equation*}
    k_{\text{RBF}}(\boldsymbol{x}, \boldsymbol{x'}) = \sigma_{f}^{2}\cdot\text{exp}\Big(\frac{-\lVert\boldsymbol{x} - \boldsymbol{x'}\rVert^{2}}{2\ell^{2}}\Big),
\end{equation*}

\noindent used for all experiments in the main thesis, the exponential kernel (Exp)

\begin{equation*}
    k_{\text{exp}}(\boldsymbol{x}, \boldsymbol{x'}) = \sigma_{f}^{2}\cdot\text{exp}\Big(\frac{-\lVert\boldsymbol{x} - \boldsymbol{x'}\rVert}{\ell}\Big),
\end{equation*}

\noindent a special instance of the Mat\'{e}rn kernel for values of $\nu = \frac{1}{2}$ \citep{2006_Rasmussen}, as well as the Mat\'{e}rn 5/2 kernel

\begin{equation*}
    k_{\text{Mat\'{e}rn}(5/2)}(\boldsymbol{x}, \boldsymbol{x'}) = \sigma_{f}^{2}\cdot \Big(1 + \frac{\sqrt{5} \lVert\boldsymbol{x} - \boldsymbol{x'}\rVert}{\ell} + \frac{5 \lVert\boldsymbol{x} - \boldsymbol{x'}\rVert^2}{3\ell^2}\Big) \cdot \text{exp}\Big(\frac{-\sqrt{5} \lVert\boldsymbol{x} - \boldsymbol{x'}\rVert}{\ell}\Big),
\end{equation*}

\noindent which is one of the most popular kernels for large-scale empirical studies \citep{2018_Wilson, 2020_Grosnit}. It should be noted that while the equations are written assuming a single scalar lengthscale, in practice for the experiments in greater than 1D, each lengthscale is optimised per dimension under the marginal likelihood. For all experiments the same kernel is chosen for both \textsc{gp}s of the \textsc{mlhgp} model i.e. the \textsc{gp} modelling the objective as well as the \textsc{gp} modelling the noise. 100 points are used for initialisation in the Branin-Hoo and Goldstein-Price functions and 144 points are used for the Hosaki function. $\beta$ is set to 0.5 for the Branin-Hoo and Hosaki functions and $\frac{1}{11}$ for the Goldstein-Price function. $\gamma$ is set to 500 for all experiments. The results are shown in \autoref{fig:bran_kernel}, \autoref{fig:gold_kernel} and \autoref{fig:hos_kernel} for the Branin-Hoo function, Goldstein-Price function and Hosaki functions respectively. There is no significant difference in performance using each kernel save for the Branin-Hoo function where ANPEI underperforms using the somewhat rougher exponential kernel.

\begin{figure*}
\centering
\subfigure[ANPEI]{\label{fig:bo_brank}\includegraphics[width=0.495\textwidth]{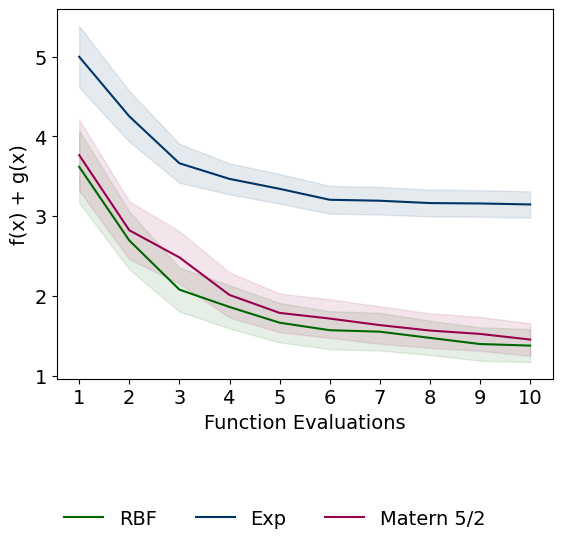}}
\subfigure[HAEI]{\label{fig:bo_2_brank}\includegraphics[width=0.495\textwidth]{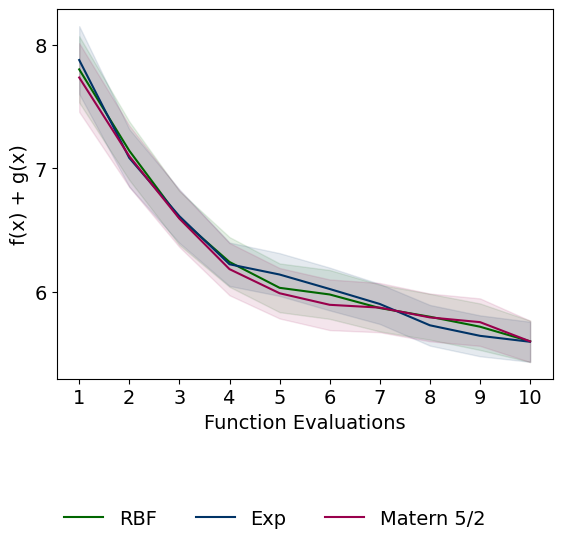}}
\caption{Branin-Hoo function kernel comparison.}
\label{fig:bran_kernel}
\end{figure*}

\begin{figure*}
\centering
\subfigure[ANPEI]{\label{fig:bo_g-p}\includegraphics[width=0.486\textwidth]{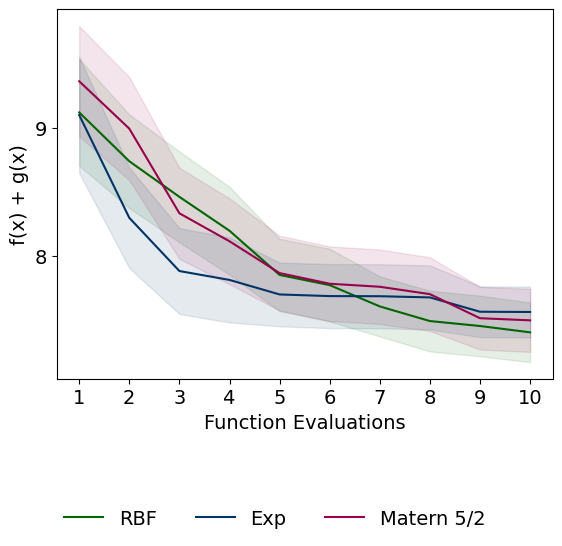}}
\subfigure[HAEI]{\label{fig:bo_2_g-p}\includegraphics[width=0.495\textwidth]{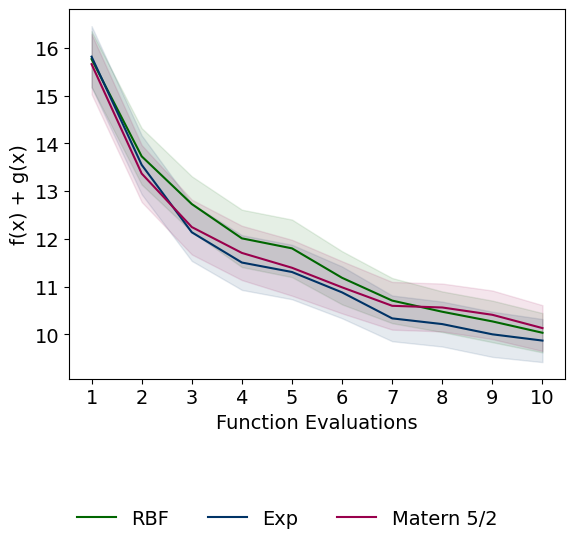}}
\caption{Goldstein-Price function kernel comparison.}
\label{fig:gold_kernel}
\end{figure*}

\begin{figure*}
\centering
\subfigure[ANPEI]{\label{fig:bo_hosk}\includegraphics[width=0.495\textwidth]{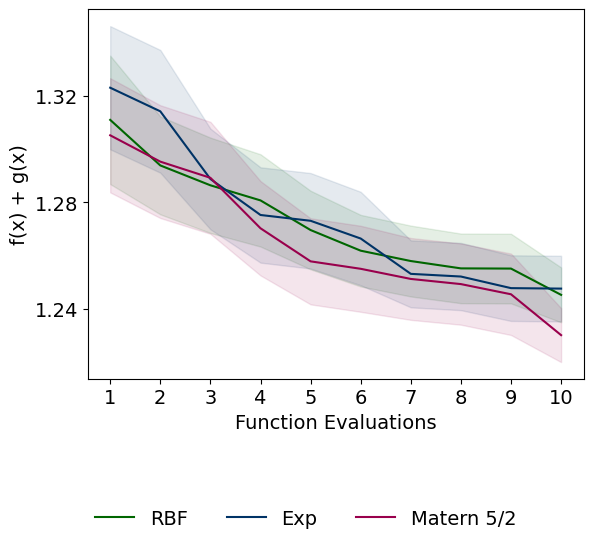}}
\subfigure[HAEI]{\label{fig:bo_2_hosk}\includegraphics[width=0.495\textwidth]{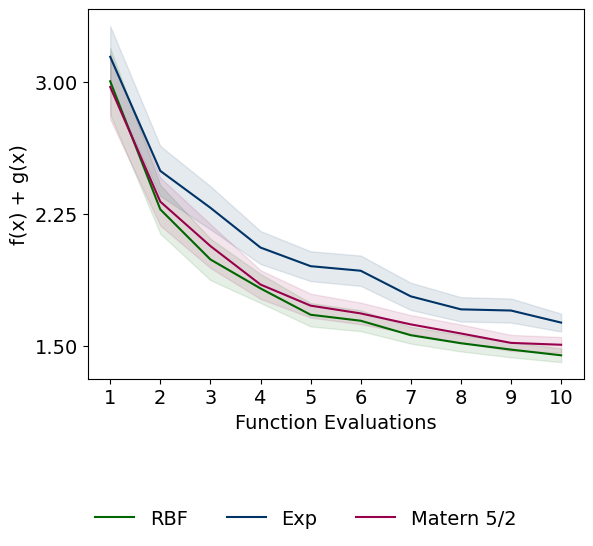}}
\caption{Hosaki function kernel comparison.}
\label{fig:hos_kernel}
\end{figure*}

\ifpdf
    \graphicspath{{Appendix3/Figs/}{Appendix3/Figs/PDF/}{Chapter1/Figs/}}
\else
    \graphicspath{{Appendix3/Figs/}{Appendix3/Figs/}}
\fi

\chapter{Molecular Discovery with Gaussian Processes}

\section{Sources of Experimental Data}
\label{exp_source}

Properties were collated from the photoswitch literature originally by Aditya Raymond Thawani. Emphasis was placed on obtaining a broad range of functional groups attached to the photoswitch scaffold. The set of literature articles consulted included \citet{d1,data1,data2,data3,data4,data5,data6,data8,2011_Jacquemin,data10,data11,data12,data13,data14,data15,data16,data17,data18,data19}.

\section{Dataset Visualisations}

The choice of molecular representation is known to be a key factor in the performance of machine learning algorithms on molecules \citep{2017_Faber, 2018_Wu, 2020_Faber}. Commonly-used representations such as fingerprint and fragment-based descriptors are high dimensional and as such, it can be challenging to interpret the inductive bias induced by the representation. To visualise the high-dimensional representation space of the Photoswitch Dataset the data matrix was projected to two dimensions using the UMAP algorithm. \citep{2018_UMAP}. The manifolds were compared under the Morgan fingerprint representation and a fragment-based representation computed using RDKit \citep{rdkit}. 512-bit Morgan fingerprints were generated with a bond radius of 2, setting the nearest neighbours parameter in the UMAP algorithm to a value of 50. The resulting visualisation was produced using the ASAP package (available at \href{https://github.com/BingqingCheng/ASAP}{https://github.com/BingqingCheng/ASAP}) and is shown in \autoref{umap}.

\begin{figure}[!htbp]
    \begin{center}
        \includegraphics[width=1.1\textwidth]{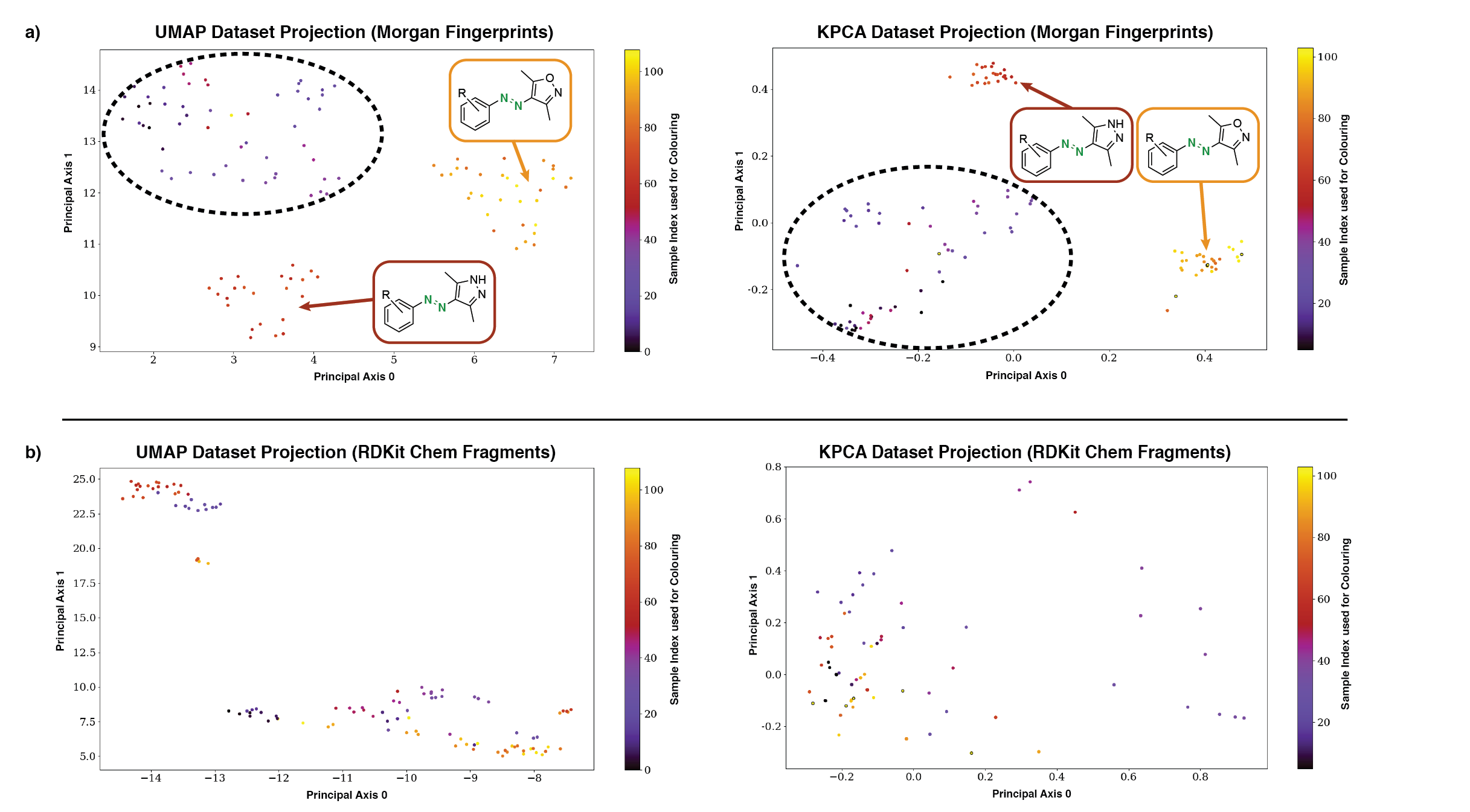}
    \end{center}
    \caption{a) UMAP and k-PCA projections of the dataset, using Morgan fingerprints, correctly identify clusters of chemically similar molecules. The regions demarcated by dashed black lines are composed of miscellaneous azoheteroarenes; no grouping was noted here due to the limited ($\leq$	10) examples per class. b) Similar projections using RDKit Fragment descriptors fails to identify any such clusters.}
    \label{umap}
\end{figure}

The structure of the manifold located under the Morgan fingerprint representation identifies meaningful subgroups of azophotoswitches when compared to the fragment-based representation. To demonstrate that the finding is due to the representation and not the dimensionality reduction algorithm the manifolds identified by k-PCA using a cosine kernel are included. Both algorithms identify the same manifold structure in the Morgan fingerprint representation. 

\section{Further Experiments}

The subsections below detail further experiments carried out during the design of the machine learning prediction pipeline.

\subsection{Property Prediction}
\label{benchmark_ml}

For representations, 2048-bit Morgan fingerprints with a bond radius of 3, implemented in RDKit, were used \citep{rdkit}. 85-dimensional fragment features computed using the RDKit descriptors module were used. The Dscribe library \citep{2020_Himanen} was used to compute (Smooth Overlap of Atomic Positions) (SOAP) descriptors using a \texttt{rcut} parameter of 3.0, a \texttt{sigma} value of 0.2, a \texttt{nmax} parameter of 12, and a \texttt{lmax} parameter of 8. An REMatch kernel was used with polynomial base kernel of degree 3.0, \texttt{gamma} = 1.0, \texttt{coef0} = 0, \texttt{alpha} = 0.5, and \texttt{threshold} = $1e^{-6}$.

Performance was evaluated on 20 random train/test splits in a ratio of 80/20 using the root mean square error (RMSE), mean absolute error (MAE) and coefficient of determination ($R^2$) as performance metrics, reporting the mean and standard error for each metric (Table \ref{property_pred1}). The following models were evaluated: Random Forest (RF), Gaussian Processes (\textsc{gp}), Attentive Neural Processes (ANP), \citep{2018_Kim} Graph Convolutional Networks (GCNs) \citep{2017_Kipf}, Graph Attention Networks (GATs) \citep{2018_Velickovic}, Directed Message-Passing Neural Networks (DMPNNs) \citep{2019_Yang}, and the following representations: Morgan fingerprints \citep{2010_Rogers}, RDKit fragments \citep{rdkit}, SOAP \citep{2013_Bartok}, the simplified molecular-input line-Entry system (SMILES) \citep{1988_Weininger}, and self-referencing embedded strings (SELFIES) \citep{2020_Krenn}. In addition, a new hybrid representation was introduced and termed \say{fragprints}. Fragprints are formed by concatenating the fragment and fingerprint vectors. For the purpose of the benchmark, hyperparameter selection for \textsc{gp}-based approaches was performed by optimising the marginal likelihood on the train set whereas for other methods cross-validation was performed using the Hyperopt-Sklearn library \citep{2019_Komer} for Sklearn models such as RF, and 1000 randomly sampled configurations for other models.
 
The RF model was trained using scikit-learn \citep{scikit-learn} with 1000 estimators and a maximum depth of 300. A \textsc{gp} was implemented in GPflow \citep{GPflow} using a Tanimoto kernel \citep{2005_Ralaivola, 2020_flowmo} for fingerprint, fragment and fragprint representations, and the subset string kernel of \citet{2020_Moss} (following the exact experimental setup in \citet{2020_flowmo}) for the character-based SMILES and SELFIES representations. Additionally, a multioutput Gaussian process (\textsc{mogp}) was trained based on the intrinsic coregionalisation model (ICM) \citep{2007_Williams} to leverage information in the multitask setting. For all \textsc{gp} models, the mean function was set to be the empirical mean of the data and the kernel variance and likelihood variance treated as hyperparameters, optimising their values under the marginal likelihood. For the ANP, 2 hidden layers of dimension 32 were used for each of the decoder, latent decoder and the deterministic encoder respectively. 8-dimensional latent variables $r$ and $z$ were used and optimisation was run for 500 iterations with the Adam optimiser \citep{2014_Adam} using a learning rate of 0.001. 

For the ANP, principal components regression was performed by reducing the representation dimension to 50. GCNs and GATs were implemented in the DGL-LifeSci library \citep{2020_Li}. Node features included one-hot representations of atom-type, atom degree, the number of implicit hydrogen atoms attached to each atom, the total number of hydrogen atoms per atom, atom hybridisation, the formal charge and the number of radical electrons on the atom. Edge features contained one-hot encodings of bond-type and Booleans indicating the stereogenic configuration of the bond and whether the bond was conjugated or in a ring. For the GCN, two hidden layers with 32 hidden units and ReLU activations were used, applying BatchNorm \citep{2015_Ioffe} to both layers. For the GAT, two hidden layers with 32 units each, 4 attention heads, and an alpha value of 0.2 were used in both layers with ELU activations. A single DMPNN model was trained for 50 epochs with additional normalised 2D RDKit features. All remaining parameters were set to the default values in \citet{2019_Yang}. SchNet \citep{2017_Schutt} was not benchmarked because it is designed for the prediction of molecular energies and atomic forces. All experiments were performed on the CPU of a MacBook Pro using a 2.3 GHz 8-Core Intel Core i9 processor.

Standardisation was applied (by subtracting the mean and dividing by the standard deviation) to the property values in all experiments. The results of the aforementioned models and representations are given in \autoref{property_pred1}. Additional results including Message-passing neural networks (MPNN) \citep{2017_Gilmer}, a black-box alpha divergence minimisation Bayesian neural network (BNN) \citep{2016_Lobato}, and an LSTM with augmented SMILES, SMILES-X \citep{2020_Lambard}, are presented in \autoref{property_pred2}. It should be noted that featurisations using standard molecular descriptors are more than competitive with neural representations for this dataset. The best-performing representation/model pair on the most data-rich \emph{E} isomer $\pi-\pi^{*}$ task was the \textsc{mogp}$^*$-Tanimoto kernel and the introduced hybrid descriptor set \say{fragprints}. Importantly, there is weak evidence that the \textsc{mogp}$^*$ is able to leverage multitask learning in learning correlations between the transition wavelengths of the isomers, a modelling feature that may be particularly useful in the low-data regimes characteristic of experimental datasets. A Wilcoxon signed-rank test \citep{1945_Wilcoxon} is carried out in order to determine whether the performance differential between the \textsc{gp}/fragprints combination and the \textsc{mogp}$^*$/fragprints combination is statistically significant. In this instance, the \textsc{mogp}$^*$ is provided with auxiliary task labels for test molecules where available (i.e. labels for tasks that are not being predicted). The null hypothesis is that there is no significant difference arising from multitask learning. In the case of the \emph{E} isomer $\pi-\pi^*$ transition, the resultant p-value is $0.33$, meaning that the null hypothesis cannot be rejected at the 95\% confidence level. In the case of the \emph{Z} isomer $\pi-\pi^*$ transition, the resultant p-value is $0.06$, meaning also that the null hypothesis cannot be rejected at the 95\% confidence level. In this latter case, however, rejection of the null hypothesis depends on the confidence level threshold specified. As such, it is concluded that only weak evidence is available to support the benefits of multitask learning over single task learning.  

\begin{table}[!htbp]
\caption{Test set performance in predicting the transition wavelengths of the \emph{E} and \emph{Z} isomers. Best-performing models are highlighted in bold. MOGP$^*$ denotes a multioutput \textsc{gp} such that auxiliary task labels (i.e. not the task being predicted) for test molecules are provided to the model where available.}
\label{tab1_photo}
\setlength{\extrarowheight}{2pt}
\begin{adjustbox}{width={\textwidth},totalheight={\textheight},keepaspectratio}
\begin{tabular}{l|cccc}
\toprule
& \emph{E} isomer $\pi-\pi{^*}$ (nm) & \emph{E} isomer \emph{n}$-\pi{^*}$ (nm) & \emph{Z} isomer $\pi-\pi{^*}$ (nm) & \emph{Z} isomer \emph{n}$-\pi{^*}$ (nm) \\ \hline
\multicolumn{1}{c|}{\underline{\textbf{RMSE}}} & & & & \\
RF + Morgan & $25.3 \pm \: 0.9$ & $\textbf{10.2} \pm \: \textbf{0.4}$ & $14.0 \pm \: 0.6$ & $11.1 \pm \: 0.4$ \\ 
RF + Fragments & $26.4 \pm \: 1.1$ & $11.4 \pm \: 0.5$ & $17.0 \pm \: 0.8$& $14.2 \pm \: 0.6$ \\
RF + Fragprints & $23.4 \pm \: 0.9$ & $11.0 \pm \: 0.4$ & $14.2 \pm \: 0.6$ &  $11.3 \pm \: 0.6$ \\ 
GP + Morgan & $23.4 \pm \: 0.8$ & $11.4 \pm \: 0.5$ & $13.2 \pm \: 0.7$ & $\textbf{11.0} \pm \: \textbf{0.7}$ \\
GP + Fragments & $26.3 \pm \: 0.8$ & $11.6 \pm \: 0.5$ & $15.5 \pm \: 0.8$ & $12.6 \pm \: 0.5$ \\
GP + Fragprints & $20.9 \pm \: 0.7$ & $11.1 \pm \: 0.5$ & $13.1 \pm \: 0.6$ & $11.4 \pm \: 0.7$ \\
GP + SOAP & $21.0 \pm \: 0.6$ & $22.7 \pm \: 0.6$ & $17.8 \pm \: 0.8$ & $15.0 \pm \: 0.5$ \\ 
GP + SMILES & $26.0 \pm \: 0.8$ & $12.3 \pm \: 0.4$& $12.5 \pm \: 0.5$ &  $11.8 \pm \: 0.6$\\ 
GP + SELFIES & $23.5 \pm \: 0.7$ &$12.9 \pm \: 0.5$  &$14.4 \pm \: 0.5$ & $12.2 \pm \: 0.5$\\ 

MOGP + Morgan & $23.6 \pm \: 0.8$ & $11.7 \pm \: 0.5$ & $15.5 \pm \: 0.6$ & $11.1 \pm \: 0.7$ \\
MOGP + Fragments & $27.0 \pm \: 0.9$ & $11.9 \pm \: 0.6$ & $16.4 \pm \: 0.9$ & $13.1 \pm \: 0.6$ \\
MOGP + Fragprints & $21.2 \pm \: 0.7$ & $11.3 \pm \: 0.5$ & $13.5 \pm \: 0.6$ & $11.4 \pm \: 0.7$ \\

MOGP$^*$ + Morgan & $22.6 \pm \: 0.8$ & $11.6 \pm \: 0.4$ & $12.3 \pm \: 0.7$ & $10.9 \pm \: 0.7$ \\
MOGP$^*$ + Fragments & $26.9 \pm \: 0.8$ & $12.1 \pm \: 0.6$ & $16.2 \pm \: 0.8$ & $13.8 \pm \: 0.6$ \\
MOGP$^*$ + Fragprints & $\textbf{20.4} \pm \: \textbf{0.7}$ & $11.2 \pm \: 0.5$ & $\textbf{11.3} \pm \: \textbf{0.4}$ & $11.4 \pm \: 0.7$ \\

ANP + Morgan & $28.1 \pm \: 1.3$ & $13.6 \pm \: 0.5$ & $13.5 \pm \: 0.6$ & $\textbf{11.0} \pm \: \textbf{0.6}$ \\
ANP + Fragments & $27.9 \pm \: 1.1$ & $13.8 \pm \: 0.9$ & $17.2 \pm \: 0.8$ & $14.1 \pm \: 0.7$ \\
ANP + Fragprints & $27.0 \pm \: 0.8$ & $11.6 \pm \: 0.5$ & $14.5 \pm \: 0.8$ & $11.3 \pm \: 0.7$ \\ 
GCN & $22.0 \pm \: 0.8$  & $12.8 \pm \: 0.8$ & $16.3 \pm \: 0.8$ & $13.1 \pm \: 0.8$\\
GAT & $26.4 \pm \: 1.1$ & $16.9 \pm \: 1.9$ & $19.6 \pm \: 1.0$ & $14.5 \pm \: 0.8$\\
DMPNN & $27.1 \pm \: 1.4$ & $13.9 \pm \: 0.6$ & $17.5 \pm \: 0.7$ & $13.8 \pm \: 0.4$ \\[5pt]\hline \multicolumn{1}{c|}{\underline{\textbf{MAE}}} & & & & \\
RF + Morgan & $15.5 \pm \: 0.5$ & $\textbf{7.3} \pm \: \textbf{0.3}$ & $10.1 \pm \: 0.4$ & $\textbf{6.6} \pm \: \textbf{0.3}$ \\ 
RF + Fragments & $16.4 \pm \: 0.5$ & $8.5 \pm \: 0.3$ & $12.2 \pm \: 0.6$ & $9.0 \pm \: 0.4$ \\ 
RF + Fragprints & $13.9 \pm \: 0.4$ & $7.7 \pm \: 0.3$ & $10.0 \pm \: 0.4$ & $6.8 \pm \: 0.3$ \\ 
GP + Morgan & $15.2 \pm \: 0.4$ & $8.4 \pm \: 0.3$ & $9.8 \pm \: 0.4$ & $6.9 \pm \: 0.3$ \\
GP + Fragments & $17.3 \pm \: 0.4$ & $8.6 \pm \: 0.3$ & $11.5 \pm \: 0.5$ & $8.2 \pm \: 0.3$ \\
GP + Fragprints & $13.3 \pm \: 0.3$ & $8.2 \pm \: 0.3$ & $9.8 \pm \: 0.4$ & $7.1 \pm \: 0.3$ \\
GP + SOAP & $14.3 \pm \: 0.3$ & $19.3 \pm \: 0.5$ & $12.9 \pm \: 0.6$ & $11.4 \pm \: 0.4$ \\
GP + SMILES &$16.6 \pm \: 0.5$ &$8.6 \pm \: 0.3$ & $9.4 \pm \: 0.4$ & $7.4 \pm \: 0.3$ \\ 
GP + SELFIES & $14.7 \pm \: 0.7$ & $8.8 \pm \: 0.3$&  $11.1 \pm \: 0.3$& $8.1 \pm \: 0.2$\\ 

MOGP + Morgan & $15.3 \pm \: 0.4$ & $8.6 \pm \: 0.3$ & $11.9 \pm \: 0.5$ & $7.0 \pm \: 0.3$ \\
MOGP + Fragments & $17.6 \pm \: 0.5$ & $8.8 \pm \: 0.4$ & $12.1 \pm \: 0.6$ & $8.3 \pm \: 0.3$ \\
MOGP + Fragprints & $13.5 \pm \: 0.3$ & $8.3 \pm \: 0.3$ & $10.2 \pm \: 0.5$ & $7.1 \pm \: 0.3$ \\

MOGP$^*$ + Morgan & $14.4 \pm \: 0.4$ & $8.5 \pm \: 0.3$ & $9.6 \pm \: 0.4$ & $6.9 \pm \: 0.4$ \\
MOGP$^*$ + Fragments & $17.2 \pm \: 0.4$ & $8.9 \pm \: 0.3$ & $11.9 \pm \: 0.5$ & $8.5 \pm \: 0.4$ \\
MOGP$^*$ + Fragprints & $\textbf{13.1} \pm \: \textbf{0.3}$ & $8.3 \pm \: 0.3$ & $\textbf{8.8} \pm \: \textbf{0.3}$ & $7.1 \pm \: 0.4$ \\

ANP + Morgan & $17.9 \pm \: 0.7$ & $10.1 \pm \: 0.4$ & $10.0 \pm \: 0.4$ & $7.2 \pm \: 0.3$ \\
ANP + Fragments & $17.4 \pm \: 0.6$ & $9.4 \pm \: 0.4$ & $12.3 \pm \: 0.6$ & $8.9 \pm \: 0.4$ \\
ANP + Fragprints & $18.1 \pm \: 0.5$ & $8.6 \pm \: 0.3$ & $10.4 \pm \: 0.5$ & $7.0 \pm \: 0.3$ \\ 
GCN & $13.9 \pm \: 0.3$ & $8.6 \pm \: 0.3$ & $11.6 \pm \: 0.5$ & $8.6 \pm \: 0.5$ \\
GAT & $18.1 \pm \: 0.7$ & $10.7 \pm \: 0.6$ & $14.4 \pm \: 0.8$ & $10.8 \pm \: 0.7$\\
DMPNN & $17.1 \pm \: 0.8$ & $10.6 \pm \: 0.4$ & $12.8 \pm \: 0.6$ & $9.8 \pm \: 0.3$\\[5pt] \hline \multicolumn{1}{c|}{\emph{\underline{\textbf{R$^{2}$}}}} & & & & \\
RF + Morgan & $0.85 \pm \: 0.01$ & $\textbf{0.80} \pm \: \textbf{0.01}$ & $0.25 \pm \: 0.06$ & $0.36 \pm \: 0.06$ \\ 
RF + Fragments & $0.83 \pm \: 0.01$ & $0.75 \pm \: 0.02$ & $-0.15 \pm \: 0.11$ & $-0.05 \pm \: 0.07$ \\
RF + Fragprints & $0.87 \pm \: 0.01$ & $0.77 \pm \: 0.02$ & $0.23 \pm \: 0.07$ & $0.33 \pm \: 0.06$ \\ 
GP + Morgan & $0.87 \pm \: 0.01$ & $0.76 \pm \: 0.01$ & $0.34 \pm \: 0.05$ &  $\textbf{0.38} \pm \: \textbf{0.05}$ \\
GP + Fragments & $0.84 \pm \: 0.01$ & $0.74 \pm \: 0.02$ & $0.07 \pm \: 0.08$ & $0.19 \pm \: 0.05$ \\
GP + Fragprints & $\textbf{0.90} \pm \: \textbf{0.01}$ & $0.77 \pm \: 0.02$ & $0.35 \pm \: 0.05$ & $0.33 \pm \: 0.05$\\
GP + SOAP & $0.89 \pm \: 0.01$ & $-0.08 \pm \: 0.03$ & $-0.05 \pm \: 0.02$ & $-0.07 \pm \: 0.02$\\
GP + SMILES & $0.84 \pm \: 0.02$& $0.72 \pm \: 0.02$&  $0.39 \pm \: 0.05$&  $0.29 \pm \: 0.04$\\ GP + SELFIES & $0.86 \pm \: 0.01$ & $0.68 \pm \: 0.02$ &$0.20 \pm \: 0.05$ & $0.23 \pm \: 0.04$\\ 

MOGP + Morgan & $0.87 \pm \: 0.01$ & $0.75 \pm \: 0.01$ & $0.06 \pm \: 0.08$ & $0.37 \pm \: 0.05$ \\
MOGP + Fragments & $0.83 \pm \: 0.01$ & $0.73 \pm \: 0.02$ & $-0.05 \pm \: 0.10$ & $0.11 \pm \: 0.06$ \\
MOGP + Fragprints & $0.89 \pm \: 0.01$ & $0.76 \pm \: 0.02$ & $0.30 \pm \: 0.06$ & $0.33 \pm \: 0.05$ \\

MOGP$^*$ + Morgan & $0.88 \pm \: 0.01$ & $0.75 \pm \: 0.01$ & $0.34 \pm \: 0.12$ & $0.39 \pm \: 0.05$ \\
MOGP$^*$ + Fragments & $0.83 \pm \: 0.01$ & $0.72 \pm \: 0.02$ & $-0.06 \pm \: 0.12$ & $0.00 \pm \: 0.08$ \\
MOGP$^*$ + Fragprints & $\textbf{0.90} \pm \: \textbf{0.01}$ & $0.76 \pm \: 0.01$ & $\textbf{0.49} \pm \: \textbf{0.05}$ & $0.33 \pm \: 0.06$ \\

ANP + Morgan & $0.70 \pm \: 0.02$ & $0.66 \pm \: 0.02$ & $0.30 \pm \: 0.06$ & $\textbf{0.38} \pm \: \textbf{0.05}$ \\ 
ANP + Fragments & $0.81 \pm \: 0.01$ & $0.62 \pm \: 0.05$ & $-0.16 \pm \: 0.11$ & $-0.06 \pm \: 0.10$ \\
ANP + Fragprints & $0.83 \pm \: 0.01$ & $0.75 \pm \: 0.01$ & $0.18 \pm \: 0.08$ & $0.35 \pm \: 0.05$ \\ 
GCN & $0.87 \pm \: 0.01$ & $0.66 \pm \: 0.03$ & $-0.41 \pm \: 0.22$ & $-0.92 \pm \: 0.3$\\
GAT & $0.81 \pm \: 0.02$ & $0.57 \pm \: 0.04$ & $0.39 \pm \: 0.17$ & $-1.07 \pm \: 0.4$\\
DMPNN & $0.82 \pm \: 0.02$ & $0.63 \pm \: 0.02$ & $-0.05 \pm \: 0.07$ & $0.11 \pm \: 0.04$\\[5 pt]
\bottomrule
\end{tabular}
\end{adjustbox}
\label{property_pred1}
\end{table}

Results with additional models on the property prediction benchmark for which extensive hyperparameter tuning was not undertaken, are presented in \autoref{property_pred2}. The black-box alpha divergence minimisation BNN was implemented in the Theano library \citep{2016_Theano} and is based on the implementation of \citep{2016_Lobato}. The network has 2 hidden layers of size 25 with ReLU activations. The alpha parameter was set to 0.5, the prior variance for the variational distribution q was set to 1, and 100 samples were taken to approximate the expectation over the variational distribution. For all tasks the network was trained using 8 iterations of the Adam optimiser \citep{2014_Adam} with a batch size of 32 and a learning rate of 0.05. The MPNN was trained for 100 epochs in the case of the \emph{E} isomer $\pi-\pi{^*}$ task and 200 epochs in the case of the other tasks with a learning rate of 0.001 and a batch size of 32. The model architecture was taken to be the library default with the same node and edge features used for the GCN and GAT models in the main paper. The SMILES-X implementation remained the same as that of \citet{2020_Lambard} save for the difference that the network was trained for 40 epochs without \textsc{bo} over model architectures. In the case of SMILES-X 3 random train/test splits were used instead of 20 for the \emph{Z} isomer tasks whereas 2 splits were used for the \emph{E} isomer \emph{n}$-\pi{^*}$ task. For the \emph{E} isomer $\pi-\pi{^*}$ prediction task results are missing due to insufficient RAM on the machine used to run the experiments. 

\begin{table}[!htbp]
\caption{Test set performance in predicting the transition wavelengths of the \emph{E} and \emph{Z} isomers.}
\label{tab2_photo}
\setlength{\extrarowheight}{2pt}
\resizebox{\textwidth}{!}{%
\begin{tabular}{l|cccc}
\toprule
& \emph{E} isomer $\pi-\pi{^*}$ (nm) & \emph{E} isomer \emph{n}$-\pi{^*}$ (nm) & \emph{Z} isomer $\pi-\pi{^*}$ (nm) & \emph{Z} isomer \emph{n}$-\pi{^*}$ (nm) \\ \hline
\multicolumn{1}{c|}{\underline{\textbf{RMSE}}} & & & & \\
BNN + Morgan & $27.0 \pm \: 0.9$ & $12.9 \pm \: 0.6$ & $13.9 \pm \: 0.6$ & $12.7 \pm \: 0.4$ \\ 
BNN + Fragments & $31.2 \pm \: 1.1$ & $14.8 \pm \: 0.8$ & $16.9 \pm \: 0.8$ & $12.7 \pm \: 0.4$ \\
BNN + Fragprints & $26.7 \pm \: 0.8$ & $13.1 \pm \: 0.5$ & $14.9 \pm \: 0.5$ &  $13.0 \pm \: 0.6$ \\ 
MPNN & $24.8 \pm \: 0.8$ & $12.5 \pm \: 0.6$ & $16.7 \pm \: 0.8$ & $12.8 \pm \: 0.7$ \\
SMILES-X &  & $25.1 \pm \: 4.2$ & $17.8 \pm \: 0.6$ & $14.8 \pm \: 0.9$
\\[5pt]\hline \multicolumn{1}{c|}{\underline{\textbf{MAE}}} & & & & \\
BNN + Morgan & $19.0 \pm \: 0.6$ & $9.9 \pm \: 0.4$ & $10.2 \pm \: 0.5$ & $8.6 \pm \: 0.3$ \\ 
BNN + Fragments & $22.4 \pm \: 0.8$ & $10.6 \pm \: 0.4$ & $12.9 \pm \: 0.6$ & $8.6 \pm \: 0.3$ \\ 
BNN + Fragprints & $19.1 \pm \: 0.6$ & $10.1 \pm \: 0.5$ & $10.8 \pm \: 0.4$ & $9.3 \pm \: 0.5$ \\ 
MPNN & $15.4 \pm \: 0.8$ & $8.6 \pm \: 0.3$ & $11.6 \pm \: 0.6$ & $8.4 \pm \: 0.4$ \\
SMILES-X &  & $20.6 \pm \: 3.1$ & $11.6 \pm \: 1.0$ & $11.2 \pm \: 1.0$ \\ \hline
\multicolumn{1}{c|}{\emph{\underline{\textbf{R$^{2}$}}}} & & & & \\
BNN + Morgan & $0.83 \pm \: 0.01$ & $0.69 \pm \: 0.02$ & $0.23 \pm \: 0.08$ & $0.18 \pm \: 0.05$ \\ 
BNN + Fragments & $0.77 \pm \: 0.01$ & $0.58 \pm \: 0.04$ & $-0.15 \pm \: 0.14$ & $0.18 \pm \: 0.05$ \\
BNN + Fragprints & $0.83 \pm \: 0.01$ & $0.68 \pm \: 0.02$ & $0.14 \pm \: 0.06$ & $0.11 \pm \: 0.08$ \\ 
MPNN & $0.83 \pm \: 0.01$ & $0.63 \pm \: 0.06$ & $-0.70 \pm \: 0.34$ &  $-0.68 \pm \: 0.27$ \\
SMILES-X &  & $-0.44 \pm \: 0.30$ & $-0.08 \pm \: 0.06$ & $-0.09 \pm \: 0.04$ \\
\bottomrule
\end{tabular}}
\label{property_pred2}
\end{table}

\subsection{Prediction Error as a Guide to Representation Selection}
\label{representation}

On the \emph{E} isomer $\pi-\pi{^*}$ transition wavelength prediction task, occasionally marked discrepancies were noted in the predictions made under the Morgan fingerprint and fragment representations. The resultant analysis motivated the expansion of the molecular feature set to include both representations as \say{fragprints}.

\subsection{Impact of Dataset Choice}
\label{sec:big_data}

In this section, the generalisation performance was evaluated for a model trained on the \emph{E} isomer $\pi-\pi{^*}$ values of a large dataset of 6142 out-of-domain molecules (including non-azoarene photoswitches) from \citet{2019_Beard}, with experimentally-determined labels. A RF regressor was (due to scalability issues with the \textsc{mogp} on 6000+ data points) implemented in the scikit-learn library with 1000 estimators and a max depth of 300 on the fragprint representation of the molecules. In \autoref{tab_merge2} results are presented for the case when the train set consists of the large dataset of 6142 molecules and the test set consists of the entire photoswitch dataset. Results are also presented on the original \emph{E} isomer $\pi-\pi{^*}$ transition wavelength prediction task where the train set of each random 80/20 train/test split was augmented with the molecules from the large dataset. The results indicate that the data for out-of-domain molecules provides no benefit for the prediction task and even degrades performance, when amalgamated, relative to training on in-domain data only.

\begin{table}[]\caption{Performance comparison of curated dataset against large non-curated dataset.}\label{tab_merge2}
\centering
\setlength{\extrarowheight}{2pt}
\begin{tabular}{lllcl}
\toprule
Dataset & Size & RMSE ($\downarrow$) & MAE ($\downarrow$) & $R^2$ ($\uparrow$) \\ \hline
Large Non-Curated & 6142 & $85.2$ & $72.5$ & $-0.66$ \\
Large Non-Curated + Curated & 6469 & $36.9 \pm 1.2$ &  $22.7 \pm 0.7$ & $0.67 \pm 0.02$ \\ 
Curated & 314 & $\textbf{23.4} \pm \textbf{0.9}$ & $\textbf{13.9} \pm \textbf{0.4}$ & $\textbf{0.87} \pm \textbf{0.01}$ \\
\bottomrule
\end{tabular}
\end{table}

Based on these results the importance of designing synthetic molecular machine learning benchmarks with a real-world application in mind is emphasised, as well as the importance of involving synthetic chemists in the curation process. By targeted data collation on a narrow and well-defined region of chemical space where the molecules are in-domain relative to the task, it becomes possible to mitigate generalisation error.

\subsection{Human Performance Benchmark}\label{sec:human_app}

Below in \autoref{tab_new} the full results breakdown of the human performance benchmark study is provided.

\begin{table}[]\caption{Results breakdown for the human expert performance benchmark predicting the transition wavelength (nm) of the \emph{E} isomer $\pi-\pi^*$ transition for 5 molecules. Closest prediction for each molecule is underlined and highlighted in bold. \textsc{mogp} achieves the lowest MAE relative to all individual human participants.}\label{tab_new}
\centering
\setlength{\extrarowheight}{2pt}
\begin{tabular}{l|lllll|l}
\toprule
 & Mol 1 & Mol 2 & Mol 3 & Mol 4 & Mol 5 & MAE ($\downarrow$) \\
True Value & \textbf{329} & \textbf{407} & \textbf{333} & \textbf{540} & \textbf{565} & \\ \midrule
Postdoc 1 & 325 & 360 & 410 & 490 & 490 & 54.7 \\
PhD 1 & 350 & 400 & 530 & 410 & 425 & 93.3 \\
PhD 2 &  380 & 280 & 530 & 600 & 250 & 177.5 \\
Postdoc 2 & \underline{\textbf{330}} & 350 & 500 & 475 & 500 & 66.7 \\
PhD 3 & 325 & 350 & 350 & \underline{\textbf{540}} & 550 & 16.3 \\
Postdoc 3 & 350 & 370 & 520 & 600 & 500 & 97.5 \\
PhD 4 & \underline{\textbf{330}} & 380 & 390 & 520 & 580 & 34.2 \\
Undergraduate 1 & 340 & 420 & 400 & \underline{\textbf{540}} & 570 & 41.8 \\
Postdoc 4 & 321 & 345 & \underline{\textbf{340}} & 500 & 520 & 28.7 \\
PhD 5 & \underline{\textbf{330}} & 360 & \underline{\textbf{340}} & 500 & 520 & 24.2 \\
PhD 6 & 303 & 367 & 435 & 411 & 450 & 78.7 \\
PhD 7 &  280 & 350 & 450 & 430 & 460 & 85.5 \\
PhD 8 &  270 & 390 & 420 & 420 & 440 & 73.8 \\
PhD 9 & \underline{\textbf{330}} & 310 & 462 & 512 & 512 & 55.3 \\ \midrule
\textsc{mogp} & 321 & \underline{\textbf{413}} & 354 & 518 & \underline{\textbf{569}} & \underline{\textbf{11.9}} \\

\bottomrule
 \end{tabular}
 \end{table}

\subsection{Confidence-Error Curves}
\label{conf_error}

An advantage of Bayesian models for the real-world prediction task is the ability to produce calibrated uncertainty estimates. If correlated with prediction error, a model's uncertainty may act as an additional decision-making criterion for the selection of candidates for lab synthesis. To investigate the benefits afforded by uncertainty estimates, confidence-error curves were produced using the \textsc{gp}-Tanimoto model in conjunction with the fingerprints representation. The confidence-error curves for the RMSE and MAE metrics are shown in \autoref{uc_rmse_plots} and \autoref{uc_mae_plots} respectively. The x-axis, confidence percentile, may be obtained simply by ranking each model prediction of the test set in terms of the predictive variance at the location of that test input. As an example, molecules that lie in the 80th confidence percentile will be the 20\% of test set molecules with the lowest model uncertainty. The prediction error is then measured at each confidence percentile across 200 random train/test splits to see whether the model's confidence is correlated with the prediction error. It is observed that across all tasks, the \textsc{gp}-Tanimoto model's uncertainty estimates are positively correlated with prediction error, offering a proof of concept that model uncertainty can be incorporated into the decision process for candidate selection.

\begin{figure*}[!htbp]
\centering
\subfigure[\emph{E} Isomer $\pi-\pi{^*}$]{\label{fig:4ppa}\includegraphics[width=0.49\textwidth]{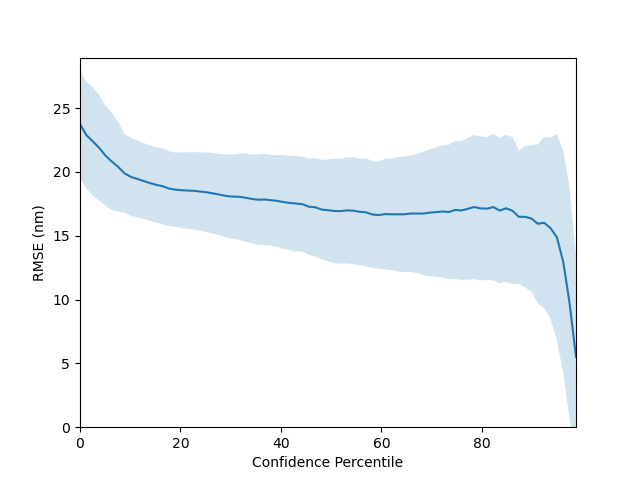}}  
\subfigure[\emph{E} Isomer \emph{n}$-\pi{^*}$]{\label{fig:4ppb}\includegraphics[width=0.49\textwidth]{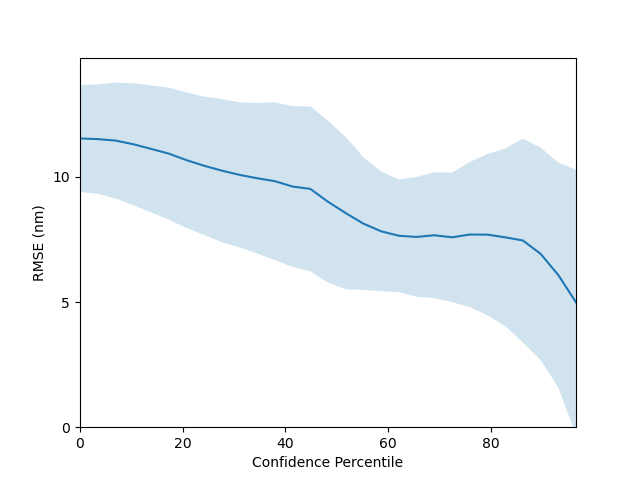}}
\subfigure[\emph{Z} Isomer $\pi-\pi{^*}$]{\label{fig:4ppc}\includegraphics[width=0.49\textwidth]{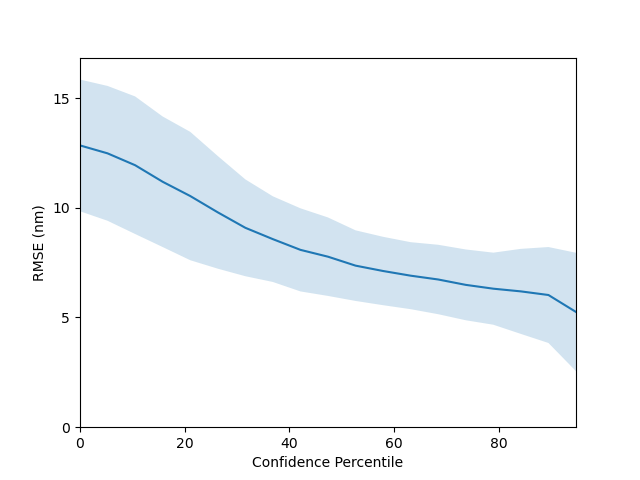}}  
\subfigure[\emph{Z} Isomer \emph{n}$-\pi{^*}$]{\label{fig:4ppd}\includegraphics[width=0.49\textwidth]{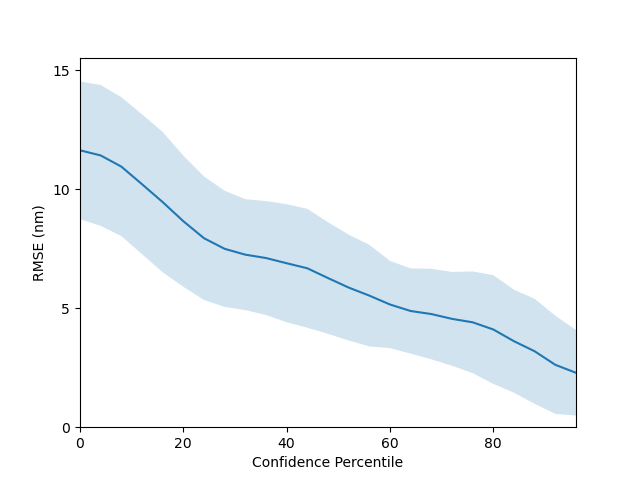}}  
\caption{RMSE confidence-error curves for property prediction using \textsc{gp} regression.}
\label{uc_rmse_plots}
\end{figure*}

\begin{figure*}[!htbp]
\centering
\subfigure[\emph{E} Isomer $\pi-\pi{^*}$]{\label{fig:4pa}\includegraphics[width=0.49\textwidth]{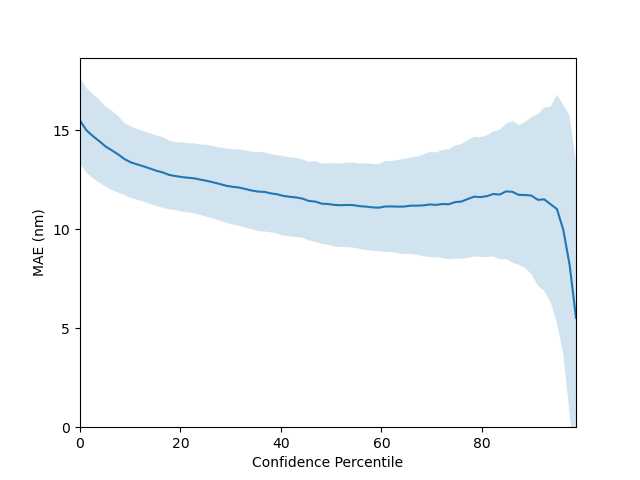}}  
\subfigure[\emph{E} Isomer \emph{n}$-\pi{^*}$]{\label{fig:4pb}\includegraphics[width=0.49\textwidth]{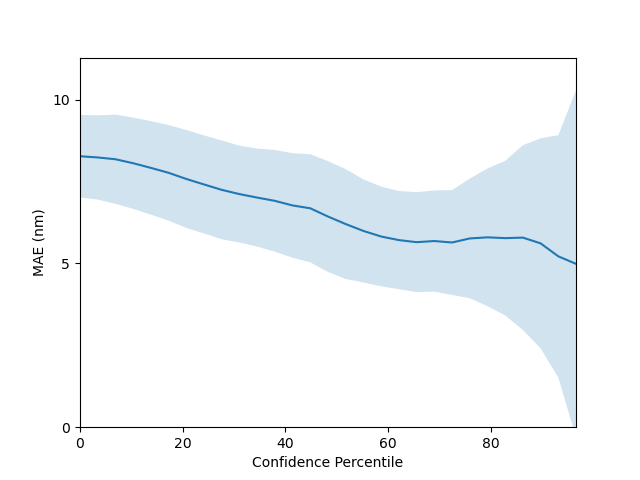}}
\subfigure[\emph{Z} Isomer $\pi-\pi{^*}$]{\label{fig:4pc}\includegraphics[width=0.49\textwidth]{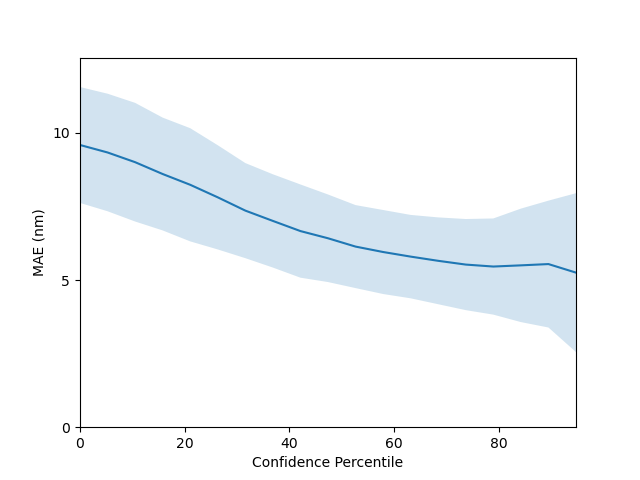}}  
\subfigure[\emph{Z} Isomer \emph{n}$-\pi{^*}$]{\label{fig:4pd}\includegraphics[width=0.49\textwidth]{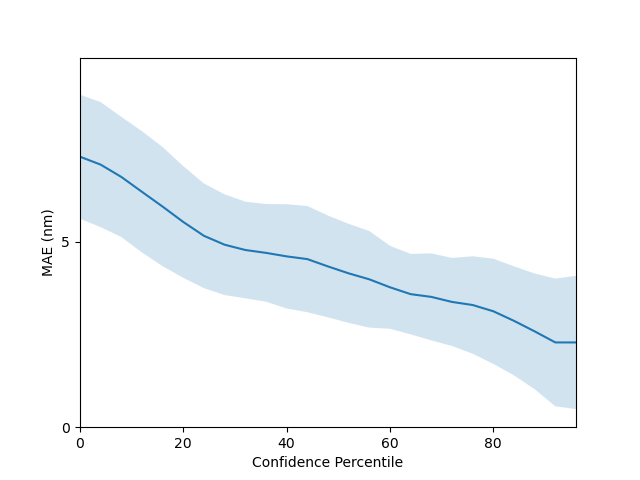}}  
\caption{MAE confidence-error curves for property prediction using \textsc{gp} regression.}
\label{uc_mae_plots}
\end{figure*}

\subsection{TD-DFT Benchmark}
\label{spearman_section}

Below, in \autoref{correlation} and \autoref{signed_error} further plots are included analysing the performance of the methods on the TD-DFT performance comparison benchmark. These plots motivated the use of the Lasso-correction to the TD-DFT predictions.

\begin{figure}[p]
  \centering
  \includegraphics[width=.45\textwidth]{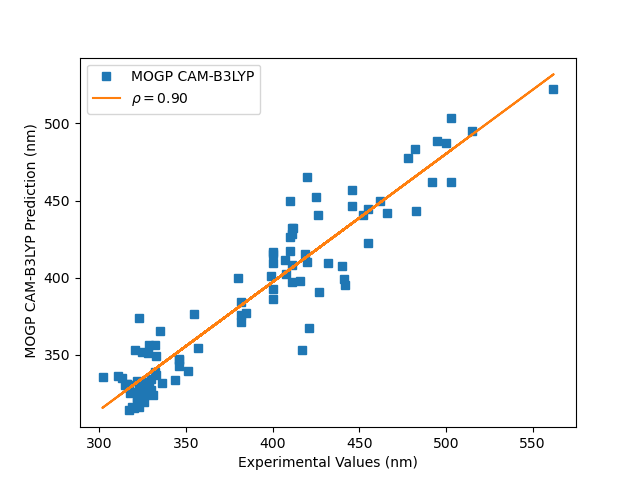}
  \hspace{1cm}
  \includegraphics[width=.45\textwidth]{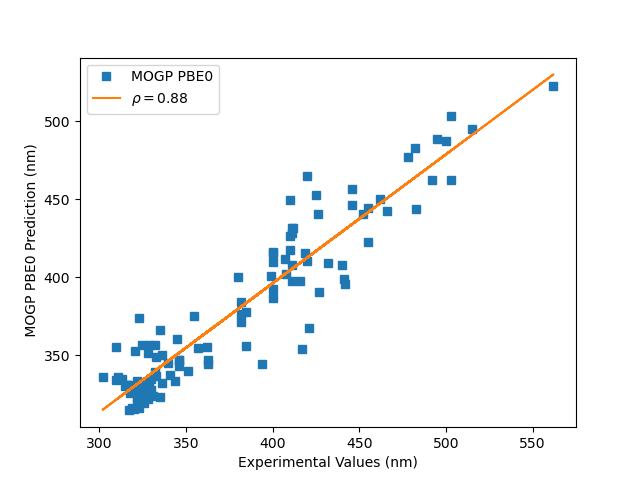}

  \vspace{1cm}

  \includegraphics[width=.45\textwidth]{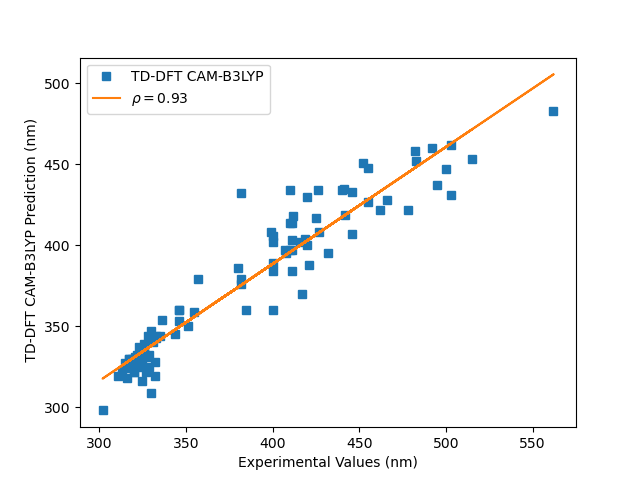}
  \hspace{1cm}
  \includegraphics[width=.45\textwidth]{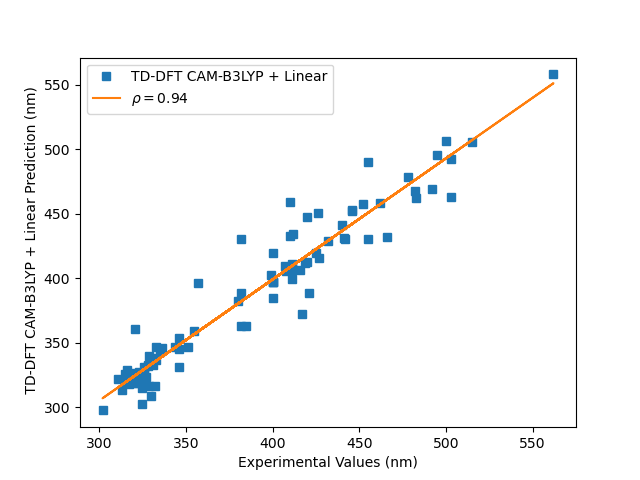}

  \vspace{1cm}

  \includegraphics[width=.45\textwidth]{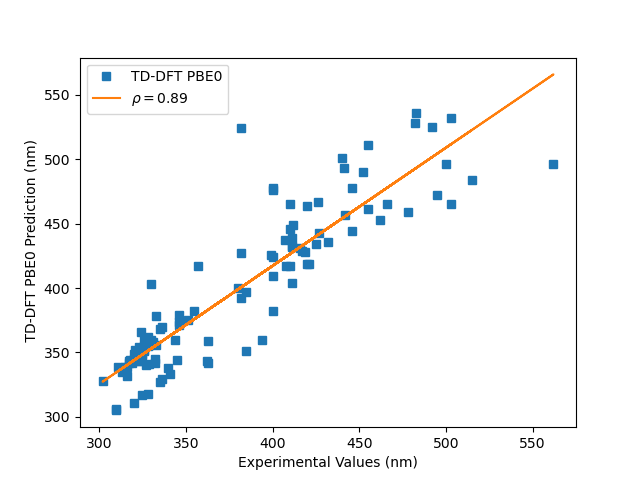}
  \hspace{1cm}
  \includegraphics[width=.45\textwidth]{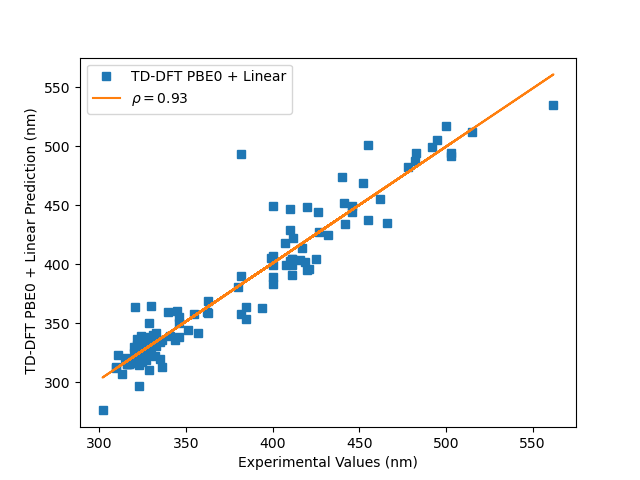}

  \caption{Regression plots for each method on the TD-DFT performance comparison benchmark with the Spearman rank-order correlation coefficient given as $\rho$. One may observe that the correlation between predictions and ground truth experimental values increases with the linear Lasso correction to the TD-DFT methods.}
  \label{correlation}
\end{figure}

\begin{figure}[p]
  \centering
  \includegraphics[width=.45\textwidth]{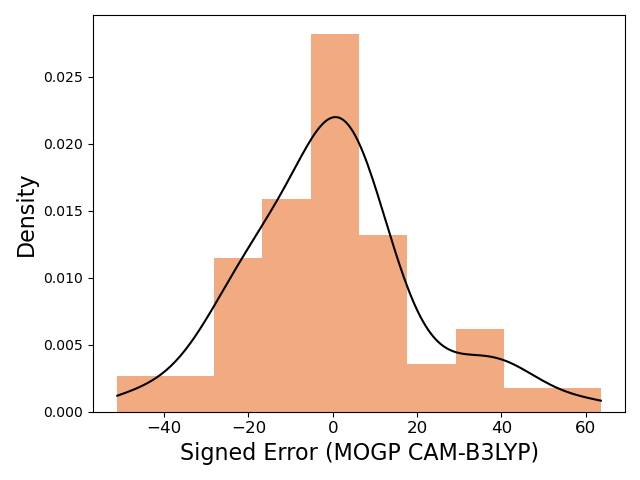}
  \hspace{1cm}
  \includegraphics[width=.45\textwidth]{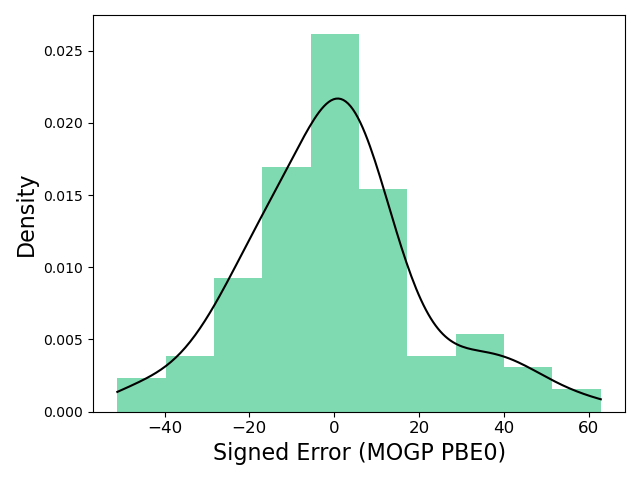}

  \vspace{1cm}

  \includegraphics[width=.45\textwidth]{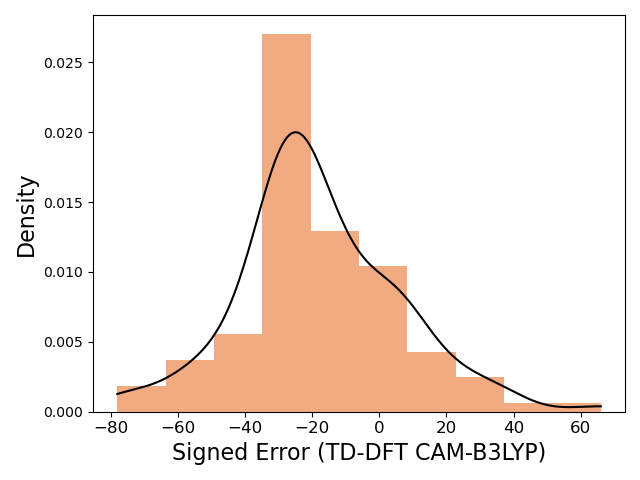}
  \hspace{1cm}
  \includegraphics[width=.45\textwidth]{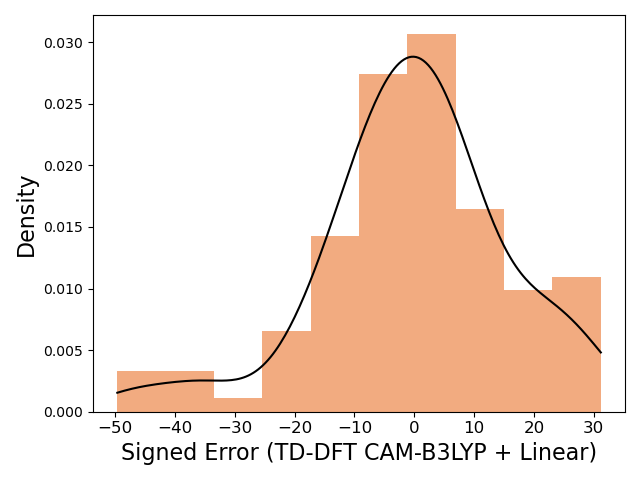}

  \vspace{1cm}

  \includegraphics[width=.45\textwidth]{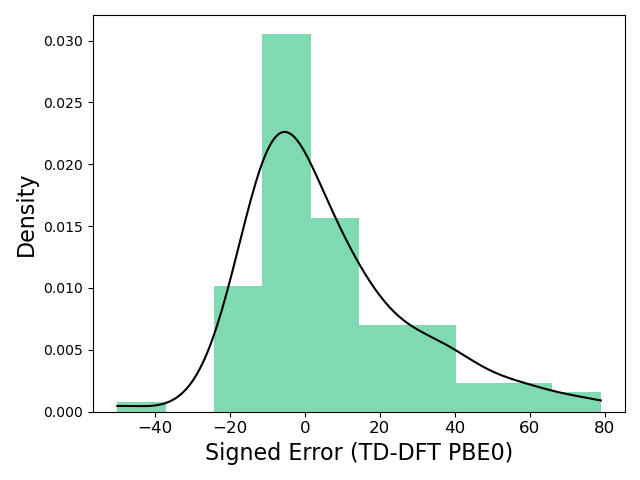}
  \hspace{1cm}
  \includegraphics[width=.45\textwidth]{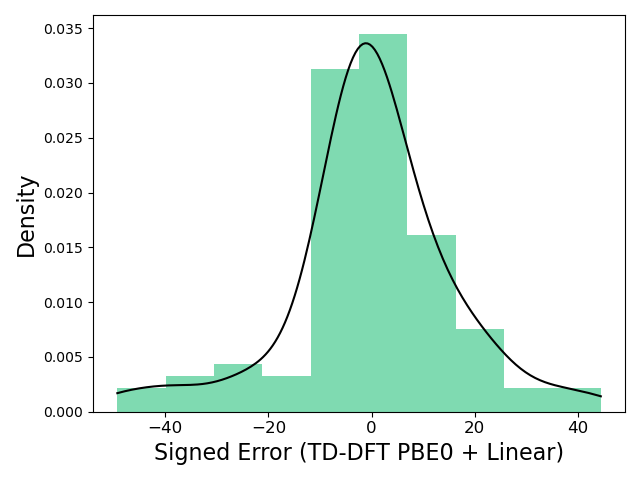}

  \caption{Signed error distributions for each method on the TD-DFT performance comparison benchmark. Signed error is recorded for each heldout molecule in leave-one-out-validation. Gaussian kernel density estimates are overlaid on the histograms. One may observe that the linear Lasso correction for the TD-DFT methods has a centering effect on the error distribution.}
  \label{signed_error}
\end{figure}

\section{Further Screening Details}
\label{exp_app}

Reagents and solvents were obtained from commercial sources (MolPort) and used as supplied. Experimental measurements were performed by Jake L. Greenfield at Imperial College London and below is included his account of the experimental procedure.

\subsection{UV-Vis Absorption Spectroscopy}

UV-Vis absorption spectra were obtained on an Agilent 8453 UV-Visible Spectrophotometer G1103A. A sampler holder with four open faces was used to enable in-situ irradiation (90° to the measurement beam). Samples were prepared in a UV Quartz cuvette with a path length of 10 mm. Solutions of the compounds were prepared in HPLC grade DMSO at a concentration of 25 $\mu\text{M}$.

\subsection{Photoswitching}

Samples were irradiated with a custom-built irradiation set up using 365 nm (3 × 800 mW Nichia NCSU276A LEDs, FWHM 9 nm), 405 nm (3 × 770 mW Nichia NCSU119C LEDs, FWHM 11 nm), 450 nm (3 × 900 mW Nichia NCSC219B-V1 LEDs, FWHM 18 nm), 495 nm (3 × 750 mW Nichia NCSE219B-V1 LEDs, FWHM 32 nm), 525 nm (3 × 450 mW NCSG219B-V1 LEDs, FWHM 38 nm) and 630 nm (3 × 780 mW Nichia NCSR219B-V1 LEDs, FWHM 16 nm) light sources. Samples were irradiated until no further change in the UV-vis absorption spectra was observed, indicating that the Photostationary State (PSS) was reached. The PSS, and the \say{predicted pure Z} spectra was determined using UV-vis following the procedure reported by \citet{Fischer1967}.

\section{Novelty of Screened Candidates relative to The Photoswitch Dataset}

In \autoref{discovered}, for each of the 6 candidates satisfying both performance criteria, some indication as to the novelty of the discovered photoswitch candidates is provided by giving the 3 closest molecules by Tanimoto similarity from the Photoswitch Dataset.

\begin{figure}[p]
  \centering
  \includegraphics[width=.45\textwidth]{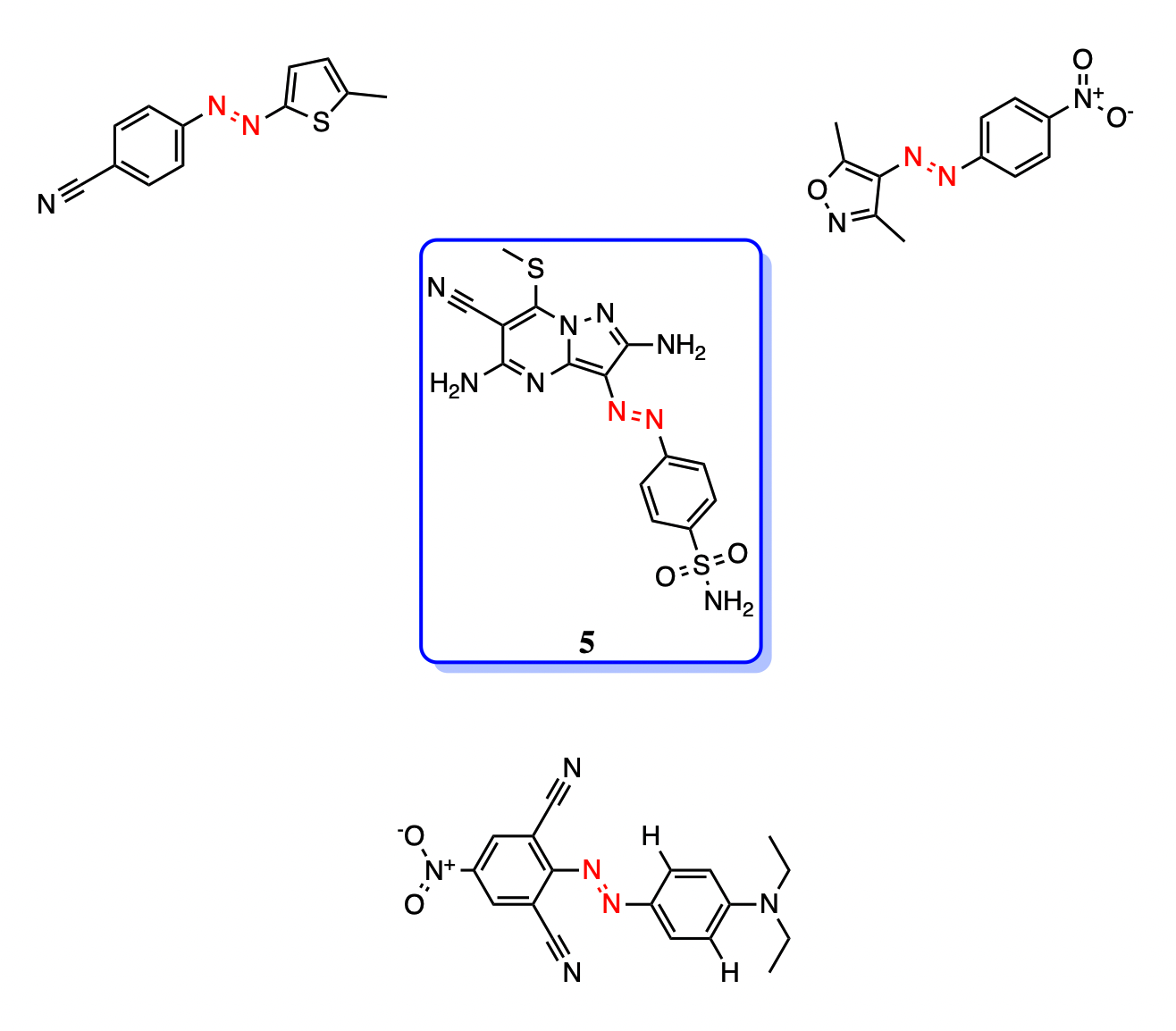}
  \hspace{1cm}
  \includegraphics[width=.45\textwidth]{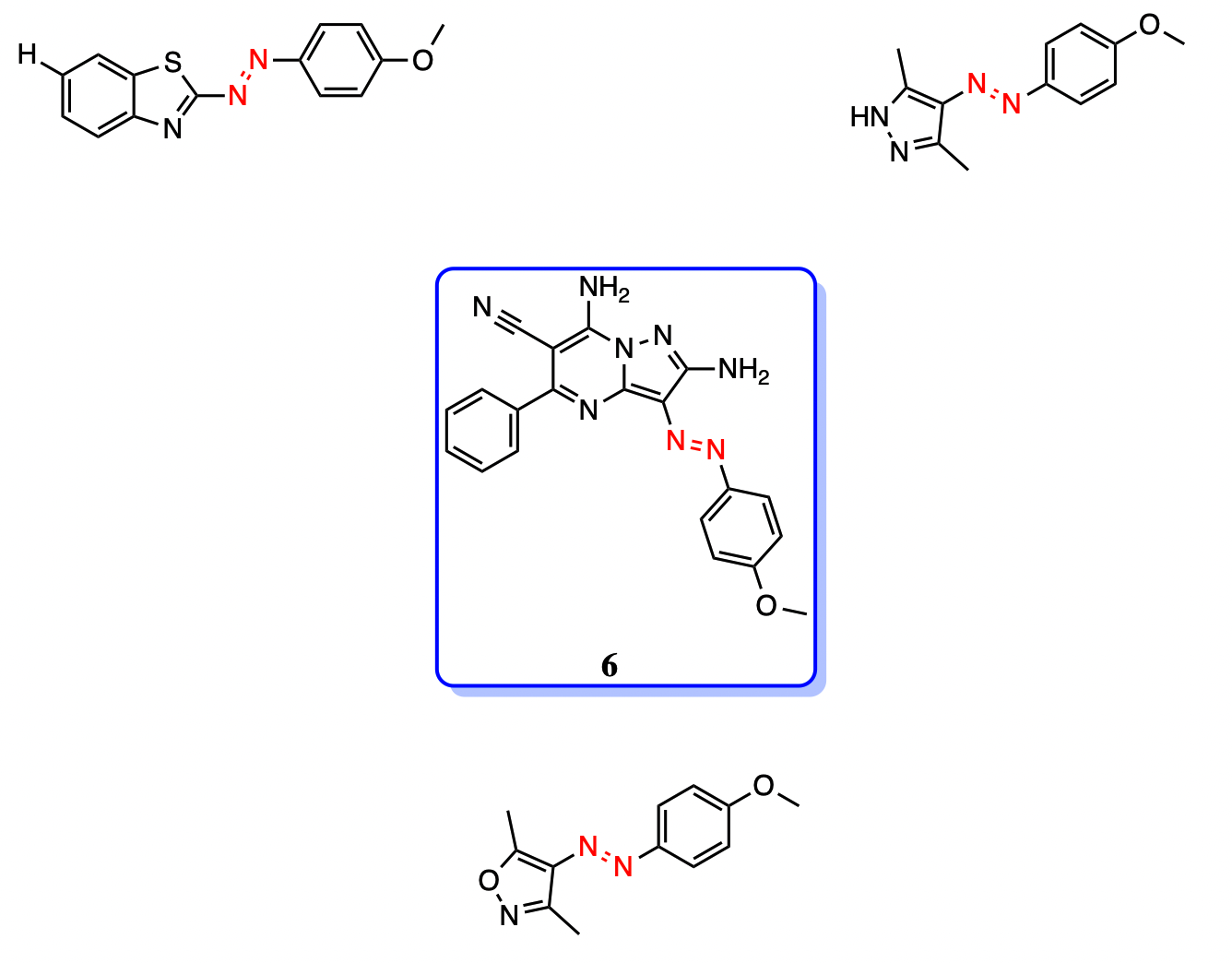}

  \vspace{1cm}

  \includegraphics[width=.45\textwidth]{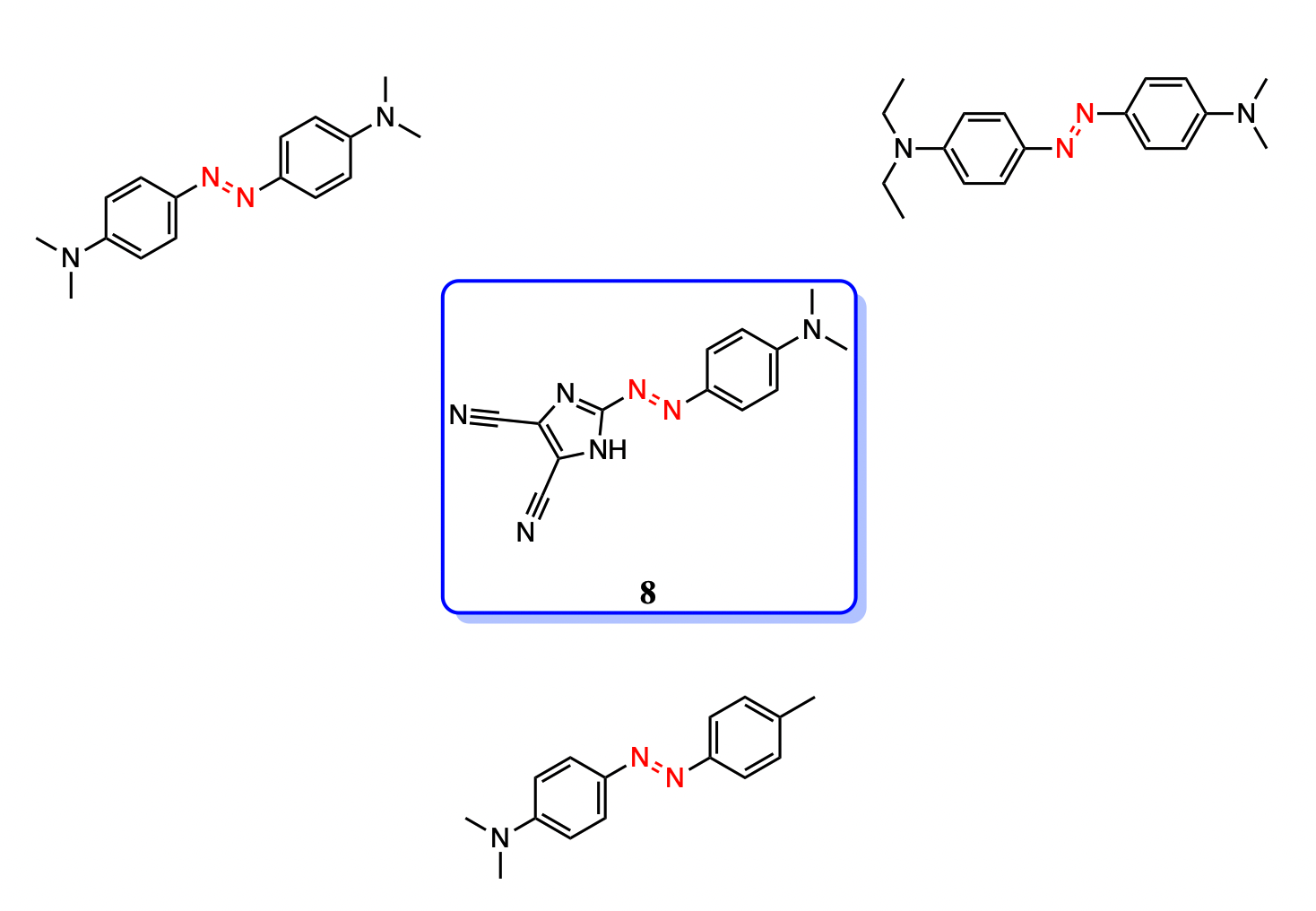}
  \hspace{1cm}
  \includegraphics[width=.45\textwidth]{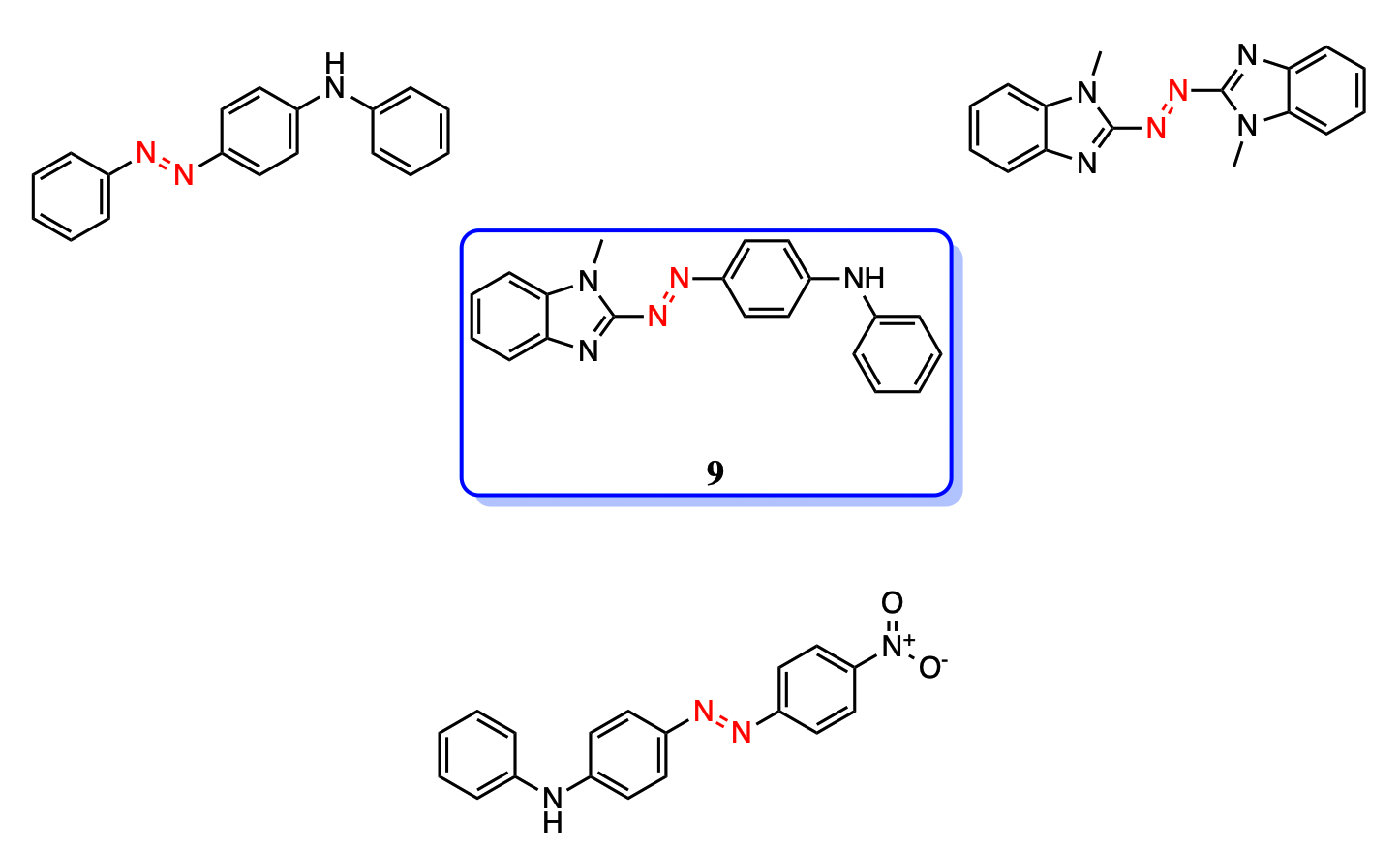}

  \vspace{1cm}

  \includegraphics[width=.45\textwidth]{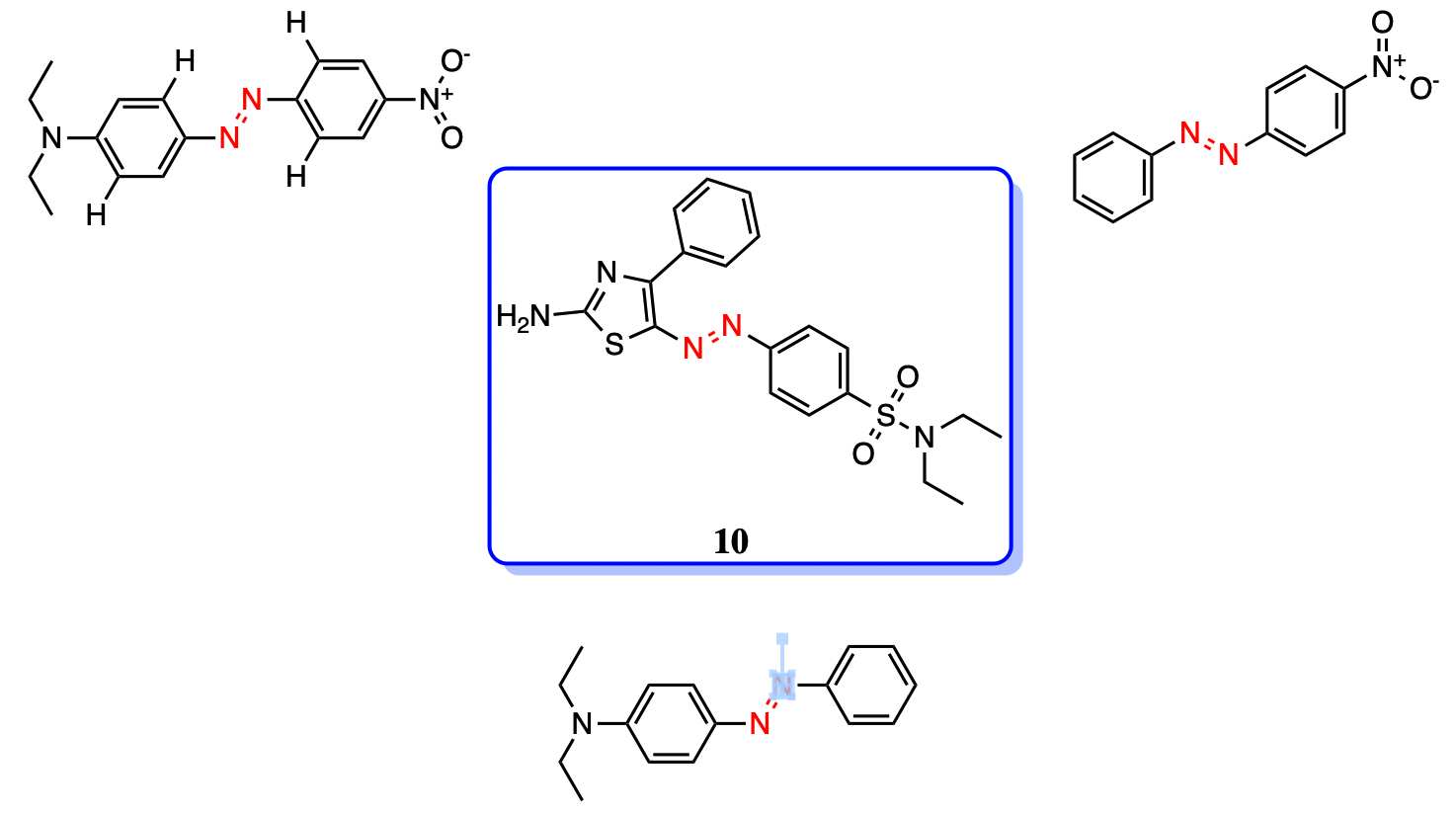}
  \hspace{1cm}
  \includegraphics[width=.45\textwidth]{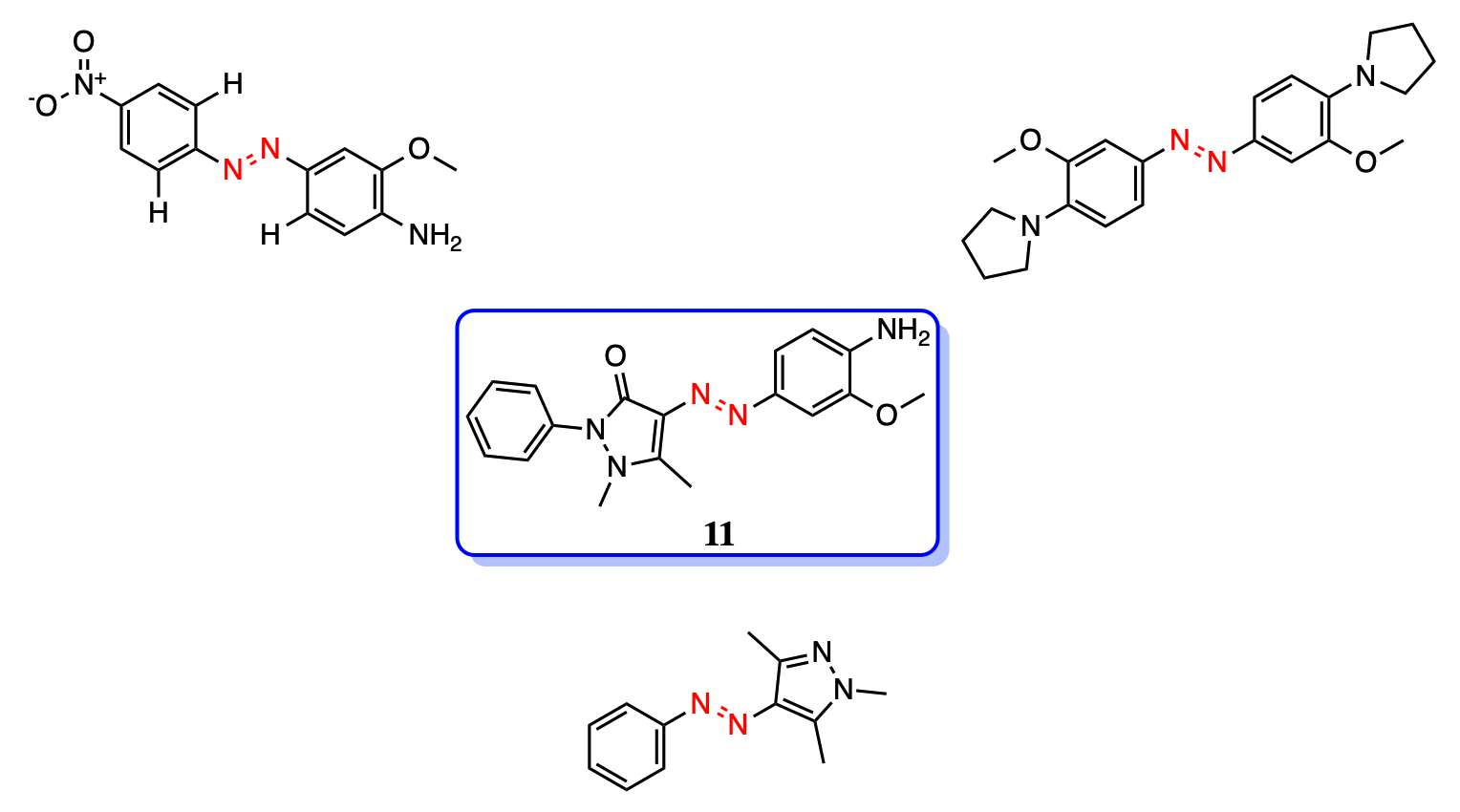}

  \caption{All 6 experimentally-tested candidates satisfying both performance criteria together with the 3 closest molecules by Tanimoto similarity in the Photoswitch Dataset.}
  \label{discovered}
\end{figure}

\end{appendices}

\printthesisindex %

\end{document}